\documentclass[a4paper,11pt]{article}

\input{macros.tex}

\title{
\toptitlebar
{{\center\baselineskip 18pt
                      {\Large\bf Why Do We Need Warm-up? A Theoretical Perspective}}
} 
\bottomtitlebar}
\date{}
\author[1]{Foivos Alimisis$^*$}
\author[1]{Rustem Islamov$^*$}
\author[1]{Aurelien Lucchi}
\affil[1]{University of Basel, Switzerland}

\begin{document}

\maketitle

\def\thefootnote{*}\footnotetext{Equal contribution. The authors are listed in the alphabetical order.}\def\thefootnote{\arabic{footnote}}
\begin{abstract}
  Learning rate warm-up -- increasing the learning rate at the beginning of training -- has become a ubiquitous heuristic in modern deep learning, yet its theoretical foundations remain poorly understood. In this work, we provide a principled explanation for why warm-up improves training. We rely on a generalization of the $(L_0, L_1)$-smoothness condition, which bounds local curvature as a linear function of the loss suboptimality and exhibits desirable closure properties. We show -- both theoretically and empirically -- that this condition is satisfied by common neural architectures and accurately captures the curvature of the optimization landscape early in training. Adapting the learning rate in response to this curvature condition naturally induces a warm-up–like schedule, and we show that this choice yields provably faster convergence guarantees than using a fixed learning rate. Experiments on language and vision models show that the resulting one-parameter warm-up schedule can match tuned linear warm-up and improve over no warm-up.
\end{abstract}

\section{Introduction}


A common learning-rate (LR) schedule linearly increases the LR at the beginning of training, {\it warm-up stage} \citep{goyal2017accurate, vaswani2017attention}, and gradually decreases it near the end, {\it decay stage} \citep{loshchilov2016sgdr, vaswani2017attention, hoffmann2022empirical, zhang2023dive, dremov2025training}.

Decaying the LR is a classical requirement in the theoretical analysis of \algname{SGD}, ensuring convergence under broad conditions~\citep{defazio2023optimal, gower2021sgd}, and it has been consistently observed to improve empirical performance~\citep{loshchilov2016sgdr, hu2024minicpm, hagele2024scaling}. Recent work further demonstrates that \emph{decaying} step-sizes can improve theoretical guarantees by yielding tighter bounds~\citep{schaipp2025surprising}. By contrast, the practice of linearly increasing the LR at the start of training (warm-up phase) has become nearly ubiquitous in modern deep learning~\citep{he2016deep, hu2024minicpm, hagele2024scaling}, yet a clear theoretical understanding of why it helps optimization remains elusive. A growing body of empirical work points to several advantages of warm-up, including: $(i)$ mitigating training instabilities \citep{kosson2024analyzing, goyal2017accurate, zhang2023dive}, reducing the variance of stochastic gradients \citep{liu2019variance}, and improving the robustness to the choice of the peak LR \citep{wortsman2023small, kalra2024warmup}. 
However, these explanations remain fragmented and likely point to secondary benefits rather than the core mechanism. This raises the central question we address in this paper:

\textit{
Is there an intrinsic property of neural network loss landscapes that justifies the use of a learning-rate warm-up schedule?
}

We identify such a property and make the following contributions:

\begin{enumerate}
    
    \item We introduce a generalized smoothness condition,$(H_0, H_1)$-smoothness, in which local curvature is bounded by an affine function of the loss suboptimality. This condition enjoys useful closure properties under finite sums and affine transformations and generalizes $(L_0, L_1)$-smoothness.

    \item We provide both theoretical and empirical evidence that the $(H_0, H_1)$-smoothness condition holds for common neural network architectures, under both mean-squared error (MSE) and cross-entropy (CE) losses. During the initial phase of training, $(H_0, H_1)$-smoothness realistically captures the loss-curvature dynamics.

    \item We theoretically demonstrate that, for functions that satisfy the $(H_0, H_1)$ condition, Gradient Descent (\algname{GD}) achieves faster convergence with a warm-up schedule than with a fixed LR. We prove this by deriving upper complexity bounds for \algname{GD} with the adaptive schedule and lower complexity bounds for fixed-step \algname{GD}.

    \item  We provide empirical evidence that the resulting tuned one-parameter schedule matches the performance of tuned linear warm-up on language and vision models, and improves upon training without warm-up.
\end{enumerate}

\section{Related Works}

\paragraph{Warm-up.}

LR scheduling plays a central role in the success of modern deep learning training pipelines. A wide range of scheduling strategies, including LR decay, annealing, and warm-up, have been developed to improve convergence and generalization \citep{mccandlish2018empirical, sutskever2013importance, touvron2023llama}.

Among these different strategies, warm-up has become a key component in modern training pipelines, particularly for Transformers~\citep{vaswani2017attention, goyal2017accurate}. It is commonly credited with enhancing training stability~\citep{kosson2024analyzing, gotmare2018closer}, improving robustness to the choice of LR~\citep{wortsman2023small}, and enabling the use of larger peak LR~\citep{kalra2024warmup}. 
Warm-up has also been linked to improved generalization, either by reducing mini-batch gradient noise~\citep{liu2019variance}, encouraging convergence to flatter minima~\citep{smith2020generalization}, or by complementing other scheduling techniques~\citep{huang2020improving, xiong2020layer, wortsman2023small}. From a geometric perspective, \citet{gilmer2021loss, roulet2024stepping} observed that warm-up induces a curvature reduction phase in which the largest Hessian eigenvalue decreases.

Although warm-up is well supported by empirical evidence \citep{vaswani2017attention, wortsman2023small, dremov2025training}, its theoretical foundations remain limited. Most existing convergence analyses of (stochastic) gradient-based optimizers focus on the decay phase. For example, \citet{wen2024understanding} uses a river-valley model to study neural loss landscapes, but their framework focuses on the stable and decay stages of the LR. Likewise, \citet{schaipp2025surprising, attia2025benefits} showed that decaying LR provides theoretical benefits and that convergence bounds closely align with empirical training curves, yet their analysis does not account for the warm-up phase.
\citet{kondo2025accelerating} analyze a scheme with exponentially increasing batch size and LR, showing faster convergence for (\algname{GD}). Yet, the requirement of rapidly growing batches limits its practicality.

Finally, several complementary explanations for the role of warm-up have been proposed. For instance, \citet{xiong2020layer} attribute the necessity of warm-up in Transformer training primarily to the placement of layer normalization. In a different vein, \citet{kosson2024analyzing} demonstrate that explicitly constraining the norm of parameter updates—similar to gradient clipping—can only partially reduce the reliance on warm-up.

Despite extensive prior research on warm-up, we are not aware of any theoretical framework that explains its benefits in terms of convergence. In this work, we address this gap by relying on a smoothness-type condition that upper bounds the curvature of the landscape using an affine expression of the function suboptimality. This condition represents well the loss landscape at the initial stage of training. For this function class, the curvature-adaptive step-size naturally increases as the loss decreases, giving a warm-up mechanism.

\paragraph{Curvature dynamics.}

The distinct phases of curvature evolution observed during neural network training have been studied extensively. \citet{cohen2021gradient} showed empirically that curvature typically increases over the course of training until it reaches a learning-rate--dependent threshold. Beyond this point, gradient descent often enters the so-called \emph{edge-of-stability} regime, in which curvature fluctuates around this threshold. These progressive sharpening and edge-of-stability phases are now relatively well understood within the deep learning community. More recent work has identified an additional \emph{initial} phase that precedes progressive sharpening and is characterized by a reduction in curvature \citep{kalra2023universal,kalra2024warmup}. This behavior is consistent with both our theoretical analysis and our experimental observations. Unlike the later phases, however, this early curvature-reduction regime remains poorly understood and is likely the phase in which learning-rate warm-up plays a critical role.

\paragraph{Generalized Smoothness.}

The conventional smoothness assumption in optimization theory requires the Hessian to satisfy a uniform bound $\|\nabla^2 f(w)\| \leq L$, but this constraint proves to be overly restrictive when applied to neural network training, as noted by \citet{zhang2019gradient}. To address this limitation, they introduced the more flexible $(L_0, L_1)$-smoothness condition, which allows the Hessian norm to grow linearly with the gradient magnitude: $\|\nabla^2 f(w)\| \leq L_0 + L_1\|\nabla f(w)\|$ for non-negative constants $L_0, L_1 \geq 0$. This relaxed framework naturally motivates gradient normalization techniques---both soft normalization and hard clipping---as optimal LR strategies that can significantly improve gradient descent convergence rates~\citep{zhang2020improved,zhao2021convergence,faw2023beyond,wang2023convergence,gorbunov2024methods,vankov2024optimizing,li2023convex}.

Despite its advantages, the $(L_0, L_1)$-smoothness condition has important limitations in practice, especially in explaining warm-up schedules. More precisely, at the beginning of training, the gradient-dependent nature of the $(L_0, L_1)$-smoothness condition leads to counterintuitive implications for LR scheduling: in some cases, the gradient norm is observed to increase during the early iterations~\citep{xie2023overlooked, defazio2023optimal, defazio2025gradients}. As a result, the $(L_0, L_1)$-bound becomes increasingly loose, which theoretically prescribes \emph{decreasing} step-sizes through gradient clipping. This stands in direct contrast to empirical best practices, where \emph{increasing} LR is typically employed at the beginning of training. 

These theoretical and practical inconsistencies highlight the need for a more sophisticated smoothness characterization that can adequately capture and explain LR warm-up dynamics. Since dependence on the gradient norm can lead to the wrong qualitative step-size prediction during early training, a natural replacement is the function-value suboptimality, which decays monotonically and gives a direct measure of the optimization target. We name this modified smoothness class as $(H_0, H_1)$-smoothness. Interestingly, a recent work by \citet{vaswani2025armijo} made a similar observation in a different context, showing that Armijo line search can achieve faster convergence than \algname{GD} with a constant step-size. Their analysis verifies this condition for several simple models but relies on additional assumptions from \citet{taheri2023fast}: $(i)$ bounding the gradient norm by the function suboptimality, $(ii)$ adopting the unrealistic exponential loss, $(iii)$ assuming data separability, and $(iv)$ restricting trainability to the input layer. In contrast, our analysis establishes the validity of the $(H_0, H_1)$-smoothness condition under a mild regularity assumption on the weights, which can be ensured either explicitly through balancedness or implicitly via standard L2 regularization. Our work proposes a curvature-based condition tailored to the early sharpness-reduction phase and uses it to justify warm-up as a stability-speed tradeoff.\footnote{A concurrent work \citep{liu2025theoretical} studies a warm-up stage using a similar condition. We discuss the differences in \Cref{sec:concurrent_work}.}

\section{The $(H_0,H_1)$-Smoothness Condition}
\label{sec:new_smoothness}

As an initial proof of concept, we analyze the \emph{derivative of the spectral norm of the Hessian} of a simple $1\times1\times1$ network with linear activation $\phi(x)=x$:
\begin{equation*}
    f(u,v) = (y - uvx)^2
\end{equation*}
over the course of the gradient flow ODE, assuming that $x,y \neq 0$. For this model, we provide the following result.
\begin{restatable}{proposition}{propsimplenetwork}
\label{prop:1x1x1 network}
During the trajectory $(u(t),v(t))$ of gradient flow starting sufficiently close to $(0,0)$, the Hessian $H(t)$ of the loss function $f(u(t),v(t)) = (y-u(t) v(t)x)^2$, satisfies

(i) $\frac{d}{dt} \|H(t)\|_2 < 0$, if $\frac{x^2 (u(0)-v(0))^2}{4} \leq (y-2u(t)v(t)x)^2$.

(ii) $\frac{d}{dt} \|H(t)\|_2 > 0$, when $y<2u(t)v(t)x$.

\end{restatable}

The proof can be found in Appendix~\ref{app:1x1x1 network}. Proposition~\ref{prop:1x1x1 network} reveals two distinct phases of curvature evolution under initialization close to~$0$, as is standard in practice. During the initial phase of training, condition~$(i)$ holds, leading to curvature reduction. This behavior has also been observed in prior work \citep{kalra2023universal,kalra2024warmup}, though it remains poorly understood. As training progresses, gradient flow drives the parameters into a regime in which condition~$(ii)$ is satisfied, resulting in progressive sharpening. This phenomenon was identified early in the study of the edge of stability \citep{cohen2021gradient} and has attracted significant attention in recent years. There are three remarks worth highlighting about this proposition:

\begin{enumerate}
\item If $u(0)=v(0)$, the two phases are continuous, and the transition from curvature reduction to progressive sharpening occurs precisely when $u(t)v(t)=y/2$.

\item The switching behavior predicted by this simple model closely matches our empirical observations in language model training (see Figure~\ref{fig:llm_verification_second_row}).

\item The proposition does not capture later training stages, such as the edge of stability, as it relies on gradient flow dynamics. These regimes can be characterized more precisely using alternative approaches, such as the central flow framework of \citet{cohenunderstanding},  which is beyond the scope of this work.
\end{enumerate}

These observations motivate us to focus on an earlier curvature reduction phase, which has been largely overlooked but is closely tied to LR warm-up. We characterize this early stage by relating curvature to function value suboptimality, as summarized in the following definition.

\begin{definition}\label{asmp:H0H1_smoothness}
    A function $f \colon\R^d\to\R$ with minimum $f^*~>~-\infty$ is called $(H_0, H_1)$-smooth for some $H_0,H_1\ge 0$, if for any $w\in\R^d$ it holds
    \begin{equation*}
    \|\nabla^2f(w)\|_2 \le H_0 + H_1(f(w) - f^*).
    \end{equation*}
    $\cH\eqdef \{f\colon\R^d\to\R\mid f\text{ is }(H_0,H_1)\text{-smooth}\}$ denotes the class of all $(H_0,H_1)$-smooth functions.
\end{definition}

Notably, any $(L_0, L_1)$-smooth function also satisfies $(H_0, H_1)$-smoothness, implying that the $(H_0, H_1)$-smooth class contains the previously studied $(L_0, L_1)$-smooth class, demonstrating the generality of \Cref{asmp:H0H1_smoothness}. Moreover, we show that $\cH$ is closed under finite sums and affine transformations, in contrast to the $(L_0, L_1)$-smooth class, for which neither operation is preserved, as demonstrated by simple counterexamples. Formal statements and proofs are deferred to \Cref{sec:more_general_class}. Finally, \Cref{asmp:H0H1_smoothness} admits a natural extension in which the linear dependence on the suboptimality $f(w)-f^*$ is replaced by $\cL(f(w)-f^*)$, where $\cL$ is a general increasing function, in the spirit of \citet{li2023convex}. We leave the study of this extension to future work.

\subsection{Theoretical Justification of $(H_0,H_1)$-Smoothness}

In this section, we show that under mild regularity conditions on the weights, enforced either explicitly via constraints or implicitly through $\ell_2$ regularization, the $(H_0, H_1)$-smoothness condition holds for several basic deep learning architectures. 

\subsubsection{Standard Feedforward Networks}
\label{sec:theory_MLP}

We begin with vanilla feedforward networks and distinguish two regimes: $(i)$ the initial phase of training and $(ii)$ the global regime. We focus on the initial phase, where both sharpness and loss suboptimality decay, making our bounds predictive. Detailed proofs are in \Cref{app:proofs_neural_nets}.

\paragraph{Results under Balancedness.} 
A well-known property of gradient flow in feedforward neural networks is the balanced evolution of the weight matrices $\{W_i\}_{i=1}^{\ell}$. For linear networks, this implies $W_i(t)^\top W_i(t) = W_{i+1}(t) W_{i+1}(t)^\top$, while for networks with homogeneous nonlinear activations it yields $\|W_i(t)\|_{\rm F} = \|W_{i+1}(t)\|_{\rm F}$~\citep[Theorem~2.2 and Corollary~2.1]{du2018algorithmic}. The latter condition is weaker than the former.

\begin{restatable}{proposition}{thmdeeplinear}
\label{prop:deep_linear}
   Consider a deep linear network with $\ell$ layers and MSE loss:
\[
    f(W) \equiv f(W_1,\dots, W_{\ell}) = \|Y- W_1 W_2 \dots W_{\ell} X\|_{\rm F}^2,
\]
where $Y \in \mathbb{R}^{c \times m}$ are the labels, $X \in \mathbb{R}^{d \times m} (d\leq m)$ is the input, and $W_i \in \mathbb{R}^{n_{i-1} \times n_i}$, where $n_0 = c$ and $n_{\ell} = d$ are layer dimensions. In the weakly balanced subspace, $\|W_1\|_F = \|W_2\|_F$, we have:

$(i)$ If $\|W_1\|_{\mathrm F}^\ell \|X\|_2 < (1-\frac{1}{\ell}) \sqrt{f(W)}$, then it holds
\begin{equation*}
\|\nabla^2 f(W)\| \leq H_0^{\text{warm-up}} + H_1^{\text{warm-up}} (f(W)-f^*)
\end{equation*}
with constants defined in Eq.~\eqref{eq:warmup_constants_linear_network}.

$(ii)$ If we additionally assume strong balancedness, then it holds
\begin{equation*}
\|\nabla^2 f(W)\| \leq H_0 + H_1 (f(W)-f^*) \hspace{2mm} \text{for all} \hspace{1mm} W,
\end{equation*}
with constants defined in Eq.~\eqref{eq:global_constants_linear_network}.

\end{restatable}

We note that condition $(i)$ is consistent with the early phase of training, where weight norms are typically small while the loss remains large, especially under commonly used near-zero initializations. Under standard conditions on the data matrices, $H_0^{\text{warm-up}}, H_1^{\text{warm-up}}$ are smaller than the global constants $H_0, H_1$ respectively. In Proposition \ref{prop:deep_non_linear_many_nonlinearities}, we generalize this result to a non-linear network with leaky-ReLU non-linearities.

\paragraph{Results under L2 Regularization.}
Analogous to balancedness, another approach to constrain the weight space is through L2 regularization. Here, we present a result that validates the $(H_0, H_1)$-smoothness condition for two-layer neural networks with general activation functions and cross-entropy loss. Its proof can be found in Appendix \ref{app:proofs_neural_nets}.

\begin{restatable}{proposition}{crossentropy}\label{prop:crossentropy}
    Consider a $2$-layer non-linear model with cross-entropy loss and L2 regularization:
    \begin{align*}
        &f(W) \equiv f(W_1, W_2) = \\& -Y \log(P)^\top-(\1-Y) \log(\1-P)^\top  \\& + \frac{\lambda_1}{2} \|W_1\|_{\rm F}^2 \;+\; \frac{\lambda_2}{2} \|W_2\|_{\rm F}^2,
    \end{align*}
where $Y\in\R^{1\times m}$ are true labels, and $P=\sigma(W_1 \phi(W_2 X))$ is the output of the model with the activation function $\phi$ such that $|\phi(s)| \leq C_1 |s|, |\phi'(s)| \leq C_2$ and $|\phi''(s)| \leq C_3$ for all $s\in\R$, sigmoid function $\sigma,$ and weight matrices $W_1\in\R^{1\times n_1}, W_2\in\R^{n_1 \times d}$.
Then, it holds

$(i)$ If $\|W_1\|_F, \|W_2\|_F$ are sufficiently small compared to the loss value in the sense of Eq.~\eqref{eq:cross_entropy_local_condition},
then $(H_0,H_1)$-smoothness holds
with constants defined as in Eq.~\eqref{eq:warmup_cross_entropy}.

$(ii)$ For all $W$, $(H_0,H_1)$-smoothness holds for constants $H_0$ and $H_1$ defined as in Eq.~\eqref{eq:CE_regularization_global_constants}.

\end{restatable}

Similarly to Proposition \ref{prop:deep_linear}, the bound derived in case (i) is sharper than the one derived in case (ii). This is natural, as case (i) concerns only the initial slice of the landscape.

Notably, we show in \Cref{app:l0l1_failure} that $(L_0, L_1)$-smoothness fails to hold even for simple two-layer networks under $\ell_2$ regularization or weight balancedness.

\subsubsection{Transformers}
\label{sec:theory_transformers}

In this section, we study a simple transformer architecture with a single attention layer trained under $\ell_2$ regularization, following the setup of \citet{zhang2024trained}. We defer the description of the model to Appendix \ref{app:transformers_proofs} and state only an informal version of the main result here: the corresponding in-context loss function is $(H_0, H_1)$-smooth. The proof of this result can be found in Appendix \ref{app:transformers_proofs} (see formal version in Prop. \ref{prop:transformers_in_context_formal}).

\begin{restatable}{proposition}{transformers}(informal)
\label{prop:in_context_transformers}
Consider a transformer model for in-context learning in continuous variables with MSE loss, or in binary variables with CE loss. Under mild regularity conditions on the distributions of the input and output spaces, it holds
\begin{equation*}
    \|\nabla^2 f(P,Q)\|_2 \leq
    H_0 + H_1 (f(P,Q)-f^*),
\end{equation*}
where $f$ is the in-context loss function, $f^*$ the minimum, $P,Q$ are parameters of the model, and $H_0$ and $H_1$ are some positive constants. When $\|P\|_F$ and $\|Q\|_F$ are small compared to the loss value, one can sharpen $H_0$ and $H_1$ similarly to Propositions \ref{prop:deep_linear} and \ref{prop:crossentropy}.
\end{restatable}

\begin{figure*}
    \centering
    \begin{tabular}{ccc}
        \includegraphics[width=0.3\linewidth]{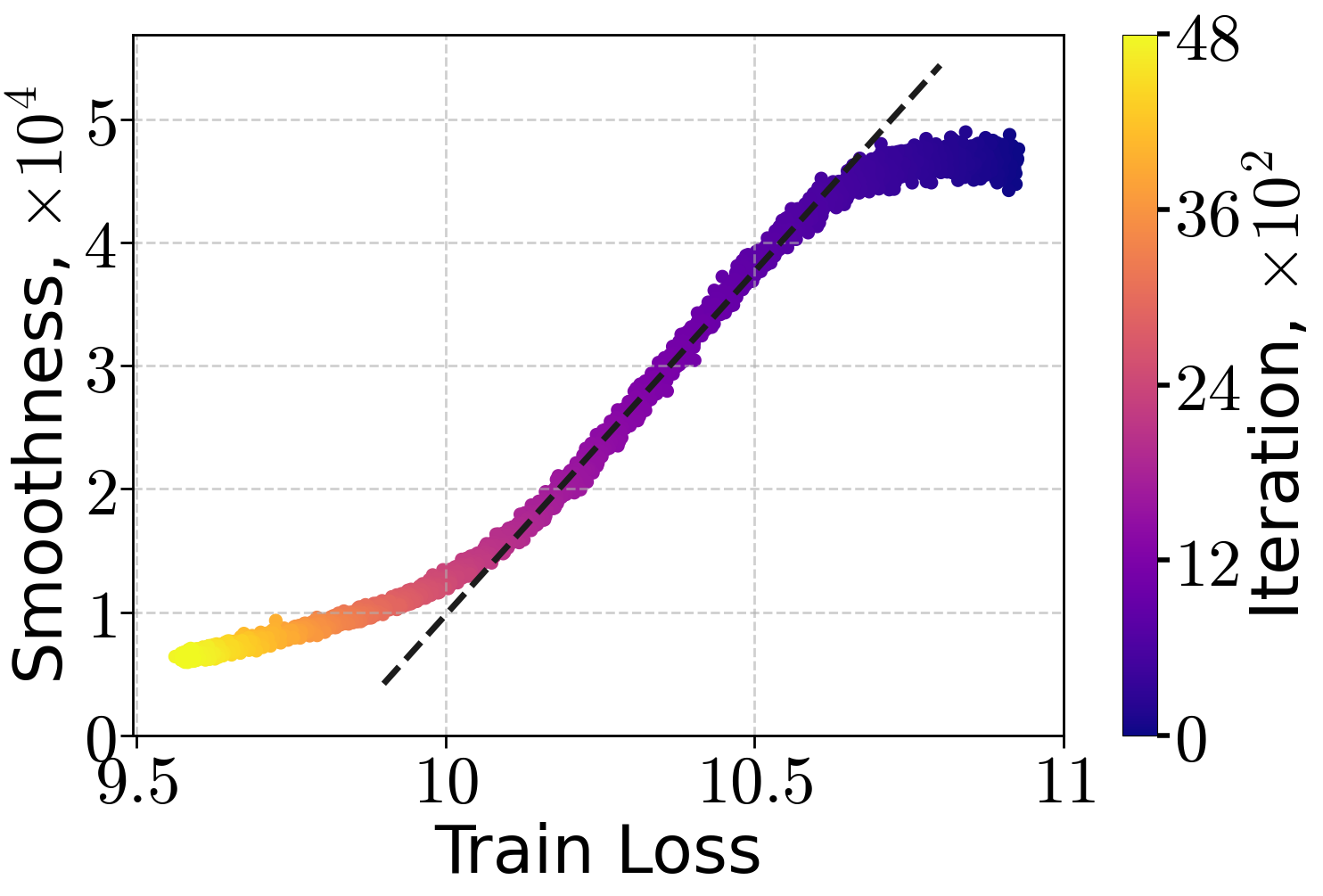} & 
        \includegraphics[width=0.3\linewidth]{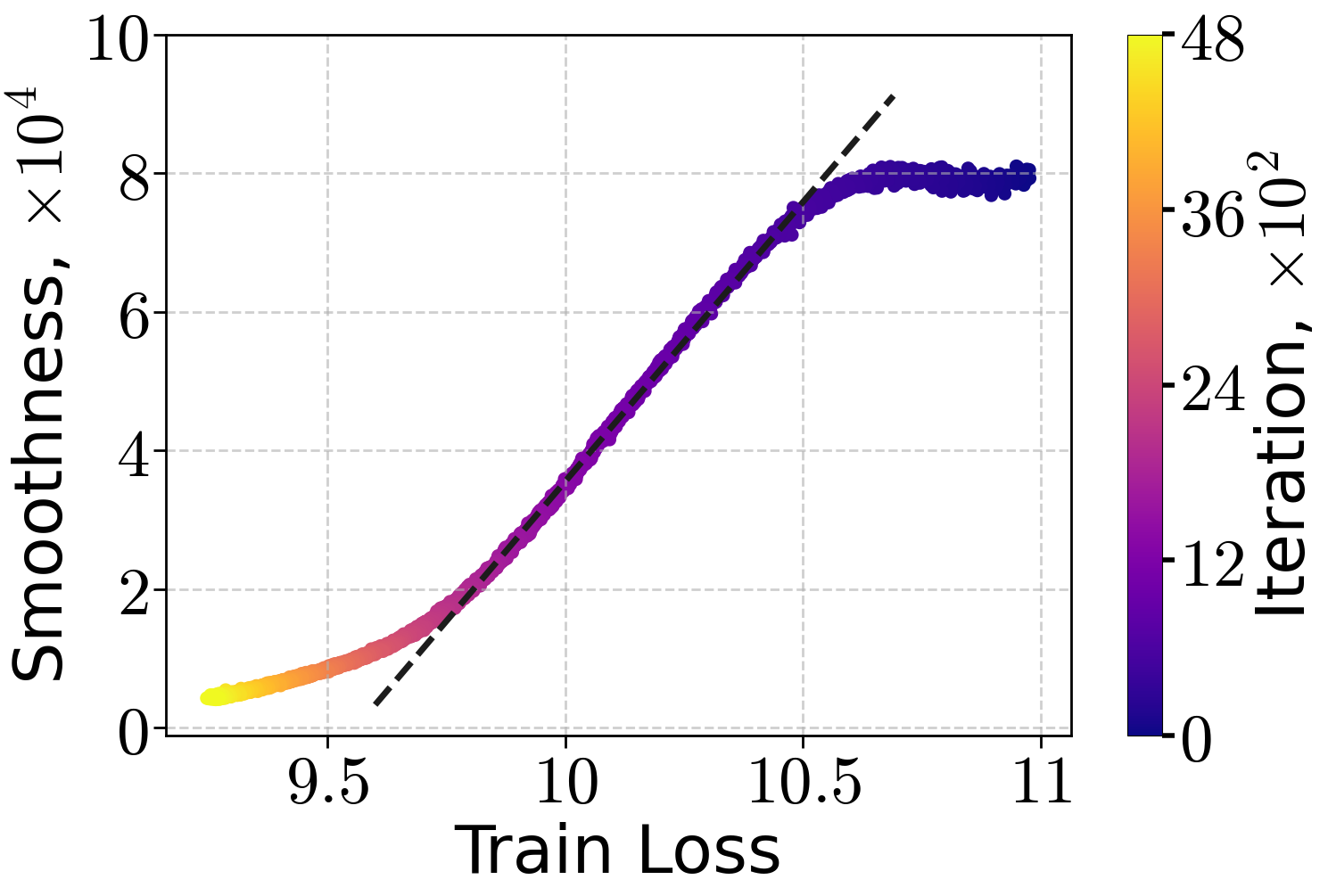} &
        \includegraphics[width=0.3\linewidth]{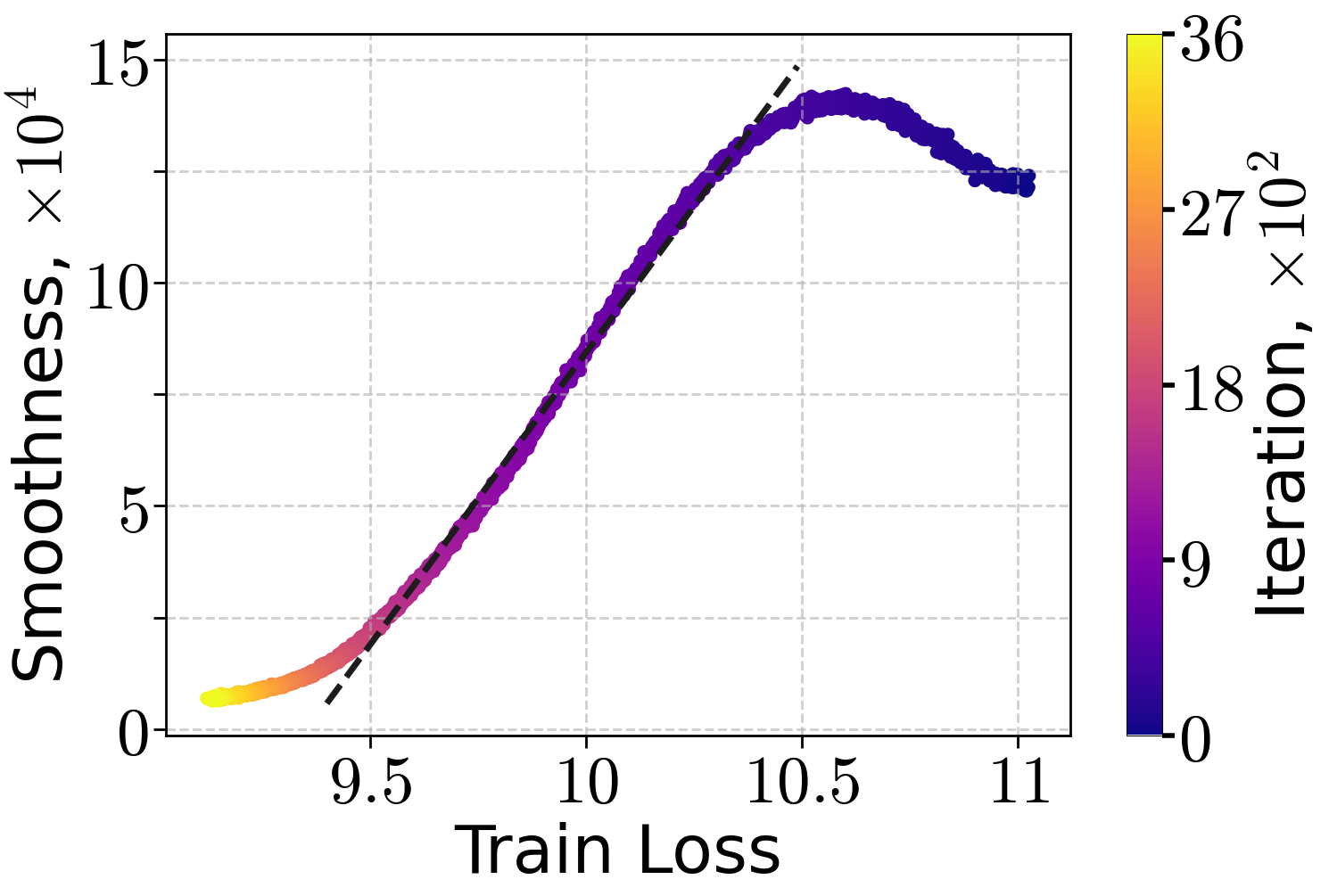} \\
        {\small 70M on FineWeb} &
        {\small 160M on FineWeb} &
        {\small 410M on FineWeb}
    \end{tabular}
    \caption{Local smoothness approximation versus training loss for language models of varying sizes on the FineWeb dataset, using \algname{SGD} with gradient clipping $1$ with a constant LR of $10^{-4}$. Each dot represents estimated local smoothness and stochastic training loss, with color indicating training progress, while the black dashed line shows the best linear fit. For much of early training, the relation is well-approximated by a line.}
    \label{fig:llm_verification}
\end{figure*}

\begin{figure*}
    \centering
    \begin{tabular}{cc}
        \includegraphics[width=0.3\linewidth]{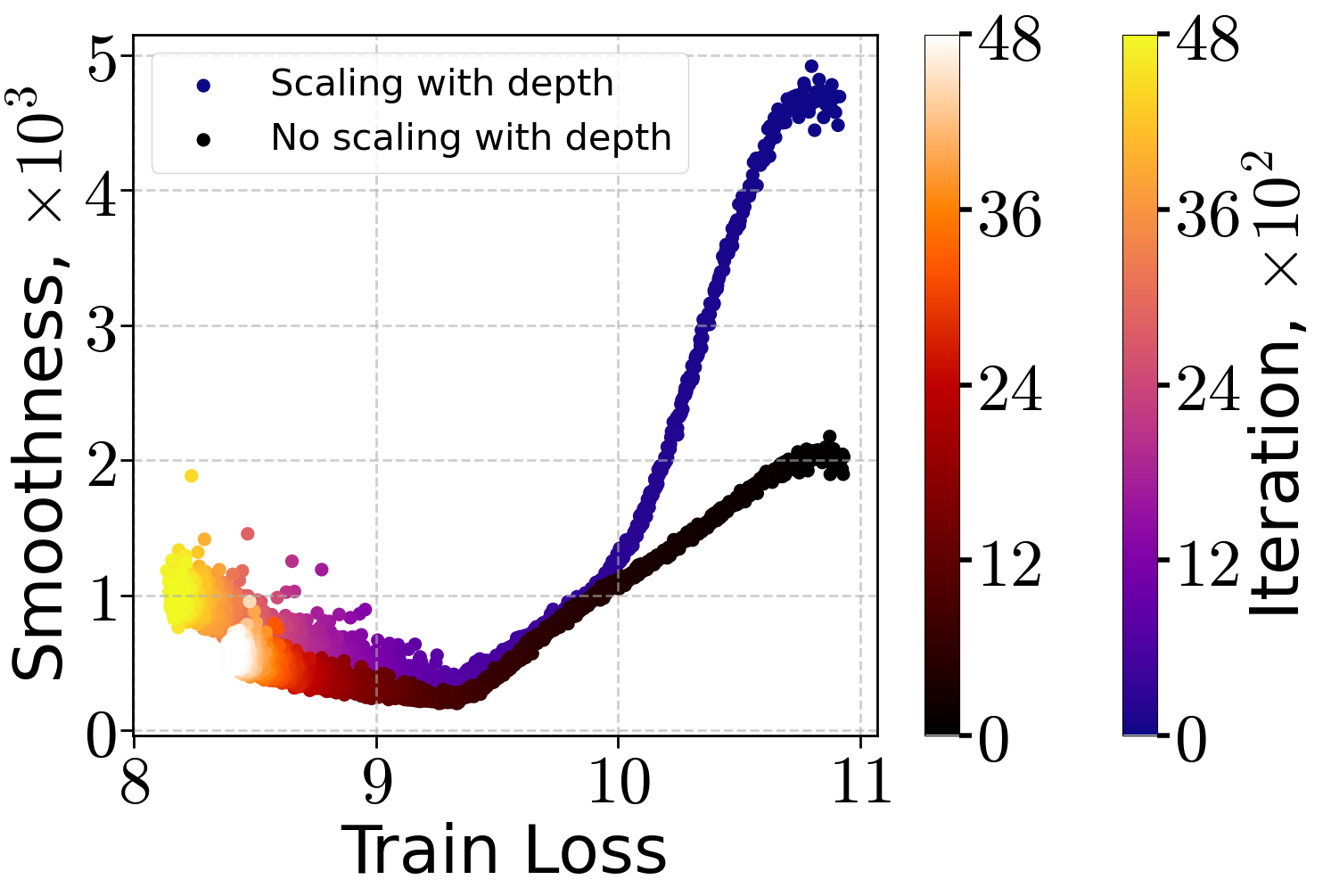} &

        \includegraphics[width=0.3\linewidth]{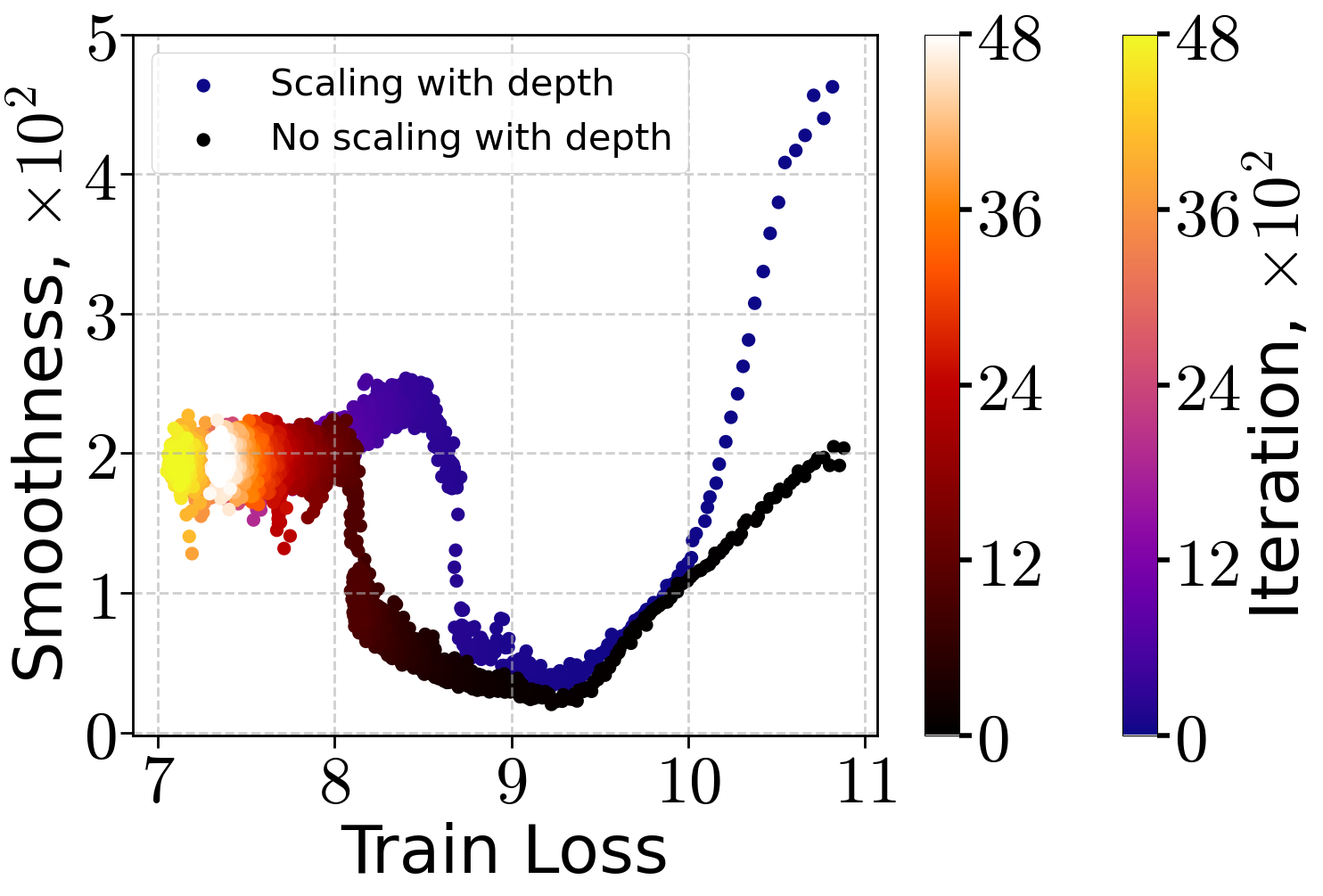} \\
        \small{LR $10^{-3}$} &
    \small{LR $10^{-2}$} 
    \end{tabular}
    \caption{Training of a 70M model on FineWeb dataset with \algname{SGD} with  gradient clipping to $1$, varying fixed learning rate and initialization scheme. Each left color bar corresponds to initialization around $0$ with variance that depends on the depth of the model (as is usually done in modern transformers), while the right one to initialization around $0$ with fixed variance.}
    \label{fig:llm_verification_second_row}
\end{figure*}

\subsection{Empirical Justification of $(H_0,H_1)$-Smoothness}\label{sec:empirical_justification}

We next turn to verifying the proposed condition in practical settings. Specifically, we examine Transformer-based language models with 70M, 160M, and 410M parameters using the NanoGPT implementation \citep{radford2019language, karpathy2022nanogpt}. Experiments are conducted on the FineWeb dataset \citep{penedo2024fineweb} using \algname{SGD} with gradient clipping at $1$ and a small constant learning rate of $10^{-4}$. Using such a conservative LR and gradient clipping allows the optimizer to progress slowly, thereby probing the landscape around initialization in more detail. As a stochastic proxy for local smoothness at iteration $k$, we compute $$\frac{\|\nabla f_{S_k}(w_{k+1})-\nabla f_{S_{k-1}}(w_k)\|}{\|w_{k+1} -w_k\|},$$ where $S_k$ denotes the mini-batch at iteration $k$, following prior work \citep{zhang2019gradient, riabinin2025gluon}. As shown in \Cref{fig:llm_verification}, the estimated smoothness decays approximately linearly, indicating that the proposed condition provides a reasonable smoothness approximation for real-world models.

To provide a more complete picture of smoothness dynamics over the full training trajectory and the role of initialization, we train a 70M language model using larger constant learning rates and two different initialization schemes. The results are shown in Figure~\ref{fig:llm_verification_second_row} . We observe an initial curvature reduction phase, previously identified in \citep{kalra2024warmup,kalra2023universal}, followed by progressive sharpening and the edge of stability phases \citep{cohen2021gradient}. 

Notably, the initial sharpness is higher when the initialization is scaled with depth, a common practice in language model training \citep{sun2025curse,zhang2019improving}. We hypothesize that this effect may partly explain the need for learning rate warm-up in Transformer training, and defer a more detailed discussion of initialization effects to \Cref{sec:initialization}. Moreover, the transition from curvature reduction to progressive sharpening consistently occurs at the same loss value, independent of the learning rate and initialization, further highlighting the link between sharpness and loss. The initial curvature reduction phase accounts for approximately $20\%$ of the total training loss reduction.

We next turn to image classification on ImageNet-32 \citep{chrabaszcz2017downsampled}, training both ResNet50 \citep{he2016deep} and ViT-Tiny \citep{dosovitskiy2020vit}. The results, shown in \Cref{fig:resnet50_vit_verification}, indicate that a linear function provides a good approximation of the relationship between local smoothness and training loss. Compared to language models, however, the points are more widely dispersed and have larger variance. Taken together, \Cref{fig:llm_verification,fig:resnet50_vit_verification} support the view that $(H_0, H_1)$-smoothness offers a reasonable approximation of smoothness in early training.

\begin{figure}
    \centering
    \begin{tabular}{cc}
    \includegraphics[width=0.3\linewidth]{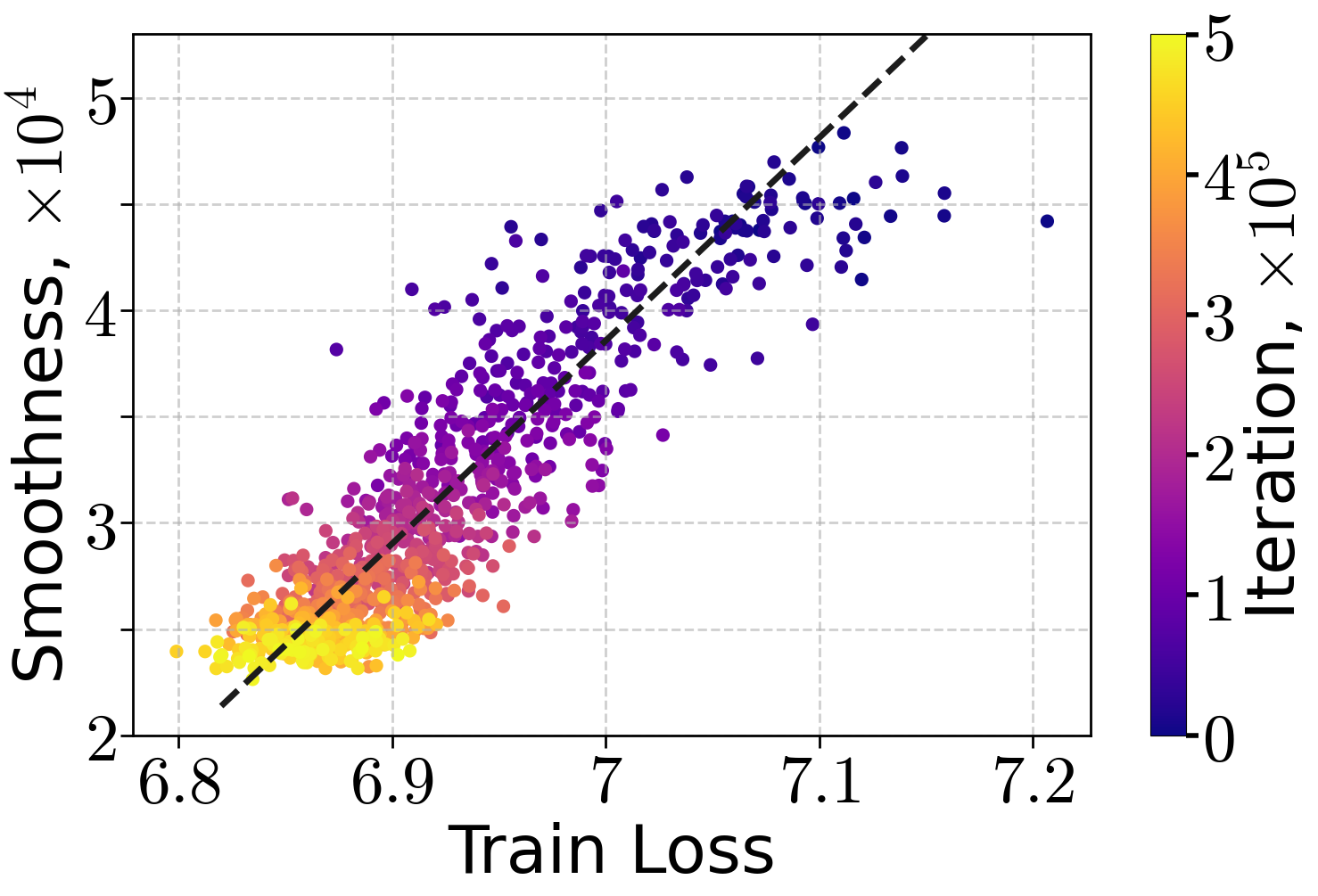} & 
    \includegraphics[width=0.3\linewidth]{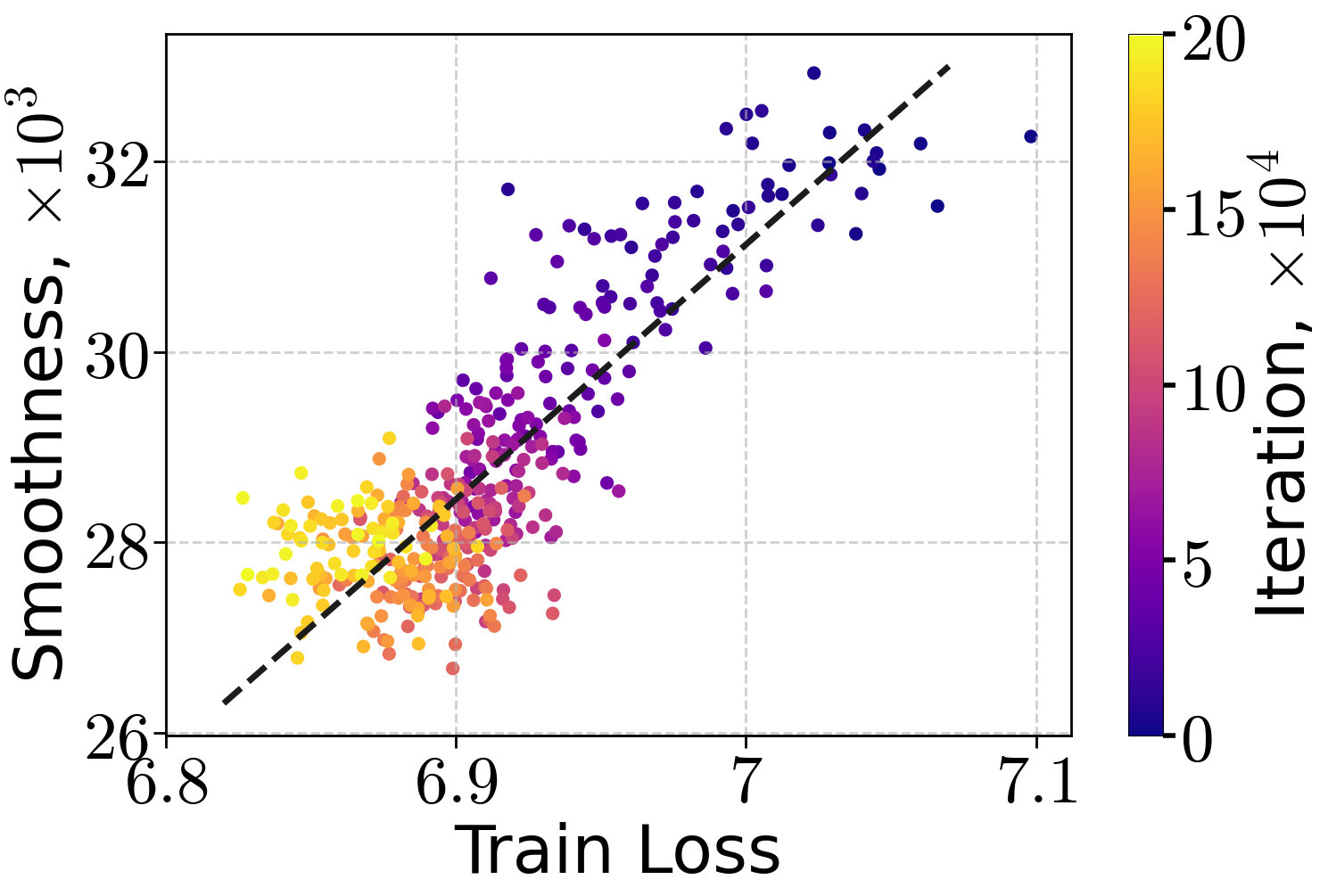} \\
    ResNet50 on ImageNet-32 &
    ViT-Tiny on ImageNet-32
    \end{tabular}
    \caption{Stochastic local-smoothness proxy versus training loss for ResNet50 (left) and ViT-Tiny (right) on ImageNet-32, trained with \algname{SGD} and a constant LR $10^{-4}.$ }
    \label{fig:resnet50_vit_verification}
    
\end{figure}

\section{Theoretical Analysis under $(H_0,H_1)$-Smoothness}\label{sec:convergence_analysis}

We study the minimization problem $\min_{w} f(w)$, which appears in various machine learning applications. Here $w\in\R^d$ denotes parameters of some model,  $d$ is the number of parameters, and $f$ is the loss that measures the performance. We define $f^*\eqdef \min_w f(w) > -\infty$ as the optimal loss. The set $\cS$ contains all global minimizers of the objective $f.$ Importantly, our theoretical results apply to the regime in which \Cref{asmp:H0H1_smoothness} holds, which in our setting is most relevant during the early sharpness-reduction phase of neural network training. 
The proofs of this section are deferred to \Cref{sec:conv_proofs_non_convex} and \ref{sec:lower_bounds_proofs}.

\subsection{Notation and Assumptions}

We conduct our analysis for well-known classes of non-convex functions, presented below.

\begin{definition}[\citet{liu2023aiming}]\label{asmp:aiming_condition}
    A function $h$ satisfies the Aiming condition with a constant $\theta > 0$ around the set $\cX$, if $\langle\nabla h(w), w - \pi_{\cX}(w) \rangle \ge \theta(h(w) - h^*)$ holds for all $w\in\R^d$. Here, $\pi_{\cX}(w)$ is the projection of $w$ onto the set $\cX$, and $h^*\eqdef \min_{w\in\R^d}h(w).$
\end{definition}

\begin{definition}[\citet{polyak1963gradient}]\label{asmp:pl_condition}
   A function $h$ satisfies Polyak-{\L}ojasiewicz  (PL) condition with a constant $\mu > 0$, if $\|\nabla h(w)\|^2 \ge 2\mu(h(w)-h^*)$ holds for all $w\in\R^d$.
\end{definition}

\subsection{Lower Bounds and Convergence of GD with Constant Step-size}

To enable a meaningful comparison between the step-size schedule suggested by the $(H_0, H_1)$-condition and an alternative fixed step-size strategy, we derive lower complexity bounds for the latter. This approach follows the idea of Theorem 4 in \cite{zhang2019gradient}: we first consider a rapidly growing function and show that, for \algname{GD} to converge, the step-size must be sufficiently small. Next, we examine a slowly growing function and demonstrate that this previously derived step-size constraint leads to slow convergence of the algorithm. The complete proof can be found in Appendix \ref{sec:lower_bounds_proofs}.

\begin{restatable}{theorem}{lowerbounds}\label{th:lower_bounds}
    Let $f$ belong to the class $\mathcal{H}$ of $(H_0,H_1)$-smooth functions. Then it holds:
\begin{enumerate}
    \item To satisfy $\|\nabla f(w_K)\| \;\leq \varepsilon$ for a general non-convex function $f$, \algname{GD} with constant step-size initialized at $w_0$, needs at least
    \begin{equation*}
        K \geq \frac{H_1 (f(w_0)-f^*)}{\log (f(w_0)-f^*)+1} \frac{f(w_0)-f^*-2 \epsilon^2}{8 \epsilon^2}
    \end{equation*}
    iterations.

    \item To satisfy $f(w_K)-f^* \leq \varepsilon$ for convex function $f$, \algname{GD} with constant step-size initialized at $w_0$, needs at least
    \begin{align*}
    K \geq \frac{H_1 (f(w_0)-f^*)}{\log(f(w_0)-f^*)+1} \frac{f(w_0)-f^*-\epsilon}{4 \epsilon}
    \end{align*}
    iterations.

    \item To satisfy $f(w_K)-f^* \leq \varepsilon$ for $\mu$-PL function $f$ (but not necessarily  convex), \algname{GD} with constant step-size initialized at $w_0$, needs at least
\begin{equation*}
    K \geq \frac{H_1}{4 \mu} \frac{(f(w_0)-f^*)}{\log(f(w_0)-f^*)+1} \log\left(\frac{f(w_0)-f^*}{\epsilon} \right)
\end{equation*}
iterations.

\end{enumerate}

\end{restatable}

This result covers the one in \citep{zhang2019gradient} as a special case, and it also covers convex and $\mu$-PL functions. These lower bounds concern fixed step-sizes chosen to be stable over the relevant $(H_0, H_1)$-smooth class; the initialization dependence appears because such a step-size must account for the worst early curvature of a non-uniformly smooth landscape. We show next that, for $(H_0, H_1)$-smooth functions, a gradually increasing step-size schedule can accelerate convergence while maintaining training stability.

\subsection{Convergence of GD with Adaptive Warm-up}

Next, we turn to the analysis of \algname{GD} under Assumption~\ref{asmp:H0H1_smoothness} with an adaptive step-size of the form
\begin{equation}
\label{eq:adaptive_stepsize}
   \eta_k := \frac{1}{10H_0 + 20H_1(f(w_k)-f^*)} 
\end{equation}
 prescribed by $(H_0, H_1)$-smoothness. The constants $H_0, H_1, f^*$ are unknown in practical training; \Cref{sec:experiments_main} therefore uses a one-parameter surrogate in which a threshold C controls when the loss-dependent denominator becomes active. Since the function suboptimality decreases at the beginning of training, the theoretical step-size follows a warm-up-like scheme. In the general non-convex case, the derived upper bound in \Cref{th:general_nonconvex_adaptive} provides only numerical (non-asymptotic) improvement over a constant schedule.

To achieve tangible improvements, additional convexity-like assumptions are necessary. This is justified by the fact that the loss landscape of neural networks exhibits additional structure; prior studies indicate that neural network loss surfaces often display a convex-like geometry locally \citep{kleinberg2018alternative, guille2023no, islamov2024loss, tran2024empirical}. This observation has motivated relaxations of convexity, such as the aiming condition \citep{liu2023aiming} and quasar-convexity \citep{hardt2018gradient}, which have been leveraged in the analysis of various gradient-based algorithms \citep{gower2021sgd, hinder2020near, fu2023accelerated}. Importantly, these conditions are satisfied by certain classes of non-convex functions \citep{hardt2018gradient, liu2023aiming}.

\begin{restatable}{theorem}{theoremaimingdeterministic}\label{th:aiming_deterministic_adaptive_step-size}
    Assume that $f$ is $(H_0,H_1)$-smooth, and it satisfies the Aiming condition with constant $\theta$ around the set of global minimizers $\cS$. Then the iterates of \algname{GD} with adaptive step-size $\theta \cdot \eta_k$ 
    satisfy  
    \begin{align*}
    & f(w_K) - f^* \le \varepsilon \quad \text{after at most} \\ &\frac{40H_0{\rm dist}(w_0,\cS)^2}{\theta^2\varepsilon} + \frac{40H_1{\rm dist}(w_0,\cS)^2}{\theta^2} \quad \text{iterations.}
    \end{align*}
     
\end{restatable}

To analyze convergence under increasing step-size, we split the iterations into two regimes corresponding to large and small function values and study them separately. The convex rate is recovered by setting $\theta = 1$. Notably, the $1/\varepsilon$ term depends only on $H_0$, as in standard convex \algname{GD} theory, while $H_1$ affects only the constant term. Comparing the bounds in \Cref{th:aiming_deterministic_adaptive_step-size} and \Cref{th:lower_bounds}, we find that \emph{\algname{GD} with a warm-up adaptive step-size converges faster than with a fixed step-size when the initialization factor $H_1(f(w_0)-f^*)/\varepsilon$ is large}, that is, when the initial loss is large. This factor can be substantial, potentially even exponential in $H_1{\rm dist}(w_0,\cS)$ \citep{gaash2025convergence}.

Next, we consider another widely studied class of structured non-convex functions, which encompasses the $\mu$-PL functions--known to hold for sufficiently over-parameterized networks \citep{liu2022loss}. Moreover, PL is considered the weakest sufficient condition ensuring linear convergence of \algname{GD} \citep{karimi2016linear}.
\begin{restatable}{theorem}{theorempleterministic}\label{th:adaptive_PL}
Assume that $f$ is $(H_0,H_1)$-smooth, and it satisfies $\mu$-PL condition. Then the iterates of \algname{GD} with adaptive step-size $\eta_k$ 
satisfy 
    \begin{align*}
        & f(w_K) - f^* \le \varepsilon \quad \text{after at most} \\ & \frac{40H_1}{\mu}(f(w_0)-f^*) + \frac{20H_0}{\mu}\log\frac{H_0}{2H_1\varepsilon} \quad \text{iterations.}
    \end{align*}

\end{restatable}

Similar to the convex case, the $\varepsilon$-dependent term in \algname{GD} with a warm-up adaptive step-size leads to faster convergence whenever $H_1(f(w_0)-f^*)$ is substantially larger than $H_0$. 

In \Cref{sec:concurrent_work}, we show that our proof techniques in both Theorems \ref{th:aiming_deterministic_adaptive_step-size} and \ref{th:adaptive_PL} extend to a more general class of functions, where the function suboptimality in Definition \ref{asmp:H0H1_smoothness} is raised to the power $\rho\ge 1$, extending the benefits of the theoretical warm-up to a broader class of functions.

\begin{figure*}
    \centering
    \begin{tabular}{ccc}

    \hspace{-2mm}\includegraphics[width=0.3\linewidth]{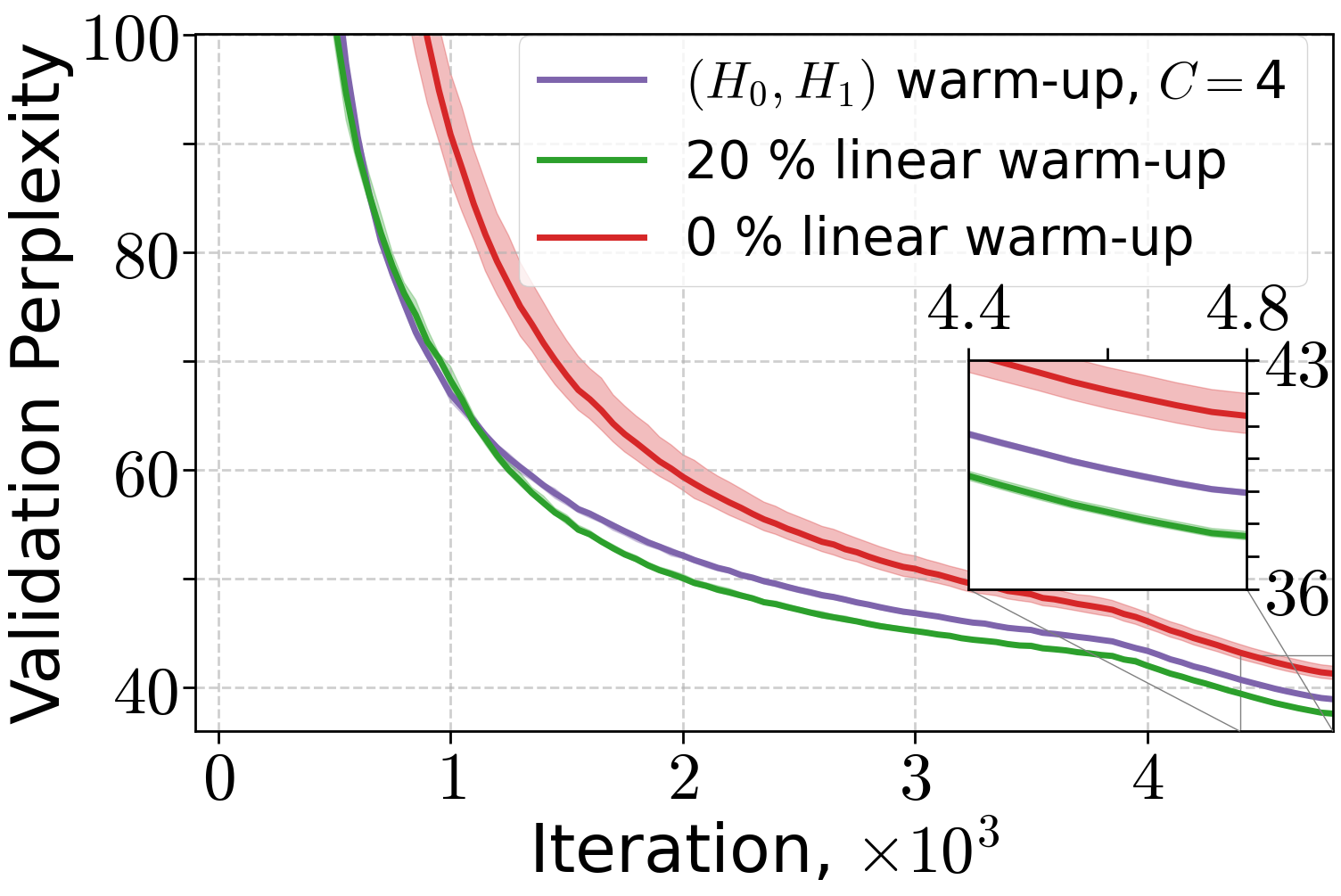} &
    \hspace{-3mm}\includegraphics[width=0.3\linewidth]{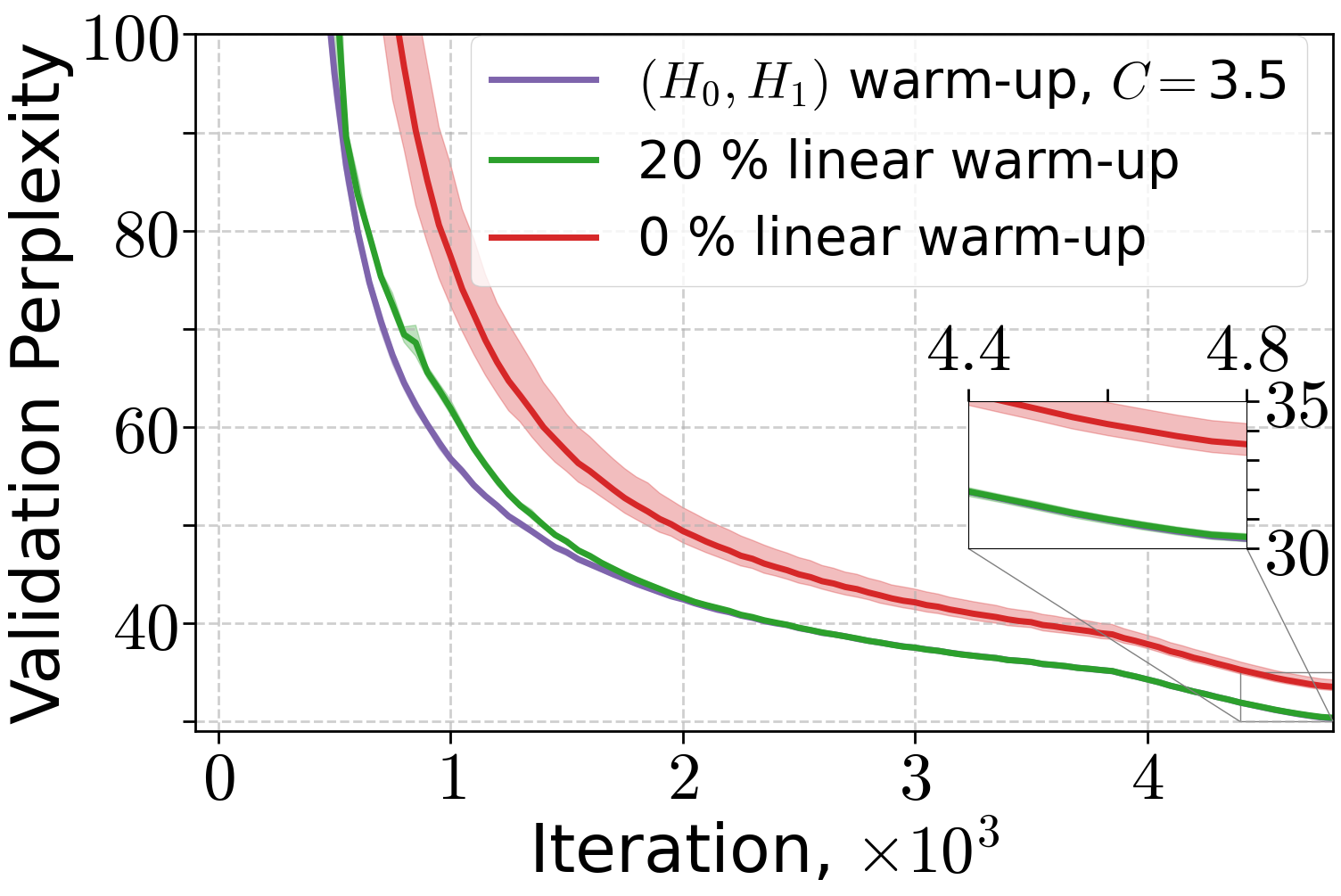}  &
    \hspace{-3mm}\includegraphics[width=0.3\linewidth]{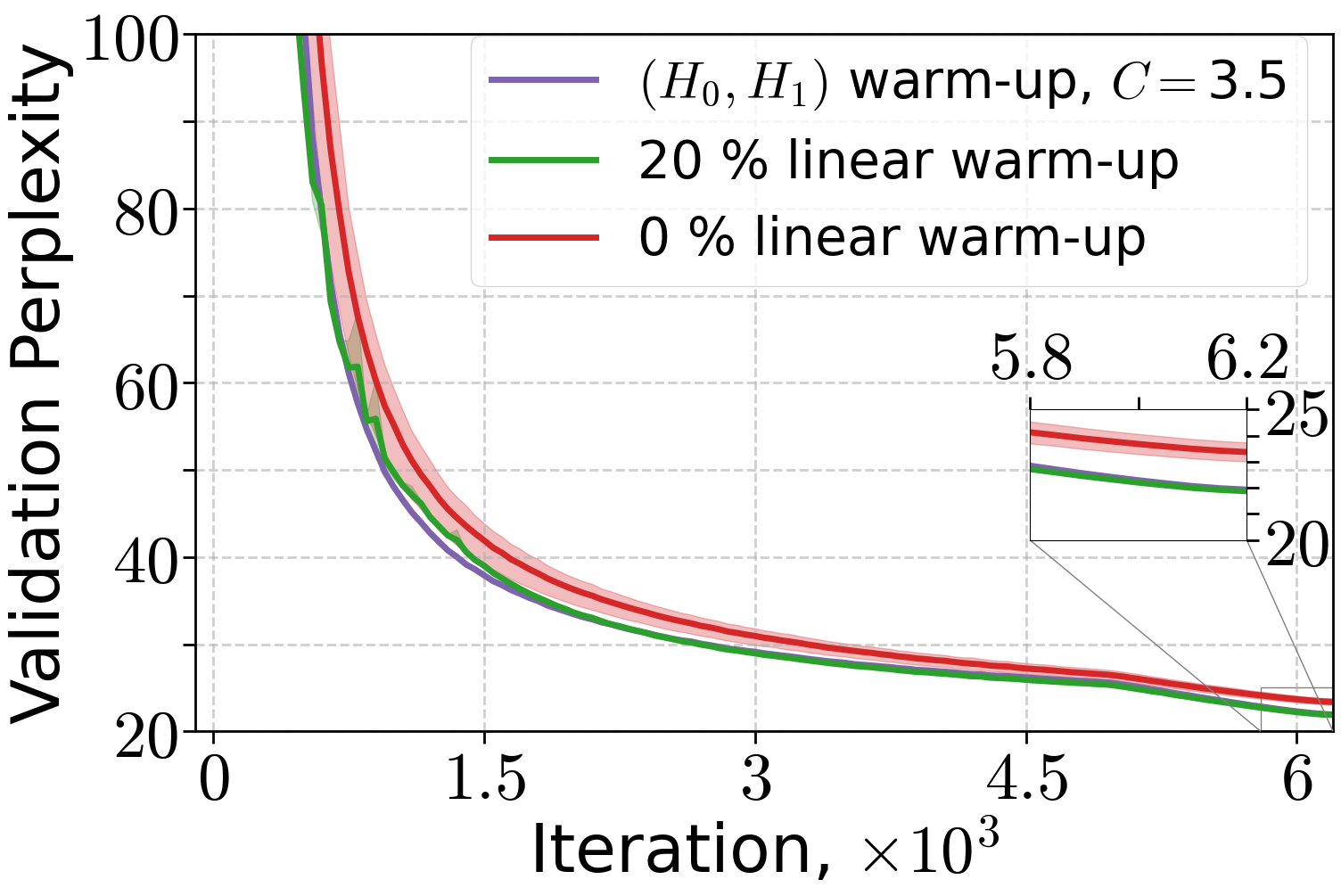}\\
    \hspace{-3mm}\includegraphics[width=0.3\linewidth]{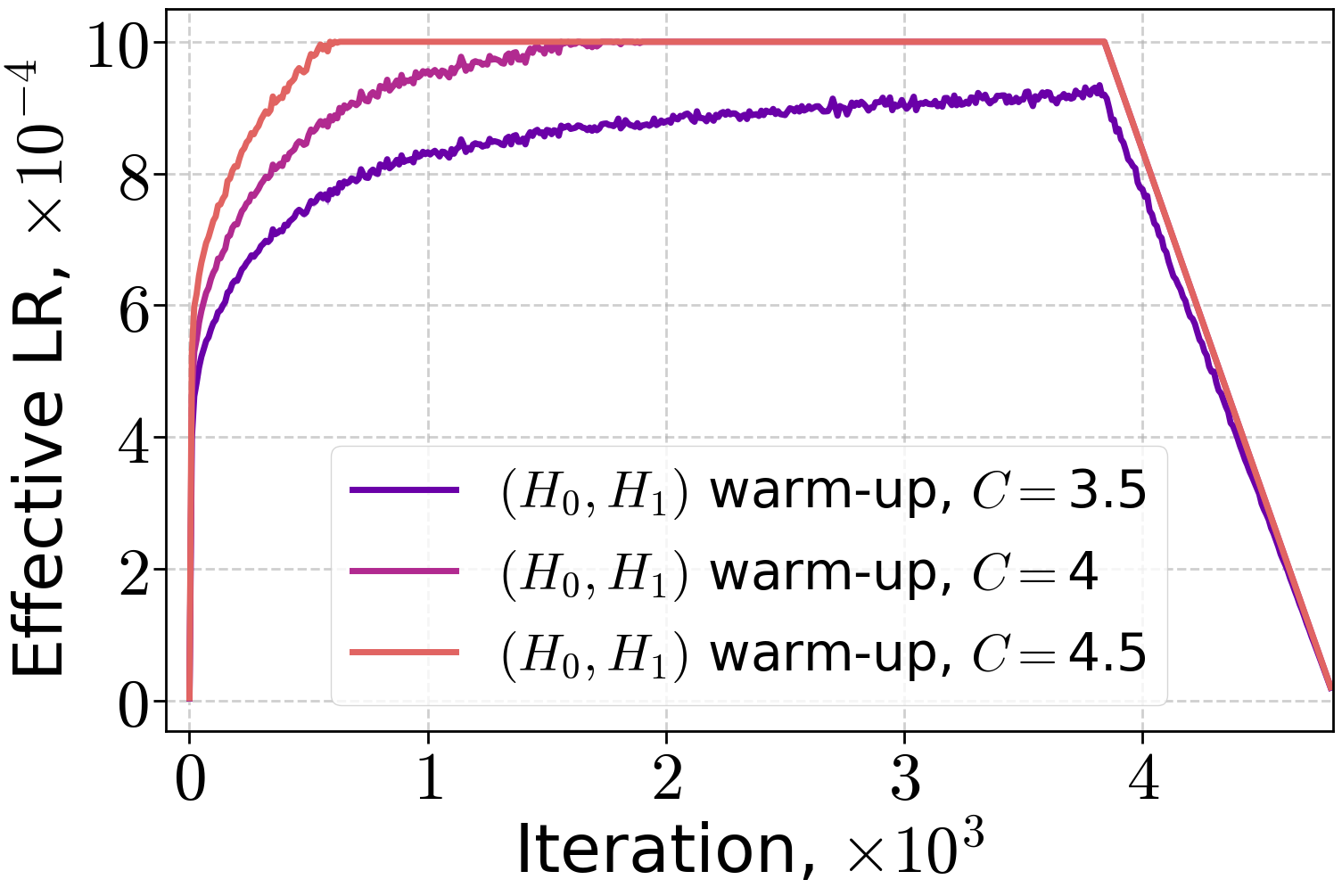} &
    \hspace{-3mm}\includegraphics[width=0.3\linewidth]{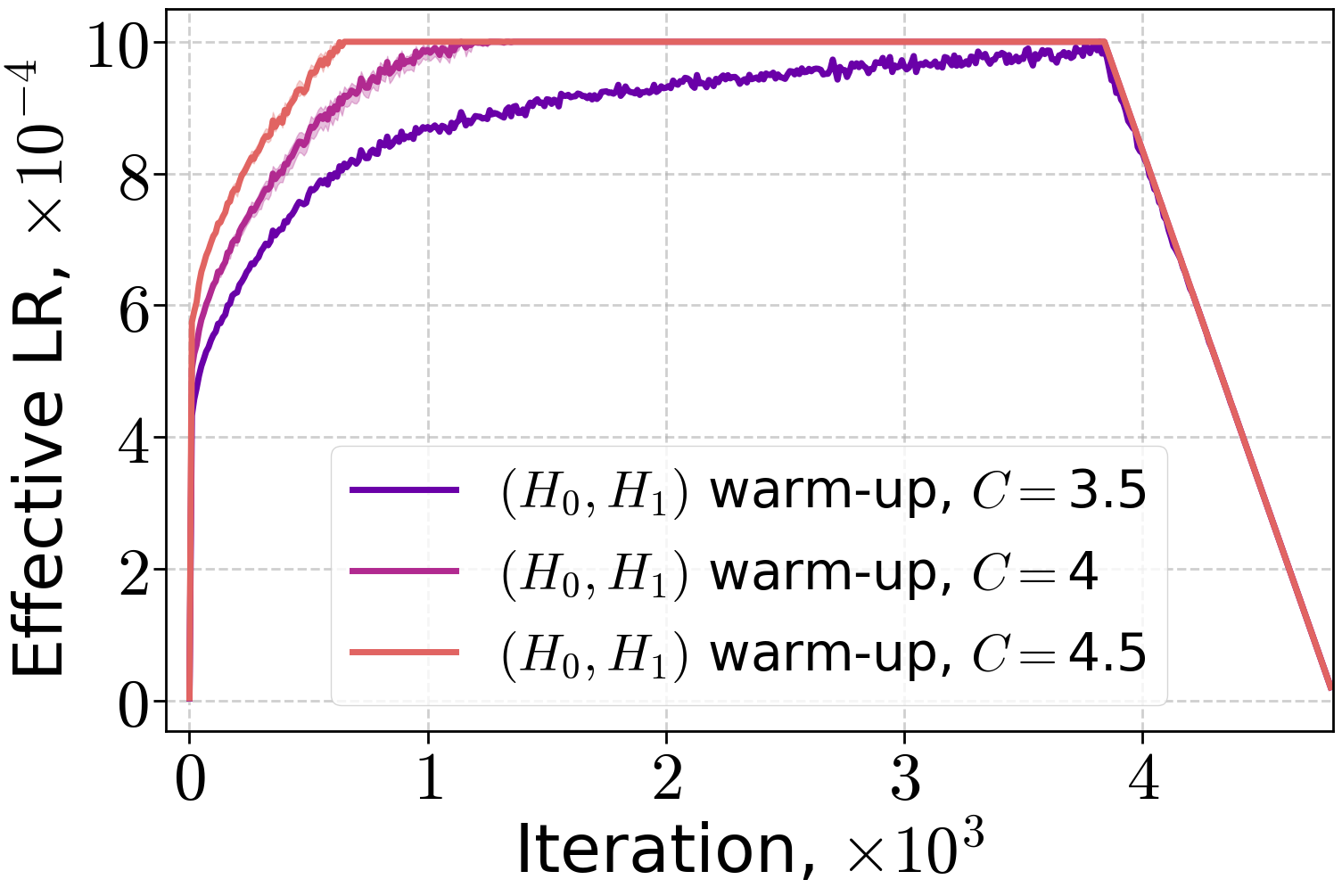} &
    \hspace{-3mm}\includegraphics[width=0.3\linewidth]{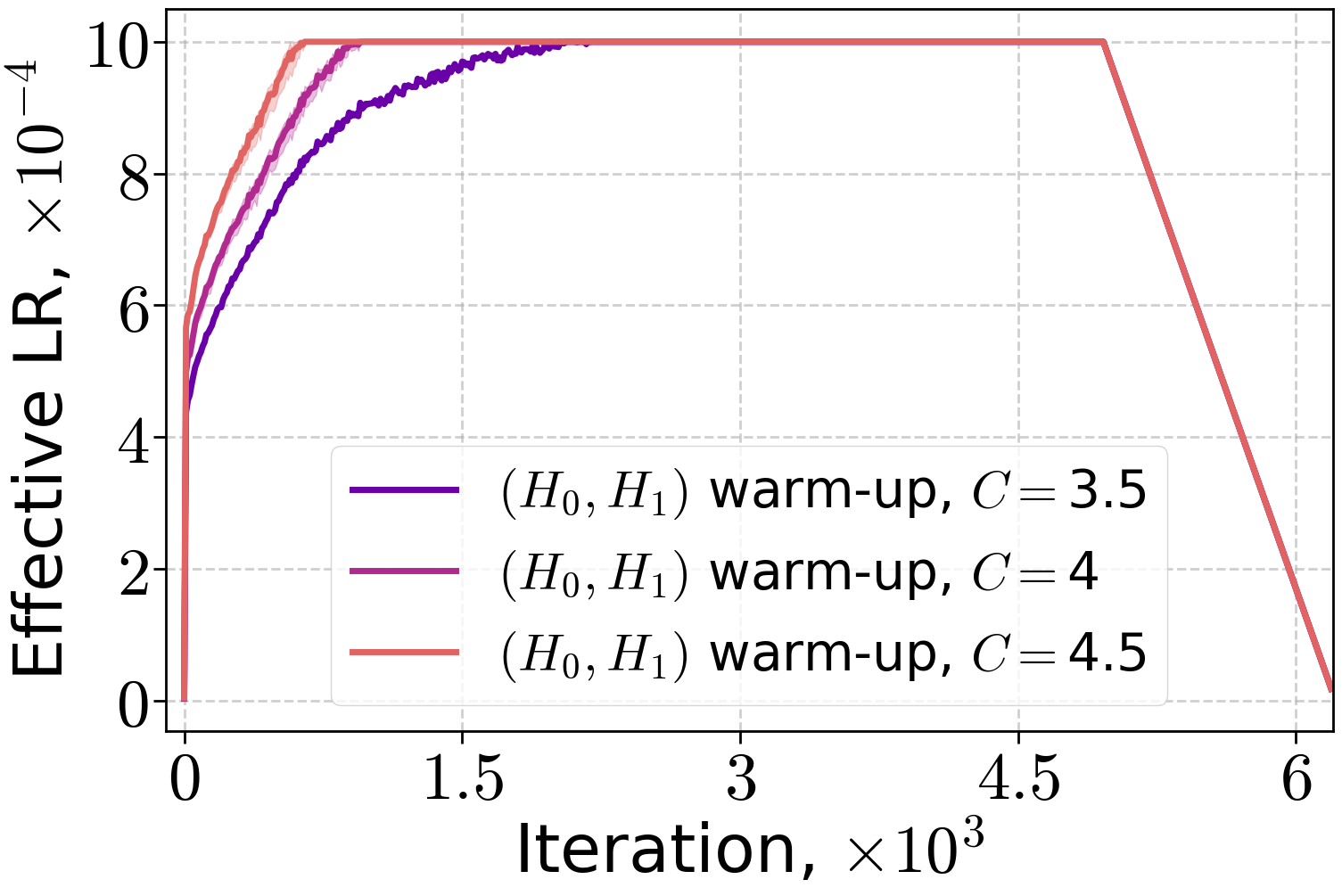} 
    \\
    70M on FineWeb &
    160M on FineWeb &
    410M on FineWeb 
    \end{tabular}
    
    \caption{{\bf Top row:} validation perplexity for 70M, 160M, and 410M language models on FineWeb under no warm-up, tuned linear warm-up, and $(H_0, H_1)$ warm-up. {\bf Bottom row:} effective LR induced by $(H_0, H_1)$ warm-up for peak LR $10^{-3}$ and different values of C. All runs use a final 20\% linear decay to $10^{-5}$.}
    \label{fig:language_models_comparison}
    
\end{figure*}

\subsection{Extension to the Stochastic Setting}

In a standard training setup, the function $f$ has a finite sum structure, namely,
\begin{align}\label{eq:finite_sum}
    f(w) \eqdef \frac{1}{n}\sum_{i=1}^nf_i(w) \tag{P}
\end{align}
where $n$ is the size of the training dataset, and each $f_i$ represents a loss on $i$-th sample. We define the minimum of each loss $f_i^* = \min_{w}f_i(w).$ To study the convergence in the stochastic setting, we need an interpolation condition, which is typically satisfied for over-parameterized networks \citep{ma2018power}. Analytically, it means that $f^*=f_i^*$ for all~$i\in[n].$ 

\begin{restatable}{theorem}{theoremaimingstochastic}\label{th:adaptive_aiming_stochastic}
    Assume that the problem \eqref{eq:finite_sum} satisfies the interpolation condition. Assume that each $f_i$ is $(H_0,H_1)$-smooth and satisfies the Aiming condition around the set of global minimizers $\cS.$ Then the iterates of \algname{SGD} $w_{k+1} = w_k - \eta_k\nabla f_{S_k}(w_k)$ with batch $S_k\subseteq [n]$ and a step-size 
    $$\eta_k = \frac{\theta}{10H_0+20H_1(f_{S_k}(w_k)-f_{S_k}^*)}$$ satisfy 
    \begin{align*}
         &\frac{1}{K+1}\sum_{k=0}^{K}\E{\min\left\{f(w_k)-f^*, \frac{H_0}{2n H_1}\right\}}\\ 
         &\hspace{5cm}\le \tfrac{20H_0{\rm dist}(w_0,\cS)^2}{\theta^2(K+1)}.
    \end{align*}
\end{restatable}
We observe that the convergence rate depends on $H_0$, mirroring the deterministic result in \Cref{th:aiming_deterministic_adaptive_step-size}. The convergence metric we use is non-standard, adopted because uniform convergence over all component functions $\{f_i\}_{i=1}^n$ cannot be ensured. Equivalently, the probability that all iterates remain outside the small-error region $f(w_k) -f^* \le \frac{H_0}{2nH_1}$ is at most $\frac{40nH_0{\rm dist}(w_0,\cS)^2}{\theta^2(K+1)}$ the suboptimality $f(w_k)-f^*$ can be larger than $\frac{H_0}{2nH_1}$ for any $k\in\{0,\dots, K\}.$ Thus, this failure probability vanishes as $K\to \infty$, implying convergence after a sufficiently large number of iterations with high probability.

As discussed earlier, our guarantees are predictive when the $(H_0,H_1)$ condition captures early sharpness reduction. Our theory shows that \algname{SGD} with an increasing step-size stabilizes by adapting to local curvature, whereas a fixed step-size must be unrealistically small under poor initialization.

\begin{remark}
    Interestingly, both Aiming \citep{liu2023aiming} and PL \cite{liu2022loss} conditions hold at least locally around the initial point of training. This matches well the study of warm-up, which concerns the initial phase of training.
\end{remark}

\section{Experiments}\label{sec:experiments_main}

We next evaluate the warm-up schedule derived from $(H_0,H_1)$-smoothness on two benchmarks: transformer language modeling on FineWeb and ViT-Tiny training
on ImageNet-32. In both settings, warm-up is standard and empirically beneficial, as confirmed by our experiments.  This section aims to highlight the merits of warm-up, particularly the gains obtained from the $(H_0, H_1)$-smoothness–driven schedule rather than to achieve state-of-the-art performance. To test the theoretically motivated warm-up schedule, we compare no warm-up and tuned linear warm-up against the following $(H_0, H_1)$-condition-based schedule: $$\frac{\eta_k}{\max\left\{1, f_{S_k}(w_k)/C\right\}}.$$ Here $f_{S_k}(w_k)$ is the stochastic loss at iteration $k$, and $C$ is a tunable threshold that controls when the denominator becomes active and therefore determines the effective warm-up length. Here $\eta_k$ follows the WSD schedule (language modeling) or cosine annealing (ViT training) with no warm-up. All training details are reported in \Cref{sec:additional_exp}.

\paragraph{Language Modeling.}
 
We train language models of three sizes: 70M, 160M, and 410M near Chinchilla optimum \citep{hoffmann2022chinchilla}. For the 70M and 160M models, we use \algname{Adam} with a WSD schedule, a 20\% decay phase, and a tuned linear warm-up length in $\{0\%,10\%,20\%\}$. For the 410M model, we use \algname{AdamW} with weight decay $\lambda=0.1$ \citep{loshchilov2017decoupled}. We report the mean of 3 runs, with the shaded area showing the min–max range.

In this setup, $\eta_k$ in $(H_0, H_1)$ warm-up follows the WSD schedule without warm-up (i.e., the LR starts directly at its peak) with a 20\% decay phase. This can be viewed as a hard counterpart of the theoretical step-size considered in our convergence analysis. We tune the parameter $C$, which determines the length of the $(H_0, H_1)$ warm-up, over the set $\{3.5, 4, 4.5\}$ (which was found to yield good results empirically). For all warm-up schedules, we tune the peak LR over $\{3 \cdot 10^{-4}, 10^{-3}, 3 \cdot 10^{-3}, 10^{-2}\}$. \Cref{fig:language_models_comparison} shows that the theoretically motivated $(H_0, H_1)$ warm-up performs competitively with linear warm-up, which is the standard choice in practice, and both warm-up schedules improve over training without warm-up. \Cref{fig:language_models_comparison}, bottom row, plots the effective LR. Unlike linear warm-up, the $(H_0, H_1)$-based schedule increases nonlinearly because it is driven by the loss trajectory.

\paragraph{Image Classification.}

Next, we repeat the study on ViT-Tiny using cosine annealing for $\eta_k$ (replacing WSD) while keeping the same warm-up mechanisms. For the $(H_0, H_1)$ warm-up, we sweep $C\in\{3, 3.5, 4\}$; for linear warm-up, we vary the warm-up length in $\{0\%,5\%,10\%\}$. For each schedule, we grid-search the peak LR over $\{3 \cdot 10^{-4}, 10^{-3}, 3 \cdot 10^{-3}, 10^{-2}, 3\cdot 10^{-2}\}$. As in the previous setting, \Cref{fig:vit_comparison} shows that $(H_0, H_1)$ warm-up matches linear warm-up, and both outperform training with no warm-up. The right panel of \Cref{fig:vit_comparison} plots the effective LR, showing again that the loss-driven warm-up profile is nonlinear. Similar to the previous case, the warm-up substantially differs from the linear warm-up. We report the mean of three runs, with the shaded area showing the min–max range.

\begin{figure}
    \centering
    \begin{tabular}{ccc}

   \includegraphics[width=0.3\linewidth]{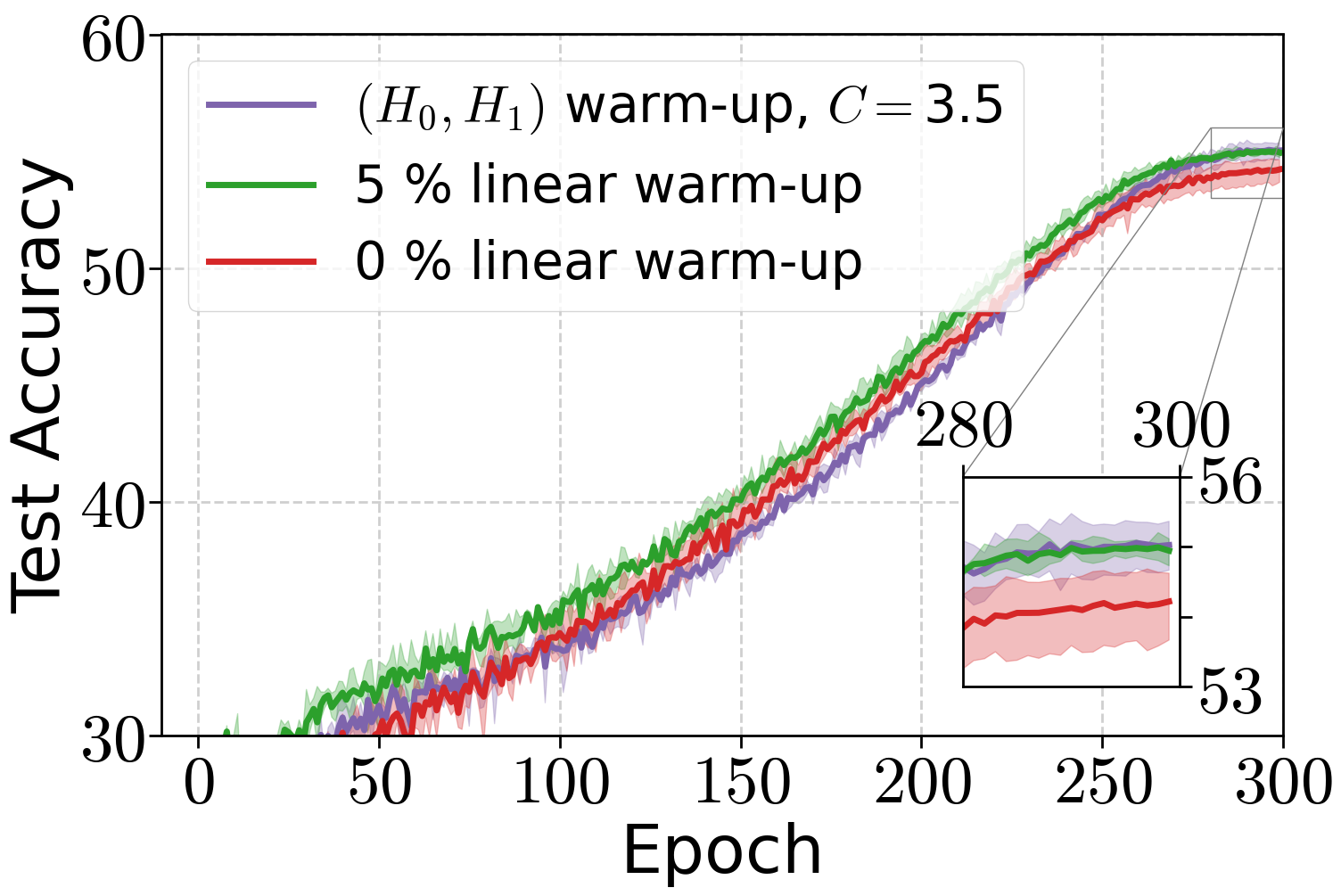} & 
    \includegraphics[width=0.3\linewidth]{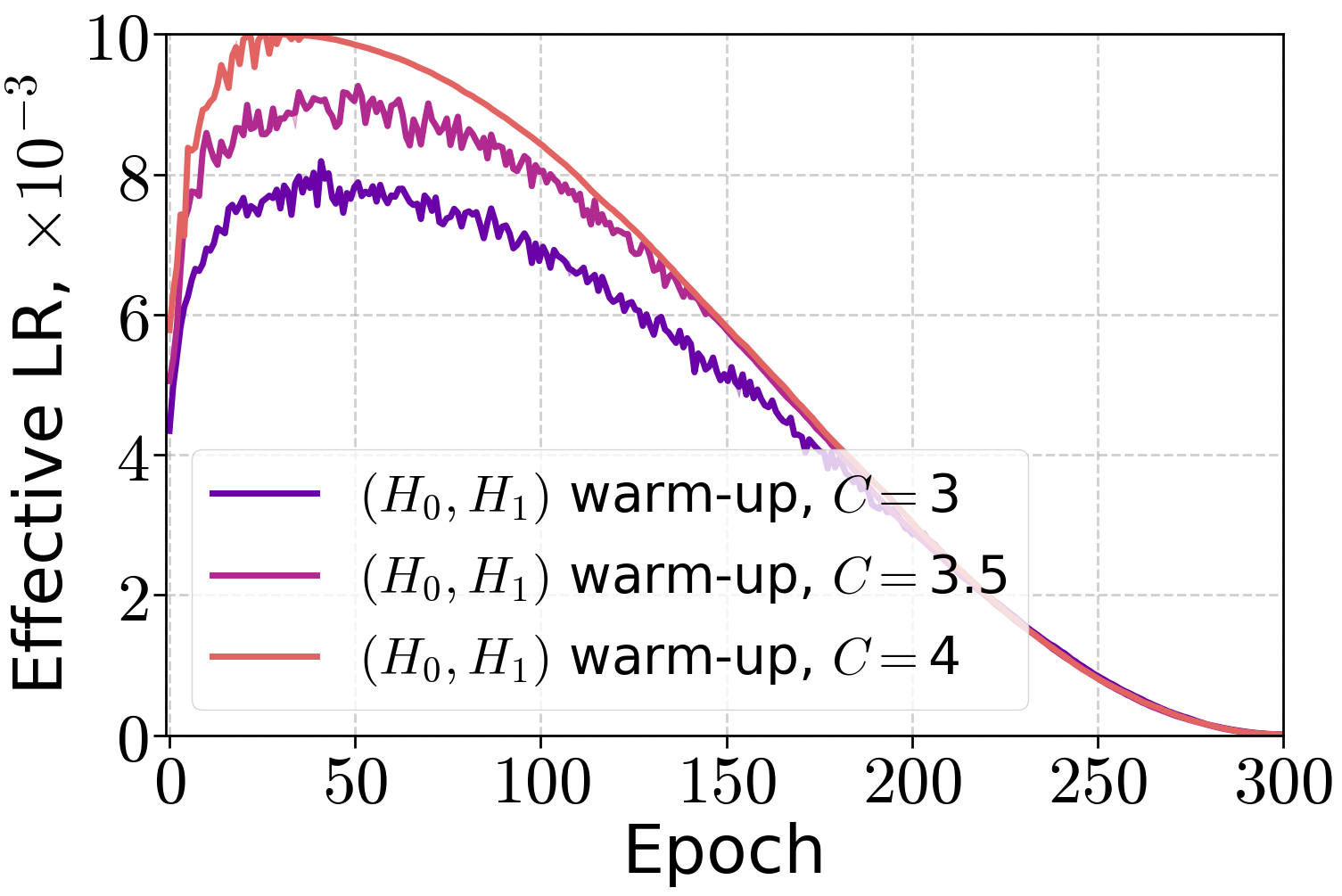}
    \end{tabular}
    \caption{Test accuracy and effective LR for ViT-Tiny on ImageNet-32 with \algname{AdamW} ($\lambda=0.05$) under no warm-up, tuned linear warm-up, and $(H_0, H_1)$ warm-up. All schedules use cosine decay after the warm-up phase.}
    \label{fig:vit_comparison}
    
\end{figure}

\section{Conclusion, Limitations, and Future Work}

We proposed \((H_0,H_1)\)-smoothness as a curvature--loss model for the early sharpness-reduction phase of neural-network training. The condition yields a curvature-adaptive step-size that increases as the loss decreases, providing a theoretical mechanism for learning-rate warm-up: fixed step-sizes must be conservative near initialization, whereas the adaptive schedule can remain stable while increasing the step size in lower-curvature regions. We proved faster convergence guarantees than fixed-step gradient descent under convex, Aiming, and PL-type structure, extended the analysis to a stochastic finite-sum setting under interpolation, and showed empirically that the resulting loss-driven schedule can match tuned linear warm-up on language and vision tasks.

Future work should characterize the duration of this phase, develop layer-wise or block-wise variants of the condition, and generalize the results beyond Euclidean geometry \citep{bernstein2024old,islamov2026noneuclidean}. Another important direction is to extend the analysis to practical optimizers such as \algname{Adam} \citep{kingma2015adam}, \algname{AdamW} \citep{loshchilov2017decoupled}, \algname{Muon} \citep{jordan2024muon}, \algname{Scion} \citep{pethick2025training}, and \algname{SCG} \citep{mokhtari2020stochastic}. Finally, analyzing the connection between Polyak-type step-sizes
\citep{loizou2021sps,orvieto2024adaptive,islamov2026enhancing} and warm-up is a promising direction, as recent work suggests that such step-sizes can adapt to curvature and enable automatic warm-up \citep{defazio2026schedulefree}.

\section*{Acknowledgements}

Foivos Alimisis, Rustem Islamov, and Aurelien Lucchi acknowledge the financial support of the Swiss National Foundation, SNF grants No 207392 and 223031. The authors thank Eduard Gorbunov for fruitful discussions, which allowed us to improve the work.

\section*{Impact Statement}

This work studies theoretical and empirical explanations of learning-rate warm-up. Its most direct potential benefit is improved training stability and reduced hyperparameter-search cost, which may reduce wasted compute. The work does not introduce new datasets, deployed systems, or capabilities; its societal risks are therefore indirect and primarily those associated with more efficient training of machine-learning models.

\bibliography{refs}

@book{vershynin2018high, 
    place={Cambridge}, 
    edition={2}, 
    series={Cambridge Series in Statistical and Probabilistic Mathematics}, 
    title={High-Dimensional Probability: An Introduction with Applications in Data Science}, 
    publisher={Cambridge University Press},
    author={Vershynin, Roman}, 
    year={2026}, 
    collection={Cambridge Series in Statistical and Probabilistic Mathematics}
}

@article{du2018algorithmic,
  title={Algorithmic regularization in learning deep homogeneous models: Layers are automatically balanced},
  author={Du, Simon S and Hu, Wei and Lee, Jason D},
  journal={Advances in neural information processing systems},
  year={2018}
}

@inproceedings{vaswani2025armijo,
    title={Armijo Line-search Can Make (Stochastic) Gradient Descent Provably Faster},
    author={Sharan Vaswani and Reza Babanezhad Harikandeh},
    booktitle={Forty-second International Conference on Machine Learning},
    year={2025}
}

@article{zhang2024trained,
  title={Trained transformers learn linear models in-context},
  author={Zhang, Ruiqi and Frei, Spencer and Bartlett, Peter L},
  journal={Journal of Machine Learning Research},
  volume={25},
  number={49},
  pages={1--55},
  year={2024}
}

@article{hagele2024scaling,
  title={Scaling laws and compute-optimal training beyond fixed training durations},
  author={H{\"a}gele, Alex and Bakouch, Elie and Kosson, Atli and Von Werra, Leandro and Jaggi, Martin and others},
  journal={Advances in Neural Information Processing Systems},
  year={2024}
}

@inproceedings{taheri2023fast,
  title={Fast convergence in learning two-layer neural networks with separable data},
  author={Taheri, Hossein and Thrampoulidis, Christos},
  booktitle={Proceedings of the AAAI Conference on Artificial Intelligence},
  year={2023}
}

@article{roulet2024stepping,
  title={Stepping on the edge: Curvature aware learning rate tuners},
  author={Roulet, Vincent and Agarwala, Atish and Grill, Jean-Bastien and Swirszcz, Grzegorz and Blondel, Mathieu and Pedregosa, Fabian},
  journal={Advances in Neural Information Processing Systems},
  year={2024}
}

@book{zhang2023dive,
  title={Dive into deep learning},
  author={Zhang, Aston and Lipton, Zachary C and Li, Mu and Smola, Alexander J},
  year={2023},
  publisher={Cambridge University Press}
}

@article{defazio2023optimal,
  title={Optimal linear decay learning rate schedules and further refinements},
  author={Defazio, Aaron and Cutkosky, Ashok and Mehta, Harsh and Mishchenko, Konstantin},
  journal={arXiv preprint arXiv:2310.07831},
  year={2023}
}

@inproceedings{hu2024minicpm,
title={Mini{CPM}: Unveiling the Potential of Small Language Models with Scalable Training Strategies},
author={Shengding Hu and Yuge Tu and Xu Han and Ganqu Cui and Chaoqun He and Weilin Zhao and Xiang Long and Zhi Zheng and Yewei Fang and Yuxiang Huang and Xinrong Zhang and Zhen Leng Thai and Chongyi Wang and Yuan Yao and Chenyang Zhao and Jie Zhou and Jie Cai and Zhongwu Zhai and Ning Ding and Chao Jia and Guoyang Zeng and dahai li and Zhiyuan Liu and Maosong Sun},
booktitle={First Conference on Language Modeling},
year={2024}
}

@inproceedings{hoffmann2022empirical,
 author = {Hoffmann, Jordan and Borgeaud, Sebastian and Mensch, Arthur and Buchatskaya, Elena and Cai, Trevor and Rutherford, Eliza and de Las Casas, Diego and Hendricks, Lisa Anne and Welbl, Johannes and Clark, Aidan and Hennigan, Thomas and Noland, Eric and Millican, Katherine and van den Driessche, George and Damoc, Bogdan and Guy, Aurelia and Osindero, Simon and Simonyan, Kar\'{e}n and Elsen, Erich and Vinyals, Oriol and Rae, Jack and Sifre, Laurent},
 booktitle = {Advances in Neural Information Processing Systems},
 publisher = {Curran Associates, Inc.},
 title = {An empirical analysis of compute-optimal large language model training},
 year = {2022}
}

@inproceedings{loshchilov2016sgdr,
title={{SGDR}: Stochastic Gradient Descent with Warm Restarts},
author={Ilya Loshchilov and Frank Hutter},
booktitle={International Conference on Learning Representations},
year={2017}
}

@article{chrabaszcz2017downsampled,
  title={A downsampled variant of imagenet as an alternative to the cifar datasets},
  author={Chrabaszcz, Patryk and Loshchilov, Ilya and Hutter, Frank},
  journal={arXiv preprint arXiv:1707.08819},
  year={2017}
}

@article{hoffmann2022chinchilla,
      title={Training Compute-Optimal Large Language Models}, 
      author={Jordan Hoffmann and Sebastian Borgeaud and Arthur Mensch and Elena Buchatskaya and Trevor Cai and Eliza Rutherford and Diego de Las Casas and Lisa Anne Hendricks and Johannes Welbl and Aidan Clark and Tom Hennigan and Eric Noland and Katie Millican and George van den Driessche and Bogdan Damoc and Aurelia Guy and Simon Osindero and Karen Simonyan and Erich Elsen and Jack W. Rae and Oriol Vinyals and Laurent Sifre},
      year={2022},
      journal={arXiv preprint arXiv:2203.15556},
}

@misc{jordan2024muon,
  title        = {{Muon}: An Optimizer for Hidden Layers in Neural Networks},
  author       = {Jordan, Keller and Jin, Yuchen and Boza, Vlado and You, Jiacheng and Cesista, Franz and Newhouse, Laker and Bernstein, Jeremy},
  year         = {2024},
  howpublished = {Blog post},
  url          = {https://kellerjordan.github.io/posts/muon/}
}

@inproceedings{pethick2025training,
  title     = {Training Deep Learning Models with Norm-Constrained {LMO}s},
  author    = {Pethick, Thomas and Xie, Wanyun and Antonakopoulos, Kimon and Zhu, Zhenyu and Silveti-Falls, Antonio and Cevher, Volkan},
  booktitle = {Proceedings of the 42nd International Conference on Machine Learning},
  year      = {2025}
}

@article{hendrycks2016gaussian,
  title={Gaussian error linear units (gelus)},
  author={Hendrycks, Dan and Gimpel, Kevin},
  journal={arXiv preprint arXiv:1606.08415},
  year={2016}
}

@article{mokhtari2020stochastic,
  title   = {Stochastic Conditional Gradient Methods: From Convex Minimization to Submodular Maximization},
  author  = {Mokhtari, Aryan and Hassani, Hamed and Karbasi, Amin},
  journal = {Journal of Machine Learning Research},
  year    = {2020}
}

@article{ba2016layer,
  title={Layer normalization},
  author={Ba, Jimmy Lei and Kiros, Jamie Ryan and Hinton, Geoffrey E},
  journal={arXiv preprint arXiv:1607.06450},
  year={2016}
}

@article{attia2025benefits,
  title={Benefits of Learning Rate Annealing for
Tuning-Robustness in Stochastic Optimization},
  author={Amit Attia and Tomer Koren},
  journal={arXiv preprint arXiv:2503.09411},
  year={2025}
}

@misc{ajroldi2024plainlm,
  author = {Niccolò Ajroldi},
  title = {plainLM: Language Model Pretraining in PyTorch},
  year = {2024},
  howpublished = {\url{https://github.com/Niccolo-Ajroldi/plainLM}}
}

@misc{ajroldi2025vision,
  author = {Niccolò Ajroldi},
  title = {vision: Vision Model Pretraining in PyTorch},
  year = {2025},
  howpublished = {\url{https://github.com/Niccolo-Ajroldi/vision}}
}

@inproceedings{kingma2015adam,
  title     = {{Adam}: A Method for Stochastic Optimization},
  author    = {Kingma, Diederik P. and Ba, Jimmy},
  booktitle = {International Conference on Learning Representations},
  year      = {2015}
}

@inproceedings{loshchilov2017decoupled,
    title={Decoupled Weight Decay Regularization},
    author={Ilya Loshchilov and Frank Hutter},
    booktitle={International Conference on Learning Representations},
    year={2019}
}

@InProceedings{loizou2021sps,
  title = 	 { Stochastic Polyak Step-size for SGD: An Adaptive Learning Rate for Fast Convergence },
  author =       {Loizou, Nicolas and Vaswani, Sharan and Hadj Laradji, Issam and Lacoste-Julien, Simon},
  booktitle = 	 {Proceedings of The 24th International Conference on Artificial Intelligence and Statistics},
  year = 	 {2021}
}

@article{orvieto2024adaptive,
  title={An adaptive stochastic gradient method with non-negative gauss-newton stepsizes},
  author={Orvieto, Antonio and Xiao, Lin},
  journal={arXiv preprint arXiv:2407.04358},
  year={2024}
}

@article{kondo2025accelerating,
  title={Accelerating sgdm via learning rate and batch size schedules: A lyapunov-based analysis},
  author={Kondo, Yuichi and Iiduka, Hideaki},
  journal={arXiv preprint arXiv:2508.03105},
  year={2025}
}

@article{defazio2026schedulefree,
  title={ScheduleFree{+}: Scaling Learning-Rate-Free \& Schedule-Free Learning to Large Language Models},
  author={Defazio, Aaron},
  journal={arXiv preprint arXiv:2605.19095},
  year={2026}
}

@inproceedings{islamov2026enhancing,
    title={Enhancing Optimizer Stability: Momentum Adaptation of The {NGN} Step-size},
    author={Rustem Islamov and Niccol{\`o} Ajroldi and Antonio Orvieto and Aurelien Lucchi},
    booktitle={The Thirty-ninth Annual Conference on Neural Information Processing Systems},
    year={2026}
}

@InProceedings{islamov2026noneuclidean,
	title = 	 {Non-Euclidean Gradient Descent Operates at the Edge of Stability},
	author =       {Rustem Islamov and Michael Crawshaw and Jeremy Cohen and Robert Gower},
	booktitle = 	 {Proceedings of the 43rd International Conference on Machine Learning},
	year = 	 {2026}
}

@inproceedings{he2016deep,
  title={Deep residual learning for image recognition},
  author={He, Kaiming and Zhang, Xiangyu and Ren, Shaoqing and Sun, Jian},
  booktitle={Proceedings of the IEEE conference on computer vision and pattern recognition},
  year={2016}
}

@article{radford2019language,
  title={Language models are unsupervised multitask learners},
  author={Radford, Alec and Wu, Jeffrey and Child, Rewon and Luan, David and Amodei, Dario and Sutskever, Ilya and others},
  journal={OpenAI blog},
  year={2019}
}

@misc{karpathy2022nanogpt,
  author = {Andrej Karpathy},
  title = {\text{NanoGPT}},
  year = {2022},
  publisher = {GitHub},
  journal = {GitHub repository},
  howpublished = {\url{https://github.com/karpathy/nanoGPT}}
}

@article{penedo2024fineweb,
  title={The fineweb datasets: Decanting the web for the finest text data at scale},
  author={Penedo, Guilherme and Kydl{\'\i}{\v{c}}ek, Hynek and Lozhkov, Anton and Mitchell, Margaret and Raffel, Colin A and Von Werra, Leandro and Wolf, Thomas and others},
  journal={Advances in Neural Information Processing Systems},
  year={2024}
}

@article{kawaguchi2016deep,
  title={Deep learning without poor local minima},
  author={Kawaguchi, Kenji},
  journal={Advances in neural information processing systems},
  year={2016}
}

@article{kalra2024warmup,
  title={Why warmup the learning rate? underlying mechanisms and improvements},
  author={Kalra, Dayal Singh and Barkeshli, Maissam},
  journal={Advances in Neural Information Processing Systems},
  year={2024}
}

@article{liu2019variance,
  title={On the variance of the adaptive learning rate and beyond},
  author={Liu, Liyuan and Jiang, Haoming and He, Pengcheng and Chen, Weizhu and Liu, Xiaodong and Gao, Jianfeng and Han, Jiawei},
  journal={arXiv preprint arXiv:1908.03265},
  year={2019}
}

@article{touvron2023llama,
  title={Llama 2: Open foundation and fine-tuned chat models},
  author={Touvron, Hugo and Martin, Louis and Stone, Kevin and Albert, Peter and Almahairi, Amjad and Babaei, Yasmine and Bashlykov, Nikolay and Batra, Soumya and Bhargava, Prajjwal and Bhosale, Shruti and others},
  journal={arXiv preprint arXiv:2307.09288},
  year={2023}
}

@inproceedings{xiong2020layer,
  title={On layer normalization in the transformer architecture},
  author={Xiong, Ruibin and Yang, Yunchang and He, Di and Zheng, Kai and Zheng, Shuxin and Xing, Chen and Zhang, Huishuai and Lan, Yanyan and Wang, Liwei and Liu, Tieyan},
  booktitle={International conference on machine learning},
  year={2020},
}

@inproceedings{huang2020improving,
  title={Improving transformer optimization through better initialization},
  author={Huang, Xiao Shi and Perez, Felipe and Ba, Jimmy and Volkovs, Maksims},
  booktitle={International Conference on Machine Learning},
  year={2020}
}

@inproceedings{smith2020generalization,
  title={On the generalization benefit of noise in stochastic gradient descent},
  author={Smith, Samuel and Elsen, Erich and De, Soham},
  booktitle={International Conference on Machine Learning},
  year={2020}
}

@inproceedings{gotmare2018closer,
title={A Closer Look at Deep Learning Heuristics: Learning rate restarts, Warmup and Distillation},
author={Akhilesh Gotmare and Nitish Shirish Keskar and Caiming Xiong and Richard Socher},
booktitle={International Conference on Learning Representations},
year={2019}
}

@inproceedings{gilmer2021loss,
title={A Loss Curvature Perspective on Training Instabilities of Deep Learning Models},
author={Justin Gilmer and Behrooz Ghorbani and Ankush Garg and Sneha Kudugunta and Behnam Neyshabur and David Cardoze and George Edward Dahl and Zachary Nado and Orhan Firat},
booktitle={International Conference on Learning Representations},
year={2022}
}

@inproceedings{sutskever2013importance,
  title={On the importance of initialization and momentum in deep learning},
  author={Sutskever, Ilya and Martens, James and Dahl, George and Hinton, Geoffrey},
  booktitle={International conference on machine learning},
  year={2013}
}

@inproceedings{cohen2021gradient,
    title={Gradient Descent on Neural Networks Typically Occurs at the Edge of Stability},
    author={Jeremy Cohen and Simran Kaur and Yuanzhi Li and J Zico Kolter and Ameet Talwalkar},
    booktitle={International Conference on Learning Representations},
    year={2021}
}

@article{kalra2023universal,
  title={Universal sharpness dynamics in neural network training: Fixed point analysis, edge of stability, and route to chaos},
  author={Kalra, Dayal Singh and He, Tianyu and Barkeshli, Maissam},
  journal={arXiv preprint arXiv:2311.02076},
  year={2023}
}

@article{mccandlish2018empirical,
  title={An empirical model of large-batch training},
  author={McCandlish, Sam and Kaplan, Jared and Amodei, Dario and Team, OpenAI Dota},
  journal={arXiv preprint arXiv:1812.06162},
  year={2018}
}

@inproceedings{wortsman2023small,
    title={Small-scale proxies for large-scale Transformer training instabilities},
    author={Mitchell Wortsman and Peter J Liu and Lechao Xiao and Katie E Everett and Alexander A Alemi and Ben Adlam and John D Co-Reyes and Izzeddin Gur and Abhishek Kumar and Roman Novak and Jeffrey Pennington and Jascha Sohl-Dickstein and Kelvin Xu and Jaehoon Lee and Justin Gilmer and Simon Kornblith},
    booktitle={The Twelfth International Conference on Learning Representations},
    year={2024}
}

@inproceedings{dosovitskiy2020vit,
    title={An Image is Worth 16x16 Words: Transformers for Image Recognition at Scale},
    author={Alexey Dosovitskiy and Lucas Beyer and Alexander Kolesnikov and Dirk Weissenborn and Xiaohua Zhai and Thomas Unterthiner and Mostafa Dehghani and Matthias Minderer and Georg Heigold and Sylvain Gelly and Jakob Uszkoreit and Neil Houlsby},
    booktitle={International Conference on Learning Representations},
    year={2021}
}

@inproceedings{riabinin2025gluon,
  title={Gluon: Making Muon \& Scion Great Again!(Bridging Theory and Practice of LMO-based Optimizers for LLMs)},
  author={Riabinin, Artem and Shulgin, Egor and Gruntkowska, Kaja and Richt{\'a}rik, Peter},
  booktitle={International {C}onference on {M}achine {L}earning},
  year={2026}
}

@article{dremov2025training,
title={Training Dynamics of the Cooldown Stage in Warmup-Stable-Decay Learning Rate Scheduler},
author={Aleksandr Dremov and Alexander H{\"a}gele and Atli Kosson and Martin Jaggi},
journal={Transactions on Machine Learning Research},
year={2025}
}

@article{shazeer2020glu,
  title={Glu variants improve transformer},
  author={Shazeer, Noam},
  journal={arXiv preprint arXiv:2002.05202},
  year={2020}
}

@article{zhang2019root,
  title={Root mean square layer normalization},
  author={Zhang, Biao and Sennrich, Rico},
  journal={Advances in neural information processing systems},
  year={2019}
}

@article{su2024roformer,
title = {RoFormer: Enhanced transformer with Rotary Position Embedding},
journal = {Neurocomputing},
volume = {568},
pages = {127063},
year = {2024},
issn = {0925-2312},
author = {Jianlin Su and Murtadha Ahmed and Yu Lu and Shengfeng Pan and Wen Bo and Yunfeng Liu}
}

@inproceedings{vaswani2017attention,
 author = {Vaswani, Ashish and Shazeer, Noam and Parmar, Niki and Uszkoreit, Jakob and Jones, Llion and Gomez, Aidan N and Kaiser, \L ukasz and Polosukhin, Illia},
 booktitle = {Advances in Neural Information Processing Systems},
 title = {Attention is All You Need},
 volume = {30},
 year = {2017}
}

@article{goyal2017accurate,
  title={Accurate, large minibatch sgd: Training imagenet in 1 hour},
  author={Goyal, Priya and Doll{\'a}r, Piotr and Girshick, Ross and Noordhuis, Pieter and Wesolowski, Lukasz and Kyrola, Aapo and Tulloch, Andrew and Jia, Yangqing and He, Kaiming},
  journal={arXiv preprint arXiv:1706.02677},
  year={2017}
}

@article{kosson2024analyzing,
  title={Analyzing \& reducing the need for learning rate warmup in GPT training},
  author={Kosson, Atli and Messmer, Bettina and Jaggi, Martin},
  journal={Advances in Neural Information Processing Systems},
  year={2024}
}

@inproceedings{hinder2020near,
  title={Near-optimal methods for minimizing star-convex functions and beyond},
  author={Hinder, Oliver and Sidford, Aaron and Sohoni, Nimit},
  booktitle={Conference on learning theory},
  year={2020}
}

@inproceedings{karimi2016linear,
  title={Linear convergence of gradient and proximal-gradient methods under the polyak-{\l}ojasiewicz condition},
  author={Karimi, Hamed and Nutini, Julie and Schmidt, Mark},
  booktitle={Joint European conference on machine learning and knowledge discovery in databases},
  year={2016}
}

@article{polyak1963gradient,
title = {Gradient methods for the minimisation of functionals},
journal = {USSR Computational Mathematics and Mathematical Physics},
volume = {3},
number = {4},
pages = {864-878},
year = {1963},
issn = {0041-5553},
doi = {https://doi.org/10.1016/0041-5553(63)90382-3},
author = {B.T. Polyak},
}

@inproceedings{
gaash2025convergence,
title={Convergence of Clipped {SGD} on Convex \$(L\_0,L\_1)\$-Smooth Functions},
author={Ofir Gaash and Kfir Yehuda Levy and Yair Carmon},
booktitle={The Thirty-ninth Annual Conference on Neural Information Processing Systems},
year={2026}
}

@article{liu2022loss,
  title={Loss landscapes and optimization in over-parameterized non-linear systems and neural networks},
  author={Liu, Chaoyue and Zhu, Libin and Belkin, Mikhail},
  journal={Applied and Computational Harmonic Analysis},
  volume={59},
  pages={85--116},
  year={2022},
  publisher={Elsevier}
}

@article{defazio2025gradients,
  title={Why Gradients Rapidly Increase Near the End of Training},
  author={Defazio, Aaron},
  journal={arXiv preprint arXiv:2506.02285},
  year={2025}
}

@inproceedings{xie2023overlooked,
    title={On the Overlooked Pitfalls of Weight Decay and How to Mitigate Them: A Gradient-Norm Perspective},
    author={Zeke Xie and zhiqiang xu and Jingzhao Zhang and Issei Sato and Masashi Sugiyama},
    booktitle={Thirty-seventh Conference on Neural Information Processing Systems},
    year={2023},
    url={https://openreview.net/forum?id=vnGcubtzR1}
}

@InProceedings{schaipp2025surprising,
  title = 	 {The Surprising Agreement Between Convex Optimization Theory and Learning-Rate Scheduling for Large Model Training},
  author =       {Schaipp, Fabian and H\"{a}gele, Alexander and Taylor, Adrien and Simsekli, Umut and Bach, Francis},
  booktitle = 	 {Proceedings of the 42nd International Conference on Machine Learning},
  year = 	 {2025}
}

@inproceedings{kleinberg2018alternative,
  title={An alternative view: When does SGD escape local minima?},
  author={Kleinberg, Bobby and Li, Yuanzhi and Yuan, Yang},
  booktitle={International {C}onference on {M}achine {L}earning},
  year={2018}
}

@inproceedings{fu2023accelerated,
  title={Accelerated stochastic optimization methods under quasar-convexity},
  author={Fu, Qiang and Xu, Dongchu and Wilson, Ashia Camage},
  booktitle={International Conference on Machine Learning},
  year={2023},
  organization={PMLR}
}

@inproceedings{gower2021sgd,
  title={Sgd for structured nonconvex functions: Learning rates, minibatching and interpolation},
  author={Gower, Robert and Sebbouh, Othmane and Loizou, Nicolas},
  booktitle={International Conference on Artificial Intelligence and Statistics},
  year={2021}
}

@inproceedings{ma2018power,
  title={The power of interpolation: Understanding the effectiveness of SGD in modern over-parametrized learning},
  author={Ma, Siyuan and Bassily, Raef and Belkin, Mikhail},
  booktitle={International Conference on Machine Learning},
  year={2018}
}

@article{hardt2018gradient,
  title={Gradient descent learns linear dynamical systems},
  author={Hardt, Moritz and Ma, Tengyu and Recht, Benjamin},
  journal={Journal of Machine Learning Research},
  year={2018}
}

@article{tran2024empirical,
  title={Reevaluating Theoretical Analysis Methods for Optimization in Deep Learning},
  author={Tran, Hoang and Zhang, Qinzi and Cutkosky, Ashok},
  journal={arXiv preprint arXiv:2407.01825},
  year={2024}
}

@InProceedings{guille2023no,
  title = 	 {No Wrong Turns: The Simple Geometry Of Neural Networks Optimization Paths},
  author =       {Guille-Escuret, Charles and Naganuma, Hiroki and Fatras, Kilian and Mitliagkas, Ioannis},
  booktitle = 	 {Proceedings of the 41st International Conference on Machine Learning},
  year = 	 {2024}
}

@article{islamov2024loss,
  title={Loss landscape characterization of neural networks without over-parametrization},
  author={Islamov, Rustem and Ajroldi, Niccol{\`o} and Orvieto, Antonio and Lucchi, Aurelien},
  journal={Advances in Neural Information Processing Systems},
  year={2024}
}

@inproceedings{liu2023aiming,
    title={Aiming towards the minimizers: fast convergence of {SGD} for overparametrized problems},
    author={Chaoyue Liu and Dmitriy Drusvyatskiy and Misha Belkin and Damek Davis and Yian Ma},
    booktitle={Thirty-seventh Conference on Neural Information Processing Systems},
    year={2023},
    url={https://openreview.net/forum?id=ZBB8EFO7ma}
}

@article{khaled2020better,
title={Better Theory for {SGD} in the  Nonconvex World},
author={Ahmed Khaled and Peter Richt{\'a}rik},
journal={Transactions on Machine Learning Research},
year={2023},
note={Survey Certification}
}

@misc{karpathy2015tinyshakspeare,
  author={Karpathy, Andrej},
  title={char-rnn},
  year={2015},
  howpublished={\url{https://github.com/karpathy/char-rnn}}
}

@techreport{krizhevsky2009learning,
  title       = {Learning Multiple Layers of Features from Tiny Images},
  author      = {Krizhevsky, Alex},
  year        = {2009},
  institution = {University of Toronto}
}

@article{liu2025theoretical,
  title={Theoretical Analysis on how Learning Rate Warmup Accelerates Convergence},
  author={Liu, Yuxing and Ge, Yuze and Pan, Rui and Kang, An and Zhang, Tong},
  journal={arXiv preprint arXiv:2509.07972},
  year={2025}
}

@inproceedings{vankov2024optimizing,
 author = {Vankov, Daniil and Rodomanov, Anton and Nedich, Angelia and Sankar, Lalitha and Stich, Sebastian},
 booktitle = {International Conference on Learning Representations},
 title = {Optimizing (L\_0, L\_1)-Smooth Functions by Gradient Methods},
 year = {2025}
}

@inproceedings{
gorbunov2024methods,
title={Methods for Convex \$(L\_0,L\_1)\$-Smooth Optimization: Clipping, Acceleration, and Adaptivity},
author={Eduard Gorbunov and Nazarii Tupitsa and Sayantan Choudhury and Alen Aliev and Peter Richt{\'a}rik and Samuel Horv{\'a}th and Martin Tak{\'a}{\v{c}}},
booktitle={The Thirteenth International Conference on Learning Representations},
year={2025}
}

@inproceedings{wang2023convergence,
  title={Convergence of adagrad for non-convex objectives: Simple proofs and relaxed assumptions},
  author={Wang, Bohan and Zhang, Huishuai and Ma, Zhiming and Chen, Wei},
  booktitle={The Thirty Sixth Annual Conference on Learning Theory},
  year={2023}
}

@inproceedings{faw2023beyond,
  title={Beyond uniform smoothness: A stopped analysis of adaptive sgd},
  author={Faw, Matthew and Rout, Litu and Caramanis, Constantine and Shakkottai, Sanjay},
  booktitle={The Thirty Sixth Annual Conference on Learning Theory},
  year={2023}
}

@article{zhao2021convergence,
  title={On the convergence and improvement of stochastic normalized gradient descent},
  author={Zhao, Shen-Yi and Xie, Yin-Peng and Li, Wu-Jun},
  journal={Science China Information Sciences},
  year={2021}
}

@article{zhang2020improved,
  title={Improved analysis of clipping algorithms for non-convex optimization},
  author={Zhang, Bohang and Jin, Jikai and Fang, Cong and Wang, Liwei},
  journal={Advances in Neural Information Processing Systems},
  year={2020}
}

@inproceedings{zhang2019gradient,
  author       = {Jingzhao Zhang and
                  Tianxing He and
                  Suvrit Sra and
                  Ali Jadbabaie},
  title        = {Why Gradient Clipping Accelerates Training: {A} Theoretical Justification
                  for Adaptivity},
  booktitle    = {8th International Conference on Learning Representations},
  year         = {2020}
}

@inproceedings{bernstein2024old,
title={Old Optimizer, New Norm: An Anthology},
author={Jeremy Bernstein and Laker Newhouse},
booktitle={OPT 2024: Optimization for Machine Learning},
year={2024}
}

@inproceedings{wen2024understanding,
  author       = {Kaiyue Wen and
                  Zhiyuan Li and
                  Jason S. Wang and
                  David Leo Wright Hall and
                  Percy Liang and
                  Tengyu Ma},
  title        = {Understanding Warmup-Stable-Decay Learning Rates: {A} River Valley
                  Loss Landscape View},
  booktitle    = {The Thirteenth International Conference on Learning Representations},
  year         = {2025}
}

@article{magnus1985matrix,
  title={Matrix differential calculus with applications to simple, Hadamard, and Kronecker products},
  author={Magnus, Jan R},
  journal={Journal of Mathematical Psychology},
  year={1985}
}

@article{li2023convex,
  title={Convex and non-convex optimization under generalized smoothness},
  author={Li, Haochuan and Qian, Jian and Tian, Yi and Rakhlin, Alexander and Jadbabaie, Ali},
  journal={Advances in Neural Information Processing Systems},
  year={2023}
}

@inproceedings{zhang2019improving,
    title = {Improving Deep Transformer with Depth-Scaled Initialization and Merged Attention},
    author = {Zhang, Biao  and
      Titov, Ivan  and
      Sennrich, Rico},
    booktitle = {Proceedings of the 2019 Conference on Empirical Methods in Natural Language Processing and the 9th International Joint Conference on Natural Language Processing (EMNLP-IJCNLP)},
    year = {2019}
}

@inproceedings{sun2025curse,
 author = {Sun, Wenfang and Song, Xinyuan and Li, Pengxiang and Yin, Lu and Zheng, Yefeng and Liu, Shiwei},
 booktitle = {Advances in Neural Information Processing Systems},
 title = {The Curse of Depth in Large Language Models},
 year = {2025}
}

@inproceedings{cohenunderstanding,
  title={Understanding Optimization in Deep Learning with Central Flows},
  author={Cohen, Jeremy and Damian, Alex and Talwalkar, Ameet and Kolter, J Zico and Lee, Jason D},
  booktitle={The Thirteenth International Conference on Learning Representations},
  year={2025}
}
\bibliographystyle{plainnat}

\newpage
\appendix
\counterwithin{figure}{section}
\counterwithin{table}{section}

\vbox{
  {\hrule height 2pt \vskip 0.15in \vskip -\parskip}
  \centering
  {\LARGE\bf Appendix\par}
  {\vskip 0.2in \vskip -\parskip \hrule height 0.5pt \vskip 0.09in}
}

\newcommand\invisiblepart[1]{%
  \refstepcounter{part}%
  \addcontentsline{toc}{part}{\protect\numberline{\thepart}#1}%
}

\invisiblepart{Appendix}
\setcounter{tocdepth}{2}
\localtableofcontents

\appendixtrue

\section{Comparison to \citet{liu2025theoretical}}\label{sec:concurrent_work}

In this section, we provide a detailed comparison against a concurrent work by \citet{liu2025theoretical}. In the following section, we discuss in more detail the differences between our work and their work. 

\subsection{The proposed Conditions}

\citet{liu2025theoretical} proposed the following condition with a general power $\rho > 0$:
\begin{equation}\label{eq:h0h1_rho}
\|\nabla^2f(w)\| \le K_0 + K_{1}(f(w)-f^*)^\rho.
\end{equation}
The condition we study in the main part of the paper is a special case of \eqref{eq:h0h1_rho} with $\rho=1$. \citet{liu2025theoretical} proves the convergence in the convex setting under \eqref{eq:h0h1_rho}, demonstrating benefits of the theoretical warm-up schedule. The proposed theoretical step-size is similar to ours in \eqref{eq:adaptive_stepsize}. However, their results can be simply recovered from our analysis for the $\rho=1$ case.

Indeed, assuming $\rho >1$ and that the iterates $\{w_k\}_{k=0}^{K}$ stay in the set $\{w \hspace{1mm} | \hspace{1mm} f(w)-f^* \le f(w_0)-f^*\}$, which is the case for \algname{GD}, we can simplify \eqref{eq:h0h1_rho} as follows 
\begin{equation}\label{eq:loser_bound}
\|\nabla^2f(w)\|_2 \le K_0 + K_\rho(f(w)-f^*)^\rho \le K_0 + K_\rho(f(w_0)-f^*)^{\rho-1}(f(w)-f^*),
\end{equation}
i.e., \Cref{asmp:H0H1_smoothness} holds with $H_0 = K_0$ and $H_1 = K_\rho(f(w_0)-f^*)^{\rho-1}.$ Therefore, the results of \Cref{th:aiming_deterministic_adaptive_step-size} apply, leading to the iteration complexity of \algname{GD} with adaptive warm-up schedule of the form
\[
K = \cO\left(\frac{K_0{\rm dist}(w_0,\cS)^2}{\theta^2\varepsilon} + \frac{K_\rho(f(w_0)-f^*)^{\rho-1}{\rm dist}(w_0,\cS)^2}{\theta^2}\right).
\]
This matches the bound in \citet{liu2025theoretical} up to constants when $\theta=1$, and shows that the adaptive schedule converges faster whenever $(f(w_0)-f^*)/\varepsilon \gg 1$. Given the simplification in \eqref{eq:loser_bound}, it remains open whether the convergence under the general condition \eqref{eq:h0h1_rho} can be further tightened.

In Proposition \ref{prop:deep_non_linear_many_nonlinearities}, we show that deep non-linear networks with Leaky-ReLU activations satisfy \eqref{eq:h0h1_rho}, albeit under stronger assumptions than \Cref{prop:deep_linear}. Moreover, \Cref{prop:crossentropy} covers L2-regularized networks with two layers and arbitrary activations. If one considers deeper networks, $\rho$ increases with the number of layers $\ell$.

\subsection{Theoretical Evidence}

We show that \Cref{asmp:H0H1_smoothness} holds for several standard architectures: $(i)$ an $\ell$-layer linear network with MSE loss under balancedness (\Cref{prop:deep_linear}); $(ii)$ an $\ell$-layer nonlinear network with leaky-ReLU activations (\Cref{prop:deep_non_linear_many_nonlinearities}); and $(iii)$ two-layer network with cross-entropy losses (\Cref{prop:crossentropy}) under L2 regularization. These results provide clear theoretical evidence that the proposed condition is a good proxy for neural-network smoothness. By contrast, concurrent work by \citet{liu2025theoretical} validates the condition only on a toy MLP model and a recurrent model with fewer than four parameters and sigmoid activation, leaving its extension beyond such simple cases unclear. Moreover, the representativeness of the proposed condition at the initial phase of training does not seem to be discussed.

\subsection{Experimental Evidence}

\citet{liu2025theoretical} empirically test the proposed condition on two setups: ResNet18 trained on CIFAR-10 \citep{krizhevsky2009learning} and a small NanoGPT-style transformer (6 blocks, 384-dimensional embeddings, 6 attention heads) trained on the Tiny Shakespeare dataset \citep{karpathy2015tinyshakspeare}. They present smoothness versus training loss in a log-log scale, which leaves open how well the condition aligns with empirical behavior. Moreover, they show that these models train efficiently without learning-rate warm-up, and that warm-up offers no benefit (see Table 1 in \citet{liu2025theoretical}). These settings leave open how the condition behaves in larger regimes where warm-up is empirically important. Our experiments complement theirs by evaluating larger language and vision models and by comparing against tuned linear warm-up. In contrast, we validate the condition on much larger models, where warm-up is crucial for achieving improved performance (see \Cref{fig:language_models_comparison,fig:vit_comparison}).

\subsection{Theoretical Analysis}

\paragraph{Deterministic Setting.} \citet{liu2025theoretical} establish convergence guarantees for \algname{GD} in both convex and general non-convex regimes. In the non-convex case, their worst-case analysis shows that an adaptive warm-up schedule offers no advantage over a constant step-size beyond numerical factors, which is consistent with our findings. Benefits emerge in the convex setting, which is again similar to our results. In contrast, we extend the analysis beyond convexity, demonstrating the benefits of warm-up under Aiming and PL conditions, which are known to hold for sufficiently wide networks. 

\paragraph{Stochastic Setting.} The convergence guarantees in \citet{liu2025theoretical} and our work are not directly comparable. We focus on finite-sum minimization in expectation under the Aiming condition and interpolation, showing that, under these assumptions, \algname{SGD} with an adaptive warm-up schedule attains performance comparable to \algname{GD} in the deterministic case. By contrast, \citet{liu2025theoretical} analyzes general non-convex objectives under almost surely bounded variance or the almost-sure ABD assumption \citep{khaled2020better} with high probability, where their results do not demonstrate warm-up benefits analogous to those in the deterministic setting. 

\paragraph{Lower Bounds.} Both works provide lower bounds for constant step-size \algname{GD}, however, the ``hard-to-optimize" functions and proof techniques are different. 

\subsection{Empirical Results} 

We test the theoretically inspired LR warm-up schedule of the form $\frac{\eta_k}{\max\{1, f_{S_k(w_k)}/C\}}$, where  $\eta_k$ follows WSD or cosine annealing schedule, $f_{S_k}(w_k)$ is the current stochastic loss and $C$ is the parameter of $(H_0, H_1)$ warm-up schedule. The proposed $(H_0, H_1)$ warm-up schedule is used when training language models of size 70M, 160M, and 410M on the FineWeb dataset and the ViT-Tiny model on the ImageNet-32 dataset. Our empirical results demonstrate that the proposed theoretical warm-up with tuned $C$ matches the performance of a linear warm-up and improves over a no-warm-up baseline. In contrast, \citet{liu2025theoretical} conducts experiments when training the ResNet18 model on the CIFAR10 dataset. Their results demonstrate that the theoretical warm-up schedule $\frac{1}{4\sqrt{2}+4}\max\{\frac{1}{K_0}, \frac{1}{3K_1f_{S_k}(w_k)}\}$ with tuned $K_0$ and $K_1$ matches the performance of linear and no warm-up baselines.

\section{Arithmetics of $(H_0,H_1)$-smooth Functions }\label{sec:more_general_class}

First, we provide a formal proof of the conjecture mentioned in \Cref{sec:new_smoothness}. In other words, the following result demonstrates that the class of $(H_0, H_1)$-smooth functions contains all $(L_0, L_1)$-smooth functions.

\begin{proposition}
    Assume that $f$ is $(L_0,L_1)$-smooth and bounded from below, i.e., $\|\nabla^2f(w)\| \le L_0 + L_1\|\nabla f(w)\|$ and $f^* > -\infty.$ Then $f$ satisfies \Cref{asmp:H0H1_smoothness} with 
    \[
    H_0 = L_0 + \frac{L_0L_1}{\nu},\quad H_1 = \frac{4L_1^2+\nu L_1}{2\nu},
    \]
    where $\nu$ satisfies the equality $\nu = e^{-\nu}$\footnote{One can check numerically that $\nu\in(0.56,0.57).$}.
\end{proposition}
\begin{proof}
    We start with Lemma 2.2 in \citet{gorbunov2024methods}
    \begin{align*}
    &\|\nabla f(w)\|^2 \le \frac{2}{\nu}(L_0 + L_1\|\nabla f(w)\|)(f(w)-f^*) \\
    \Leftrightarrow\; &\|\nabla f(w)\|^2 -\frac{2L_1}{\nu}\|\nabla f(w)\|(f(w)-f^*) - \frac{2L_0}{\nu}(f(w)-f^*)  \le 0.
    \end{align*}
    We need to solve this quadratic inequality w.r.t. $\|\nabla f(w)\|.$ The discriminant is 
    \[
        \frac{4L_1^2}{\nu^2}(f(w)-f^*)^2 + 4\cdot 1\cdot \frac{2L_0}{\nu}(f(w)-f^*) > 0 = \frac{4L_1^2}{\nu^2}(f(w)-f^*)^2 + \frac{8L_0}{\nu}(f(w)-f^*) > 0,
    \]
    i.e., it is positive. Since $\|\nabla f(w)\|\ge 0,$ we should also satisfy 
    \begin{align*}
        \|\nabla f(w)\| &\le \frac{\frac{2L_1}{\nu}(f(w)-f^*) + \sqrt{\frac{4L_1^2}{\nu^2}(f(w)-f^*)^2 + \frac{8L_0}{\nu}(f(w)-f^*)}}{2}\\
        &\overset{(i)}{\le} \frac{L_1}{\nu}(f(w)-f^*) + \sqrt{\frac{L_1^2}{\nu^2}(f(w)-f^*)^2 } + \sqrt{\frac{2L_0}{\nu}(f(w)-f^*)}\\
        &\overset{(ii)}{\le} \frac{2L_1}{\nu}(f(w)-f^*) + \frac{L_0}{\nu} + \frac{1}{2}(f(w)-f^*)\\
        &= \frac{L_0}{\nu} + \frac{4L_1+\nu}{2\nu}(f(w)-f^*),
    \end{align*}
    where $(i)$ follows from the inequality $\sqrt{a+b} \le \sqrt{a} + \sqrt{b}$ for any $a,b\ge 0$, $(ii)$ -- from the inequality $\sqrt{ab} \le \frac{a}{2} + \frac{b}{2}$ for any $a,b\ge 0$. Therefore, we obtain
    \begin{align*}
        \|\nabla^2f(w)\| &\le L_0 + L_1\|\nabla f(w)\|\\
        &\le L_0 + \frac{L_0L_1}{\nu} + \frac{4L_1^2+\nu L_1}{2\nu}(f(w)-f^*),
    \end{align*}
    which means that the function $f$ is $(H_0,H_1)$-smooth.

\end{proof}

Next, we demonstrate that operations like summation preserve $(H_0, H_1)$-smoothness. First, we show that the class of $(H_0, H_1)$-smooth functions is closed under summation. 
\begin{proposition}Let $f$ and $g$ be $(H_0^f,H_1^f)$-smooth and $(H_0^g,H_1^g)$-smooth respectively. Then $h \eqdef f + g$ is $(H_0,H_1)$-smooth with 
\[
H_0 = (H_0^f + H_0^g + \max\{H_1^f,H_1^g\} h^* - H_1^ff^* - H_1^gg^*), \quad \text{and} \quad H_1 = \max\{H_1^f,H_1^g\}.
\]
\end{proposition}
\begin{proof}
    By \Cref{asmp:H0H1_smoothness}, we have 
    \[
    \|\nabla^2f(w)\|\le H_0^f + H_1^f(f(w)-f^*), \quad \|\nabla^2g(w)\|\le H_0^g + H_1^g(g(w)-g^*).
    \]
\end{proof}
Therefore, we have 
\begin{align*}
    \|\nabla^2h(w)\| &= \|\nabla^2f(w) + \nabla^2g(w)\|\\
    &\le \|\nabla^2f(w)\| + \|\nabla^2g(w)\|\\
    &\le H_0^f + H_1^f(f(w)-f^*) + H_0^g + H_1^g(g(w)-g^*)\\
    &\le (H_0^f + H_0^g) + \max\{H_1^f,H_1^g\}(f(w)+g(w)) - H_1^ff^* - H_1^gg^*\\
    &= \underbrace{(H_0^f + H_0^g + \max\{H_1^f,H_1^g\} h^* - H_1^ff^* - H_1^gg^*)}_{\eqdef H_0} + \underbrace{\max\{H_1^f,H_1^g\}}_{\eqdef H_1}(h(w) - h^*).
\end{align*}
Note that $h^* \ge f^* + g^*.$ Therefore, we have 
\begin{align*}
    \max\{H_1^f,H_1^g\}h^* - H_1^ff^* - H_1^gg^* \ge H_1^fh^* + H_1^gh^* - H_1^ff^* - H_1^gg^*  \ge 0,
\end{align*}
i.e., $H_0 \ge 0$.

The next proposition shows that the class of $(H_0,H_1)$-smooth functions is closed under affine transformation.
\begin{proposition}
    Let $g\colon\R^{q}\to\R$ be $(H_0^g,H_1^g)$-smooth, $A\in\R^{q\times p}$ be an arbitrary matrix, and $b\in\R^q$ be an arbitrary vector. We define $f\colon\R^p\to\R$ as $f(w) \eqdef g(Aw+b)$. Then $f$ is $(H_0^f,H_1^f)$-smooth with 
    \[
    H_0^f =  \|A\|^2(H_0^g + H_1(f^*-g^*)),\quad H_1^f =  \|A\|^2H_1^g,
    \]
    where $f^* = \min_{w\in\R^p}f(w), g^*=\min_{y\in\R^q}g(y).$
\end{proposition}
\begin{proof}
    First, note that
    \[
    f^* = \min_{w\in\R^p} g(Aw+b) \ge \min_{y\in\R^q} g(y) = g^*,
    \]
    since the first minimum is taken in ${\rm Im}(A).$ Second, note that $\nabla^2f(w) = A^\top \nabla^2g(Aw+b)A.$ Therefore, 
    \begin{align*}
        \|\nabla^2f(w)\| &= \|A^\top\nabla^2g(Aw+b)A\|\\
        &\le \|A^\top\|\cdot \|\nabla^2g(Aw+b)\|\cdot \|A\|\\
        &\le \|A\|^2 \cdot (H_0^g + H_1^g(g(Aw+b)-g^*))\\
        &= \|A\|^2H_0^g + \|A\|^2H_1^g(f(w)-f^* + f^* - g^*)\\
        &= \|A\|^2(H_0^g + H_1(f^*-g^*)) + \|A\|^2H_1^g(f(w)-f^*).
    \end{align*}
\end{proof}

In the next proposition, we demonstrate that the class of $(L_0,L_1)$-smooth functions is not closed under summation.

\begin{proposition}
    There exist two $(L_0, L_1)$-smooth functions $f_1,f_2\colon\R\to\R$ such that their sum $f=f_1+f_2$ does not belong to the class of $(L_0, L_1)$-smooth functions. 
\end{proposition}
\begin{proof}
    Let us consider two functions $f_1$ and $f_2$ defined as 
    \[
    f_1(w) = \int_0^w (u+\sin(u^2))\diff u, \quad f_2(w) = \int_0^w (-v+\sin(v^2))\diff v.
    \]
    Then we have
    \[
    f_1^\prime(w) = w+\sin(w^2), \quad f_1^{\prime\prime}(w) = 1 + 2w\cos(w^2), \quad f_2^\prime(w) = -w+\sin(w^2), \quad f_2^{\prime\prime}(w) = -1 + 2w\cos(w^2).
    \]
    Therefore, we have 
    \[
    |f_{1,2}^{\prime\prime}(w)| \le 1 + |2w\cos(w^2)| \le 1 + 2|w|,
    \]
    and 
    \[
    |f_{1,2}^\prime(w)| \ge |\pm w + \sin(w^2)| \ge |w| - |\sin(w^2)| \ge |w|-1.
    \]
    This implies that for $|w|\ge 1$
    \[
    |f_{1,2}^{\prime\prime}(w)| \le 1 + 2|w| \le 3 + 3(|w|-1) \le 3 + 3|f_{1,2}^{\prime}(w)|.
    \]
    For $|w|\le 1$, we have $|f_{1,2}^{\prime\prime}(w)| \le 3$. Thus, both functions are $(L_0, L_1)$-smooth with $L_0=L_1=3.$ They sum is $f(w) = 2\sin(w^2),$ for which we have 
    \[
    f^\prime(w) = 2\sin(w^2), \quad f^{\prime\prime}(w) = 4w\cos(w^2).
    \]
    Now we consider points $\{w_m\}_{m=1}^{\infty}$ with $w_m = \sqrt{m\pi}$. At these points, we have
    \[
    f^\prime(w_m) = 0, \quad f^{\prime\prime}(w_m) = 4w_m \to \infty.
    \]
    If $f$ were $(L_0, L_1)$-smooth, then we would have 
    \[
    |f^{\prime\prime}(w_m)| \le L_0+L_1|f^{\prime}(w_m)| \le L_0.
    \]
    This contradiction concludes the proof.
\end{proof}

We now show that there exists an affine transformation that does not preserve $(L_0, L_1)$-smoothness.

\begin{proposition}
    There exist a $(L_0, L_1)$-smooth function $g \colon\R^2\to\R$ and a matrix $A\in\R^{2\times 1}$ such that a function $f(w) = g(Aw)$ does not belong to the class of $(L_0, L_1)$-smooth functions. 
\end{proposition}
\begin{proof}
    Let us consider $A = \begin{pmatrix}
        1\\0
    \end{pmatrix},$ $b=0,$ and $g(y_1, y_2) = h(y_1)e^{y_2}$ with $h(y_1) = \cos(y_1)e^{y_1}.$ We know that 
    \[
    h^\prime(y_1) = e^{y_1}(\cos(y_1) - \sin(y_1)), \quad h^{\prime\prime}(y_1) = -2\sin(y_1)e^{y_1}.
    \]
    Therefore, 
    \[
    \nabla g(y) = e^{y_2}\begin{pmatrix}
        h^{\prime}(y_1)\\
        h(y_1)
    \end{pmatrix},
    \quad \nabla^2g(y) = e^{y_2}\begin{pmatrix}
        h^{\prime\prime}(y_1) & h^{\prime}(y_1)\\
        h^{\prime}(y_1) & h(y_1)
    \end{pmatrix}.
    \]
    Note that 
    \begin{align*}
    |h^{\prime\prime}(y_1)| &= 2e^{y_1}|\sin(y_1)| \le 2e^{y_1}|\cos(y_1)| + 2e^{y_1}|\cos(y_1)-\sin(y_1)| = 2|h(y_1)| + 2|h^{\prime}(y_1)|.
    \end{align*}
    Therefore, we have 
    \begin{align*}
        \|\nabla^2g(y)\|_2 &\le \|\nabla^2g(y)\|_{\rm F}\\
        &= e^{y_2}\sqrt{(h^{\prime\prime}(y_1))^2 + 2(h^{\prime}(y_1))^2 + (h(y_1))^2}\\
        &\le e^{y_2}\sqrt{4(h(y_1) + h^{\prime}(y_1))^2 + 2(h^{\prime}(y_1))^2 + (h(y_1))^2}\\
        &\le e^{y_2}\sqrt{8(h(y_1))^2 + 8(h^{\prime}(y_1))^2 + 2(h^{\prime}(y_1))^2 + (h(y_1))^2}\\
        &\le \sqrt{10}e^{y_2}\sqrt{(h(y_1))^2 + (h^{\prime}(y_1))^2 }
    \end{align*}
    Note that $\|\nabla g(y)\| = e^{y_2}\sqrt{(h(y_1))^2 + (h^{\prime}(y_1))^2 }$. Therefore, we obtain the bound $\|\nabla^2g(y)\|_2 \le \sqrt{10}\|\nabla g(y)\|.$ Now we consider the function $f(w) = g(Aw) = g(w, 0) = h(w).$ For $f$, we have
    \[
    f^{\prime}(w) = e^{w}(\cos(w)-\sin(w)), \quad f^{\prime\prime}(w) = -2\sin(w)e^{w}.
    \]
    We consider the points $\{w_m\}_{m=1}{\infty}$ with $w_m = \frac{\pi}{4}+2\pi m$. Therefore, $\cos(w_m)=\sin(w_m) = \sqrt{2}/2$. This implies, that at these points $f^{\prime}(w_m) = e^{w_m}(\sqrt{2}/2 - \sqrt{2}/2) = 0$ and $f^{\prime\prime}(w_m) = -\sqrt{2}e^{w_m}$. Thus, we obtain that $|f^{\prime\prime}(w_m)| \to \infty$ with $m\to \infty,$ while $|f^{\prime}(w_m)| = 0$. This implies that $f$ does not satisfy $(L_0, L_1)$-smoothness for any $L_0, L_1 \ge 0.$
    
\end{proof}

\section{Proof of Proposition \ref{prop:1x1x1 network}}
\label{app:1x1x1 network}
\propsimplenetwork*

\begin{proof}
The parameters $u(t)$ and $v(t)$ evolve over time via gradient flow as:
$
\frac{du}{dt} = -g_u \quad \text{and} \quad \frac{dv}{dt} = -g_v,
$
where $g_u = \frac{\partial f(u,v)}{\partial u}, g_v = \frac{\partial f(u,v)}{\partial v}$\\
Let $R := y - uvx$ be the residual error. We have
\begin{align*}
    g_u &= \frac{\partial f}{\partial u} = 2(y - uvx) \cdot (-vx) = -2vx R \\
    g_v &= \frac{\partial f}{\partial v} = 2(y - uvx) \cdot (-ux) = -2ux R.
\end{align*}
The gradient vector is $g = \begin{pmatrix} g_u \\ g_v \end{pmatrix}$.

The Hessian matrix is $H = \nabla^2 f = \begin{pmatrix} H_{11} & H_{12} \\ H_{21} & H_{22} \end{pmatrix}$. We can calculate its components directly:
\begin{align*}
    A := H_{11} &= \frac{\partial g_u}{\partial u} = \frac{\partial}{\partial u}(-2vyx + 2u v^2 x^2) = 2v^2x^2 \\
    C := H_{22} &= \frac{\partial g_v}{\partial v} = \frac{\partial}{\partial v}(-2uyx + 2u^2 v x^2) = 2u^2x^2 \\
    B := H_{12} &= \frac{\partial g_u}{\partial v} = \frac{\partial}{\partial v}(-2vyx + 2u v^2 x^2) = -2yx + 4uvx^2 = 2x(2uvx - y) \\
    H_{21} &= \frac{\partial g_v}{\partial u} = \frac{\partial}{\partial u}(-2uyx + 2u^2 v x^2) = -2yx + 4uvx^2 = B.
\end{align*}
So, the Hessian is:
\begin{equation*}
    H(t) = \begin{pmatrix} 2v^2x^2 & 2x(2uvx - y) \\ 2x(2uvx - y) & 2u^2x^2 \end{pmatrix}.
\end{equation*}

The spectral norm $\|H\|_2$ of a $2 \times 2$ symmetric matrix of the form $\begin{bmatrix}
    A &B \\ B & C
\end{bmatrix}$ is its largest eigenvalue in absolute value. The two eigenvalues $\lambda_{\pm}$ are found by solving the characteristic Eq.~$\det(H - \lambda I) = 0$:
$$
\lambda^2 - (A+C)\lambda + (AC - B^2) = 0.
$$
The solution of this quadratic Eq.~gives the two eigenvalues:
$$
\lambda_{\pm} = \frac{(A+C) \pm \sqrt{(A+C)^2 - 4(AC - B^2)}}{2} = \frac{(A+C) \pm \sqrt{(A-C)^2 + 4B^2}}{2}.
$$
The spectral norm is the larger of these in magnitude, i.e.
$$
\|H\|_2 = \lambda_{+} = \frac{(A+C) + \sqrt{(A-C)^2 + 4B^2}}{2}.
$$
For simplicity, we define $S_{\lambda} := \frac{1}{2} \sqrt{(A-C)^2 + 4B^2}$. Then we have:
$
\|H\|_2 = \lambda_{+} = \frac{1}{2}(A+C) + S_{\lambda}.
$

Our goal is to compute $\frac{d}{dt} \|H\|_2$ and analyze its sign.

First, we must find $\dot{H}$, which requires the 3rd derivatives of $f$:
\begin{itemize}
    \item $\frac{\partial H_{11}}{\partial u} = \frac{\partial}{\partial u}(2v^2x^2) = 0$
    \item $\frac{\partial H_{11}}{\partial v} = \frac{\partial}{\partial v}(2v^2x^2) = 4vx^2$
    
    \item $\frac{\partial H_{12}}{\partial u} = \frac{\partial}{\partial u}(4uvx^2 - 2xy) = 4vx^2$
    \item $\frac{\partial H_{12}}{\partial v} = \frac{\partial}{\partial v}(4uvx^2 - 2xy) = 4ux^2$
    
    \item $\frac{\partial H_{22}}{\partial u} = \frac{\partial}{\partial u}(2u^2x^2) = 4ux^2$
    \item $\frac{\partial H_{22}}{\partial v} = \frac{\partial}{\partial v}(2u^2x^2) = 0$
\end{itemize}

Now we find the components of $\dot{H}$ using the chain rule: $$\dot{H}_{ij} = \frac{\partial H_{ij}}{\partial u} \dot{u} + \frac{\partial H_{ij}}{\partial v} \dot{v} = -g_u(\frac{\partial H_{ij}}{\partial u}) - g_v(\frac{\partial H_{ij}}{\partial v}).$$
\begin{itemize}
    \item $\dot{H}_{11} = -g_u(0) - g_v(4vx^2) = -g_v(4vx^2)$
    \item $\dot{H}_{22} = -g_u(4ux^2) - g_v(0) = -g_u(4ux^2)$
    \item $\dot{H}_{12} = -g_u(4vx^2) - g_v(4ux^2)$
\end{itemize}

We substitute $g_u = -2vxR$ and $g_v = -2uxR$ and we get
\begin{itemize}
    \item $\dot{H}_{11} = -(-2uxR)(4vx^2) = 8uvx^3 R$
    \item $\dot{H}_{22} = -(-2vxR)(4ux^2) = 8uvx^3 R$
\end{itemize}
This reveals that $\dot{H}_{11} = \dot{H}_{22}$, which will be useful later. We denote these terms by $D_H$:
\begin{equation*}
    D_H := \dot{H}_{11} = \dot{H}_{22} = 8uvx^3 R.
\end{equation*}
The off-diagonal term is:
\begin{align*}
    E_H := \dot{H}_{12} &= -(-2vxR)(4vx^2) - (-2uxR)(4ux^2) \nonumber \\
    &= 8v^2x^3 R + 8u^2x^3 R = 8x^3(u^2+v^2)R
\end{align*}
So, $\dot{H} = \begin{pmatrix} D_H & E_H \\ E_H & D_H \end{pmatrix}$.

For a $2 \times 2$ symmetric matrix $H$ with $A=H_{11}$ and $C=H_{22}$, the derivative of the largest eigenvalue $\lambda_+$ is:
$$
\frac{d\lambda_+}{dt} = \frac{1}{2}(\dot{A} + \dot{C}) + \frac{(A-C)(\dot{A}-\dot{C}) + 4B\dot{B}}{4 S_{\lambda}}.
$$
Our simplification $\dot{A} = \dot{C} = D_H$ gives a cleaner expression:
$$
\frac{d\lambda_+}{dt} = \frac{1}{2}(D_H + D_H) + \frac{(A-C) \cdot 0 + 4B E_H}{4 S_{\lambda}} = D_H + \frac{B E_H}{S_{\lambda}}.
$$
Substituting $S_{\lambda} = \frac{1}{2} \sqrt{(A-C)^2 + 4B^2}$, we get
\begin{equation*}
    \frac{d}{dt} \|H\|_2 = D_H + \frac{2 B E_H}{\sqrt{(A-C)^2 + 4B^2}}.
\end{equation*}
Substituting our components $A, B, C, D_H, E_H$ gives the final formula:
\begin{equation*}
    \frac{d}{dt} \|H\|_2 = 8uvx^3 R + \frac{2 \cdot [2x(2uvx - y)] \cdot [8x^3(u^2+v^2)R]}{\sqrt{(2v^2x^2 - 2u^2x^2)^2 + 4(2x(2uvx - y))^2}}.
\end{equation*}

To simplify the notation, we denote $B_{\text{sign}} := y - 2uvx$.
Then $2uvx - y = -B_{\text{sign}}$ and plugging it in the previous formula, we have
\begin{align*}
    \frac{d}{dt} \|H\|_2 &= 8uvx^3 R + \frac{2 \cdot (2x(-B_{\text{sign}})) \cdot (8x^3(u^2+v^2)R)}{2\sqrt{x^4(v^2-u^2)^2 + 4x^2(-B_{\text{sign}})^2}} \\
    &= 8uvx^3 R - \frac{32 x^4 (u^2+v^2) R B_{\text{sign}}}{2|x|\sqrt{x^2(v^2-u^2)^2 + 4B_{\text{sign}}^2}}.
\end{align*}
Let $c = u^2 - v^2$ (a constant, as we will prove later).
To simplify the fraction, we use the identity $\frac{x^4}{|x|} = \frac{x^4}{x \cdot \text{sign}(x)} = x^3 \cdot \text{sign}(x)$ and we get
$$
\frac{d}{dt} \|H\|_2 = 8uvx^3 R - 16 \cdot \text{sign}(x) \cdot x^3 (u^2+v^2) R \cdot \frac{B_{\text{sign}}}{\sqrt{x^2c^2 + 4B_{\text{sign}}^2}}.
$$
Factoring out the common term $8x^3 R$, we arrive at the general formula for the derivative, valid for all $x,y,u,v$:
\begin{equation} \label{eq:deriv_general}
    \frac{d}{dt} \|H\|_2 = 8x^3 R \left( uv - 2 \cdot \text{sign}(x) \cdot (u^2+v^2) \cdot \frac{B_{\text{sign}}}{\sqrt{x^2c^2 + 4B_{\text{sign}}^2}} \right).
\end{equation}

Before analyzing specific paths, we prove that the quantity $u^2 - v^2$ stays constant over the course of gradient flow.

Indeed, we compute the time-derivative of the quantity $(u^2 - v^2)$:
$$
\frac{d}{dt}(u^2 - v^2) = 2u \dot{u} - 2v \dot{v}.
$$
From the gradient flow definition, $\dot{u} = -g_u = 2vxR$ and $\dot{v} = -g_v = 2uxR$.
Substituting these in:
$$
\frac{d}{dt}(u^2 - v^2) = 2u(2vxR) - 2v(2uxR) = 4uvxR - 4uvxR = 0.
$$
Because the time-derivative is always zero, the quantity $(u^2 - v^2)$ is a conserved quantity. Any trajectory is constrained to a manifold where $u^2 - v^2 = c$, where $c$ is determined by the random initial conditions $c = u(0)^2 - v(0)^2$.

We now pass to the analysis of the sign of $\frac{d}{dt} \|H\|_2$. For notational simplicity, let $S = u^2+v^2$ and we have
$$
\frac{d}{dt} \|H\|_2 = \underbrace{8x^3 R}_{\text{Term 1}} \cdot \underbrace{\left( uv - \frac{2 \text{sign}(x) \cdot S \cdot B_{\text{sign}}}{\sqrt{x^2c^2 + 4B_{\text{sign}}^2}} \right)}_{\text{Term 2}}
$$
We can see that the trajectory consists in general of three phases, defined by the boundaries $x^2c^2 = 4B_{\text{sign}}^2$ and $B_{\text{sign}} = 0$.

\paragraph{First phase:}
We define the initial regime by the condition $x^2c^2 \le 4B_{\text{sign}}^2$. This holds near the origin because $c \approx 0$ and $B_{\text{sign}} \approx y \neq 0$. In this regime: $\text{sign}(R) \approx \text{sign}(y)$ and $\text{sign}(B_{\text{sign}}) \approx \text{sign}(y)$.
\begin{enumerate}
    \item \textbf{Term 1:} $\text{sign}(\text{Term 1}) = \text{sign}(x^3 R) \approx \text{sign}(x^3 y) = \text{sign}(xy)$.
    \item \textbf{Term 2:} Let $F = \frac{2 S |B_{\text{sign}}|}{\sqrt{x^2c^2 + 4B_{\text{sign}}^2}}$.
    We use the identity $uv = \text{sign}(uv) \cdot \frac{1}{2}\sqrt{S^2 - c^2}$ (easy to prove).
\end{enumerate}

We must analyze two scenarios
\begin{itemize}
    \item \textbf{Case A: $\text{sign}(uv) = \text{sign}(xy)$}. Then,
    $$
    \text{Term 2} = \text{sign}(xy) \left[ \frac{1}{2}\sqrt{S^2 - c^2} - |F| \right].
    $$
    We have $\frac{1}{2}\sqrt{S^2 - c^2} \leq \frac{S}{2}$. We also have
    $|F| \geq \frac{2 S |B_{\text{sign}}|}{\sqrt{8B_{\text{sign}}^2}} = \frac{S}{\sqrt{2}} > \frac{S}{2}$. Thus, the term inside the bracket is negative and 
    \begin{equation*}
        \text{sign}\left(\frac{d}{dt} \|H\|_2\right) = -\text{sign}(xy)^2 = -1.
    \end{equation*}
    
    \item \textbf{Case B: $\text{sign}(uv) = -\text{sign}(xy)$}
    $$
    \text{Term 2} = -\text{sign}(xy) \left[ \frac{1}{2}\sqrt{S^2 - c^2} + |F| \right].
    $$
In this case, it holds
\begin{equation*}
   \text{sign}\left(\frac{d}{dt} \|H\|_2\right) = -\text{sign}(xy)^2 \text{sign}\left(\frac{1}{2}\sqrt{S^2 - c^2} + |F|\right) = -1.
\end{equation*}
    
    
\end{itemize}
In both cases, the derivative of the spectral norm is negative, which means that we enter a flattening phase.

\paragraph{Second phase:}
This initial flattening phase is only guaranteed as long as our assumption $x^2c^2 \le 4B_{\text{sign}}^2$ holds.
As the flow moves, $uvx$ increases, so $B_{\text{sign}} = y - 2uvx$ gets smaller. Eventually, we will enter an ambiguous phase where $x^2c^2 > 4B_{\text{sign}}^2$, but we still have $\text{sign}(B_{\text{sign}}) = \text{sign}(y)$ (because the flow has not yet reached $uvx = y/2$).
In this region, our bounding logic for $|F|$ is inconclusive, and the sign of the derivative is unknown. However, if $c = 0$, then this phase is vacuous and we pass directly to the next progressive sharpening phase.

\paragraph{Third phase:}
The flow then crosses the mid-point at $uvx = y/2$, which means $\text{sign}(B_{\text{sign}})$ finally flips to $-\text{sign}(y)$.
The flow is now in a ``correct" quadrant (having been repelled by the origin), so $\text{sign}(uv) = \text{sign}(xy)$. We have:
\begin{itemize}
    \item $\text{sign}(\text{Term 1}) = \text{sign}(x^3 R) = \text{sign}(x^3 y) = \text{sign}(xy)$ (since $\text{sign}(R) = \text{sign}(y)$ always).
    \item The sign of $F$ is now $\text{sign}(F) = \text{sign}(x \cdot B_{\text{sign}}) = \text{sign}(x) \cdot (-\text{sign}(y)) = -\text{sign}(xy)$.
    \item This means $\text{sign}(uv)$ and $\text{sign}(F)$ are now opposites.
    \item $\text{Term 2} = uv - F$.
        \begin{itemize}
            \item If $\text{sign}(xy) = +1$: $\text{sign}(uv) = +1$ and $\text{sign}(F) = -1$. Term 2 is $(\text{pos}) - (\text{neg}) = +$.
            \item If $\text{sign}(xy) = -1$: $\text{sign}(uv) = -1$ and $\text{sign}(F) = +1$. Term 2 is $(\text{neg}) - (\text{pos}) = -$.
        \end{itemize}
    \item In both sub-cases, $\text{sign}(\text{Term 2})$ is $\text{sign}(xy)$.
\end{itemize}
The total sign is
$$
\text{sign}\left(\frac{d}{dt} \|H\|_2\right) = \text{sign}(\text{Term 1}) \cdot \text{sign}(\text{Term 2}) = \text{sign}(xy) \cdot \text{sign}(xy) = +1.
$$
Thus, a sharpening phase is guaranteed after the $uvx = y/2$ boundary.

\end{proof}

\section{Proofs for Section \ref{sec:theory_MLP}}
\label{app:proofs_neural_nets}

\thmdeeplinear*

\begin{proof}
The proof goes as follows: first we obtain a general upper bound for the spectral norm of the Hessian. Then, we consider cases depending on whether we want a local bound close to initialization or a global bound.

\paragraph{Upper bound for the Hessian norm:}

    One can find an explicit formula for the Hessian of such neural network in \cite{kawaguchi2016deep}, Lemma 4.3.

    The Hessian of $f$ in vectorized form has blocks in the $(i,j)$ position for $j < i$, that are of the form 
    \begin{align*}
        \frac{\partial^2 f}{\partial \textnormal{vec}(W_i) \textnormal{vec}(W_j)} &= 2((W_1 \dots W_{i-1}) \otimes (W_{i+1} \dots W_{\ell} X)^\top)^\top  ((W_1 \dots W_{j-1}) \otimes (W_{j+1} \dots W_{\ell} X)^\top) \\ & + 2((W_{j+1} \dots W_{i-1})^\top \otimes (W_{i+1} \dots W_{\ell} X)) (I_{n_j} \otimes ((W_1 \dots W_{\ell} X-Y)^\top W_1 \dots W_{j-1} )),
    \end{align*}
where $W_1 W_0, W_{\ell+1} W_{\ell} : = I$.

For $j=i$, we have
\begin{align*}
        \frac{\partial^2 f}{\partial \textnormal{vec}(W_i) \textnormal{vec}(W_j)} &= 2((W_1 \dots W_{i-1}) \otimes (W_{i+1} \dots W_{\ell} X)^\top)^\top  ((W_1 \dots W_{j-1}) \otimes (W_{j+1} \dots W_{\ell} X)^\top).
\end{align*}

The spectral norm of the Hessian in vectorized form is upper bounded by the sum of the spectral norms of each such block.
Indeed, let $M$ be an $N \times N$ block symmetric matrix:
\[
M = \begin{pmatrix}
M_{11} & M_{12} & \cdots & M_{1N} \\
M_{12}^\top & M_{22} & \cdots & M_{2N} \\
\vdots & \vdots & \ddots & \vdots \\
M_{1N}^\top & M_{2N}^\top & \cdots & M_{NN}
\end{pmatrix}
\]
where each $M_{ij}$ is a matrix block. 

A fundamental result for block matrices states that the spectral norm of a block matrix is bounded by the spectral norm of the matrix formed by the spectral norms of its blocks. Let us define a real symmetric $N \times N$ matrix $\tilde{M}$ where each element $(\tilde{M})_{ij}$ is the spectral norm of the corresponding block $M_{ij}$:
\[
\tilde{M} = \begin{pmatrix}
\|M_{11}\|_2 & \|M_{12}\|_2 & \cdots & \|M_{1N}\|_2 \\
\|M_{12}\|_2 & \|M_{22}\|_2 & \cdots & \|M_{2N}\|_2 \\
\vdots & \vdots & \ddots & \vdots \\
\|M_{1N}\|_2 & \|M_{2N}\|_2 & \cdots & \|M_{NN}\|_2
\end{pmatrix}
\]
The inequality is then:
\begin{equation*}
\|M\|_2 \le \|\tilde{M}\|_2
\end{equation*}

Since the spectral norm is always upper bounded by the Frobenius norm, it holds 
\begin{equation*}
\|\tilde{M}\|_2 \le \|\tilde{M}\|_{\rm F} = \sqrt{\sum_{i=1}^N \sum_{j=1}^N \|M_{ij}\|_2^2} \leq \sum_{i=1}^N \sum_{j=1}^N \|M_{ij}\|_2.
\end{equation*}

Thus, indeed, it holds
\begin{equation}
\label{eq:basic_bound_spectral_norm}
    \|M\|_2 \leq \sum_{i=1}^N \sum_{j=1}^N \|M_{ij}\|_2.
\end{equation}

Going back to the Hessian, we can upper bound the spectral norm of the $(i,j)$ block using only the weak form of balancedness $\|W_i\|_{\rm F} = \|W_{i+1}\|_{\rm F}$ (which is implied by the strong form of balancedness). 

For $1<j<i<\ell$, we have
\begin{align*}
   \left \| \frac{\partial^2 f}{\partial \textnormal{vec}(W_i) \textnormal{vec}(W_j)} \right \|_2 & =  2\|((W_1 \dots W_{i-1}) \otimes (W_{i+1} \dots W_{\ell} X)^\top)^\top  ((W_1 \dots W_{j-1}) \otimes (W_{j+1} \dots W_{\ell} X)^\top) \\ & + 2((W_{j+1} \dots W_{i-1})^\top \otimes (W_{i+1} \dots W_{\ell} X)) (I_{n_j} \otimes ((W_1 \dots W_{\ell} X-Y)^\top W_1 ... W_{j-1} )\|_2 \\ &\leq  2 \|((W_1 \dots W_{i-1}) \otimes ((W_{i+1} \dots W_{\ell} X)^\top)^\top  (W_1 \dots W_{j-1}) \otimes (W_{j+1} \dots W_{\ell} X)^\top)\|_2 \\ & +  2 \|((W_{j+1} \dots W_{i-1})^\top \otimes (W_{i+1} \dots W_{\ell} X)) (I_{n_j} \otimes ((W_1 \dots W_{\ell} X-Y)^\top W_1 ... W_{j-1} )\|_2 \\ & \leq  2 \|W_1\|_{\rm F}^{2 \ell-2} \|X\|_2^2 + 2 \|W_1\|_{\rm F}^{\ell-2} \|X\|_2 \sqrt{f(W)}.    
\end{align*}
For the last inequality, we used that for matrices $A$ and $B$
\begin{itemize}
    \item $\|A \otimes B \|_2 = \|A\|_2 \|B\|_2$.
    \item $\|A\|_2 = \|A^\top\|_2$
    \item $\|A B \|_2 \leq \|A\|_2 \|B\|_2$.
    \item $\|A\|_2 \leq \|A\|_{\rm F}$.
\end{itemize}

For $j=1$ and $1<i< \ell$, we have

\begin{align*}
        \frac{\partial^2 f}{\partial \textnormal{vec}(W_i) \textnormal{vec}(W_1)} &= 2 ((W_1 \dots W_{i-1}) \otimes (W_{i+1} \dots W_{\ell} X)^\top)^\top  (I_c \otimes (W_2 \dots W_{\ell} X)^\top) \\ & + 2 ((W_2 \dots W_{i-1})^\top \otimes (W_{i+1} \dots W_{\ell} X)) (I_{n_j} \otimes (W_1 \dots W_{\ell} X-Y)^\top),
    \end{align*}
thus
\begin{equation*}
    \left \|\frac{\partial^2 f}{\partial \textnormal{vec}(W_i) \textnormal{vec}(W_1)} \right\|_2 \leq 2 \|W_1\|_{\rm F}^{2 \ell-2} \|X\|^2_2 + 2 \|W_1\|_{\rm F}^{\ell-2} \|X\|_2 \sqrt{f(W)}.
\end{equation*}

For $j=1$ and $i=\ell$, it holds
\begin{align*}
        \frac{\partial^2 f}{\partial \textnormal{vec}(W_i) \textnormal{vec}(W_1)} &= 2 ((W_1 \dots W_{\ell-1}) \otimes X)  (I_c \otimes (W_2 \dots W_{\ell} X)^\top) \\ & + 2 ((W_2 \dots W_{\ell-1})^\top \otimes X) (I_{n_1} \otimes ((W_1 \dots W_{\ell} X-Y)^\top),
    \end{align*}
    thus again
\begin{equation*}
    \left \|\frac{\partial^2 f}{\partial \textnormal{vec}(W_{\ell}) \textnormal{vec}(W_1)} \right\|_2 \leq 2 \|W_1\|_{\rm F}^{2 \ell-2} \|X\|^2_2 + 2 \|W_1\|_{\rm F}^{\ell-2} \|X\|_2 \sqrt{f(W)}.
\end{equation*}

For the case that $1<j<\ell$ and $i = \ell$, we have 
\begin{align*}
    \frac{\partial^2 f}{\partial \textnormal{vec}(W_i) \textnormal{vec}(W_j)} &= 2 ((W_1 \dots W_{\ell-1}) \otimes X) ((W_1 \dots W_{j-1}) \otimes (W_{j+1} \dots W_{\ell} X)^\top) \\ & + 2 ((W_{j+1} \dots W_{\ell-1})^\top \otimes X) (I_{n_j} \otimes ((W_1 \dots W_{\ell} X-Y)^\top W_1 \dots W_{j-1} ).
\end{align*}

Again, we have
\begin{equation*}
    \left \|\frac{\partial^2 f}{\partial \textnormal{vec}(W_{\ell}) \textnormal{vec}(W_j)} \right\|_2 \leq 2 \|W_1\|_{\rm F}^{2 \ell-2} \|X\|^2_2 + 2 \|W_1\|_{\rm F}^{\ell-2} \|X\|_2 \sqrt{f(W)}.
\end{equation*}

Similarly, we have for the diagonal blocks that
\begin{equation*}
    \left \|\frac{\partial^2 f}{\partial \textnormal{vec}(W_i) \textnormal{vec}(W_j)} \right\|_2 \leq 2 \|W_1\|_{\rm F}^{2 \ell-2} \|X\|^2_2.
\end{equation*}

In summary, since we have $(\ell^2-\ell)$-many off-diagonal blocks and $\ell$-many diagonal blocks in the Hessian, its norm is bounded as
\begin{equation}
\label{eq:Hessian_bound_deepLN}
    \|\nabla^2 f(W) \|_2 \leq 2 \ell^2 \|W_1\|_{\rm F}^{2 \ell-2} \|X\|^2_2 + 2 (\ell^2-\ell) \|W_1\|_{\rm F}^{\ell-2} \|X\|_2 \sqrt{f(W)}.
\end{equation}

If $\|W_1\|_{\mathrm F}^\ell \|X\|_2 < (1-\frac{1}{\ell}) \sqrt{f(W)}$, then the first term of the previous bound is smaller than the second, thus we have
\begin{equation*}
\|\nabla^2 f(W) \|_2 \leq 4 (\ell^2-\ell) \|W_1\|_{\rm F}^{\ell-2} \|X\|_2 \sqrt{f(W)}.
\end{equation*}
It also holds $\|W_1\|_F^{\ell-2} \leq (1-\frac{1}{\ell})^{\frac{\ell-2}{\ell}} \|X\|_2^{-\frac{\ell-2}{\ell}} \sqrt{f(W)}^{\frac{\ell-2}{\ell}}$. In total, we have
\begin{equation*}
\label{eq:warmup-linear}
\|\nabla^2 f(W)\|_2 \le H_0^{\text{warm-up}}+H_1^{\text{warm-up}}(f(W)-f^*),
\end{equation*}
where 
\begin{equation}
\label{eq:warmup_constants_linear_network}
H_0^{\text{warm-up}}:=4(\ell^2-\ell) \left(1-\frac{1}{\ell}\right)^{\frac{\ell-2}{\ell}} \|X\|_2^{\frac{2}{\ell}}\bigl(1+(f^*)^{\frac{\ell-1}{\ell}}\bigr), \qquad H_1^{\text{warm-up}}:= 4 (\ell^2-\ell) \left(1-\frac{1}{\ell}\right)^{\frac{\ell-2}{\ell}} \|X\|_2^{\frac{2}{\ell}}.
\end{equation}

For the latter, we used the bound
$$ f(W)^{\frac{\ell-1}{\ell}} = (f(W)-f^* +f^*)^{\frac{\ell-1}{\ell}} \leq (f(W)-f^*)^{\frac{\ell-1}{\ell}} + (f^*)^{\frac{\ell-1}{\ell}} \leq 1 + (f(W)-f^*) + (f^*)^{\frac{\ell-1}{\ell}}.$$

This concludes the result of (i).

Now, we move to the global bound. For that, we need to obtain a general lower bound for the loss value. It holds
\begin{align}
\label{eq:eq1}
    \|W_1 \dots W_{\ell} X\|^2_{\rm F} = \text{Tr}(X^\top W_{\ell}^\top \dots W_2^\top W_1^\top W_1 W_2 \dots W_{\ell} X)  \geq \lambda_{\min}(X X^\top) \text{Tr}( W_{\ell}^\top \dots W_2^\top W_1^\top W_1 W_2 \dots W_{\ell}).
\end{align}
In order to deal with the last term, we use the strong balancedness assumption:
\begin{align*}
   & W_{\ell}^\top \dots W_4^\top W_3^\top W_2^\top W_1^\top W_1 W_2 W_3 W_4 \dots W_{\ell}  = W_{\ell}^\top \dots W_4^\top W_3^\top W_2^\top W_2 W_2^\top W_2 W_3 W_4 \dots W_{\ell} = \\ & W_{\ell}^\top \dots W_4^\top W_3^\top W_3 W_3^\top W_3 W_3^\top W_3 W_4 \dots W_{\ell} = W_{\ell}^\top \dots W_4^\top W_4 W_4^\top W_3^\top W_3 W_4 W_4^\top W_4 \dots W_{\ell} = \\ & W_{\ell}^\top \dots W_5 W_5^\top W_4^\top W_3^\top W_3 W_4 W_5 W_5^\top \dots W_{\ell}
\end{align*}
and the process continuous until we reach the expression $$(W_{\ell}^\top W_{\ell}) W_{\ell}^\top W_{\ell-1}^\top \dots W_6^\top W_5^\top W_4^\top W_3^\top W_3 W_4 W_5 W_6 \dots W_{\ell-1} W_{\ell} (W_{\ell}^\top W_{\ell}).$$ We can now do the same process starting from $W_3$ and so on. Repeating this process $\ell/2$ times if $\ell$ is even and $(\ell-1)/2$ if $\ell$ is odd, we arrive to the expression
\begin{equation*}
   \underbrace{(W_{\ell}^\top W_{\ell}) \dots (W_{\ell}^\top W_{\ell})}_{\ell- \text{times}} = (W_{\ell}^\top W_{\ell})^{\ell}. 
\end{equation*}
Since the eigenvalues of $(W_{\ell}^\top W_{\ell})^\ell$ are $\ell$ powers of the eigenvalues of $W_{\ell}^\top W_{\ell}$, we can use the generalized mean inequality and derive
\begin{equation*}
   \frac{\text{Tr}((W_{\ell}^\top W_{\ell})^\ell)}{d} \geq \frac{\text{Tr}((W_{\ell}^\top W_{\ell}))^\ell}{d^{\ell}} = \frac{\|W_{\ell}\|_{\rm F}^{2 \ell}}{d^{\ell}} = \frac{\|W_1\|_{\rm F}^{2 \ell}}{d^{\ell}},  
\end{equation*}
thus 
\begin{equation}
\label{eq:eq2}
    \text{Tr}((W_{\ell}^\top W_{\ell})^\ell ) \geq \frac{\|W_1\|_{\rm F}^{2 \ell}}{d^{\ell-1}}.
\end{equation}

Notice that we made use of the weak balancedness assumption $\|W_{\ell}\|_{\rm F} = \|W_1\|_{\rm F}$.

Combining inequalities \eqref{eq:eq1} and \eqref{eq:eq2}, we get
\begin{equation}
\label{eq:basic_deep_linear}
    \|W_1 \dots W_{\ell} X\|_{\rm F} \geq \sqrt{\lambda_{\min}(X X^\top)} \frac{\|W_1\|_{\rm F}^{\ell}}{d^{\frac{\ell-1}{2}}}.
\end{equation}

Combining this with the triangle inequality $\|W_1 \dots W_{\ell} X\|_{\rm F} \leq \|Y\|_F +\sqrt{f(W)}$, we get the bounds
\begin{equation*}
\|W_1\|_{\mathrm F}^{2\ell-2} = \Big(\|W_1\|_{\mathrm F}^\ell\Big)^{\frac{2\ell-2}{\ell}}
\le K^{\frac{2\ell-2}{\ell}}\big(\|Y\|_{\mathrm F}+\sqrt{f(W)}\big)^{\frac{2\ell-2}{\ell}}
\end{equation*}
and
\begin{equation*}
\|W_1\|_{\mathrm F}^{\ell-2} = \Big(\|W_1\|_{\mathrm F}^\ell\Big)^{\frac{\ell-2}{\ell}}
\le K^{\frac{\ell-2}{\ell}}\big\|Y\|_{\mathrm F}+\sqrt{f(W)}\big)^{\frac{\ell-2}{\ell}},
\end{equation*}
where
\begin{equation}
\label{eq:Kdef}
K:=\frac{d^{\frac{\ell-1}{2}}}{\sqrt{\lambda_{\min}(XX^\top)}}.
\end{equation}

Plugging these in inequality \eqref{eq:Hessian_bound_deepLN}, we get

\begin{align}
\label{eq:Hessian_bound_global_LN}
\norm{\nabla^2 f(W)}_2
\le 2\ell^2K^{\frac{2\ell-2}{\ell}}\big(\norm{Y}_{\mathrm F}+\sqrt{f(W)}\big)^{\frac{2\ell-2}{\ell}}\norm{X}_2^2
+2(\ell^2-\ell)K^{\frac{\ell-2}{\ell}}\big(\norm{Y}_{\mathrm F}+\sqrt{f(W)}\big)^{\frac{\ell-2}{\ell}}\norm{X}_2\sqrt{f(W)}.
\end{align}
 
We now separate $\norm{Y}_{\mathrm F}$ and $\sqrt{f(W)}$.
For the first exponent $p_1:=\frac{2\ell-2}{\ell}\ge1$, we use Jensen's inequality:
\begin{equation*}
\big(\norm{Y}_{\mathrm F}+\sqrt{f}\big)^{p_1}
\le 2^{p_1-1}\Big(\norm{Y}_{\mathrm F}^{p_1}+f^{\frac{p_1}{2}}\Big)
=2^{p_1-1}\Big(\norm{Y}_{\mathrm F}^{\frac{2\ell-2}{\ell}}+f^{\frac{\ell-1}{\ell}}\Big).
\end{equation*}
For the second exponent $p_2:=\frac{\ell-2}{\ell}\in[0,1)$, we use sub-additivity $(a+b)^{p_2} \le a^{p_2} + b^{p_2}$.
We bound this further by $2^{p_2}(a^{p_2} + b^{p_2})$ (since $2^{p_2}\ge 1$). Thus:
\begin{equation*}
\big(\norm{Y}_{\mathrm F}+\sqrt{f}\big)^{p_2}\sqrt{f}
\le \Big(\norm{Y}_{\mathrm F}^{p_2} + f^{\frac{p_2}{2}}\Big)\sqrt{f}
\le 2^{p_2}\Big(\norm{Y}_{\mathrm F}^{\frac{\ell-2}{\ell}}\sqrt{f} + f^{\frac{\ell-1}{\ell}}\Big).
\end{equation*}
We absorbe the factors of $2^{p_1-1}$ and $2^{p_2}$ into $K$ to forming $2K$.

Finally we linearize the remaining nonlinearity $\sqrt{f(W)}$.
Using the standard Young's inequality valid for any $\beta>0$, we have
\[
\sqrt{f(W)} \le \frac{f(W)}{2\beta} + \frac{\beta}{2}
= \frac{f(W)-f^*}{2\beta} + \left( \frac{f^*}{2\beta} + \frac{\beta}{2} \right).
\]
As above, we also have $f(W)^{\frac{\ell-1}{\ell}}\le 1+(f(W)-f^*)+(f^*)^{\frac{\ell-1}{\ell}}$.
Applying these to \eqref{eq:Hessian_bound_global_LN} and choosing $\beta=1$ gives the bound
$\|\nabla^2 f(W)\|_2 \leq H_0 + H_1 (f(W)-f^*)$ with 
\begin{align}
\label{eq:global_constants_linear_network}
 &H_0 :=A_0 \norm{X}_2^2 + A_1 \norm{X}_2 \left(\frac{f^*}{2} + \frac{1}{2}\right) + (A_2 \norm{X}_2^2+A_3 \norm{X}_2)\bigl(1+(f^*)^{\frac{\ell-1}{\ell}}\bigr), \nonumber \\
&H_1 := (A_1 + A_3) \|X\|_2 + A_2 \norm{X}_2^2
.
\end{align}
The constants $A_0, A_1, A_2, A_3$ are defined as
\begin{align}
\label{eq:A0123}
A_0 &:=2\ell^2 \bigl(2K\bigr)^{\frac{2\ell-2}{\ell}} \norm{Y}_{\mathrm F}^{\frac{2\ell-2}{\ell}},
&\qquad
A_2 &:=2\ell^2 \bigl(2K\bigr)^{\frac{2\ell-2}{\ell}}, \\
A_1 &:=2(\ell^2-\ell) \bigl(2K\bigr)^{\frac{\ell-2}{\ell}} \norm{Y}_{\mathrm F}^{\frac{\ell-2}{\ell}},
&
A_3 &:=2(\ell^2-\ell) \bigl(2K\bigr)^{\frac{\ell-2}{\ell}}
\end{align}
and $K$ as in Eq.~\eqref{eq:Kdef}.

\end{proof}

\begin{proposition}
\label{prop:deep_non_linear_many_nonlinearities}

Let $f$ be defined as
\begin{equation*}
    f(W) \equiv f(W_1,...,W_{\ell}) = \|Y- \underbrace{W_1 \phi_1(W_2 \phi(W_3 \dots \phi_{\ell-1}(W_{\ell} X) \dots ))}_F\|_F^2
\end{equation*}
where $\phi_i$ is leaky-ReLU activation function with slopes $1$ and $b_i$, i.e., $\phi_i(x)=\max \lbrace b_i x,x \rbrace$, $0 < b_i \leq 1$.
Assume that over the course of \algname{GD}:
\newline $\bullet$ $\lambda_{\min}(W_i^\top W_i) \geq h_i >0,$ for $i=1,\dots,\ell-1$.
\newline $\bullet$ The layers $W_i$ are weakly balanced, i.e., $\|W_1\|_{\rm F} = \ldots =\|W_{\ell}\|_{\rm F}.$

Then,

i) If it holds $\|W_1\|_F^{\ell} \|X\|_2 < (1-\frac{1}{\ell}) \sqrt{f(W)}$, then
we have $\|\nabla^2 f(W)\|_2 \leq H_0^{\text{warm-up}} + H_1^{\text{warm-up}} (f(W)-f^*)$, with $H_0^{\text{warm-up}}$ and $H_1^{\text{warm-up}}$ defined as in Eq.~\eqref{eq:warmup_constants_nonlinear_network}.

ii) It holds $\|\nabla^2 f(W)\|_2 \leq H_0^{\text{warm-up}} + H_1^{\text{warm-up}} (f(W)-f^*)$ for all $W$, with $H_0$ and $H_1$ defined as in Eq.~\eqref{eq:H01-post-final}.

\end{proposition}

\begin{proof}
The proof is divided into two parts, similarly to the proof of Proposition \ref{prop:deep_linear}: the first obtains an upper bound for the norm of the Hessian, while the second obtains a lower bound on the loss value.

The first part in the proof of Proposition \ref{prop:deep_linear} was easy, as one has ready formulas for the Hessian. In this case, the situation is more involved and we come up with a more general process to estimate the spectral norm of the Hessian based on the gradient finite differences.

\paragraph{Upper bound for the Hessian norm:}
To simplify the notation, we set
    \begin{align*}
        &Z_{\ell} = W_{\ell} X \\
        &A_{\ell-1} = \phi_{\ell-1} (Z_{\ell}) \\
        & Z_{\ell-1} = W_{\ell-1} A_{\ell-1} \\
        &\vdots 
        \\
        & Z_2 = W_2 A_2 \\ &
        A_1 = \phi_1(Z_2) \\ &
        Z_1 = W_1 A_1 = F.
    \end{align*}
    
    By the backpropagation algorithm for the gradient, we have that the gradient of $f$ can be computed as
    \begin{align*}
        \frac{\partial f}{\partial W_i} = \delta_i A_i^\top
    \end{align*}
    where $\delta_i$ is defined recursively as
    \begin{align*}
        &\delta_1 = -2(Y-F) \\ &
        \delta_2 = W_1^\top \delta_1 \odot \phi_1'(Z_2) \\ &
        \vdots \\ &
        \delta_i = W_{i-1}^\top \delta_{i-1} \odot \phi_{i-1}'(Z_i).
    \end{align*}

We need to upper bound the difference of the gradient defined in two distinct, sufficiently close points $W=(W_1,\dots,W_{\ell})$ and $\bar W = (\bar W_1,\dots, \bar W_{\ell})$. We also define
\begin{equation*}
    \textnormal{dist}(W, \bar W) := \sqrt{\sum_{i=1}^{\ell} \|W_i- \bar W_i\|^2_{\rm F}}.
\end{equation*}

It holds that
\begin{equation*}
    \|\nabla f(W)-\nabla f(\bar W)\|_{\rm F} \leq \sum_{i=1}^{\ell} \left\|\frac{\partial f}{\partial W_i}(W)-\frac{\partial f}{\partial W_i}(\bar W) \right\|_{\rm F}.
\end{equation*}

We have
\begin{equation}
\label{eq:basic_bound_partial_derivatives}
    \left\|\frac{\partial f}{\partial W_i}(W)-\frac{\partial f}{\partial W_i}(\bar W) \right\|_{\rm F} = \|\delta_i A_i^\top - \bar \delta_i \bar A_i^\top\|_{\rm F} \leq \|\delta_i\|_{\rm F} \|A_i-\bar A_i\|_{\rm F}+\|\bar A_i\|_{\rm F} \|\delta_i-\bar \delta_i\|_{\rm F}.
\end{equation}
Here we use a bar to denote the sequences of matrices related to the point $\bar W$. We deal with the four sequences appearing in this upper bound one by one, starting from $\bar A_i$. We can equivalently deal with $\bar A_i$ as the only difference will be to substitute $\bar W$ in place of $W$.

We have
\begin{align*}
    A_i =  \phi_i(W_{i+1}A_{i+1}), \hspace{1mm} \text{for} \hspace{1mm} i=1, \dots ,\ell-2,
\end{align*}
 thus
\begin{align*}
    \|A_i\|_{\rm F} =  \|\phi_i(W_{i+1}A_{i+1})\|_{\rm F} \leq \|W_{i+1} A_{i+1}\|_{\rm F} = \|W_1\|_{\rm F} \|A_{i+1}\|_{\rm F}. 
\end{align*}
The inequality follows from the fact that $\phi_i$ is leaky-ReLU, thus $|\phi_i(x)| \leq |x|$ and the last equality by the weakly balanced assumption, i.e. that $\|W_i\|_{\rm F} = \|W_1\|_{\rm F}$.

This implies that
\begin{equation}
\label{eq:bound_A_i}
    \|A_i\|_{\rm F} \leq \|W_1\|^{\ell-1-i} \|A_{\ell-1}\| = \|W_1\|^{\ell-1-i} \|\phi_{\ell-1}(W_{\ell} X)\|_{\rm F} \leq \|W_1\|_{\rm F}^{\ell-i} \|X\|_2.
\end{equation}

Similarly, it holds
\begin{equation}
\label{eq:bar_A_i}
    \|\bar A_i\|_{\rm F} \leq \|\bar W_1\|_{\rm F}^{\ell-i} \|X\|_2.
\end{equation}

Now, we deal with $A_i-\bar A_i$:
\begin{align*}
    &\|A_i-\bar A_i\|_{\rm F} = \nonumber\|\phi_i(W_{i+1}A_{i+1}) - \phi_i(\bar W_{i+1} \bar A_{i+1}) \|_{\rm F} \leq \|W_{i+1}A_{i+1} - \bar W_{i+1} \bar A_{i+1}\|_{\rm F} \leq \\ & \|A_{i+1}\|_{\rm F} \|W_{i+1}-\bar W_{i+1}\|_{\rm F} + \|\bar W_{i+1}\|_{\rm F} \|A_{i+1} - \bar A_{i+1}\|_{\rm F}  \leq \\ & \|A_{i+1}\|_{\rm F} \textnormal{dist}(W,\bar W) + \|\bar W_1\|_{\rm F} \|A_{i+1}-\bar A_{i+1}\|_{\rm F}.
\end{align*}

By an induction argument, we can get the bound
\begin{equation*}
    \|A_i-\bar A_i\|_{\rm F} \leq \left(\sum_{k=i+1}^{\ell-1} \|A_k\| \|\bar W_1\|^{k-i-1}\right) \textnormal{dist}(W,\bar W) + \|\bar W_1\|_{\rm F}^{\ell-i-1} \|A_{\ell-1}-\bar A_{\ell-1} \|_{\rm F}
\end{equation*}
and by inequality \eqref{eq:bound_A_i}, we have
\begin{align}
\label{eq:A_i-barA_i}
     \|A_i-\bar A_i\|_{\rm F} & \leq \nonumber \left(\sum_{k=i+1}^{\ell-1} \|W_1\|_{\rm F}^{\ell-k} \|\bar W_1\|^{k-i-1}\right) \textnormal{dist}(W,\bar W) \|X\|_2 + \|\bar W_1\|_{\rm F}^{\ell-i-1} \|W_{\ell}-\bar W_{\ell}\|_{\rm F} \|X\|_2 \\ & \leq \left(\sum_{k=i+1}^{\ell-1} \|W_1\|_{\rm F}^{\ell-k} \|\bar W_1\|^{k-i-1}\right) \textnormal{dist}(W,\bar W) \|X\|_2 + \|\bar W_1\|_{\rm F}^{\ell-i-1} \textnormal{dist}(W,\bar W) \|X\|_2.
\end{align}

Now we move to $\delta_i$. It holds
\begin{equation*}
    \|\delta_i\|_{\rm F} = \|W_{i-1}^\top \delta_{i-1} \odot \phi_{i-1}'(Z_i)\|_{\rm F} \leq \|W_{i-1}\|_{\rm F} \|\delta_{i-1}\|_{\rm F} = \|W_1\|_{\rm F} \|\delta_{i-1}\|_{\rm F}.
\end{equation*}
This implies that
\begin{equation*}
     \|\delta_i\|_{\rm F} \leq \|W_1\|_{\rm F}^{i-1} \|\delta_1\|_{\rm F} = 2 \|W_1\|_{\rm F}^{i-1} \sqrt{f(W)}
\end{equation*}
and similarly
\begin{equation}
\label{eq:bar_delta}
    \|\bar \delta_i\|_{\rm F} \leq 2 \|\bar W_1\|_{\rm F}^{i-1} \sqrt{f(\bar W)}.
\end{equation}

For the sequence $\delta_i-\bar \delta_i$, we have
\begin{align*}
    \|\delta_i - \bar \delta_i\|_{\rm F} = \| W_{i-1}^\top \delta_{i-1} \odot \phi_{i-1}'(Z_i) - \bar W_{i-1}^\top \bar \delta_{i-1} \odot \phi_{i-1}'(\bar Z_i)\|_{\rm F}
\end{align*}
and since all entries of $Z_i$ are non-zero and $\bar Z_i$ is taken sufficiently close to $Z_i$, these two points feature the same activation pattern, thus $\phi_{i-1}'(Z_i) = \phi_{i-1}'(\bar Z_i)$. This gives
\begin{align*}
    \|\delta_i - \bar \delta_i\|_{\rm F} & \leq \| W_{i-1}^\top \delta_{i-1} - \bar W_{i-1}^\top \bar \delta_{i-1}\|_{\rm F} \leq \|W_{i-1}\|_{\rm F} \| \delta_{i-1}-\bar \delta_{i-1}\|_{\rm F} + \|\bar \delta_{i-1}\|_{\rm F} \|W_{i+1}-\bar W_{i+1}\|_{\rm F} 
 \\ &\leq \|W_1\|_{\rm F} \| \delta_{i-1}-\bar \delta_{i-1}\|_{\rm F} + \|\bar \delta_{i-1}\|_{\rm F} \textnormal{dist}(W,\bar W).
\end{align*}

By induction, we have
\begin{align*}
    \|\delta_i - \bar \delta_i\|_{\rm F} & \leq \sum_{k=i-1}^1 \|\bar \delta_k\|_{\rm F} \| W_1 \|_{\rm F}^{i-1-k} \textnormal{dist}(W,\bar W) + \|W_1\|_{\rm F}^{i-1} \| \delta_1 - \bar \delta_1\|_{\rm F} \\ & \leq 2 \sqrt{f(\bar W)} \sum_{k=i-1}^1  \|\bar W_1\|_{\rm F}^{k-1} \| W_1 \|_{\rm F}^{i-1-k} \textnormal{dist}(W,\bar W) + \|W_1\|_{\rm F}^{i-1} \| \delta_1 - \bar \delta_1\|_{\rm F}. 
\end{align*}
The second inequality in the previous derivation follows by inequality \eqref{eq:bar_delta}.

For $\|\delta_1-\bar \delta_1\|_{\rm F}$, we have
\begin{align*}
  &\|\delta_1-\bar \delta_1\|_{\rm F} = 2 \|W_1 A_1-\bar W_1 \bar A_1\|_{\rm F} \leq 2\|W_1\|_{\rm F}\|A_1-\bar A_1\|_{\rm F}+ 2 \|\bar A_1\|_{\rm F} \|W_1-\bar W_1\|_{\rm F} \leq \\ &  2\|W_1\|_{\rm F} \left(\left(\sum_{k=2}^{\ell-1} \|W_1\|_{\rm F}^{\ell-k} \|\bar W_1\|^{k-2}\right) + \|\bar W_1\|_{\rm F}^{\ell-2} \right) \textnormal{dist}(W,\bar W) \|X\|_2 + 2\|\bar W_1\|_{\rm F}^{\ell-1} \textnormal{dist}(W,\bar W) \|X\|_2 = \\ & 2\left(\|W_1\|_{\rm F} \left(\left(\sum_{k=2}^{\ell-1} \|W_1\|_{\rm F}^{\ell-k} \|\bar W_1\|^{k-2}\right) + \|\bar W_1\|_{\rm F}^{\ell-2} \right) + \|\bar W_1\|_{\rm F}^{\ell-1} \right) \textnormal{dist}(W,\bar W) \|X\|_2.
\end{align*}

Thus,
\begin{align}
\label{eq:eq8}
    &\|\delta_i - \bar \delta_i\|_{\rm F}  \nonumber \leq 2 \sqrt{f(\bar W)} \sum_{k=i-1}^1  \|\bar W_1\|_{\rm F}^{k-1} \| W_1 \|_{\rm F}^{i-1-k} \textnormal{dist}(W,\bar W) + \\& 2 \|W_1\|_{\rm F}^{i-1} \left(\|W_1\|_{\rm F} \left(\left(\sum_{k=2}^{\ell-1} \|W_1\|_{\rm F}^{\ell-k} \|\bar W_1\|^{k-2}\right) + \|\bar W_1\|_{\rm F}^{\ell-2} \right) + \|\bar W_1\|_{\rm F}^{\ell-1} \right) \textnormal{dist}(W,\bar W) \|X\|_2. 
\end{align}

Combining inequalities \eqref{eq:basic_bound_partial_derivatives},\eqref{eq:bar_A_i},\eqref{eq:A_i-barA_i},\eqref{eq:bar_delta} and \eqref{eq:eq8},
we get
\begin{align*}
    &\left\|\frac{\partial f}{\partial W_i}(W)-\frac{\partial f}{\partial W_i}(\bar W) \right\|_{\rm F} \leq \\ &  2\|\bar W_1\|_{\rm F}^{i-1} \sqrt{f(\bar W)} \left(\left(\sum_{k=i+1}^{\ell-1} \|W_1\|_{\rm F}^{\ell-k} \|\bar W_1\|^{k-i-1}\right) + \|\bar W_1\|_{\rm F}^{\ell-i-1}\right)  \textnormal{dist}(W,\bar W) \|X\|_2 + \\ & 2 \|\bar W_1\|_{\rm F}^{\ell-i} \|X\|_2 \sqrt{f(\bar W)} \sum_{k=i-1}^1  \|\bar W_1\|_{\rm F}^{k-1} \| W_1 \|_{\rm F}^{i-1-k} \textnormal{dist}(W,\bar W) + \\ & 2\|\bar W_1\|_{\rm F}^{\ell-i} \|W_1\|_{\rm F}^{i-1} \left(\|W_1\|_{\rm F} \left(\left(\sum_{k=2}^{\ell-1} \|W_1\|_{\rm F}^{\ell-k} \|\bar W_1\|^{k-2}\right) + \|\bar W_1\|_{\rm F}^{\ell-2} \right) + \|\bar W_1\|_{\rm F}^{\ell-1} \right) \textnormal{dist}(W,\bar W) \|X\|^2_2,  
\end{align*}
thus
\begin{align*}
    &\frac{\left\|\frac{\partial f}{\partial W_i}(W)-\frac{\partial f}{\partial W_i}(\bar W) \right\|_{\rm F}}{\textnormal{dist}(W,\bar W)} \leq \\ &  2\|\bar W_1\|_{\rm F}^{i-1} \sqrt{f(\bar W)} \left(\left(\sum_{k=i+1}^{\ell-1} \|W_1\|_{\rm F}^{\ell-k} \|\bar W_1\|^{k-i-1}\right) + \|\bar W_1\|_{\rm F}^{\ell-i-1}\right) \|X\|_2 + \\ & 2 \|\bar W_1\|_{\rm F}^{\ell-i} \|X\|_2 \sqrt{f(\bar W)} \sum_{k=i-1}^1  \|\bar W_1\|_{\rm F}^{k-1} \| W_1 \|_{\rm F}^{i-1-k} + \\ & 2 \|\bar W_1\|_{\rm F}^{\ell-i} \|W_1\|_{\rm F}^{i-1} \left(\|W_1\|_{\rm F} \left(\left(\sum_{k=2}^{\ell-1} \|W_1\|_{\rm F}^{\ell-k} \|\bar W_1\|^{k-2}\right) + \|\bar W_1\|_{\rm F}^{\ell-2} \right) + \|\bar W_1\|_{\rm F}^{\ell-1} \right) \|X\|^2_2
\end{align*}
and taking the limit as $\bar W \longrightarrow W$, we get
\begin{align*}
    &\lim_{\bar W \rightarrow W}\frac{\left\|\frac{\partial f}{\partial W_i}(W)-\frac{\partial f}{\partial W_i}(\bar W) \right\|_{\rm F}}{\textnormal{dist}(W,\bar W)} \leq \\ &  2 (\ell-i) \|X\|_2 \|W_1\|_{\rm F}^{\ell-2} \sqrt{f(W)}  +2 (i-1) \|X\|_2 \|W_1\|_{\rm F}^{\ell-2} \sqrt{f(W)} + 2 (\ell-1) \|W_1\|_{\rm F}^{2 \ell-2} \|X\|_2^2 = \\ & 2 (\ell-1) \|X\|_2 \sqrt{f(W)} \|W_1\|_{\rm F}^{\ell-2} + 2 \ell \|W_1\|_{\rm F}^{2 \ell-2} \|X\|_2^2.
\end{align*}
This is because, when $\bar W \longrightarrow W$, it holds $\bar W_1\longrightarrow W_1$.

For the total gradient difference, we have
\begin{align*}
  \lim_{\bar W \rightarrow W}\frac{\left\|\nabla f(W)-\nabla f(\bar W) \right\|_{\rm F}}{\textnormal{dist}(W,\bar W)} & \leq \sum_{i=1}^\ell \lim_{\bar W \rightarrow W}\frac{\left\|\frac{\partial f}{\partial W_i}(W)-\frac{\partial f}{\partial W_i}(\bar W) \right\|_{\rm F}}{\textnormal{dist}(W,\bar W)} \\ &  \leq 2 \ell (\ell-1) \|W_1\|_{\rm F}^{\ell-2} \|X\|_2 \sqrt{f(W)} + 2 \ell^2 \|W_1\|_{\rm F}^{2 \ell-2} \|X\|_2^2.
\end{align*}
It holds
\begin{equation*}
    \|\nabla^2 f(W) \|_2 = \lim_{\bar W \rightarrow W}\frac{\left\|\nabla f(W)-\nabla f(\bar W) \right\|_{\rm F}}{\textnormal{dist}(W,\bar W)},
\end{equation*}
thus
\begin{equation}
\label{eq:Hessian_bound_non_linear}
    \|\nabla^2 f(W)\|_2 \leq 2 \ell (\ell-1) \|W_1\|_{\rm F}^{\ell-2} \|X\|_2 \sqrt{f(W)} + 2 \ell^2 \|W_1\|_{\rm F}^{2 \ell-2} \|X\|_2^2.
\end{equation}

Notice that this is the same upper bound as the one provided in \eqref{eq:Hessian_bound_deepLN}.

Thus, if $\|W\|_F^{\ell} \|X\|_2 \leq (1-\frac{1}{\ell}) \sqrt{f(W)}$, then we get $\|\nabla^2 f(W)\| \leq H_0^{\text{warm-up}} + H_1^{\text{warm-up}} (f(W)-f^*)$ with exactly the same constants as in the previous proposition, i.e.
\begin{equation}
\label{eq:warmup_constants_nonlinear_network}
H_0^{\text{warm-up}}:=4(\ell^2-\ell) \left(1-\frac{1}{\ell}\right)^{\frac{\ell-2}{\ell}} \|X\|_2^{\frac{2}{\ell}}\bigl(1+(f^*)^{\frac{\ell-1}{\ell}}\bigr), \qquad H_1^{\text{warm-up}}:= 4 (\ell^2-\ell) \left(1-\frac{1}{\ell}\right)^{\frac{\ell-2}{\ell}} \|X\|_2^{\frac{2}{\ell}}.
\end{equation}

We now move to a global bound and for that we need a global lower bound for the loss value. For $i=1 ,\dots, \ell-2$, we have
\begin{align*}
     \|W_i A_i\|_F^2 \geq \lambda_{\min}(W_i^T W_i) \|A_i\|_F^2 \geq h_i \|\phi_i(W_{i+1} A_{i+1})\|^2_F \geq h_i b_i^2 \|W_{i+1} A_{i+1}\|_F^2
\end{align*}
and by induction,
\begin{align}
\label{eq:deep_leaky_relu_basic}
    \|W_1 A_1\|_F^2 & \geq \nonumber\left(\Pi_{i=1}^{\ell-2} h_i b_i^2 \right) \|W_{\ell-1} A_{\ell-1}\|_F^2 \\ & \geq \nonumber\left(\Pi_{i=1}^{\ell-2} h_i b_i^2 \right) \lambda_{\min}(W_{\ell-1}^T W_{\ell-1}) \|A_{\ell-1}\|_F^2 \\ & = \nonumber \left(\Pi_{i=1}^{\ell-2} h_i b_i^2 \right) \lambda_{\min}(W_{\ell-1}^T W_{\ell-1}) \|\phi_{\ell-1}(W_{\ell} X)\|_F^2 \\ & \geq \nonumber \left(\Pi_{i=1}^{\ell-2} h_i b_i^2 \right) h_{\ell-1} b_{\ell-1}^2 \lambda_{\min}(X X^T) \|W_{\ell}\|_F^2 \\ & = \nonumber\left(\Pi_{i=1}^{\ell-2} h_i b_i^2 \right) h_{\ell-1} b_{\ell-1}^2 \lambda_{\min}(X X^T) \|W_1\|_F^2 \\ & = \left(\Pi_{i=1}^{\ell-1} h_i b_i^2 \right) \lambda_{\min}(X X^T) \|W_1\|_F^2.
\end{align}
We have repeatedly used the assumption that $\lambda_{\min}(W_i W_i^T) \geq h_i$ and that
\begin{equation*}
    \|\phi_i(S)\|^2_F \geq b_i^2 \|S\|_F^2,
\end{equation*}
for any matrix $S$.

To derive inequality \eqref{eq:deep_leaky_relu_basic}, we also used the weak balancedness assumption, that is, all $\|W_i\|_F$ are equal.

Defining \begin{equation}
\label{eq:K_gd_def}
K_{gd} := \left( \lambda_{\min}(XX^\top) \prod_{i=1}^{\ell-1} h_i b_i^2 \right)^{-\frac{1}{2}}
\end{equation}
we get
$$ \|W_1\|_F \leq K_{gd} \|W_1 A_1\|_F. $$

By the triangle inequality, $\norm{W_1 A_1}_F \le \norm{Y}_F + \sqrt{f(W)}$. Thus:
\begin{equation}
\label{eq:W1_bound_nonlinear}
\norm{W_1}_F \le K_{gd} \left( \norm{Y}_F + \sqrt{f(W)} \right).
\end{equation}

Now we substitute this into the sharp Hessian bound \eqref{eq:Hessian_bound_non_linear}.
\paragraph{Term 1:}
\begin{align*}
2\ell^2 \norm{W_1}_F^{2\ell-2} \norm{X}_2^2
\le 2\ell^2 K_{gd}^{2\ell-2} \left( \norm{Y}_F + \sqrt{f} \right)^{2\ell-2} \norm{X}_2^2 
\le 2\ell^2 K_{gd}^{2\ell-2} 2^{2\ell-3} \left( \norm{Y}_F^{2\ell-2} + f^{\ell-1} \right) \norm{X}_2^2.
\end{align*}
(Using Jensen's inequality $(a+b)^p \le 2^{p-1}(a^p+b^p)$ with $p=2\ell-2$).

\paragraph{Term 2:}
\begin{align*}
2(\ell^2-\ell) \norm{W_1}_F^{\ell-2} \norm{X}_2 \sqrt{f}
&\le 2(\ell^2-\ell) K_{gd}^{\ell-2} \left( \norm{Y}_F + \sqrt{f} \right)^{\ell-2} \norm{X}_2 \sqrt{f} \\
&\le 2(\ell^2-\ell) K_{gd}^{\ell-2} 2^{\ell-3} \left( \norm{Y}_F^{\ell-2} + f^{\frac{\ell-2}{2}} \right) \norm{X}_2 \sqrt{f} \\
&= 2(\ell^2-\ell) K_{gd}^{\ell-2} 2^{\ell-3} \left( \norm{Y}_F^{\ell-2}\sqrt{f} + f^{\frac{\ell-1}{2}} \right) \norm{X}_2.
\end{align*}

We now use the inequality $u \leq 1+ u^p$, which holds for all $u \geq 0$ and $p \geq 1$. In our case, it translates to 
\begin{equation*}
    \sqrt{f} \leq 1+f^{\ell-1}
\end{equation*}
and
\begin{equation*}
    f^{\frac{\ell-1}{2}} \leq 1 + f^{\ell-1}.
\end{equation*}

Term 2 can be bounded as
\begin{align*}
    & 2(\ell^2-\ell) K_{gd}^{\ell-2} 2^{\ell-3} \left( \norm{Y}_F^{\ell-2}\sqrt{f} + f^{\frac{\ell-1}{2}} \right) \norm{X}_2 \leq 2(\ell^2-\ell) K_{gd}^{\ell-2} 2^{\ell-3} \left(1+\|Y\|_F^{\ell-2} \right) \|X\|_2 \left(1 + f^{\ell-1} \right) = \\ & 2(\ell^2-\ell) K_{gd}^{\ell-2} 2^{\ell-3} \left(1+\|Y\|_F^{\ell-2} \right) \|X\|_2 + 2(\ell^2-\ell) K_{gd}^{\ell-2} 2^{\ell-3} \left(1+\|Y\|_F^{\ell-2} \right) \|X\|_2 f^{\ell-1}.
\end{align*}

Summing with the bound of Term 1 yields the inequality $\|\nabla^2 f(W)\|_2 \leq H_0 + H_1 (f(W)-f^*)$, with
\begin{align}
\label{eq:H01-post-final}
H_0 &:= 2\ell^2 K_{gd}^{2\ell-2} 2^{2\ell-3} \norm{Y}_F^{2\ell-2} \norm{X}_2^2 + 2(\ell^2-\ell) K_{gd}^{\ell-2} 2^{\ell-3} (1+ \norm{Y}_F^{\ell-2}) \norm{X}_2, \nonumber \\
H_1 &:= 2\ell^2 K_{gd}^{2\ell-2} 2^{2\ell-3} \norm{X}_2^2 + 2(\ell^2-\ell) K_{gd}^{\ell-2} 2^{\ell-3} (1+ \|Y\|_F^{\ell-2}) \norm{X}_2,
\end{align}
where $K_{gd}$ is defined as in Eq.~\eqref{eq:K_gd_def}.

\end{proof}

\crossentropy*

\begin{proof}

We start by calculating the gradients and Hessians of $f$. The Hessian of the regularization part is just $(\lambda_1+\lambda_2) I$. We denote the main part of the loss as 
\begin{equation*}
    \bar f(W) = -Y \log(P)^\top-(\1-Y) \log(\1-P)^\top.
\end{equation*}

Again, it holds
\begin{equation*}
    \|\nabla^2 f(W)\|_2 \leq \| \nabla^2 \bar f(W)\|_2 + (\lambda_1+\lambda_2).
\end{equation*}

Some useful notation is
\begin{align*}
    & A := W_2 X \\
    & H := \phi(A) \\
    & Z := W_1 H \\
    & P := \sigma(Z).
\end{align*}

The gradient of $\bar L$ with respect to $\textnormal{vec}(W_1)$ is 
\begin{equation*}
    \frac{\partial \bar f}{\partial Z} \cdot \frac{\partial Z}{\partial \textnormal{vec}(W_1)}. 
\end{equation*}

It holds
\begin{equation*}
    \frac{ \partial \bar f}{\partial P} = -Y \odot \frac{1}{P} + (\1-Y) \odot \frac{1}{\1-P}
\end{equation*}
where $1/\text{vector}$ is used to denote entry-wise inversion.

We also have
\begin{equation*}
    \frac{\partial P}{\partial Z} = \sigma'(Z) = P \odot (\1-P).
\end{equation*}

Thus,
\begin{equation*}
  \frac{\partial \bar f}{\partial Z} = \frac{\partial \bar f}{\partial P} \odot \frac{\partial P}{\partial Z} = P-Y.
\end{equation*}

We denote the vectorized form of this term by $R$ since it plays the role of a residual. Since $P-Y$ is a row vector, its vectorized form is just its transpose, however, we will often keep the standard form $R = \textnormal{vec}(P-Y)$ to ensure compatibility with previous calculations. 

It holds
\begin{equation*}
    \frac{\partial \bar f}{\partial \textnormal{vec}(W_1)} = \frac{\partial \bar f}{\partial Z} \frac{\partial Z}{\partial \textnormal{vec}(W_1)} = R^\top H^\top = R^\top \phi(W_2 X)^\top .
\end{equation*}

This is a row vector, thus we transpose it to bring it to column form:

\begin{equation*}
    \frac{\partial \bar f}{\partial \textnormal{vec}(W_1)} = H R = \textnormal{vec}((P-Y) H^\top) = \textnormal{vec}((P-Y) \phi(W_2 X)^\top)
\end{equation*}

For the partial derivative with respect to $\textnormal{vec}(W_2)$, we have
\begin{align*}
   &\frac{\partial \bar f}{\partial \textnormal{vec}(W_2)} = \frac{\partial \bar f}{\partial Z} \cdot \frac{\partial Z}{\partial \textnormal{vec}(W_2) } = R^\top \frac{\partial Z}{\partial \textnormal{vec}(W_2) }
\end{align*}
and
\begin{align*}
    \frac{\partial R}{\partial \textnormal{vec}(W_2)} = -(I_m \otimes W_1) \frac{\partial \textnormal{vec}(\phi(W_2 X))}{\partial \textnormal{vec}(W_2)} = -(I_m \otimes W_1) \frac{\partial \textnormal{vec}(\phi(W_2 X))}{\partial \textnormal{vec}(W_2 X)} \frac{\partial \textnormal{vec}(W_2 X)}{\partial \textnormal{vec}(W_2)}
\end{align*}

$\frac{\partial \textnormal{vec}(\phi(W_2 X))}{\partial \textnormal{vec}(W_2)}$ is the diagonal matrix $\textnormal{diag}(\textnormal{vec}(\phi'(W_2 X))$.

Since $\textnormal{vec}(W_2 X) = (X^\top \otimes I_{n_1}) \textnormal{vec}(W_2)$, the gradient $\frac{\partial \textnormal{vec}(W_2 X)}{\partial \textnormal{vec}(W_2)}$ is 
\begin{equation*}
    \frac{\partial \textnormal{vec}(W_2 X)}{\partial \textnormal{vec}(W_2)} = X^\top \otimes I_{n_1}.
\end{equation*}

Putting it all together, we have
\begin{equation*}
   \frac{\partial f}{\partial \textnormal{vec}(W_2)} = R^\top (I_m \otimes W_1) \textnormal{diag}(\textnormal{vec}(\phi'(W_2 X)) (X^\top \otimes I_{n_1}).
\end{equation*}

Writing that again as column vector yields
\begin{equation*}
    (X \otimes I_{n_1}) \textnormal{diag}(\textnormal{vec}(\phi'(W_2 X))) (I_m \otimes W_1^\top) R.
\end{equation*}

After some modifications, we can write
\begin{align*}
    &\textnormal{diag}(\textnormal{vec}(\phi'(W_2 X))) (I_m \otimes W_1^\top) R = \\ & \textnormal{diag}(\textnormal{vec}(\phi'(W_2 X))) \textnormal{vec}(W_1^\top (P-Y))  = \\ & \textnormal{vec}(W_1^\top (P-Y) \odot \phi'(W_2 X)).
\end{align*}
where $\odot$ is the Hadamard product.

This means that we can write the previous gradient as
\begin{equation*}
    \textnormal{vec}(((W_1^\top (P-Y)) \odot \phi'(W_2 X) ) X^\top).
\end{equation*}

We now move to the calculation of the Hessian.

For the first block, we have
\begin{align*}
    \frac{\partial^2 \bar f}{\partial \textnormal{vec}(W_1) \textnormal{vec}(W_1)^\top } &= \phi(W_2 X) \frac{\partial R}{\partial \textnormal{vec}(W_1)^\top} \\ & = \phi(W_2 X) \frac{\partial \textnormal{vec}(P-Y)}{\partial \textnormal{vec}(W_1)^\top} \\ & =  \phi(W_2 X) \textnormal{diag}(P \odot (\1-P)) \phi(W_2 X)^\top.
\end{align*}

For the off-diagonal blocks, it suffices to compute one of them, as they are symmetric.

We use the product rule (see \cite{magnus1985matrix}, Theorem 9)
\begin{equation*}
    \frac{\partial \textnormal{vec}(A(W) B (W))}{\partial \textnormal{vec}(W)^\top} = (B(W)^\top \otimes I) \frac{\partial \textnormal{vec}(A(W))}{\partial \textnormal{vec}(W)^\top} + (I \otimes A(W)) \frac{\partial \textnormal{vec}(B(W))}{\partial \textnormal{vec}(W)^\top}.
\end{equation*}
We have
\begin{align*}
    \frac{\partial}{\partial \textnormal{vec}(W_2)^\top} \frac{\partial \bar f}{\partial \textnormal{vec}(W_1)} 
& = (\phi(W_2 X) \otimes I_1) \frac{\partial \textnormal{vec} (P-Y) }{\partial \textnormal{vec}(W_2)^\top} \\& + (I_{n_1} \otimes (P-Y) ) \frac{\partial \textnormal{vec} (\phi(W_2 X)^\top) }{\partial \textnormal{vec}(W_2)^\top}. \end{align*}

In order to proceed, we need to write $\textnormal{vec}(\phi(W_2 X)^\top)$ in terms of $\textnormal{vec}(\phi(W_2 X))$, and this can be done formally using the so-called commutation matrix:
\begin{equation*}
   \textnormal{vec}(\phi(W_2 X)^\top) = K_{n_1 m}  \textnormal{vec}(\phi(W_2 X)).
\end{equation*}

For the first partial derivative in the sum, we have
\begin{align*}
  \frac{\partial \textnormal{vec} (P-Y) }{\partial \textnormal{vec}(W_2)^\top} &= \frac{\partial \textnormal{vec}(P) }{\partial \textnormal{vec}(Z)} \frac{\partial \textnormal{vec}(Z) }{\partial \textnormal{vec}(W_2)^\top} \\&= \textnormal{diag}(P \odot (\1-P)) \frac{\partial \textnormal{vec} (W_1\phi(W_2 X)) }{\partial \textnormal{vec}(W_2)^\top}\\ &  = \textnormal{diag}(P \odot (\1-P))(I_m \otimes W_1)\frac{\partial \textnormal{vec} (\phi(W_2 X)) }{\partial \textnormal{vec}(W_2)^\top} \\ & = \textnormal{diag}(P \odot (\1-P)) (I_m \otimes W_1) \textnormal{diag}(\textnormal{vec}(\phi'(W_2 X))) \frac{\partial \textnormal{vec}(W_2 X)}{\partial \textnormal{vec}(W_2)^\top} \\ & = \textnormal{diag}(P \odot (\1-P)) (I_m \otimes W_1) \textnormal{diag}(\textnormal{vec}(\phi'(W_2 X))) (X^\top \otimes I_{n_1}).
\end{align*}

As it is evident in the previous calculation
\begin{equation*}
    \frac{\partial \textnormal{vec} (\phi(W_2 X)) }{\partial \textnormal{vec}(W_2)^\top} = \textnormal{diag}(\textnormal{vec}(\phi'(W_2 X))) (X^\top \otimes I_{n_1}).
\end{equation*}

Putting it all together, we get
\begin{align*}
    \frac{\partial^2 \bar f}{\partial \textnormal{vec}(W_1) \textnormal{vec}(W_2)^\top } & = \phi(W_2 X)\textnormal{diag}(P \odot (\1-P) (I_m \otimes W_1) \textnormal{diag}(\textnormal{vec}(\phi'(W_2 X))) (X^\top \otimes I_{n_1}) \\ & +(I_{n_1} \otimes (P-Y)) K_{n_1 m} \textnormal{diag}(\textnormal{vec}(\phi'(W_2 X))) (X^\top \otimes I_{n_1}) \\ & =
    (\phi(W_2 X)\textnormal{diag}(P \odot (\1-P)) (I_m \otimes W_1) \\ & + (I_{n_1} \otimes (P-Y)K_{n_1 m})) \textnormal{diag}(\textnormal{vec}(\phi'(W_2 X))) (X^\top \otimes I_{n_1}).
\end{align*}

We conclude with the calculation of the last block.
To differentiate $\textnormal{vec}(((W_1^\top R) \odot \phi'(W_2 X)) X^\top)$, we can use the product rule for the Hadamard product, see \cite{magnus1985matrix} (Theorem 10):

\begin{align*}
    &\frac{\partial\textnormal{vec}((W_1^\top R) \odot \phi'(W_2 X))}{\partial \text{vec}(W_2)^\top} =  \textnormal{diag}(\text{vec}(\phi'(W_2 X)) \frac{\partial \text{vec}(W_1^\top R)}{\partial \text{vec}(W_2)^\top} + \textnormal{diag}(\textnormal{vec}(W_1^\top R)) \frac{\partial \phi'(W_2 X)}{\partial \text{vec}(W_2)^\top}.
\end{align*}

For the first term of the last sum, we have by previous calculations that
\begin{equation*}
    \frac{\partial \text{vec}(W_1^\top R)}{\partial \text{vec}(W_2)^\top} = (I_m \otimes W_1^\top) \textnormal{diag}(P \odot (\1-P)) (I_m \otimes W_1) \textnormal{diag}(\textnormal{vec}(\phi'(W_2 X))) (X^\top \otimes I_{n_1}).
\end{equation*}

For the second term of the last sum, we have
\begin{equation*}
    \frac{\partial \phi'(W_2 X)}{\partial \text{vec}(W_2)^\top} = \textnormal{diag}(\text{vec}(\phi''(W_2 X))) (X^\top \otimes I_{n_1}).
\end{equation*}

In total, we have
\begin{align*}
    \frac{\partial^2 \bar f}{\partial \textnormal{vec}(W_2) \text{vec}(W_2)^\top} =& (X \otimes I_{n_1}) \textnormal{diag}(\textnormal{vec}(\phi'(W_2 X))) (I_m \otimes W_1^\top) \textnormal{diag}(P \odot (\1-P)) \\ & (I_m \otimes W_1) \textnormal{diag}(\textnormal{vec}(\phi'(W_2 X))) (X^\top \otimes I_{n_1}) \\  + & (X \otimes I_{n_1}) \textnormal{diag}(\textnormal{vec}(W_1^\top R))\textnormal{diag}(\text{vec}(\phi''(W_2 X))) (X^\top \otimes I_{n_1}).
\end{align*}

This completes the calculation of all four blocks of the Hessian of $\bar f$.

To upper bound $\| \nabla^2 \bar f(W) \|_2$, we can write
\begin{equation*}
    \| \nabla^2 \bar f(W) \|_2 \leq \left\|\frac{\partial^2 \bar f}{\partial \textnormal{vec}(W_1) \textnormal{vec}(W_1)^\top } \right\|_2 + 2 \left\|\frac{\partial^2 \bar f}{\partial \textnormal{vec}(W_1) \textnormal{vec}(W_2)^\top } \right\|_2 + \left\|\frac{\partial^2 \bar f}{\partial \textnormal{vec}(W_2) \textnormal{vec}(W_2)^\top } \right\|_2. 
\end{equation*}

It holds
\begin{align*}
    \left\|\frac{\partial^2 \bar f}{\partial \textnormal{vec}(W_1) \textnormal{vec}(W_1)^\top } \right\|_2 & \leq \|\textnormal{diag}(P \odot (\1-P))\|_2 \|\phi(W_2 X) \phi(W_2 X)^\top\|_2 \\ &  \leq \|\textnormal{diag}(P \odot (\1-P))\|_2 \|\phi(W_2 X) \phi(W_2 X)^\top\|_{\rm F} \leq C_1^2 \|W_2\|_{\rm F}^2 \|X\|_2^2,
\end{align*}
since all entries of $P \odot (\1-P)$ are upper bounded by $1$ in absolute value.

For the off-diagonal blocks, it holds

\begin{align*}
    \left\|\frac{\partial^2 \bar f}{\partial \textnormal{vec}(W_1) \textnormal{vec}(W_2)^\top } \right\|_2 & \leq ( \|\phi(W_2 X)\|_2 \|W_1\|_2 + \|P - Y\|_2) C_2 \|X\|_2\\ & \leq C_2 (C_1 \|W_1\|_{\rm F} \|W_2\|_{\rm F} \|X\|_2 + \|P - Y\|_{\rm F}) \|X\|_2 \\ & \leq C_2 (C_1 \|W_1\|_{\rm F} \|W_2\|_{\rm F} \|X\|_2 + \sqrt{2f(W)}) \|X\|_2
\end{align*}
and
\begin{align*}
    \left\|\frac{\partial^2 \bar f}{\partial \textnormal{vec}(W_2) \textnormal{vec}(W_2)^\top } \right\|_2 & \leq \|X\|^2_2 C_2^2 \|W_1^\top\|_2 \|W_1\|_2 + \|X\|^2_2 C_3 \|W_1^\top (P-Y)\|_2 \\ & \leq \|X\|^2_2 C_2^2 \|W_1\|^2_{\rm F} + \|X\|^2_2 C_3 \|W_1\|_{\rm F} \|P-Y\|_{\rm F} \\ & \leq \|X\|^2_2 C_2^2 \|W_1\|^2_{\rm F} + \|X\|^2_2 C_3 \|W_1\|_{\rm F} \sqrt{2 f(W)}
\end{align*}

In the previous inequalities we used the standard inequality $\|P-Y\|_F \leq \sqrt{2 \bar f(W)} \leq \sqrt{2f(W)}$. 

We define the structural constants:
\begin{align}
\label{eq:CE_regularization_constants}
    A_1 &:= C_2^2 \norm{X}^2_2, \nonumber \\
    A_2 &:= C_1^2 \norm{X}^2_2, \nonumber\\
    A_{12} &:= 2 C_1 C_2 \norm{X}^2_2, \nonumber\\
    B_{\text{mix}} &:= \sqrt{2} C_3 \norm{X}^2_2, \nonumber\\
    B_{\text{res}} &:= 2\sqrt{2} C_2 \norm{X}_2.
\end{align}

Using the explicit block-wise bounds, we get that the Hessian spectral norm is bounded by:
\begin{equation*}
\label{eq:exact_decomp_ce}
    \norm{\nabla^2 f(W)}_2 \le (\lambda_1 + \lambda_2) + T_{\text{quad}}(W) + T_{\text{mixed}}(W) + T_{\text{res}}(W),
\end{equation*}
where:
\begin{align*}
    T_{\text{quad}}(W) &:= A_1 \norm{W_1}_{\mathrm F}^2 + A_2 \norm{W_2}_{\mathrm F}^2 + A_{12} \norm{W_1}_{\mathrm F} \norm{W_2}_{\mathrm F}, \\
    T_{\text{mixed}}(W) &:= B_{\text{mix}} \norm{W_1}_{\mathrm F} \sqrt{\bar{f}(W)}, \\
    T_{\text{res}}(W) &:= B_{\text{res}} \sqrt{\bar{f}(W)}.
\end{align*}

For (i), the condition
\begin{equation}
\label{eq:cross_entropy_local_condition}
    A_1 \norm{W_1}^2 + A_2 \norm{W_2}^2 + A_{12} \norm{W_1}\norm{W_2} \le  B_{\text{mix}} \norm{W_1} \sqrt{f} + B_{\text{res}} \sqrt{f}
\end{equation}
implies $T_{\text{quad}}(W) \leq T_{\text{mixed}}(W) + T_{\text{res}}(W)$.
Since weight norms are non-negative, we can also drop the terms involving $\norm{W_2}$ to find a bound for $\norm{W_1}$:
\begin{equation*}
    A_1 \norm{W_1}_F^2 - ( B_{\text{mix}} \sqrt{f}) \norm{W_1}_F - ( B_{\text{res}} \sqrt{f}) \le 0.
\end{equation*}
Solving this quadratic inequality for $\norm{W_1}_F$ (finding the positive root) gives
\begin{equation*}
    \norm{W_1}_F \le \frac{ B_{\text{mix}}}{A_1} \sqrt{f} + \sqrt{\frac{ B_{\text{res}}}{A_1}} f^{1/4}.
\end{equation*}
Substituting this back into $T_{\text{mixed}} = B_{\text{mix}} \norm{W_1}_F \sqrt{f}$ yields
\begin{equation*}
    T_{\text{mixed}} \le \frac{ B_{\text{mix}}^2}{A_1} f + B_{\text{mix}}\sqrt{\frac{ B_{\text{res}}}{A_1}} f^{3/4}.
\end{equation*}
The total bound is $(\lambda_1+\lambda_2) + 2 (T_{\text{mixed}} + T_{\text{res}})$.
We linearize $\sqrt{f}$ and $f^{3/4}$ (using $\sqrt{f} \leq (f-f^*)+f^*+1$ and $f^{3/4} \le \frac{3}{4}(f-f^*) +\frac{3}{4} f^* + \frac{1}{4}$) and group coefficients to obtain
\begin{align}
\label{eq:warmup_cross_entropy}
H_0^{\text{warm-up}} &:= (\lambda_1 + \lambda_2) + 2 \left( B_{\text{res}}\left(\frac{f^*}{2} + \frac{1}{2}\right) + \frac{ B_{\text{mix}}^2}{A_1}f^* + B_{\text{mix}}\sqrt{\frac{ B_{\text{res}}}{A_1}} \left(\frac{3}{4}f^* + \frac{1}{4}\right) \right), \nonumber \\ 
    H_1^{\text{warm-up}} &:= 2 \left( B_{\text{res}} + \frac{ B_{\text{mix}}^2}{A_1} + \frac{3}{4} B_{\text{mix}}\sqrt{\frac{ B_{\text{res}}}{A_1}} \right).
\end{align}

We proceed now to the global bound (ii).
It holds
\begin{equation*}
    A_{12} \norm{W_1}\norm{W_2} \le \frac{A_{12}}{2} \norm{W_1}^2 + \frac{A_{12}}{2} \norm{W_2}^2.
\end{equation*}
and
\begin{equation*}
    T_{\text{mixed}} = B_{\text{mix}} \norm{W_1} \sqrt{f} \le \frac{B_{\text{mix}}}{2} \norm{W_1}^2 + \frac{B_{\text{mix}}}{2} f.
\end{equation*}
By the definition of the function $f$, we can bound the weight norms as
$\norm{W_i}_F^2 \le \frac{2 f}{\lambda_i}$. This implies
\begin{align*}
    T_{\text{quad}} &\le \left( \frac{2 A_1 + A_{12}}{\lambda_1} + \frac{2 A_2 + A_{12}}{\lambda_2} \right) f. \\
    T_{\text{mixed}} &\le \left( \frac{B_{\text{mix}}}{\lambda_1} + \frac{B_{\text{mix}}}{2} \right) f.
\end{align*}

For the $T_{\text{res}}$, we bound it as
\begin{equation*}
T_{\text{res}} \leq B_{\text{res}} \sqrt{f(W)} \leq B_{\text{res}} ((f(W)-f^*)+f^*+1). 
\end{equation*}

In total, we have
\begin{equation*}
\|\nabla^2 f(W)\| \leq H_0 + H_1 (f(W)-f^*)
\end{equation*}
with

\begin{align}
\label{eq:CE_regularization_global_constants}
     \nonumber H_0 &:= (\lambda_1 + \lambda_2) + B_{\text{res}} (f^*+1), \\
     H_1 &:= B_{\text{res}} + \left( \frac{2 A_1 + A_{12} + B_{\text{mix}}}{\lambda_1} + \frac{2 A_2 + A_{12}}{\lambda_2} + \frac{B_{\text{mix}}}{2} \right).   
\end{align}

\end{proof}

\section{Missing proofs from Section \ref{sec:theory_transformers}}
\label{app:transformers_proofs}

First, we describe the setup in more detail. The input data is encoded into a single matrix $Z_0 \in \mathbb{R}^{(d+1) \times (n+1)}$. This matrix contains $n$ training tokens and one query token. The training tokens cover the first $n$ columns of the matrix, while the query token the last one. The label of the query's feature is initialized at $0$.
\begin{equation*}
    Z_0 = \begin{bmatrix}
        x_1 &x_2 &\dots &x_n &x_{\text{query}} \\ y_1 &y_2 &\dots &y_n &0
    \end{bmatrix} \in \mathbb{R}^{(d+1) \times (n+1)}.
\end{equation*}
The model's objective is to predict the true value for this entry.
The model is defined by 
$$
Z_{1} = Z_0 + \frac{1}{n} P Z_0 M \cdot \phi(Z_0^T Q Z_0),
$$ where the trainable parameters are $P$ and $Q$ and $M=\text{diag}(1,1,\dots,1,0)$ is a fixed mask and $\phi$ is a general activation applied to the attention scores. The output of the model is $$\hat{y} = [Z_1]_{(d+1),(n+1)}$$ and the cost function for one task with true target  $y_{\text{true}}$ is $$f(P,Q) = \ell(\hat{y},y_{\text{true}}),$$ where $\ell$ is some loss function (MSE for continuous variables or cross-entropy for binary ones). The most interesting property of transformers is their ability to learn in-context, i.e., minimize an in-context cost function defined below.

\begin{definition}
\label{def:in_context_learning}
\label{prop:transformers_in_context}
Let $D_x$ be a distribution over an input
space $X$, $H$ a set of functions $X \rightarrow Y$, and $D_H$ a distribution over functions in $H$.
Let $\ell : Y \times Y \rightarrow\mathbb{R}$ be a loss function, $S = \lbrace {(x_1, y_1, \dots , x_n, y_n) : x_i \in X , y_i \in Y} \rbrace$ be
the set of finite-length sequences of $(x, y)$ pairs and
$F_\theta = \lbrace f_\theta : S \times X \rightarrow Y, \theta \in \Theta \rbrace$
be a class of functions parameterized by $\theta$ in some set $\Theta$. For $n > 0$, we say that a model
$f : S \times X \rightarrow Y$ is trained on in-context examples of functions in $H$ under loss $\ell$ w.r.t.
$(D_x, D_H)$ if $f = f_\theta$, where $\theta$ minimizes
minimizes $$\mathbb{E}_{j=(x_1,h(x_1),\dots,x_n ,h(x_n ),x_{query})}
[\ell(f_\theta(j), h(x_{query}))]$$
where $x_i, x_{query}$ are chosen 
i.i.d. from $D_x$ and $h \sim D_H$ is independent. $j$ represents a prompt.
\end{definition}



\begin{proposition}
\label{prop:transformer_one_task}
    Consider the aforementioned $1$-layer transformer model with a regularized loss:
\begin{equation*}
    f(P,Q) = (\hat{y} - y_{true})^2 + \frac{\lambda_P}{2} \|P\|_F^2 + \frac{\lambda_Q}{2} \|Q\|_F^2.
\end{equation*}
Assume that the input $Z$ is bounded, and the activation function and its derivatives are globally bounded:
\begin{itemize}
    \item $\|Z\|_2 \leq C_Z$
    \item $\|\phi(x)\| \leq C_1 \|x\|$
    \item $\|\phi'(x)\| \leq C_2$
    \item $\|\phi''(x)\| \leq C_3$.
\end{itemize}
Then, it holds

(i) If $K_1 \|Q\|_F^2 + K_2 \|P\|_F \|Q\|_F + K_4 \|P\|_2^F \leq \sqrt{2} K_5 \|P\|_F \sqrt{f(P,Q)} + \sqrt{2} K_3 \sqrt{f(P,Q)}$, we have
\begin{equation*}
    \|\nabla^2 f(P,Q)\|_2 \leq H_0^{\text{warm-up}} + H_1^{\text{warm-up}} (f(P,Q)-f^*),
\end{equation*}
where $H_0^{\text{warm-up}}$ and $H_1^{\text{warm-up}}$ are defined as in Eq.~\eqref{eq:warmup_constants_transformer}.

(ii) For all $(P,Q)$ it holds
\begin{equation*}
    \|\nabla^2 f(P,Q)\|_2 \leq H_0 + H_1 (f(P,Q)-f^*),
\end{equation*}
where $H_0$ and $H_1$ are defined as in Eq.~\eqref{eq:global_constants_transformer}.    
\end{proposition}

\begin{proof}
To compute the Hessian, we first vectorize the parameters and the model. For simplicity, we set $k=d+1$ and $m=n+1$.
\begin{itemize}
    \item Parameters: $p = \vecop(P) \in \R^{k^2 \times 1}$ and $q = \vecop(Q) \in \R^{k^2 \times 1}$.
    \item Constants: Let $V = Z M \in \R^{k \times m}$ and $E = e_k e_m^T \in \R^{k \times m}$.
    \item Intermediates: $S(Q) = Z^T Q Z$ and $A(Q) = \phi(S(Q))$.
\end{itemize}
Using the trace identity $e_k^T X e_m = \tr(e_m e_k^T X) = \tr(E^T X)$, we rewrite $\hat{y}$:
$$
\hat{y} = \frac{1}{n} \tr(E^T P V A(Q)) = \frac{1}{n} \tr( (V A(Q) E^T) P ).
$$
Using $\tr(\mathbf{A}^T \mathbf{B}) = \text{vec}(\mathbf{A})^T \vecop(\mathbf{B})$, we get:
$$
\hat{y} = \frac{1}{n} \vecop( (V A(Q) E^T)^T )^T \vecop(P) = \frac{1}{n} \vecop(E A(Q)^T V^T)^T p.
$$
Let $b(q) = \vecop(E A(Q)^T V^T)$. The prediction is linear in $p$:
$$
\hat{y}(p, q) = \frac{1}{n} b(q)^T p.
$$
We now expand $b(q)$ using the Kronecker product ($\otimes$) and commutation matrix ($K$):
\begin{enumerate}
    \item $s(q) = \vecop(S(Q)) = \vecop(Z^T Q Z) = (Z^T \otimes Z^T) q$.
    Let $J_S = (Z^T \otimes Z^T)$.

    \item $a(q) = \vecop(A(Q)) = \phi(s(q)) = \phi(J_S q)$.

    \item $b(q) = \vecop(E A(Q)^T V^T) = (V \otimes E) \vecop(A(Q)^T)$.

    \item $\vecop(A(Q)^T) = K_{m,m} \vecop(A(Q)) = K_{m,m} a(q)$.
\end{enumerate}
Combining these, we define the constant matrix $J_A$:
$$
b(q) = (V \otimes E) K_{m,m} a(q) = \underbrace{(V \otimes E) K_{m,m}}_{J_A} \phi(J_S q).
$$
Our final vectorized prediction is:
$$
\hat{y}(p, q) = \frac{1}{n} (J_A \phi(J_S q))^T p.
$$
We define the loss $f$ as either the Mean Squared Error (MSE) or Cross-Entropy (CE).
The gradient and Hessian of the loss are given by the chain rule:
$$
\nabla f = r \nabla \hat{y} \quad \text{and} \quad H_f = \gamma_1 (\nabla \hat{y})(\nabla \hat{y})^T + \gamma_2 r H_{\hat{y}},
$$
where $r$ is the residual and $\gamma_1, \gamma_2$ are scalar coefficients depending on the loss.
\begin{itemize}
    \item \textbf{MSE:} $f = (\hat{y}-y)^2$. Then $r = 2(\hat{y}-y)$, $\gamma_1 = 2$, $\gamma_2 = 1$. (Or equivalently absorbing the 2 into $r$, giving coefficients 2 and 2).
    \item \textbf{CE:} $f = \text{BCE}(\sigma(\hat{y}), y)$. Then $r = \sigma(\hat{y})-y$, $\gamma_1 = \sigma(\hat{y})(1-\sigma(\hat{y})) \le 1/4$, $\gamma_2 = 1$.
\end{itemize}
In both cases, the coefficients are upper bounded by 2. Thus, we perform the calculation using the upper bound coefficients (2, 2) which covers both cases:
$$
(\nabla \hat{y})(\nabla \hat{y})^T + 2r H_{\hat{y}}.
$$
Finally, we define the diagonal matrices of derivatives:
$$
D'(q) = \diag(\vecop(\phi'(J_S q))) \quad \text{and} \quad D''(q) = \diag(\vecop(\phi''(J_S q))).
$$

We first compute $\nabla \hat{y}$.

$$
\nabla_p \hat{y} = \frac{\partial \hat{y}}{\partial p} = \frac{\partial}{\partial p} \left( \frac{1}{n} b(q)^T p \right) = \frac{1}{n} b(q) = \frac{1}{n} J_A \phi(J_S q).
$$
We also have
$$
\nabla_q \hat{y} = \frac{\partial \hat{y}}{\partial q} = \frac{\partial}{\partial q} \left( \frac{1}{n} p^T b(q) \right) = \frac{1}{n} \left( \frac{\partial b(q)}{\partial q} \right)^T p.
$$
We need the Jacobian of $b(q)$, $\frac{\partial b(q)}{\partial q^T}$.
$$
\frac{\partial b(q)}{\partial q^T} = \frac{\partial (J_A \phi(J_S q))}{\partial q^T} = J_A \frac{\partial (\phi(J_S q))}{\partial q^T} = J_A D'(q) \frac{\partial (J_S q)}{\partial q^T} = J_A D'(q) J_S
$$
Plugging this back in (and transposing):
$$
\nabla_q \hat{y} = \frac{1}{n} (J_A D'(q) J_S)^T p = \frac{1}{n} J_S^T D'(q)^T J_A^T p = \frac{1}{n} J_S^T D'(q) J_A^T p.
$$

We will compute the four blocks of $H_f$ by first finding the four blocks of $H_{\hat{y}}$.

\paragraph{$H_{\hat{y}, pp}$:}
We differentiate $\nabla_p \hat{y}$ w.r.t $p^T$:
$$
H_{\hat{y}, pp} = \frac{\partial}{\partial p^T} (\nabla_p \hat{y}) = \frac{\partial}{\partial p^T} \left( \frac{1}{n} b(q) \right) = \mathbf{0}.
$$
since $b(q)$ does not depend on $p$.

\paragraph{$H_{\hat{y}, pq}$:}
We differentiate $\nabla_p \hat{y}$ w.r.t $q^T$:
$$
H_{\hat{y}, pq} = \frac{\partial}{\partial q^T} (\nabla_p \hat{y}) = \frac{\partial}{\partial q^T} \left( \frac{1}{n} b(q) \right) = \frac{1}{n} \frac{\partial b(q)}{\partial q^T}.
$$
We have already computed $\frac{\partial b(q)}{\partial q^T}$:
$$
H_{\hat{y}, pq} = \frac{1}{n} J_A D'(q) J_S.
$$

\paragraph{$H_{\hat{y}, qp}$:}
This block is the transpose of the previous one:
$$
H_{\hat{y}, qp} = H_{\hat{y}, pq}^T = \left( \frac{1}{n} J_A D'(q) J_S \right)^T = \frac{1}{n} J_S^T D'(q) J_A^T.
$$

\paragraph{$H_{\hat{y}, qq}$:}
We differentiate $\nabla_q \hat{y}$ w.r.t $q^T$:
$$
H_{\hat{y}, qq} = \frac{\partial}{\partial q^T} (\nabla_q \hat{y}) = \frac{\partial}{\partial q^T} \left( \frac{1}{n} J_S^T D'(q) J_A^T p \right).
$$
Let $v = J_A^T p$ be a constant vector. We need to find $\frac{1}{n} J_S^T \left( \frac{\partial (D'(q) v)}{\partial q^T} \right)$.
Using the rule for differentiating a diagonal matrix $\frac{\partial (\diag(X) v)}{\partial y^T} = \diag(v) \frac{\partial X}{\partial y^T}$, we have:
$$
\frac{\partial (D'(q) v)}{\partial q^T} = \frac{\partial (\diag(\vecop(\phi'(J_S q))) v)}{\partial q^T} = \diag(v) \frac{\partial (\vecop(\phi'(J_S q)))}{\partial q^T}.
$$
The derivative of the vector $\vecop(\phi'(J_S q)$ is $D''(q) J_S$.
$$
\frac{\partial (D'(q) v)}{\partial q^T} = \diag(v) D''(q) J_S = \diag(J_A^T p) D''(q) J_S.
$$
Plugging this back into the expression for $H_{\hat{y}, qq}$:
$$
H_{\hat{y}, qq} = \frac{1}{n} J_S^T \diag(J_A^T p) D''(q) J_S.
$$

We now put together the four blocks of $H_f \approx 2 (\nabla \hat{y}) (\nabla \hat{y})^T + 2r H_{\hat{y}}$.

$$
H_{pp} = 2 (\nabla_p \hat{y}) (\nabla_p \hat{y})^T + 2r H_{\hat{y}, pp} = 2 \left( \frac{1}{n} b(q) \right) \left( \frac{1}{n} b(q) \right)^T + \mathbf{0}.
$$
$$
H_{pp} = \frac{2}{n^2} (J_A \phi(J_S q)) (J_A \phi(J_S q))^T
$$

$$
H_{qq} = 2 (\nabla_q \hat{y}) (\nabla_q \hat{y})^T + 2r H_{\hat{y}, qq}
$$
$$
H_{qq} = \frac{2}{n^2} \left( J_S^T D'(q) J_A^T p \right) \left(J_S^T D'(q) J_A^T p \right)^T + \frac{2r}{n} \left[ J_S^T \diag(J_A^T p) D''(q) J_S \right]
$$

$$
H_{pq} = 2 (\nabla_p \hat{y}) (\nabla_q \hat{y})^T + 2r H_{\hat{y}, pq}
$$
$$
H_{pq} = 2 \left( \frac{1}{n} J_A \phi(J_S q) \right) \left( \frac{1}{n} J_S^T D'(q) J_A^T p \right)^T + 2r \left( \frac{1}{n} J_A D'(q) J_S \right)
$$
$$
H_{pq} = \frac{2}{n^2} (J_A \phi(J_S q)) (p^T J_A D'(q) J_S) + \frac{2r}{n} (J_A D'(q) J_S)
$$

$$
H_{qp} = 2 (\nabla_q \hat{y}) (\nabla_p \hat{y})^T + 2r H_{\hat{y}, qp} = H_{pq}^T
$$
$$
H_{qp} = \frac{2}{n^2} (J_S^T D'(q) J_A^T p) (J_A \phi(J_S q))^T + \frac{2r}{n} (J_S^T D'(q) J_A^T).
$$

Now we pass in bounding the regulatized loss.
We use the notation $f$ for the regulatized and $\bar{f}$ unregularized loss.
The Hessian of the regularized loss $f$ is
$$
    \nabla^2 f(p, q) = \nabla^2 \bar{f}(p, q) + \begin{bmatrix} \lambda_P I & 0 \\ 0 & \lambda_Q I \end{bmatrix}.
$$
Using the triangle inequality for the spectral norm ($\|\cdot\|_2$):
\begin{equation*}
    \|\nabla^2 f(p, q)\|_2 \leq \|\nabla^2 \bar{f}(p, q)\|_2 + \left\| \begin{bmatrix} \lambda_P I & 0 \\ 0 & \lambda_Q I \end{bmatrix} \right\|_2 \leq \|\nabla^2 \bar{f}(p, q)\|_2+(\lambda_P + \lambda_Q).
\end{equation*}
We proceed by computing an upper bound for $\|\nabla^2 \bar{f}(p, q)\|_2$.

From our previous calculations, we have the four blocks:
\begin{align*}
    H_{pp} &= \frac{2}{n^2} (J_A \phi(J_S q)) (J_A \phi(J_S q))^T \\
    H_{qq} &= \frac{2}{n^2} (J_S^T D'(q) J_A^T p) (J_S^T D'(q) J_A^T p)^T + \frac{2r}{n} \left[ J_S^T \diag(J_A^T p) D''(q) J_S \right] \\
    H_{pq} &= \frac{2}{n^2} (J_A \phi(J_S q)) (p^T J_A D'(q) J_S) + \frac{2r}{n} (J_A D'(q) J_S) \\
    H_{qp} &= H_{pq}^T
\end{align*}
We now upper bound the spectral norm of the full Hessian by the sum of the norms of its blocks:
\begin{equation*}
    \| \nabla^2 \bar{f}(p, q) \|_2 \leq \|H_{pp}\|_2 + 2 \|H_{pq}\|_2 + \|H_{qq}\|_2.
\end{equation*}

Since $\|Z\|_2 \leq C_Z$, it holds $\|J_S\|_2 \leq C_Z^2$ and $\|J_A\|_2 \leq C_Z$.

\paragraph{Bounding $H_{pp}$:}
$H_{pp}$ is a rank-1 matrix.
\begin{align*}
    \|H_{pp}\|_2 &= \frac{2}{n^2} \|J_A \phi(J_S q)\|_2^2 \leq \frac{2}{n^2} \|J_A\|_2^2 \|\phi(J_S q)\|_2^2
    \leq \frac{2}{n^2} C_Z^6 C_1^2 \|q\|_2^2.
\end{align*}

\paragraph{Bounding $H_{pq}$:}
We use the triangle inequality.
\begin{align*}
    \|H_{pq}\|_2 &\leq \left\| \frac{2}{n^2} (J_A \phi(J_S q)) (p^T J_A D'(q) J_S) \right\|_2 + \left\| \frac{2r}{n} (J_A D'(q) J_S) \right\|_2 \\
    &\leq \frac{2}{n^2} \|J_A \phi(J_S q)\|_2 \cdot \|p^T J_A D'(q) J_S\|_2 + \frac{2|r|}{n} \|J_A D'(q) J_S\|_2 \\
    &\leq \frac{2}{n^2} (C_Z^3 C_1 \|q\|_2) \cdot (\|p\|_2 \|J_A\|_2 \|D'(q)\|_2 \|J_S\|_2) + \frac{2|r|}{n} (\|J_A\|_2 \|D'(q)\|_2 \|J_S\|_2) \\
    &\leq \frac{2}{n^2} (C_Z^3 C_1 \|q\|_2) \cdot (\|p\|_2 C_2 C_Z^3) + \frac{2|r|}{n} C_2 C_Z^3 \\
    &= \frac{2}{n^2} C_Z^6 C_1 C_2 \|p\|_2 \|q\|_2 + \frac{2}{n} C_2 C_Z^3 |r|.
\end{align*}

\paragraph{Bounding $H_{qq}$:}
We use the triangle inequality.
\begin{align*}
    \|H_{qq}\|_2 &\leq \frac{2}{n^2} \|J_S^T D'(q) J_A^T p\|_2^2 + \frac{2|r|}{n} \|J_S^T \diag(J_A^T p) D''(q) J_S\|_2 \\
    &\leq \frac{2}{n^2} (\|J_S\|_2 \|D'(q)\|_2 \|J_A\|_2 \|p\|_2)^2 + \frac{2|r|}{n} (\|J_S\|_2^2 \|\diag(J_A^T p)\|_2 \|D''(q)\|_2) \\
    &\leq \frac{2}{n^2} (C_2 C_Z^3 \|p\|_2)^2 + \frac{2|r|}{n} ( (C_Z^2)^2 \cdot (\|J_A\|_2 \|p\|_2) \cdot C_3 ) \\
    &= \frac{2}{n^2} C_2^2 C_Z^6 \|p\|_2^2 + \frac{2}{n} C_Z^5 C_3 |r| \|p\|_2.
\end{align*}

\paragraph{Overall Bound for $\bar{f}$:}
\begin{align}
\label{eq:transformers_hessian_bound}
    \|\nabla^2 \bar{f}\|_2 &\leq \|H_{pp}\|_2 + 2 \|H_{pq}\|_2 + \|H_{qq}\|_2 \nonumber \\ & 
    \leq \underbrace{\frac{2}{n^2} C_Z^6 C_1^2}_{K_1} \|q\|_2^2 + \underbrace{\frac{2}{n^2} C_Z^6 C_1 C_2}_{K_2} \|p\|_2 \|q\|_2 + \underbrace{\frac{2}{n} C_2 C_Z^3}_{K_3} |r| + \underbrace{\frac{2}{n^2} C_2^2 C_Z^6}_{K_4} \|p\|_2^2 + \underbrace{\frac{2}{n} C_Z^5 C_3}_{K_5} |r| \|p\|_2.
\end{align}

We bound the residual magnitude $|r|$ by the loss. For MSE, $|r| = \sqrt{\bar{f}}$. For Cross-Entropy, $|r| \le \sqrt{2\bar{f}}$. To unify notation, let $C_r=1$ for MSE and $C_r=\sqrt{2}$ for CE.
The Hessian spectral norm satisfies:
\begin{equation}
    \|\nabla^2 f(p,q)\|_2 \le (\lambda_P + \lambda_Q) + T_{\text{quad}}(p,q) + T_{\text{mixed}}(p,q) + T_{\text{res}}(p,q),
\end{equation}
where:
\begin{align*}
    T_{\text{quad}} &:= K_1 \|q\|_2^2 + K_2 \|p\|_2 \|q\|_2 + K_4 \|p\|_2^2, \\
    T_{\text{mixed}} &:= K_5 C_r \|p\|_2 \sqrt{\bar{f}}, \\
    T_{\text{res}} &:= K_3 C_r \sqrt{\bar{f}}.
\end{align*}

The local condition for (i) is equivalent to $T_{\text{quad}} \leq T_{\text{mixed}}+T_{\text{res}}.$ Thus,
\begin{equation*}
\|\nabla^2 f(p,q)\|_2 \leq (\lambda_P + \lambda_Q) + 2 (T_{\text{mixed}}(p,q) + T_{\text{res}}(p,q)).
\end{equation*}

Since the terms involving $\|q\|_2$ are non-negative ($K_1, K_2 \ge 0$), we can drop them to obtain a necessary condition for $\|p\|_2$:
\begin{equation*}
    K_4 \|p\|_2^2 \le K_5 C_r \sqrt{\bar{f}} \|p\|_2 + K_3 C_r \sqrt{\bar{f}}.
\end{equation*}
Rewriting this as $K_4 \|p\|_2^2 - (K_5 C_r \sqrt{\bar{f}}) \|p\|_2 - (K_3 C_r \sqrt{\bar{f}}) \le 0$, we solve for the positive root to bound $\|p\|_2$:
\begin{equation*}
    \|p\|_2 \le \frac{K_5 C_r}{K_4} \sqrt{\bar{f}} + \sqrt{\frac{K_3 C_r}{K_4}} \bar{f}^{1/4}.
\end{equation*}
We substitute this bound directly into the mixed term $T_{\text{mixed}} = K_5 C_r \|p\|_2 \sqrt{\bar{f}}$:
\begin{align*}
    T_{\text{mixed}} &\le K_5 C_r \left( \frac{K_5 C_r}{K_4} \bar{f} + \sqrt{\frac{K_3 C_r}{K_4}} \bar{f}^{3/4} \right) = \frac{(K_5 C_r)^2}{K_4} \bar{f} + C_{\text{cross}} \bar{f}^{3/4}, \quad \text{where } C_{\text{cross}} = K_5 C_r \sqrt{\frac{K_3 C_r}{K_4}}.
\end{align*}

Using $\bar{f}(W) \le f(W)$, we replace all instances of $\bar{f}$ with $f$. To obtain the affine bound $H_0^{\text{warm-up}} + H_1^{\text{warm-up}}(f(W) - f^*)$, we apply the following linearizations valid for $f \ge 0$ and $\beta > 0$:
\begin{enumerate}
    \item \textbf{Linear term:} $f = (f - f^*) + f^*$.
    \item \textbf{Cross term ($f^{3/4}$):} By Young's inequality ($ab \le a^p/p + b^q/q$ with $p=4/3, q=4$), we have $u^{3/4} \cdot 1 \le \frac{3}{4}u + \frac{1}{4}$. Thus:
    \begin{equation*}
        f^{3/4} \le \frac{3}{4}f + \frac{1}{4} = \frac{3}{4}(f - f^*) + \left( \frac{3}{4}f^* + \frac{1}{4} \right).
    \end{equation*}
    \item \textbf{Residual term ($\sqrt{f}$):} Using the robust smoothness inequality $\sqrt{u} \le \frac{u}{2\beta} + \frac{\beta}{2}$:
    \begin{equation*}
        \sqrt{f} \le \frac{f}{2\beta} + \frac{\beta}{2} = \frac{f - f^*}{2\beta} + \left( \frac{f^*}{2\beta} + \frac{\beta}{2} \right).
    \end{equation*}
\end{enumerate}

Substituting these into the general curvature bound results to
\begin{align*}
    \|\nabla^2 f\|_2 \le (\lambda_P + \lambda_Q)
     + 2 \frac{(K_5 C_r)^2}{K_4} \left((f - f^*) + f^* \right)
     + 2 C_{\text{cross}} \left( \frac{3}{4}(f - f^*) + \left( \frac{3}{4}f^* + \frac{1}{4} \right) \right)
     + 2 K_3 C_r \left( \frac{f - f^*}{2\beta} + \left( \frac{f^*}{2\beta} + \frac{\beta}{2} \right) \right).
\end{align*}

We now collect the coefficients of $(f(W) - f^*)$ into $H_0^{\text{warm-up}}$ and the constant terms into $H_1^{\text{warm-up}}$ and set $\beta=1$:

\begin{align}
\label{eq:warmup_constants_transformer}
&H_0^{\text{warm-up}} = (\lambda_P + \lambda_Q) \nonumber 
    + 2 \left( \frac{(K_5 C_r)^2}{K_4} f^* + C_{\text{cross}} \left( \frac{3}{4}f^* + \frac{1}{4} \right) + K_3 C_r \left( \frac{f^*}{2} + \frac{1}{2} \right) \right) \\
    &H_1^{\text{warm-up}} = 2 \left( \frac{(K_5 C_r)^2}{K_4} + \frac{3}{4} C_{\text{cross}} + \frac{K_3 C_r}{2} \right).
\end{align}

For the global bound (ii), we begin again with the exact decomposition of the Hessian bound:
\begin{equation*}
\label{eq:phase2_start}
    \|\nabla^2 f(W)\|_2 \le (\lambda_P + \lambda_Q) + T_{\text{quad}}(W) + T_{\text{mixed}}(W) + T_{\text{res}}(W).
\end{equation*}
Recall the definitions:
\begin{align*}
    T_{\text{quad}} &= K_1 \|q\|_2^2 + K_2 \|p\|_2 \|q\|_2 + K_4 \|p\|_2^2, \\
    T_{\text{mixed}} &= K_5 C_r \|p\|_2 \sqrt{\bar{f}}, \\
    T_{\text{res}} &= K_3 C_r \sqrt{\bar{f}}.
\end{align*}

We first apply the AM-GM inequality $ab \le \frac{1}{2}(a^2+b^2)$ to separate the coupled terms.

For $T_{\text{quad}}$:
\begin{align*}
    T_{\text{quad}} &= K_1 \|q\|_2^2 + K_2 (\|p\|_2 \|q\|_2) + K_4 \|p\|_2^2
    \\ & \le K_1 \|q\|_2^2 + \frac{K_2}{2} (\|p\|_2^2 + \|q\|_2^2) + K_4 \|p\|_2^2
    \\ & = \left( K_4 + \frac{K_2}{2} \right) \|p\|_2^2 + \left( K_1 + \frac{K_2}{2} \right) \|q\|_2^2.
\end{align*}

For $T_{\text{mixed}}$:
\begin{align*}
    T_{\text{mixed}} = K_5 C_r (\|p\|_2 \sqrt{\bar{f}})
    \le K_5 C_r \left( \frac{1}{2}\|p\|_2^2 + \frac{1}{2}\bar{f} \right)
    = \frac{K_5 C_r}{2} \|p\|_2^2 + \frac{K_5 C_r}{2} \bar{f}.
\end{align*}

We now use the regularization bounds derived from the total loss definition $f(W) \ge \frac{\lambda_P}{2}\|p\|_2^2 + \frac{\lambda_Q}{2}\|q\|_2^2$:
\begin{equation*}
    \|p\|_2^2 \le \frac{2 f(W)}{\lambda_P}, \qquad \|q\|_2^2 \le \frac{2 f(W)}{\lambda_Q}.
\end{equation*}
Substituting these into the disentangled terms (and using $\bar{f} \le f$):

For $T_{\text{quad}}$:
\begin{equation*}
    T_{\text{quad}} \le \left( \frac{2}{\lambda_P}\left( K_4 + \frac{K_2}{2} \right) + \frac{2}{\lambda_Q}\left( K_1 + \frac{K_2}{2} \right) \right) f(W).
\end{equation*}

For $T_{\text{mixed}}$:
\begin{equation*}
    T_{\text{mixed}} \le \left( \frac{2}{\lambda_P}\left( \frac{K_5 C_r}{2} \right) \right) f(W) + \frac{K_5 C_r}{2} f(W) = \left( \frac{K_5 C_r}{\lambda_P} + \frac{K_5 C_r}{2} \right) f(W).
\end{equation*}

Let us define a consolidated coefficient $C_{\text{linear}}$ for all terms that are strictly proportional to $f(W)$:
\begin{equation*}
    C_{\text{linear}} := \underbrace{\frac{2K_4 + K_2 + K_5 C_r}{\lambda_P}}_{\text{Terms from } \|p\|^2} + \underbrace{\frac{2K_1 + K_2}{\lambda_Q}}_{\text{Terms from } \|q\|^2} + \underbrace{\frac{K_5 C_r}{2}}_{\text{From } \bar{f}}.
\end{equation*}

For $T_{\text{res}} = K_3 C_r \sqrt{\bar{f}}$, we use $\sqrt{\bar{f}} \le \sqrt{f}$ and the linearization $\sqrt{f} \le \frac{f-f^*}{2\beta} + (\frac{f^*}{2\beta} + \frac{\beta}{2})$:
\begin{equation*}
    T_{\text{res}} \le K_3 C_r \left[ \frac{f - f^*}{2\beta} + \left( \frac{f^*}{2\beta} + \frac{\beta}{2} \right) \right].
\end{equation*}

Combining everything into the form $H_0 + H_1(f-f^*)$:
\begin{align*}
    \|\nabla^2 f\|_2 \le {}& (\lambda_P + \lambda_Q)
    + C_{\text{linear}} f(W) + T_{\text{res}}.
\end{align*}
We rewrite $C_{\text{linear}} f(W)$ as $C_{\text{linear}} (f - f^*) + C_{\text{linear}} f^*$.

Setting $\beta=1$ and $C_r = \sqrt{2}$, we arrive at the following expressions for the constants $H_0$ and $H_1$:
\begin{align}
\label{eq:global_constants_transformer}
& \nonumber H_0 = (\lambda_P + \lambda_Q) + \sqrt{2} K_3 \left( \frac{f^*}{2} + \frac{1}{2} \right) + f^* \left( \frac{2K_4 + K_2 + \sqrt{2} K_5}{\lambda_P} + \frac{2K_1 + K_2}{\lambda_Q} + \frac{\sqrt{2} K_5}{2} \right) \\
    &H_1 = \frac{\sqrt{2} K_3}{2} + C_{\text{linear}} = \frac{\sqrt{2} K_3}{2} + \left( \frac{2K_4 + K_2 + \sqrt{2} K_5}{\lambda_P} + \frac{2K_1 + K_2}{\lambda_Q} + \frac{\sqrt{2} K_5}{2} \right),
\end{align}
where
\begin{align*}
&K_1 = \frac{2}{n} C_Z^6 C_1^2 \\
&K_2 = \frac{2}{n^2} C_Z^6 C_1 C_2 \\
&K_3 = \frac{2}{n} C_2 C_Z^3 \\
&K_4 = \frac{2}{n^2} C_2^2 C_Z^6 \\
&K_5 = \frac{2}{n} C_Z^5 C_3.
\end{align*}

\end{proof}

We close the discussion of this section with our main result, i.e. proving an $(H_0,H_1)$-smoothness condition for the \emph{in-context} loss function over a distribution of tasks.
Before we get to its proof, we present some useful lemmas from the theory of random matrices.
\begin{lemma}
\label{le:1}
A sub-Gaussian random variable is also sub-exponential. 
\end{lemma}
\begin{lemma}
\label{le:2}
For a random matrix $Z$ whose entries are independent and sub-exponential (with zero mean), its spectral norm $\|Z\|_2$ exhibits sub-exponential tails. This means the probability of the norm being large decays very quickly.
\begin{equation*}
\mathbb{P}(\|Z\|_2 > t) \leq C \exp(-c t) \quad \text{for large } t,
\end{equation*}
for constants $C, c$.
\end{lemma}
\begin{lemma}[Finite Moments]
\label{le:finite moments}
    A random variable $X$ with sub-exponential tails (like $\|Z_j\|_2$) has finite moments:
$$
\mathbb{E}[|X|^k] < \infty \quad \text{for all } k \geq 1.
$$
\end{lemma}
\noindent For a rigorous treatment of these topics, see \cite{vershynin2018high}.

Now, we are ready for the main proof. This is the formal version of Proposition \ref{prop:transformers_in_context}.

\begin{restatable}{proposition}{transformers}
\label{prop:transformers_in_context_formal}
Consider the transformer model described in \Cref{sec:theory_transformers}, for in-context learning in continuous variables with MSE loss, or in binary variables with CE loss. $\phi$ is assumed to satisfy the conditions of Proposition \ref{prop:crossentropy}  and $f_j(P,Q)$ is the L2 regularized loss corresponding to the j-th prompt.  Consider the regularized in-context loss function
   \begin{equation*}
        f(P,Q) = \mathbb{E}_j[f_j(P,Q)], \hspace{1mm} \text{where} \hspace{1mm} f_j(P,Q)=(\hat{y}_j - y_{j,true})^2 + \frac{\lambda_P}{2} \|P\|_F^2 + \frac{\lambda_Q}{2} \|Q\|_F^2.
    \end{equation*}
    Assume that $D_x$ is sub-gaussian and the distribution of $y=h(x)$ is sub-exponential. Assume also that the variance $V_f(\theta):= \text{Var}_j(f_j(\theta)$ is finite and upper bounded by a constant $V$. Then, it holds
    
(i) If $$K_{1,j} \|Q\|_F^2 + K_{2,j} \|P\|_F \|Q\|_F + K_{4,j} \|P\|_F^2 \leq \sqrt{2} K_{5,j} \|P\|_F \sqrt{f(P,Q)} + \sqrt{2} K_{3,j} \sqrt{f(P,Q)},$$ for all $j$, where $K_{1,j},K_{2,j},K_{3,j},K_{4,j},K_{5,j}$ are the analogues of the constants $K_1,K_2,K_3,K_4,K_5$ from Eq.~\eqref{eq:transformers_hessian_bound} substituting $Z$ with the task-dependent data matrix $Z_j$,
then
\begin{equation*}
    \|\nabla^2 f(P,Q)\|_2 \leq 
    H_0^{\text{warm-up}}+H_1^{\text{warm-up}} (f(P,Q)-f^*)),
\end{equation*}
with constants $H_0^{\text{warm-up}}$ and $H_1^{\text{warm-up}}$ defined as in equation~\eqref{eq:transformers_in_context_warmup_constants}.

(ii) For any $(P,Q)$, we have
\begin{equation*}
    \norm{\nabla^2 f(P,Q)}_2 \le H_0 + H_1 (f(P,Q) - f^*).
\end{equation*}
The constants $H_0$ and $H_1$ are defined as in equation~\eqref{eq:transformers_in_context_global_constants} and explicitly incorporate the loss variance $\sqrt{V_f(\theta)}$.

\end{restatable}

\begin{proof}
For simplicity, we denote $\theta \equiv (P,Q)$.

We start with Jensen's Inequality for the spectral norm:
\begin{equation*}
    \norm{\nabla^2 f(\theta)}_2 = \norm{\EE_j [\nabla^2 f_j(\theta)]}_2 \le \EE_j \left[ \norm{\nabla^2 f_j(\theta)}_2 \right].
\end{equation*}

If it holds $K_{1,j} \|q\|_2^2 + K_{2,j} \|p\|_2 \|q\|_2 + K_{4,j} \|p\|_2^2 \leq K_{5,j} C_r \|p\|_2 \sqrt{\bar{f}} +
    K_{3,j} C_r \sqrt{\bar{f}}$ for all tasks $j$, the single-task bound from Proposition \ref{prop:transformer_one_task} applies pointwise:
\begin{equation}
    \norm{\nabla^2 f_j(\theta)}_2 \le C_{0,j}^{\text{warm-up}} + C_{1,j}^{\text{warm-up}} \sqrt{f_j(\theta)}.
\end{equation}
for some $C_{0,j}^{\text{warm-up}}$ and $C_{1,j}^{\text{warm-up}}$ random variables depending on the data of task $j$.
Taking the expectation:
\begin{equation*}
    \EE_j \left[ \norm{\nabla^2 f_j(\theta)}_2 \right] \le \EE_j [C_{0,j}^{\text{warm-up}}] + \EE_j \left[ C_{1,j}^{\text{warm-up}} \sqrt{f_j(\theta)} \right].
\end{equation*}
We define the mean constant $\bar{C}_0 = \EE_j [C_{0,j}^{\text{warm-up}}]$.
For the second term, we apply the Cauchy-Schwarz inequality for expectations:
\begin{equation*}
    \EE_j \left[ C_{1,j}^{\text{warm-up}} \sqrt{f_j(\theta)} \right] \le \sqrt{\EE_j [(C_{1,j}^{\text{warm-up}})^2]} \cdot \sqrt{\EE_j [(\sqrt{f_j(\theta)})^2]}.
\end{equation*}
Note that $\EE_j [(\sqrt{f_j})^2] = \EE_j [f_j] = f(\theta)$.
Let $\bar{C}_1^{\mathrm{rms}} = \sqrt{\EE_j [(C_{1,j}^{\text{warm-up}})^2]}$. These moments are finite because $C_{1,j}^{\text{warm-up}}$ is polynomial in $\|Z_j\|_2$, which has sub-exponential tails.
Thus, we have:
\begin{equation*}
    \norm{\nabla^2 f(\theta)}_2 \le \bar{C}_0 + \bar{C}_1^{\mathrm{rms}} \sqrt{f(\theta)}.
\end{equation*}
To obtain the affine form, we use the linearization inequality $\sqrt{u} \le \frac{u-u^*}{2\beta} + \left(\frac{u^*}{2\beta} + \frac{\beta}{2}\right)$ with $u=f(\theta)$:
\begin{equation*}
    \norm{\nabla^2 f(\theta)}_2 \le \bar{C}_0 + \bar{C}_1^{\mathrm{rms}} \left[ \frac{f(\theta)-f^*}{2\beta} + \frac{f^*}{2\beta} + \frac{\beta}{2} \right].
\end{equation*}
Collecting terms and setting $\beta=1$ yields the final constants:
\begin{equation}
\label{eq:transformers_in_context_warmup_constants}
H_0^{\text{warm-up}} = \bar{C}_0 + \bar{C}_1^{\mathrm{rms}} \left( \frac{f^*}{2} + \frac{1}{2} \right), \quad
    H_1^{\text{warm-up}} = \frac{\bar{C}_1^{\mathrm{rms}}}{2}.
\end{equation}

For the global bound, the single-task result holds unconditionally:
\begin{equation*}
    \norm{\nabla^2 f_j(\theta)}_2 \le H_{0,j} + H_{1,j} f_j(\theta).
\end{equation*}
Taking the expectation, we get
\begin{equation*}
    \EE_j \left[ \norm{\nabla^2 f_j(\theta)}_2 \right] \le \EE_j [H_{0,j}] + \EE_j \left[ H_{1,j} f_j(\theta) \right].
\end{equation*}
Let $\bar{H}_0 = \EE_j [H_{0,j}]$. For the second term, we apply Cauchy-Schwarz:
\begin{equation*}
    \EE_j \left[ H_{1,j}f_j(\theta) \right] \le \sqrt{\EE_j [(H_{1,j})^2]} \cdot \sqrt{\EE_j [f_j(\theta)^2]}.
\end{equation*}
Using the definition of variance $\EE [X^2] = (\EE[ X])^2 + \text{Var}(X)$, we have $\EE [f_j^2] = f(\theta)^2 + V_f(\theta)$.
Let $\bar{H}_1^{\mathrm{rms}} = \sqrt{\EE_j [(H_{1,j})^2]}$. Using $\sqrt{a+b} \le \sqrt{a} + \sqrt{b}$, we have
\begin{equation*}
    \sqrt{f(\theta)^2 + V_f(\theta)} \le f(\theta) + \sqrt{V_f(\theta)}.
\end{equation*}
Thus:
\begin{equation*}
    \norm{\nabla^2 f(\theta)}_2 \le \bar{H}_0 + \bar{H}_1^{\mathrm{rms}} \left( f(\theta) + \sqrt{V_f(\theta)} \right).
\end{equation*}
We write $f(\theta) = (f(\theta)-f^*) + f^*$. Absorbing the variance term into the constant and upper bounding it by $V$:
\begin{equation}
\label{eq:transformers_in_context_global_constants}
H_0 = \bar H_0 + \bar{H}_1^{\mathrm{rms}} \left( f^* + \sqrt{V} \right), \quad
    H_1 = \bar{H}_1^{\mathrm{rms}}.     
\end{equation}
The expressions $\bar{H}_0$ and $\bar{H}_1^{\mathrm{rms}}$ are finite by Lemma \ref{le:finite moments}.
\end{proof}

\section{Neural Networks are in general not $(L_0, L_1)$-smooth}\label{app:l0l1_failure}

In this section, we demonstrate that neural networks still violate the $(L_0, L_1)$-smoothness, even in the presence of L2 regularization or weight balancedness. We start with an example of a simple 2-layer neural network with L2 regularization when $(L_0, L_1)$-smoothness is violated for $L_0, L_1 \ge 0.$

\begin{proposition}
\label{prop:counterexample1}
    We consider a simple 2-layer neural network with MSE loss
    \[
    f(u,v) = \frac{1}{2}(u\sigma(v))^2 + \frac{\lambda_1}{2}u^2 + \frac{\lambda_2}{2}v^2,
    \]
    such that $\sigma(0)=0,\sigma^{\prime}(0)\neq 0$\footnote{These assumptions are satisfied for several activation functions such as \algname{tanh}, \algname{GELU}, \algname{SiLU}.}. Then $(L_0, L_1)$-smoothness does not hold for any $L_0, L_1 \ge 0.$
\end{proposition}
\begin{proof}
    For this example, the gradient and the Hessian are
    \[
    \nabla f(u,v) = \begin{bmatrix}
        u\sigma^2(v) + \lambda_1u \\
        u^2\sigma(v)\sigma^{\prime}(v) + \lambda_2v
    \end{bmatrix}, \; 
    \nabla^2 f(u,v) = \begin{bmatrix}
        \sigma^2(v) +\lambda_1 & 2u\sigma(v)\sigma^{\prime}(v) \\
        2u\sigma(v)\sigma^{\prime}(v) & u^2((\sigma^{\prime}(v))^2 + \sigma(v)\sigma^{\prime\prime}(v) ) + \lambda_2.
    \end{bmatrix}
    \]
    Let us evaluate them at the point $(u,0)$. Note that $\sigma(0)=0, \sigma^{\prime}(0)\neq 0$ by the assumption of the proposition. We obtain
    \[
    \nabla f(u,v) = \begin{bmatrix}
        \lambda_1 u\\
        0
    \end{bmatrix}, \; 
    \nabla^2 f(u,v) = \begin{bmatrix}
        \lambda_1 & 0 \\
        0 & u^2(\sigma^{\prime}(0))^2 + \lambda_2.
    \end{bmatrix}
    \]
    Therefore, we obtain that
    \[
    \|\nabla^2f(u,0)\|_2 = \max\{\lambda_1, u^2(\sigma^{\prime}(0))^2 + \lambda_2\}, \quad \|\nabla f(u,0)\|= \lambda_1|u|.
    \]
    Thus, if $(L_0, L_1)$-smoothness was true for this function, then there were constants $L_0, L_1 \ge 0$ such that 
    \begin{equation}\label{eq:lqjndojqwndqowjdnq}
    \|\nabla^2f(u,0)\|_2 =  \max\{\lambda_1, u^2(\sigma^{\prime}(0))^2 + \lambda_2\} \le L_0 + L_1\lambda_1|u|.
    \end{equation}
    Let $u \ge \frac{\sqrt{\lambda_1}}{|\sigma^{\prime}(0)|}$. Then $\|\nabla^2f(u,0)\|_2 = u^2(\sigma^{\prime}(0))^2 + \lambda_2$. Therefore, dividing both sides of \eqref{eq:lqjndojqwndqowjdnq} by $u$ we obtain
    \[
    u(\sigma^{\prime}(0))^2 \le \frac{L_0}{u} + L_1\lambda_1.
    \]
    Taking $u\to +\infty,$ we get that LHS goes to $+\infty$ while RHS goes to a constant. Therefore, $(L_0, L_1)$-smoothness is violated for any $L_0, L_1\ge0.$
\end{proof}

Next, we demonstrate that $(L_0, L_1)$-smoothness is violated under a balancedness condition as well.

\begin{proposition}
\label{prop:counterexample2}
    We consider a $2$-layer neural network with MSE loss
    \begin{equation*}
        f(W_1,W_2) = \|Y- W_1 \phi(W_2 X)\|_F^2
    \end{equation*}
    and leaky-ReLU or linear activation function, i.e. $\phi(x)=\max \lbrace x, bx \rbrace$, with $0<b\leq 1$. Then, $(L_0,L_1)$-smoothness does not hold under weak balancedness for any $L_0,L_1 \geq 0$.
\end{proposition}

\begin{proof}
We start by computing the gradients and Hessian of $f$.

For all computations, we work with a vectorized version of $f$:
\begin{equation*}
    f(W_1,W_2)=\|\textnormal{vec}(Y)-\textnormal{vec}(W_1 \phi(W_2 X))\|_{\rm F}^2.
\end{equation*}

Let us denote $$R := \textnormal{vec}(Y)-\textnormal{vec}(W_1 \phi(W_2 X)) = \textnormal{vec}(Y) - (\phi(W_2 X)^\top \otimes I_c) \textnormal{vec}(W_1). $$

For the second inequality, we used a classic property between vectorization and the Kronecker product.

The derivative with respect to $\textnormal{vec}(W_1)$ is 
\begin{align*}
    &\frac{\partial f}{\partial \textnormal{vec}(W_1)} =  
  \frac{\partial f}{\partial R} \cdot \frac{\partial R}{\partial \textnormal{vec}(W_1) }  = -2 R^\top (\phi(W_2 X)^\top \otimes I_c) = -2 R^\top (\phi(W_2 X)^\top \otimes I_c).
\end{align*}

Transposing in order to bring the vector in column form, we get
\begin{equation}
\label{eq:fisrt_grad_of_2layer}
    \frac{\partial f}{\partial \textnormal{vec}(W_1)} = 2 (\phi(W_2 X) \otimes I_c) R = - 2 \textnormal{vec}((Y-W_1 \phi(W_2 X)) \phi(W_2 X)^\top).
\end{equation}

The gradient with respect to $W_2$ is similarly
\begin{align*}
   &\frac{\partial f}{\partial \textnormal{vec}(W_2)} = \frac{\partial f}{\partial R} \cdot \frac{\partial R}{\partial \textnormal{vec}(W_2) }. 
\end{align*}
$\frac{\partial f}{\partial R}$ is again $2 R^\top$. In order to deal with $\frac{\partial R}{\partial \textnormal{vec}(W_2) }$, we write
\begin{equation*}
    R = \textnormal{vec}(Y) - (I_m \otimes W_1) \textnormal{vec}(\phi(W_2 X)).
\end{equation*}

Thus,
\begin{align*}
    \frac{\partial R}{\partial \textnormal{vec}(W_2)} = -(I_m \otimes W_1) \frac{\partial \textnormal{vec}(\phi(W_2 X))}{\partial \textnormal{vec}(W_2)} = -(I_m \otimes W_1) \frac{\partial \textnormal{vec}(\phi(W_2 X))}{\partial \textnormal{vec}(W_2 X)} \frac{\partial \textnormal{vec}(W_2 X)}{\partial \textnormal{vec}(W_2)}
\end{align*}

$\frac{\partial \textnormal{vec}(\phi(W_2 X))}{\partial \textnormal{vec}(W_2)}$ is the diagonal matrix $\textnormal{diag}(\textnormal{vec}(\phi'(W_2 X))$.

Since $\textnormal{vec}(W_2 X) = (X^\top \otimes I_{n_1}) \textnormal{vec}(W_2)$, the gradient $\frac{\partial \textnormal{vec}(W_2 X)}{\partial \textnormal{vec}(W_2)}$ is 
\begin{equation*}
    \frac{\partial \textnormal{vec}(W_2 X)}{\partial \textnormal{vec}(W_2)} = X^\top \otimes I_{n_1}.
\end{equation*}

Putting it all together, we have
\begin{equation}
\label{eq:second_grad_2layer}
   \frac{\partial f}{\partial \textnormal{vec}(W_2)} = -2 R^\top (I_m \otimes W_1) \textnormal{diag}(\textnormal{vec}(\phi'(W_2 X)) (X^\top \otimes I_{n_1}).
\end{equation}

Writing that again as column vector yields
\begin{equation*}
    -2 (X \otimes I_{n_1}) \textnormal{diag}(\textnormal{vec}(\phi'(W_2 X))) (I_m \otimes W_1^\top) R.
\end{equation*}

After some modifications, we can write
\begin{align*}
    &\textnormal{diag}(\textnormal{vec}(\phi'(W_2 X))) (I_m \otimes W_1^\top) R = \\ & \textnormal{diag}(\textnormal{vec}(\phi'(W_2 X))) \textnormal{vec}(W_1^\top (Y-W_1 \phi(W_2 X))))  = \\ & \textnormal{vec}((W_1^\top (Y-W_1 \phi(W_2 X))\odot \phi'(W_2 X)).
\end{align*}
where $\odot$ is the Hadamard product.

This means that we can write the previous gradient as
\begin{equation*}
    -2 \textnormal{vec}(((W_1^\top (Y-W_1 \phi(W_2 X))) \odot \phi'(W_2 X) ) X^\top).
\end{equation*}

For the first block of the Hessian, we differentiate $\frac{\partial f}{\partial \textnormal{vec}(W_1)}$ with respect to $\partial \textnormal{vec}(W_1)^\top$. Since $$\frac{\partial f}{\partial \textnormal{vec}(W_1)} = -2 \textnormal{vec}((Y-W_1 \phi(W_2 X)) \phi(W_2 X)^\top) = -2 (\phi(W_2 X) \otimes I_c) \textnormal{vec}(Y-W_1 \phi(W_2 X)) , $$ we have
\begin{align*}
    \frac{\partial^2 f}{\partial \textnormal{vec}(W_1) \textnormal{vec}(W_1)^\top } &= -2 (\phi(W_2 X) \otimes I_c) \frac{\partial \textnormal{vec}(Y-W_1 \phi(W_2 X))}{\partial \textnormal{vec}(W_1)^\top} \\ & = 2(\phi(W_2 X) \otimes I_c)  (\phi(W_2 X)^\top \otimes I_c) \frac{\partial \textnormal{vec}(W_1)}{\partial \textnormal{vec}(W_1)^\top} \\ & = 2 (\phi(W_2 X) \phi(W_2 X)^\top \otimes I_c).
\end{align*}

For the off-diagonal blocks, it suffices to compute only one, as they are symmetric to each other. We use the product rule (see \cite{magnus1985matrix}, Theorem 9)
\begin{equation*}
    \frac{\partial \textnormal{vec}(A(W) B (W))}{\partial \textnormal{vec}(W)^\top} = (B(W)^\top \otimes I) \frac{\partial \textnormal{vec}(A(W))}{\partial \textnormal{vec}(W)^\top} + (I \otimes A(W)) \frac{\partial \textnormal{vec}(B(W))}{\partial \textnormal{vec}(W)^\top}.
\end{equation*}
We have
\begin{align*}
    \frac{\partial}{\partial \textnormal{vec}(W_2)^\top} \frac{\partial f}{\partial \textnormal{vec}(W_1)} 
& = -2 (\phi(W_2 X) \otimes I_c) \frac{\partial \textnormal{vec} (Y-W_1\phi(W_2 X)) }{\partial \textnormal{vec}(W_2)^\top} \\& - 2 (I_{n_1} \otimes (Y-W_1\phi(W_2 X)) ) \frac{\partial \textnormal{vec} (\phi(W_2 X)^\top) }{\partial \textnormal{vec}(W_2)^\top}. 
\end{align*}

In order to proceed, we need to write $\textnormal{vec}(\phi(W_2 X)^\top)$ in terms of $\textnormal{vec}(\phi(W_2 X))$, and this can be done formally using the so-called commutation matrix:
\begin{equation*}
   \textnormal{vec}(\phi(W_2 X)^\top) = K_{n_1 m}  \textnormal{vec}(\phi(W_2 X)).
\end{equation*}

For the first partial derivative in the sum, we have
\begin{align*}
  \frac{\partial \textnormal{vec} (Y-W_1\phi(W_2 X)) }{\partial \textnormal{vec}(W_2)^\top} & = -\frac{\partial \textnormal{vec} (W_1\phi(W_2 X)) }{\partial \textnormal{vec}(W_2)^\top}  = -(I_m \otimes W_1)\frac{\partial \textnormal{vec} (\phi(W_2 X)) }{\partial \textnormal{vec}(W_2)^\top} \\ & = -(I_m \otimes W_1) \textnormal{diag}(\textnormal{vec}(\phi'(W_2 X))) \frac{\partial \textnormal{vec}(W_2 X)}{\partial \textnormal{vec}(W_2)^\top} \\ & = -(I_m \otimes W_1) \textnormal{diag}(\textnormal{vec}(\phi'(W_2 X))) (X^\top \otimes I_{n_1}).
\end{align*}

As it is evident in the previous calculation
\begin{equation*}
    \frac{\partial \textnormal{vec} (\phi(W_2 X)) }{\partial \textnormal{vec}(W_2)^\top} = \textnormal{diag}(\textnormal{vec}(\phi'(W_2 X))) (X^\top \otimes I_{n_1}).
\end{equation*}

Putting it all together, we get
\begin{align*}
    &\frac{\partial^2 f}{\partial \textnormal{vec}(W_1) \textnormal{vec}(W_2)^\top } = \\& 2 (\phi(W_2 X) \otimes W_1) \textnormal{diag}(\textnormal{vec}(\phi'(W_2 X))) (X^\top \otimes I_{n_1}) \\ - &2 (I_{n_1} \otimes (Y-W_1\phi(W_2 X))) K_{n_1 m} \textnormal{diag}(\textnormal{vec}(\phi'(W_2 X))) (X^\top \otimes I_{n_1}) = \\ &
    2 (\phi(W_2 X) \otimes W_1 +   (I_{n_1} \otimes (W_1\phi(W_2 X)-Y))K_{n_1 m}) \textnormal{diag}(\textnormal{vec}(\phi'(W_2 X))) (X^\top \otimes I_{n_1}).
\end{align*}

We also have
\begin{equation*}
    \frac{\partial^2 f}{\partial \textnormal{vec}(W_2) \textnormal{vec}(W_1)^\top } =\left(\frac{\partial^2 f}{\partial \textnormal{vec}(W_1) \textnormal{vec}(W_2)^\top }\right)^\top.
\end{equation*}

For the second derivative of $L$ with respect to $W_2$, we remind that
\begin{equation*}
    \frac{\partial f}{\partial \textnormal{vec}(W_2)} = -2 (X \otimes I_{n_1}) \textnormal{diag}(\textnormal{vec}(\phi'(W_2 X))) (I_m \otimes W_1^\top) R.
\end{equation*}

Differentiating that with respect to $\textnormal{vec}(W_2)^\top$ involves a product rule, as $W_2$ appears in $\textnormal{diag}(\textnormal{vec}(\phi'(W_2 X)))$ and in $R$. It is more convenient to bring $\frac{\partial f}{\partial \textnormal{vec}(W_2)}$ back in fully vectorized form as:
\begin{equation*}
   \frac{\partial f}{\partial \textnormal{vec}(W_2)} =
    -2 \textnormal{vec}(((W_1^\top (Y-W_1 \phi(W_2 X))) \odot \phi'(W_2 X) ) X^\top). 
\end{equation*}

We have

\begin{align*}
    &-2 \frac{\partial \textnormal{vec}(((W_1^\top (Y-W_1 \phi(W_2 X))) \odot \phi'(W_2 X) ) X^\top)}{\partial \text{vec}(W_2)^\top}  = \\ &  -2 (X \otimes I_{n_1}) \left( \frac{\partial\textnormal{vec}(W_1^\top(Y-W_1 \phi(W_2 X)) \odot \phi'(W_2 X))}{\partial \text{vec}(W_2)^\top} \right).
\end{align*}

Now we can use the product rule for the Hadamard product, see \cite{magnus1985matrix} (Theorem 10):

\begin{align*}
    &\frac{\partial\textnormal{vec}((W_1^\top(Y-W_1 \phi(W_2 X))) \odot \phi'(W_2 X))}{\partial \text{vec}(W_2)^\top} = \\ & \textnormal{diag}(\text{vec}(\phi'(W_2 X)) \frac{\partial \text{vec}(W_1^\top(Y-W_1 \phi(W_2 X)))}{\partial \text{vec}(W_2)^\top} + \textnormal{diag}(\textnormal{vec}(W_1^\top(Y-W_1 \phi(W_2 X)))) \frac{\partial \phi'(W_2 X)}{\partial \text{vec}(W_2)^\top}.
\end{align*}

For the first term of the last sum, we have by previous calculations that
\begin{equation*}
    \frac{\partial \text{vec}(W_1^\top(Y-W_1 \phi(W_2 X)))}{\partial \text{vec}(W_2)^\top} = -(I_m \otimes 
 W_1^\top W_1) \textnormal{diag}(\textnormal{vec}(\phi'(W_2 X))) (X^\top \otimes I_{n_1}).
\end{equation*}

For the second term of the last sum, we have
\begin{equation*}
    \frac{\partial \phi'(W_2 X)}{\partial \text{vec}(W_2)^\top} = \textnormal{diag}(\text{vec}(\phi''(W_2 X))) (X^\top \otimes I_{n_1}).
\end{equation*}

In total, we have
\begin{align}
\label{eq:Hessian_2layer_22block}
    &\frac{\partial^2 f}{\partial \nonumber \textnormal{vec}(W_2) \text{vec}(W_2)^\top} =  \\& 2 (X \otimes I_{n_1}) \textnormal{diag}(\textnormal{vec}(\phi'(W_2 X))) \nonumber(I_m \otimes 
 W_1^\top W_1) \textnormal{diag}(\textnormal{vec}(\phi'(W_2 X))) (X^\top \otimes I_{n_1}) \\ - & 2 (X \otimes I_{n_1}) \textnormal{diag}(\textnormal{vec}(W_1^\top(Y-W_1 \phi(W_2 X))))\textnormal{diag}(\text{vec}(\phi''(W_2 X))) (X^\top \otimes I_{n_1}).
\end{align}

This completes the calculation of gradients and Hessian of $f$.

Now, we can simply consider $X = \begin{bmatrix}
        1 & 0 & 0 \\ 0 & 1 & 0 \\ 0 & 0 & 1
    \end{bmatrix}$
and
$Y = \begin{bmatrix}
        1 & 0 & 0 \\ 0 & 2 & 0 \\ 0 & 0 & 3
    \end{bmatrix}$.
Take also $W_1 = \begin{bmatrix}
        t & 0 \\ 0 & 0 \\ 0 & 0
    \end{bmatrix}$
    and 
    $W_2 = \begin{bmatrix}
       \frac{1}{t} & 0 & 0 \\ \sqrt{t^2 - 1/t^2} & 0 & 0   
    \end{bmatrix},$
for $t > 1$ (notice that the entries of $W_2$ are positive, thus it is not affected by leaky-ReLU). It holds $\|W_1\|_F = t = \|W_2\|_F$, thus we indeed satisfy the weak balancedness condition. It also holds
\begin{equation*}
 Y- W_1 W_2 X = \begin{bmatrix}
     0 & 0 & 0 \\ 0 & 2 & 0 \\ 0 & 0 & 3 
 \end{bmatrix}.   
\end{equation*}
We can use that to compute
\begin{align*}
    W_1^T (Y - W_1 W_2 X) = \begin{bmatrix}
       \frac{1}{t} & 0 & 0 \\ 0 & 0 & 0   
    \end{bmatrix} \begin{bmatrix}
     0 & 0 & 0 \\ 0 & 2 & 0 \\ 0 & 0 & 3 
 \end{bmatrix} = \begin{bmatrix}
     0 & 0 & 0 \\ 0 & 0 & 0
 \end{bmatrix}
\end{align*}
and 
\begin{align*}
    (Y - W_1 W_2 X) X^T W_2^T =
        \begin{bmatrix}
     0 & 0 & 0 \\ 0 & 2 & 0 \\ 0 & 0 & 3 
 \end{bmatrix} \begin{bmatrix}
       \frac{1}{t} & \sqrt{t^2 - 1/t^2} \\ 0  & 0 \\ 0 & 0   
    \end{bmatrix} = \begin{bmatrix}
     0 & 0  \\ 0 & 0  \\ 0 & 0
 \end{bmatrix}. 
 \end{align*}
 It holds $\nabla_{W_1} f(W_1,W_2) = (Y - W_1 W_2 X) X^T W_2^T$ and $\nabla_{W_2} f (W_1, W_2) = W_1^T (Y - W_1 W_2 X)$, thus we have $\|\nabla f\| = 0$, while the Frobenius norm of the Hessian (thus also its spectral norm) goes to infinity at $t$ goes to infinity by Eq.~\eqref{eq:Hessian_2layer_22block}, since
 \begin{equation*}
     W_1^T  W_1 = \begin{bmatrix}
         t^2 & 0 \\ 0 & 0
     \end{bmatrix}.
 \end{equation*}
\end{proof}

\section{Useful Lemmas}
\label{sec:useful_lemmas}

To obtain convergence results, we need to bound the smoothness between two arbitrary points $w,y\in\R^d$. This can be done using $(H_0,H_1)$-smoothness. If we parametrize $w(t) \eqdef x + t(y-w)$, then from the new smoothness assumption, the smoothness constant along the segment $[x,w(t)]$ can be bounded by a function of $\chi(t) \eqdef \int_0^t (H_0 + H_1(f(w(\theta))-f^*))\diff \theta$; see the next derivations for detailed proof. In particular, the bound on the smoothness constant along $[x,y]$ is related to $\chi(1).$ Our proof techniques are inspired by the results in \citet{li2023convex}. However, due to a more general smoothness inequality, our derivations involve a more careful analysis, because we need to deal with additional terms that do not appear in the case of $(L_0, L_1)$-smoothness. This highlights the difficulty of obtaining convergence guarantees, in particular, a gradient upper bound (\Cref{lem:gradient_bound}) and quadratic upper bound (\Cref{lem:quadratic_bound}).

We start with a restatement of Lemma A.3 in \citet{li2023convex}, which can be seen as a generalization of Gr{\"o}nwall’s inequality.

The next lemma provides sufficient conditions when the smoothness constant along the segment $[x,y]$ can be bounded by some constant, which uses information at $w$ only. The concurrent work \citet{liu2025theoretical} provides similar derivations based on \citet{li2023convex}.

\begin{lemma}[Lemma A.3 in \cite{li2023convex}]\label{lem:lemmaA3} Let $\alpha\colon [a,b] \to [0, \infty)$ and $\beta\colon [0,\infty) \to (0,\infty)$ be two continuous functions. Suppose $\alpha^\prime(t) \le \beta(\alpha(t))$ almost everywhere over $(a,b).$  We have for all $K\in[a,b]$,
\[
\int_{\alpha(a)}^{\alpha(t)} \frac{1}{\beta(u)}\diff u \le t-a.
\]
    
\end{lemma}

\begin{lemma}\label{lem:c1c2_bound} Let $f$ be $(H_0,H_1)$-smooth, $\chi(t) \eqdef \int_0^t (H_0 + H_1(f(w(\theta))-f^*))\diff \theta$, and define $c_1 \coloneqq H_0 + H_1(c_2 + f(w) - f^*)$ for some $c_2 > 0$. Then it holds 
    \[
    \|\nabla f(w)\|\cdot \|y-w\| + \frac{c_1}{2}\|y-w\|^2 \le c_2 \implies \chi(1) \leq c_1.
    \]

\end{lemma}
\begin{proof}
    Let $w(t) = x + t(y-w)$. Then, by Taylor's theorem for a gradient, we have
    \begin{align}
        \|\nabla f(w(t)) - \nabla f(w)\| &= \left\|\int_0^t \nabla^2f(w(\theta))(w(t)-w) \diff \theta\right\|\notag\\
        &\overset{(i)}{\le} \int_0^t \|\nabla^2f(w(\theta))\|\diff\theta\cdot \|w(t)-w\|\notag\\
        &\overset{(ii)}{\le} \|w(t)-w\|\int_0^t (H_0 + H_1(f(w(\theta)) - f^*))\diff \theta = \|w(t)-w\|\chi(t),\label{eq:noqjn1oub}
    \end{align}
    where $(i)$ follows from the definition of a spectral norm, $(ii)$ --- from \Cref{asmp:H0H1_smoothness}, and $\chi(t) \eqdef \int_0^t (H_0 + H_1(f(w(\theta))-f^*))\diff \theta$.  Eventually, we want to bound $\|\nabla f(y) - \nabla f(w)\|$ for any $y=w(1)$ and $w$. Therefore, our goal is to bound $\chi(1)$. Differentiating $\chi(t)$ we have 
    \begin{align}\label{eq:pqkwnqlk}
        \chi^\prime(t) &= (H_0 + H_1(f(w(t)) - f^*))\notag\\
        &= H_0 + H_1(f(w) - f^*) + H_1(f(w(t)) - f(w)).
    \end{align}
    Now we need to bound the difference $f(w(t)) - f(w)$ using Taylor's theorem for the function 
    \begin{align}
        f(w(t)) - f(w) &= \int_0^t \langle\nabla f(w(\theta)), w(t)- x\rangle \diff \theta\\
        &\overset{(iii)}{\le} \langle \nabla f(w), w(t)-w\rangle + \|w(t)-w\|\int_0^t \|\nabla f(w(\theta)) - \nabla f(w)\|\diff \theta\notag\\
        &\overset{(iv)}{\le}\langle \nabla f(w), w(t)-w\rangle + \|w(t)-w\|\int_0^t\chi(\theta)\|w(\theta)-w\| \diff \theta\notag\\
        &\overset{(v)}{\le}\langle \nabla f(w), w(t)-w\rangle + \|w(t)-w\|\chi(t)\int_0^t\theta\|y-w\| \diff \theta\notag\\
        &\overset{(vi)}{\le} \langle \nabla f(w), w(t)-w\rangle + \frac{1}{2}\|y-w\|^2\chi(t),\label{eq:o3nwjdn}
    \end{align}
   where $(iii)$ follows from Cauchy-Shawrz inequality, $(iv)$ --- from \eqref{eq:noqjn1oub}, $(v)$ --- from monotonicity of $\chi(t)$ by definition, $(vi)$ --- from the fact that $t\in[0,1].$ Using Cauchy-Schwarz inequality again, we obtain
    \begin{align}\label{eq:o3nwjdnffff}
        f(w(t)) - f(w) \le \|\nabla f(w)\|\cdot \|w(t)-w\| + \frac{1}{2}\|y-w\|^2\chi(t) &\le \|\nabla f(w)\|\cdot \|y-w\|\notag\\
        &\hspace{1cm}+\; \frac{1}{2}\|y-w\|^2\chi(t),
    \end{align}
    Therefore, we obtain from \eqref{eq:pqkwnqlk} and \eqref{eq:o3nwjdnffff} that
    \begin{align*}
        \chi^\prime(t) &\le H_0 + H_1(f(w) - f^*) + H_1\|\nabla f(w)\|\cdot \|y-w\| + \frac{1}{2}H_1\|y-w\|^2\chi(t).
    \end{align*}
    Now we use \Cref{lem:lemmaA3} with $\alpha(t)=\chi(t), \beta(t) \coloneqq H_0 + H_1(f(w) - f^*) + H_1\|\nabla f(w)\|\cdot \|y-w\| + \frac{t}{2}H_1\|y-w\|^2, a=0, b=1$. We obtain
    \[
    \int_0^{\chi(1)}\frac{1}{\beta(u)}\diff u \le 1.
    \]
    Note that since $\beta(t)$ is monotonically increasing in $t$, then for any $c_1 > 0$ such that $\beta(c_1) \le c_1$ we have 
    \[
    \int_0^{\chi(1)} \frac{1}{\beta(u)}\diff u \le 1 \le \frac{c_1}{\beta(c_1)} \le \int_0^{c_1} \frac{1}{\beta(u)}\diff u.
    \]
    This implies that $\chi(1) \le c_1$ by monotonicity of the integral. Note that $\beta(c_1) \le c_1$ can be rewritten as 
    \begin{align*}
        \beta(c_1) &\le c_1\\
        H_0 + H_1(f(w) - f^*) + H_1\|\nabla f(w)\|\cdot \|y-w\| + \frac{c_1}{2}H_1\|y-w\|^2 &\le H_0 + H_1(c_2 + f(w) - f^*)\\
        \|\nabla f(w)\|\cdot \|y-w\| + \frac{c_1}{2}\|y-w\|^2 &\le c_2.
    \end{align*}
    Therefore, $\chi(1) \le c_1$ if $\beta(c_1) \le c_1$ which is equivalent to
    \[
    \|\nabla f(w)\|\cdot \|y-w\| + \frac{c_1}{2}\|y-w\|^2 \le c_2.
    \]
    
\end{proof}

Our next lemma provides a gradient bound which can be seen as a generalization of the classic $\|\nabla f(w)\|^2 \le 2L(f(w)-f^*)$ inequality, which holds in the $L$-smooth regime. The proof of \Cref{lem:gradient_bound} requires a careful choice of constants $c_1$ and $c_2$ from \Cref{lem:c1c2_bound}.

\begin{lemma}\label{lem:gradient_bound} Let $f$ be $(H_0,H_1)$-smooth. Then we have 
\[
\|\nabla f(w)\|^2 \le \frac{9}{4}(H_0 + 3H_1(f(w) - f^*))(f(w) - f^*).
\]
\end{lemma}
\begin{proof}
    Let $c_2 := 2(f(w) - f^*),$ then $c_1 = H_0 + H_1(c_2 + f(w) - f^*) = H_0 + 3H_1(f(w) - f^*).$ From \Cref{lem:c1c2_bound} we obtain that $\chi(1) \le c_1$ if 
    \begin{align*}
       & \|\nabla f(w)\|\cdot\|y-w\| + \frac{c_1}{2}\|y-w\|^2 \le 2(f(w) - f^*) \Longleftrightarrow\; \\ & \|\nabla f(w)\|\cdot \|y-w\| + \frac{1}{2}(H_0 + 3H_1(f(w) - f^*))\|y-w\|^2 \le 2(f(w) - f^*).\\ 
    \end{align*}
    This is a quadratic polynomial of $\|y-w\|$ which has one negative and one positive solution. Therefore, to satisfy the constraint above, we should have 
    \begin{align*}
        \|y-w\| \le \frac{-\|\nabla f(w)\| + \sqrt{\|\nabla f(w)\|^2 + 4c_1(f(w) - f^*)}}{c_1} \eqcolon r.
    \end{align*}
    In other words, we have $\chi(1) \le c_1$ if points $y$ and $w$ satisfy the constraint above: $\|y-w\| \le r.$ 

    We choose $y:=w-\eta \frac{\nabla f(w)}{\|\nabla f(w)\|}$, with $\eta \leq r$. Then, we have $\|y-w\| \leq r$, thus it holds $\chi(1) \leq c_1$.
    
    From \eqref{eq:o3nwjdn}, the bound $\chi(1) \le c_1$ implies 
    \[
    f(y) - f(w) \le \langle\nabla f(w), y-w\rangle + \frac{c_1}{2}\|y-w\|^2.
    \]
    Since $f(y) \ge f^*$ we can continue the inequality above as follows
     \[
    f^* - f(w) \le \langle\nabla f(w), y-w\rangle + \frac{c_1}{2}\|y-w\|^2.
    \]
    Substituting $y=w-\eta \frac{\nabla f(w)}{\|\nabla f(w)\|}$, we get
    \begin{align}\label{eq:kfnqljlonwd}
    -(f(w) - f^*) \le -\eta\|\nabla f(w)\| + \frac{c_1}{2}\eta^2 \Longleftrightarrow \frac{c_1}{2}\eta^2 - \eta\|\nabla f(w)\| + (f(w) -f^*) \ge 0.
    \end{align}
 \eqref{eq:kfnqljlonwd} is a quadratic polynomial of $\eta$, which attains the minimum at $\eta = \frac{\|\nabla f(w)\|}{c_1}.$ We consider two cases: 
    \begin{enumerate}
        \item $\frac{\|\nabla f(w)\|}{c_1} \le r:$ Plugging this bound into \eqref{eq:kfnqljlonwd} implies
        \[
        \frac{c_1}{2}\frac{\|\nabla f(w)\|^2}{c_1^2} - \|\nabla f(w)\| \frac{\|\nabla f(w)\|}{c_1} + (f(w) - f^*) \ge 0 \Leftrightarrow \|\nabla f(w)\|^2 \le 2c_1(f(w) - f^*). 
        \]

    \item $\frac{\|\nabla f(w)\|}{c_1} > r:$ In this case, the minimum of the polynomial in \eqref{eq:kfnqljlonwd} for $\eta = [0,r]$ is attained at $\eta=r$. Thus, from \eqref{eq:kfnqljlonwd} we obtain
    \[
    \frac{c_1}{2}r^2 - r\|\nabla f(w)\| + (f(w) - f^*) \ge 0 \Leftrightarrow \|\nabla f(w)\| \le \frac{c_1}{2}r + \frac{f(w) - f^*}{r}.
    \]
    We plug in the definition of $r$ in the bound above and obtain
    \begin{align*}
        &\|\nabla f(w)\| \le \frac{-\|\nabla f(w)\| + \sqrt{\|\nabla f(w)\|^2 + 4c_1(f(w) - f^*)}}{2}\\
        &\hspace{5cm}\qquad +\; \frac{c_1(f(w)-f^*)}{-\|\nabla f(w)\| + \sqrt{\|\nabla f(w)\|^2 + 4c_1(f(w) - f^*)}},\\
        &2\|\nabla f(w)\|(-\|\nabla f(w)\| + \sqrt{\|\nabla f(w)\|^2 + 4c_1(f(w) - f^*)})\\
        &\hspace{3cm}\le (-\|\nabla f(w)\| + \sqrt{\|\nabla f(w)\|^2 + 4c_1(f(w) - f^*)})^22c_1(f(w)-f^*),\\
        &-2\|\nabla f(w)\|^2 + 2\|\nabla f(w)\|\sqrt{\|\nabla f(w)\|^2 + 4c_1(f(w)-f^*)} \le \|\nabla f(w)\|^2 + \|\nabla f(w)\|^2 \\
        &\qquad +\;
        4c_1(f(w)-f^*) 
        - 2\|\nabla f(w)\|\sqrt{\|\nabla f(w)\|^2 + 4c_1(f(w)- f^*)} + 2c_1(f(w) - f^*)\\
        &4\|\nabla f(w)\|\sqrt{\|\nabla f(w)\|^2 + 4c_1(f(w)-f^*)} \le 4\|\nabla f(w)\|^2 + 6c_1(f(w) - f^*)\\
        &4\|\nabla f(w)\|^2(\|\nabla f(w)\|^2 + 4c_1(f(w)-f^*)) \le 4\|\nabla f(w)\|^4 + 9c_1^2(f(w)-f^*)^2\\
        &\hspace{8cm}+\; 12c_1\|\nabla f(w)\|^2(f(w)-f^*),\\
        &4c_1\|\nabla f(w)\|^2(f(w)-f^*) \le 9c_1^2(f(w) - f^*)^2,\\
        &\|\nabla f(w)\|^2 \le \frac{9c_1}{4}(f(w) - f^*).
    \end{align*}
    \end{enumerate}
    
    We obtain that, in both cases, we have the following bound
    \begin{align*}
        \|\nabla f(w)\|^2 
        &\le \max\left\{\frac{9}{4}, 2\right\}(H_0 + 3H_1(f(w) - f^*))(f(w) - f^*)\\
        &\le \frac{9}{4}(H_0 + 3H_1(f(w) - f^*))(f(w) - f^*).
    \end{align*}

\end{proof}

Finally, we provide a quadratic upper bound on the change of the function value. This inequality will be useful later to demonstrate that \algname{GD} returns iterates with decreasing function value.

\begin{lemma}\label{lem:quadratic_bound} Let $f$ be $(H_0,H_1)$-smooth. Then for any $y,w\in\R^d$ such that $\|y-w\| \le \frac{1}{6\sqrt{H_1}}$ we have 
\[
f(y) \le f(w) + \langle\nabla f(w), y-w \rangle + \frac{2H_0 + 2H_1(f(w) - f^*)}{2}\|y-w\|^2.
\]
    
\end{lemma}
\begin{proof}
    From \Cref{lem:c1c2_bound} we know that $\chi(1) \le c_1$ if 
    \begin{align}\label{eq:blmwnpkqnf}
        \|\nabla f(w)\|\cdot \|y-w\| + \frac{c_1}{2}\|y-w\|^2 \le c_2.
    \end{align}
    If we choose $y$ and $w$ such that 
    \[
    \|y-w\| \le \min\left\{\frac{c_2}{2\|\nabla f(w)\|}, \sqrt{\frac{c_2}{c_1}}\right\} \eqcolon r,
    \]
    then \eqref{eq:blmwnpkqnf} is satisfied. Indeed, this is due to
    \[
    \|\nabla f(w)\|\cdot \|y-w\| \le \|\nabla f(w)\| \frac{c_2}{2\|\nabla f(w)\|} = \frac{c_2}{2},
    \]
    and 
    \[
    \frac{c_1}{2}\|y-w\|^2 \le \frac{c_1}{2}\left(\sqrt{\frac{c_2}{c_1}}\right)^2 = \frac{c_2}{2}.
    \]
    Therefore, 
    \[
        \|\nabla f(w)\|\cdot\|y-w\| + \frac{c_1}{2}\|y-w\|^2 \le c_2.
    \]
    Let $a_1 \eqdef \frac{9}{4}(H_0 + 3H_1(f(w) - f^*))(f(w) - f^*)$ and $a_2 = c_2(H_0 + H_1(c_2 + f(w) - f^*)) = c_1c_2.$ Then from \Cref{lem:gradient_bound} we have $\|\nabla f(w)\|^2 \le a_1.$ Therefore, we obtain that 
    \[
    r \ge \frac{2c_2}{\sqrt{a_1} + 2\sqrt{c_1c_2}} = \frac{2c_2}{\sqrt{a_1} + 2\sqrt{a_2}}.
    \]

    Let 
    \[
    c_2 = \max\left\{f(w) - f^*, \frac{H_0}{3H_1}\right\}.
    \]
    \begin{enumerate}
    \item If $H_0 \le 3H_1(f(w) - f^*),$ then $\frac{H_0}{3H_1} \le f(w) - f^*$. Therefore, $c_2 = f(w) - f^*.$ This implies that $a_1 \le \frac{27}{2}H_1(f(w) - f^*)^2,$ and 
    \begin{align*}
    a_2 = c_2(H_0 + H_1(c_2 + f(w) - f^*)) &= (f(w) - f^*)(H_0 + 2H_1(f(w) - f^*)) \\
    &\le 5H_1(f(w) - f^*)^2.
    \end{align*}
    Moreover, in this case, 
    \[
    c_1 = H_0+H_1(c_2 + f(w) - f^*) = H_0 + 2H_1(f(w) - f^*) \le 5H_1(f(w)-f^*).
    \]
    Therefore, in this case, we obtain 
    \begin{align*}
    r \ge \frac{2(f(w) - f^*)}{\frac{3}{2}\sqrt{6H_1}(f(w) - f^*) + 2\cdot \sqrt{5H_1}(f(w) - f^*} = \frac{2}{\frac{3}{2}\sqrt{6H_1} + 2\sqrt{5H_1}} &\ge \frac{0.245}{\sqrt{H_1}}\\
    &\ge \frac{1}{5\sqrt{H_1}}.
    \end{align*}
    
    \item If $H_0 > 3H_1(f(w) - f^*)$, then $\frac{H_0}{3H_1}> f(w)-f^*$. Therefore,  $c_2 = \frac{H_0}{3H_1}.$ This implies that 
    \[
    a_1 = \frac{9}{4}(H_0 + 3H_1(f(w) - f^*))(f(w) - f^*) \le \frac{9}{2}H_0(f(w) - f^*) \le \frac{3H_0^2}{2H_1},\]
    \begin{align*}
    a_2 &= c_2(H_0 + H_1(c_2 + f(w) - f^*)) = \frac{H_0}{3H_1}(H_0 + H_1(\nicefrac{H_0}{3H_1} + f(w) - f^*))\\
    &\le \frac{H_0}{3H_1}(H_0 + H_1\cdot \nicefrac{2H_0}{3H_1}) = \frac{5H_0^2}{9H_1},
    \end{align*}
    and 
    \[
    c_1 = H_0 + H_1(c_2 + f(w)-f^*) \le H_0 + H_1\left(\frac{H_0}{3H_1} + \frac{H_0}{3H_1}\right) = \frac{5}{3}H_0.
    \]
    Therefore, in this case, we obtain
    \[
    r \ge \frac{2\frac{H_0}{3H_1}}{\frac{H_0}{2}\sqrt{\frac{6}{H_1}} + \frac{2H_0}{3}\sqrt{\frac{5}{H_1}}} = \frac{2/3}{\sqrt{H_1}(\frac{\sqrt{6}}{2} + \frac{2\sqrt{5}}{3})} \ge \frac{0.193}{\sqrt{H_1}} \ge \frac{1}{6\sqrt{H_1}}.
    \]

    \end{enumerate}
    Combining both cases, we obtain that if 
    \[
    \|y-w\| \le \min\left\{\frac{1}{5\sqrt{H_1}}, \frac{1}{6\sqrt{H_1}}\right\} = \frac{1}{6\sqrt{H_1}} \le r,
    \]
    we obtain $\chi(1) \le c_1 \le \max\{H_0 + 2H_1(f(w) - f^*), \nicefrac{5H_0}{3}\} \le \frac{5H_0}{3} + 2H_1(f(w) - f^*) \le 2H_0 + 2H_1(f(w) - f^*).$ From \eqref{eq:o3nwjdn}, this implies 
    \[
    f(y) - f(w) \le \langle\nabla f(w), y-w \rangle + \frac{2H_0 + 2H_1(f(w) - f^*)}{2}\|y-w\|^2.
    \]
\end{proof}

\section{Missing Proofs for \Cref{sec:convergence_analysis}}
\label{sec:conv_proofs_non_convex}

\subsection{Convergence for General Non-Convex Functions}

\begin{theorem}\label{th:general_nonconvex_adaptive} Let $f$ be $(H_0,H_1)$-smooth. Then the iterates of \algname{GD} $w_{k+1} = w_k - \eta_k \nabla f(w_k)$ where $\eta_k = \frac{1}{10H_0 + 20H_1(f(w_k) - f^*)}$ satisfy
\[
 \min_{k<K}\|\nabla f(w_k)\|^2 \le \frac{20(H_0+2H_1(f(w_0)-f^*))(f(w_0)-f^*)}{K}\frac{1}{1+\frac{10H_1(f(w_0)-f^*)(K-1)(K-2)}{K^2(10H_0+20H_1(f(w_0)-f^*))}}. 
\]
If $K \ge 6,$ then the rate can be simplified
\[
 \min_{k<K}\|\nabla f(w_k)\|^2 \le \frac{20(H_0+2H_1(f(w_0)-f^*))(f(w_0)-f^*)}{K}\frac{1}{1+\frac{H_1(f(w_0)-f^*)}{(2H_0+4H_1(f(w_0)-f^*))}}. 
\]
\end{theorem}

\begin{proof}
    Note that $\|w_{k+1} - w_k\| = \eta_k\|\nabla f(w_k)\|.$ Now we use \Cref{lem:gradient_bound} to obtain
    \[
    \eta_k\|\nabla f(w_k)\| \le \eta_k \frac{3}{2}\sqrt{(H_0 + 3H_1(f(w_k)-f^*))(f(w_k) -f^*)}.
    \]

    \begin{enumerate}
        \item If $H_0 \le 3H_1(f(w_k) - f^*)$, then 
        \begin{align}\label{eq:iknqwdoqwd}
        \frac{3}{2}\eta_k\sqrt{6H_1(f(w_k) - f^*)^2} \le \frac{3}{2}\eta_k \sqrt{6H_1}(f(w_k) - f^*).
        \end{align}
        We need to upper bound the above by $\frac{1}{6\sqrt{H_1}}$ to be able to use \Cref{lem:quadratic_bound}. We satisfy \eqref{eq:iknqwdoqwd} by the choice of the step-size $\eta_k$
        \[
        \eta_k\|\nabla f(w_k)\| \le \frac{3}{2}\eta_k \sqrt{6H_1}(f(w_k) -f^*) \le \frac{1}{6\sqrt{H_1}} \Leftrightarrow \eta_k \le \frac{1}{\frac{15}{2}\sqrt{6}H_1(f(w_k) - f^*)},
        \]
        where the last inequality is satisfied since 
        \[
        \eta_k = \frac{1}{10H_0 + 20H_1(f(w_k) - f^*)} \le \frac{1}{20H_1(f(w_k)-f^*)} \le \frac{1}{\frac{15}{2}\sqrt{6}H_1(f(w_k)-f^*)}.
        \]

        \item If $H_0 > 3H_1(f(w_k) -f^*)$, then
        \begin{align}\label{eq:jfnojqnqef}
        \eta_k\|\nabla f(w_k)\| \le \frac{3}{2}\eta_k\sqrt{2H_0(f(w_k)-f^*)} \le \frac{3}{2}\eta_k \sqrt{2H_0\cdot \frac{H_0}{3H_1}} = \eta_k \frac{\sqrt{3}H_0}{\sqrt{2H_1}}.
        \end{align}
        We need to upper bound the above by $\frac{1}{6\sqrt{H_1}}$ to be able to use \Cref{lem:quadratic_bound}. We satisfy \eqref{eq:jfnojqnqef} by the choice of the step-size $\eta_k$
        \[
        \eta_k\|\nabla f(w_k)\| \le \eta_k\frac{\sqrt{3}H_0}{\sqrt{2H_1}} \le \frac{1}{6\sqrt{H_1}} \Leftrightarrow \eta_k \le \frac{\sqrt{6}}{18H_0},
        \]
        where the last inequality is satisfied since 
        \[
        \eta_k = \frac{1}{10H_0 + 20H_1(f(w_k)-f^*)} \le \frac{1}{10H_0} \le \frac{\sqrt{6}}{18H_0}.
        \]
    \end{enumerate}
    Therefore, the choice of the step-size allows to use \Cref{lem:quadratic_bound} since the restriction $\|w_{k+1}-w_k\|\le \frac{1}{6\sqrt{H_1}}$ is satisfied. Therefore, we have
    \begin{align}
        f(w_{k+1}) &\overset{(i)}{\le} f(w_k) + \langle\nabla f(w_k), w_{k+1}-w_k \rangle + \frac{2H_0+2H_1(f(w_k)-f^*)}{2}\|w_{k+1}-w_k\|^2\notag\\
        &\overset{(ii)}{\le} f(w_k) -\eta_k\|\nabla f(w_k)\|^2 + (H_0 + H_1(f(w_k)-f^*))\eta_k^2\|\nabla f(w_k)\|^2\notag\\
        &= f(w_k) - \eta_k\|\nabla f(w_k)\|^2(1-\eta_k(H_0 + H_1(f(w_k) -f^*))\notag\\
        &\overset{(iii)}{\le} f(w_k) - \frac{\eta_k}{2}\|\nabla f(w_k)\|^2,\label{eq:pqmdnoqnwdq}
    \end{align}
    where $(i)$ follows from \Cref{lem:quadratic_bound}, $(ii)$ --- from \Cref{lem:gradient_bound}, $(iii)$ --- from the choice of the step-size $\eta_k \le\frac{1}{10H_0 + 20H_1(f(w_k)-f^*)}$. This implies that \algname{GD} achieves a monotone decrease of the function value. By the choice of the step-size $\eta_k = \frac{1}{10H_0 + 20H_1(f(w_k)-f^*)}$, we obtain that $\eta_k$ is increasing with $k.$ Rearranging the last inequality we obtain $\|\nabla f(w_k)\|^2 \le \frac{2}{\eta_k}(f(w_k) - f(w_{k+1}).$ Summing this inequality over iterations $\{0,\hdots, K-1\}$ we obtain

    \begin{align*}
        \frac{1}{K}\sum_{k=0}^{K-1}\|\nabla f(w_k)\|^2 & \le \frac{1}{K}\sum_{k=0}^{K-1}\frac{2}{\eta_k}(f(w_k)-f(w_{k+1}))\\
        &= \frac{1}{K}\sum_{k=0}^{K-1}(20H_0 + 40H_1(f(w_k)-f^*))(f(w_k)-f(w_{k+1}))\\
        &= \frac{20H_0}{K}\sum_{k=0}^{K-1}f(w_k)-f(w_{k+1})\\
        &\hspace{1.5cm}+\; \frac{40H_1}{K}\sum_{k=0}^{K-1}(f(w_k)-f^*)^2 - (f(w_k)-f^*)(f(w_{k+1})-f^*)\\
        &\overset{(iv)}{\le} \frac{20H_0}{K}\sum_{k=0}^{K-1}f(w_k)-f(w_{k+1})\\
        &\hspace{3cm}+\; \frac{40H_1}{K}\sum_{k=0}^{K-1}(f(w_k)-f^*)^2 - (f(w_{k+1})-f^*)^2\\
        &\le \frac{20H_0(f(w_0)-f^*)}{K} + \frac{40H_1(f(w_0)-f^*)^2}{K}.
    \end{align*}

    The current rate is the same as with a constant step-size $\eta = \frac{1}{10H_0 + 20H_1(f(w_0)-f^*)}$, i.e., we do not show improvement. Now our goal is to obtain a tighter rate for \algname{GD} using the fact that the sequence $\{\eta_k\}$ is increasing. By \eqref{eq:pqmdnoqnwdq}, we obtain
    \[
    f(w_k) \le f(w_0) - \sum_{j=0}^{k-1}\frac{\eta_j}{2}\|\nabla f(w_j)\|^2 \Rightarrow f(w_k) - f^* \le (f(w_0) - f^*) - \sum_{j=0}^{k-1}\frac{\eta_j}{2}\|\nabla f(w_j)\|^2.
    \]
    Therefore, 
    \[
    \frac{1}{\sum_{k=0}^{K-1}\eta_k}\sum_{k=0}^{K-1}\eta_k\|\nabla f(w_k)\|^2 \le \frac{2(f(w_0)-f^*)}{\sum_{k=0}^{K-1}\eta_k}.
    \]
    To provide a tighter bound, we should take into account that the step-sizes are increasing since $f(w_k)-f^*$ is decreasing. Remember that  $\eta_k = \frac{1}{10H_0+20H_1(f(w_k)-f^*)}$, then 
    \begin{align*}
        \sum_{k=0}^{K-1}\eta_k
        &= \sum_{k=0}^{K-1}\frac{1}{10H_0+20H_1(f(w_k)-f^*)}\\
        &\ge \sum_{k=0}^{K-1}\frac{1}{10H_0 + 20H_1\left(f(w_0)-f^* -\sum_{j=0}^{k-1}\frac{\eta_j}{2}\|\nabla f(w_j)\|^2\right)}.
    \end{align*}
    Let us denote $\Lambda_k = \sum_{j=0}^{k-1}\eta_j\|\nabla f(w_j)\|^2$, then
    \[
    \sum_{k=0}^{K-1}\eta_k \ge \sum_{k=0}^{K-1}\frac{1}{10H_0 + 20H_1(f(w_0)-f^*) - 10H_1 \Lambda_k}.
    \]
    Since the function $u \to g(u) \eqdef \frac{1}{10H_0 + 20H_1(f(w_0)-f^*) - 10H_1u}$ is convex in the set $\{u\in\R \mid g(u) > 0\},$ then by Jensen's inequality we have 
    \[
    \frac{1}{K}\sum_{k=0}^{K-1}g(\Lambda_k) \ge g\left(\frac{1}{K}\sum_{k=0}^{K-1}\Lambda_k\right).
    \]
    In our case, we obtain 
    \[
    \sum_{k=0}^{K-1}\eta_k \ge \sum_{k=0}^{K-1}g(\Lambda_k) \ge \frac{K}{10H_0 + 20H_1(f(w_0)-f^*) - \frac{10H_1}{K}\sum_{k=0}^{K-1}\Lambda_k}.
    \]
    Now we estimate 
    \begin{align*}
    \sum_{k=0}^{K-1}\Lambda_k= \sum_{k=0}^{K-1}\sum_{j=0}^{k-1}\eta_j\|\nabla f(w_j)\|^2 
    \ge \min_{k < K}\|\nabla f(w_k)\|^2\sum_{k=0}^{K-1}\sum_{j=0}^{k-1}\eta_j  \ge \min_{k < K}\|\nabla f(w_k)\|^2\eta_0 \frac{(K-1)K}{2},
    \end{align*}
    where we use the fact that $\eta_0 \le \eta_k$ for all $k \ge 0.$ This leads to the following bound
    \begin{align*}
        \min_{k<K}\|\nabla f(w_k)\|^2 &\le \frac{1}{\sum_{k=0}^{K-1}\eta_k}\sum_{k=0}^{K-1}\eta_k\|\nabla f(w_k)\|^2\\
        &\le \frac{2(f(w_0)-f^*)}{\frac{K}{10H_0 + 20H_1(f(w_0)-f^*) - \frac{10H_1}{K}\eta_0 \frac{(K-1)(K-2)}{2}\min_{k}\|\nabla f(w_k)\|^2 }}\\
        &\le \frac{2(10H_0 + 20H_1(f(w_0)-f^*))(f(w_0)-f^*)}{K}\\
        &\hspace{2cm} -\; \frac{10H_1(f(w_0)-f^*)(K-1)(K-2)\eta_0\min_k\|\nabla f(w_k)\|^2}{K^2}.
    \end{align*}
    Rearranging the terms, we obtain
    \begin{align*}
         \min_{k<K}\|\nabla f(w_k)\|^2 \le \frac{20(H_0+2H_1(f(w_0)-f^*))(f(w_0)-f^*)}{K}\frac{1}{1+\frac{10H_1(f(w_0)-f^*)(K-1)(K-2)}{K^2(10H_0+20H_1(f(w_0)-f^*))}}. 
    \end{align*}
    If $K\ge 6$, then $\frac{10(K-1)(K-2)}{K^2}\ge 5,$ which leads to the simplified rate.
\end{proof}

\subsection{Convergence under Aiming Condition}

\theoremaimingdeterministic*
\begin{proof}
    We start by \eqref{eq:pqmdnoqnwdq}
    \begin{align}\label{eq:ejfonqejfq}
        f(w_{k+1}) &\le f(w_k) - \frac{\eta_k}{2}\|\nabla f(w_k)\|^2 = f(w_k) - \frac{\theta}{20H_0+40H_1(f(w_k)-f^*)}\|\nabla f(w_k)\|^2.
    \end{align}
    Next, we show that the distance to the set of global minimizers $\cS$ of the function $f$ does not increase. Indeed, we have 
        \begin{align*}
            {\rm dist}(w_{k+1},\cS)^2 &\overset{(i)}{=} \|w_{k+1}-\pi_{\cS}(w_k)\|^2 \\ 
            &= \|w_k-\pi_{\cS}(w_k)\|^2 -2\eta_k\langle w_k-\pi_{\cS}(w_k), \nabla f(w_k)\rangle + \eta_k^2\|\nabla f(w_k)\|^2\\
            &\overset{(ii)}{\le} {\rm dist}(w_k,\cS)^2 -2\eta_k\theta(f(w_k)-f^*) + \eta_k^2\|\nabla f(w_k)\|^2\\
            &\overset{(iii)}{\le} {\rm dist}(w_k,\cS)^2 -2\eta_k\theta(f(w_k)-f^*)\\
            &\hspace{4cm}+\; \frac{9\eta_k^2}{4}(H_0 + 3H_1(f(w_k)-f^*))(f(w_k)-f^*)\\
            &= {\rm dist}(w_k,\cS)^2 - 2\eta_k(f(w_k)-f^*)\left(\theta - \frac{9}{8}\eta_k(H_0 + 3H_1(f(w_k)-f^*))\right),
        \end{align*}
        where $(i)$ follows from the definition of the projection, $(ii)$ follows from the definition of the Aiming condition, $(iii)$ --- from \Cref{lem:gradient_bound}. Now we use the choice of the step-size $\eta_k = \frac{\theta}{10 H_0+20 H_1(f(w_k)-f^*)}$ to obtain
        \begin{align}\label{eq:wnflekoadqojwodnqwqndq}
            {\rm dist}(w_{k+1},\cS)^2 &\le {\rm dist}(w_k,\cS)^2 - \eta_k\theta(f(w_k)-f^*).
        \end{align}
        Therefore, we have that ${\rm dist}(w_{k+1},\cS)^2 \le {\rm dist}(w_k,\cS)^2$ for any $k\ge 0.$ Now we consider two cases:
    \begin{itemize}
        \item $f(w_k) - f^* \ge \frac{H_0}{2H_1}$ (large function value).  In this case, we can lower bound the step-size as 
        \[
        \eta_k = \frac{\theta}{10H_0 + 20H_1(f(w_k)-f^*)} \ge \frac{\theta}{40H_1(f(w_k)-f^*)}.
        \]
        Therefore, from \eqref{eq:wnflekoadqojwodnqwqndq}, we obtain
        \begin{align*}
             {\rm dist}(w_{k+1},\cS)^2 &\le {\rm dist}(w_k,\cS)^2 - \eta_k\theta(f(w_k)-f^*) \\
             &\le {\rm dist}(w_k,\cS)^2 - \frac{\theta}{40H_1(f(w_k)-f^*)}\theta(f(w_k)-f^*)\\
             &= {\rm dist}(w_k,\cS)^2 - \frac{\theta^2}{40H_1}.
        \end{align*}
        Since ${\rm dist}(w_k,\cS)^2  \ge 0$, we can stay in this regime at most $T$ iterations, such that
        \[
        0 \leq {\rm dist}(w_T,\cS)^2 \le {\rm dist}(w_0,\cS)^2 - \frac{\theta^2}{40H_1} T  \Rightarrow T \eqdef \frac{40H_1{\rm dist}(w_0,\cS)^2 }{\theta^2}.
        \]

        \item  $f(w_k) - f^* \le \frac{H_0}{2H_1}$ (small function value). In this case, we can lower bound the step-size as 
        \[
            \eta_k = \frac{\theta}{10H_0+20H_1(f(w_k)-f^*)} \ge \frac{\theta}{20H_0}.
        \]
        Therefore, from \eqref{eq:wnflekoadqojwodnqwqndq}, we obtain
        \begin{align*}
             {\rm dist}(w_{k+1},\cS)^2 &\le {\rm dist}(w_k,\cS)^2 - \eta_k\theta(f(w_k)-f^*) \\
             &\le {\rm dist}(w_k,\cS)^2 - \frac{\theta^2}{40H_0}(f(w_k)-f^*).
        \end{align*}
        Rearranging the terms, we obtain  
        \begin{align}\label{eq:enfeojqqwdqwlk}
            f(w_k) - f^* \le \frac{40H_0}{\theta^2}({\rm dist}(w_k,\cS)^2 - {\rm dist}(w_{k+1},\cS)^2).
        \end{align}
        Averaging the inequalities \eqref{eq:enfeojqqwdqwlk} for $k\in\{T,\dots, K\}$, we obtain 
        \begin{align*}
        \frac{1}{K-T+1}\sum_{k=T}^{K}(f(w_k)-f^*) &\le \frac{40H_0({\rm dist}(w_0,\cS)^2-{\rm dist}(w_{K+1},\cS)^2)}{\theta^2(K-T+1)} \\
        &\le \frac{40H_0{\rm dist}(w_0,\cS)^2}{\theta^2(K-T+1)}.
        \end{align*}
        Since $f(w_k)-f^*$ is decreasing by \eqref{eq:ejfonqejfq}, we have
        \[
        f(w_K) - f^* \le \frac{40H_0{\rm dist}(w_0,\cS)^2}{\theta^2(K-T+1)}.
        \]
    \end{itemize}
    To achieve $\varepsilon$ accuracy, we need the number of iterations $K$ to be
    \begin{align*}
        f(w_K) - f^* \le \frac{40H_0{\rm dist}(w_0,\cS)^2}{\theta^2(K-T+1)} \le \varepsilon \Rightarrow K &\ge \frac{40H_0{\rm dist}(w_0,\cS)^2}{\theta^2\varepsilon} + T\\
        &= \frac{40H_0{\rm dist}(w_0,\cS)^2}{\theta^2\varepsilon} + \frac{40H_1{\rm dist}(w_0,\cS)^2}{\theta^2}. 
    \end{align*}
\end{proof}

The next theorem demonstrates that when the function suboptimality is large, we should expect a linear decrease. This gives another intuition behind the improvement from the warm-up schedule. This result demonstrates that linear convergence can be expected even beyond the PL case.
\begin{theorem}
    Assume that $f$ is $(H_0,H_1)$-smooth, and it satisfies the Aiming condition with constant $\theta$ around the set of global minimizers $\cS.$ Assume that $f(w_k) - f^* \ge \frac{H_0}{2H_1}$. Then the iterates of \algname{GD} $w_{k+1} = w_k - \eta_k\nabla f(w_k)$ with a step-size $\eta_k = \frac{\theta}{10H_0+20H_1(f(w_k)-f^*)}$ satisfy
    \[ f(w_{k+1})-f^* \leq 
    \left(1-\frac{\theta^3}{80H_1{\rm dist}(w_0,\cS)^2}\right)(f(w_k)-f^*).
    \]
\end{theorem}
\begin{proof}
    First, we use the previously derived decrease in the function value \eqref{eq:ejfonqejfq}
    \[
    f(w_{k+1}) - f^* \le f(w_k) - f^* - \frac{\theta}{20H_0+40H_1(f(w_k)-f^*)}\|\nabla f(w_k)\|^2,
    \]
    and in the distance \eqref{eq:wnflekoadqojwodnqwqndq}
    \[
    {\rm dist}(w_{k+1},\cS)^2 \le {\rm dist}(w_k,\cS)^2 - \eta_k\theta(f(w_k)-f^*).
    \]
    In particular, ${\rm dist}(w_k,\cS)^2 \le {\rm dist}(w_0,\cS)^2$. From the Aiming condition, we have 
    \begin{align}\label{eq:kjofnwejfwefw}
    \theta(f(w_k)-f^*) \le\langle\nabla f(w_k),w_k-\pi_{\cS}(w_k)\rangle &\le \|\nabla f(w_k)\|\cdot {\rm dist}(w_k,\cS)\notag\\
    &\le \|\nabla f(w_k)\|\cdot {\rm dist}(w_0,\cS).
    \end{align}
    Therefore, we obtain
    \begin{align*}
        f(w_{k+1})-f^* &\le f(w_k) - f^* - \frac{\theta}{20H_0 + 40H_1(f(w_k)-f^*)}\|\nabla f(w_k)\|^2\\
        &\overset{(i)}{\le} f(w_k) - f^* - \frac{\theta}{80H_1(f(w_k)-f^*)}\|\nabla f(w_k)\|^2\\
        &\overset{(ii)}{\le} f(w_k) - f^* - \frac{\theta}{80H_1(f(w_k)-f^*)}\frac{\theta^2(f(w_k)-f^*)^2}{{\rm dist}(w_0,\cS)^2}\\
        &= \left(1-\frac{\theta^3}{80H_1{\rm dist}(w_0,\cS)^2}\right)(f(w_k)-f^*).
    \end{align*}
    where $(i)$ follows from the bound $f(w_k)-f^* \ge \frac{H_0}{2H_1}$, $(ii)$ -- from \eqref{eq:kjofnwejfwefw}.
\end{proof}

\subsection{Convergence under Polyak-{\L}ojasiewicz Condition}

\theorempleterministic*

\begin{proof}
    We start with the Eq.~\eqref{eq:pqmdnoqnwdq} and use $\mu$-PL inequality
    \begin{align*}
        f(w_{k+1}) &\le f(w_k) - \frac{\eta_k}{2}\|\nabla f(w_k)\|^2\\
        &\le f(w_k) - \mu\eta_k(f(w_k)-f^*)\\
        &= f(w_k) - \frac{\mu(f(w_k)-f^*)}{10H_0 + 20H_1(f(w_k)-f^*)}.
    \end{align*}
    Now we consider two cases.
    \begin{itemize}
        \item $f(w_k) - f^* \ge \frac{H_0}{2H_1}$ (large function value). In this case, we have
        \begin{align*}
            f(w_{k+1}) &\le f(w_k) - \frac{\mu(f(w_k)-f^*)}{10H_0+20H_1(f(w_k)-f^*))}\\
            &\le f(w_k) - \frac{\mu(f(w_k)-f^*)}{40H_1(f(w_k)-f^*)}\\
            &= f(w_k) - \frac{\mu}{40H_1}.
        \end{align*}
        
         Since \algname{GD} decreases the function value (see \eqref{eq:pqmdnoqnwdq}), we have $f(w_t) - f^* \ge \frac{H_0}{2H_1}$ for all $K\in\{0,\dots,k\}.$ Therefore, 
        \[
        f(w_{k+1}) - f^*\le f(w_0)-f^* - \frac{\mu}{40H_1}(k+1).
        \]
        
        However, we cannot reduce the function value infinitely many times, since it is lower bounded. We can stay in this regime as long as $f(w_t)-f^*\ge \frac{H_0}{2H_1}$, therefore, \algname{GD} stays in this regime for at most $k\le \frac{40H_1}{\mu}\left(f(w_0)-f^* -\frac{H_0}{2H_1}\right) -1 \le \frac{40H_1}{\mu}(f(w_0)-f^*) - \frac{20H_0}{\mu}$ iterations. In other words, the cardinality of the set $\cT\eqdef\{k\in\{0,\dots, K-1\}\colon f(w_k) - f^* \ge \frac{H_0}{2H_1}\}$ is bounded by $T = \frac{40H_1}{\mu}(f(w_0)-f^*) - \frac{20H_0}{\mu}.$

        \item $f(w_k) -f^* \le \frac{H_0}{2H_1}$ (small function value). In this case, we have 
        \begin{align}\label{eq:nfoqjwndqqweqw}
            f(w_{k+1}) &\le f(w_k) - \frac{\mu(f(w_k)-f^*)}{10H_0+20H_1(f(w_k)-f^*))}\notag\\
            &\le f(w_k) - \frac{\mu(f(w_k)-f^*)}{20H_0}.
        \end{align}
        Since the function along the trajectory of \algname{GD} does not increase (see \eqref{eq:pqmdnoqnwdq}), we stay in this regime for the rest of the training. Therefore, summing up \eqref{eq:nfoqjwndqqweqw} for all iterations $k\in\{T,\dots, K-1\}$ we obtain
        \begin{align*}
           f(w_K) - f^* &\le \left(1-\frac{\mu}{20H_0}\right)(f(w_{K-1})-f^*)\\
           &\le \hdots\\
           &\le \left(1-\frac{\mu}{20H_0}\right)^{K-T}(f(w_T)-f^*).
        \end{align*}
        Since $f(w_T)-f^* \le f(w_0)-f^* - \frac{\mu T}{40H_1},$ we get the rate
        \begin{align*}
            & f(w_K)-f^*\\
            \le\; &\left(1-\frac{\mu}{20H_0}\right)^{K-T}\left(f(w_{0})-f^* - \frac{\mu }{40H_1}\left(\frac{40H_1}{\mu}(f(w_0)-f^*) - \frac{20H_0}{\mu}\right)\right)\\
            =\; &\left(1-\frac{\mu}{20H_0}\right)^{K-T}\frac{H_0}{2H_1}.
        \end{align*}

    \end{itemize}
    To achieve $f(w_K) - f^*\le \varepsilon$ we need to satisfy
    \begin{align*}
        f(w_K)-f^*\le \left(1-\frac{\mu}{20H_0}\right)^{K-T}\frac{H_0}{2H_1} \le \varepsilon \Rightarrow K &\ge T + \frac{20H_0}{\mu}\log\frac{H_0}{2H_1\varepsilon}\\
        &= \frac{40H_1}{\mu}(f(w_0)-f^*) + \frac{20H_0}{\mu}\log\frac{H_0}{2H_1\varepsilon}.
    \end{align*}
\end{proof}

\subsection{Convergence in the Stochastic Setting}

\theoremaimingstochastic*
\begin{proof}
    We show that the distance to the set of global minimizers $\cS$ of the function $f$ does not increase. Indeed, we have
    \begin{align*}
        {\rm dist}(w_{k+1},\cS)^2 &= \|w_{k+1}-\pi_{\cS}(w_{k+1})\|^2\\
        &\le \|w_{k+1}-\pi_{\cS}(w_k)\|^2\\
        &= \|w_k - \pi_{\cS}(w_k)\|^2 - 2\eta_k\langle w_k-\pi_{\cS}(w_k),\nabla f_{S_k}(w_k) \rangle + \eta_k^2\|\nabla f_{S_k}(w_k)\|^2\\
        &\overset{(i)}{\le} \|w_k-\pi_{\cS}(w_k)\|^2 - 2\theta\eta_k(f_{S_k}(w_k) - f_{S_k}^*) + \eta_k^2\|\nabla f_{S_k}(w_k)\|^2\\
        &\overset{(ii)}{\le} {\rm dist}(w_k,\cS)^2 - 2\theta\eta_k(f_{S_k}(w_k)-f_{S_k}^*)\\
        &\hspace{4cm} +\; \frac{9}{4}\eta_k^2(H_0+3H_1(f_{S_k}(w_k)-f_{S_k}^*))(f_{S_k}(w_k)-f_{S_k}^*)\\
        &\overset{(iii)}{=} {\rm dist}(w_k,\cS)^2 - 2\eta_k(f_{S_k}(w_k)-f_{S_k}(w^*))\\
        &\hspace{3cm}+\; \frac{9}{4}\eta_k^2(H_0+3H_1(f_{S_k}(w_k)-f_{S_k}(w^*)))(f_{S_k}(w_k)-f_{S_k}(w^*))\\
        &\hspace{-0.5cm}={\rm dist}(w_k,\cS)^2 - 2\eta_k(f_{S_k}(w_k)-f_{S_k}(w^*))\left(\theta-\frac{9}{8}\eta_k(H_0+3H_1(f_{S_k}(w_k)-f_{S_k}(w^*)))\right)
    \end{align*}
    where $(i)$ follows from \Cref{asmp:aiming_condition}, $(ii)$ --- from \Cref{lem:gradient_bound}, $(iii)$ --- from the interpolation condition. Now we use the choice of the step-size \[
    \eta_k = \frac{\theta}{10H_0+20H_1(f_{S_k}(w_k)-f_{S_k}^*)} = \frac{\theta}{10H_0+20H_1(f_{S_k}(w_k)-f_{S_k}(w^*))}
    \]
    to obtain
    \begin{align}\label{eq:konqjwdonqwojdqnwod}
        {\rm dist}(w_{k+1},\cS)^2 \le {\rm dist}(w_k,\cS)^2 - \eta_k\theta(f_{S_k}(w_k)-f_{S_k}(w^*)).
    \end{align}
    Therefore, we have that ${\rm dist}(w_{k+1},\cS)^2 \le {\rm dist}(w_k,\cS)^2$ for any $k\ge 0.$ Now we consider two cases:
    \begin{itemize}
        \item $f_{S_k}(w_k)-f_{S_k}(w^*)\ge \frac{H_0}{2H_1}$ (large function value). In this case, we can lower bound the step-size $\eta_k$ as
        \[
        \eta_k = \frac{\theta}{10H_0+20H_1(f_{S_k}(w_k)-f_{S_k}(w^*))} \ge \frac{\theta}{40H_1(f_{S_k}(w_k)-f_{S_k}(w^*))}.
        \]
        Therefore, from \eqref{eq:konqjwdonqwojdqnwod}, we obtain
        \begin{align}
            {\rm dist}(w_{k+1},\cS)^2 &\le {\rm dist}(w_k,\cS)^2 - \eta_k\theta(f_{S_k}(w_k)-f_{S_k}(w^*))\notag\\
            &\le {\rm dist}(w_k,\cS)^2 - \frac{\theta^2}{40H_1(f_{S_k}(w_k)-f_{S_k}(w^*))}(f_{S_k}(w_k)-f_{S_k}(w^*))\notag\\
            &= {\rm dist}(w_k,\cS)^2 - \frac{\theta^2}{40H_1}.\label{eq:qdqlwjndqlwkdb}
        \end{align}

        \item $f_{S_k}(w_k)-f_{S_k}(w^*) \le \frac{H_0}{2H_1}$ (small function value). In this case, we can lower bound the step-size $\eta_k$ as
        \[
        \eta_k = \frac{\theta}{10H_0+20H_1(f_{S_k}(w_k)-f_{S_k}(w^*))} \ge \frac{\theta}{20H_0}.
        \]
        Therefore, from \eqref{eq:konqjwdonqwojdqnwod}, we obtain
        \begin{align}\label{eq:fkehqbjdqbwkd}
            {\rm dist}(w_{k+1},\cS)^2 &\le 
            {\rm dist}(w_k,\cS)^2 - \eta_k\theta(f_{S_k}(w_k)-f_{S_k}(w^*))\notag\\
            &\le {\rm dist}(w_k,\cS)^2 - \frac{\theta^2}{20H_0}(f_{S_k}(w_k)-f_{S_k}(w^*)).
        \end{align}
    \end{itemize}
    To combine descent inequalities \eqref{eq:qdqlwjndqlwkdb} and \eqref{eq:fkehqbjdqbwkd}, we introduce the even $E(w_k)\eqdef \left\{f_{S_k}(w_k)-f_{S_k}(w^*) \ge \frac{H_0}{2H_1} \mid w_k\right\}$ for given $w_k$ and its indicator function $\1_{E(w_k)}$, i.e., for given $w_k$, $\1_{E(w_k)} = 1$ if $f_{S_k}(w_k)-f_{S_k}(w^*) \ge \frac{H_0}{2H_1},$ and $\1_{E(w_k)} = 0$ if $f_{S_k}(w_k)-f_{S_k}(w^*) < \frac{H_0}{2H_1}.$ Then the descent in the general case can be written as
    \begin{align}\label{eq:ionqowjbdnqwj}
        {\rm dist}(w_{k+1},\cS)^2 \le {\rm dist}(w_k,\cS)^2 - \1_{E(w_k)}\frac{\theta^2}{40H_1} - (1-\1_{E(w_k)})\frac{\theta^2}{20H_0}(f_{S_k}(w_k)-f_{S_k}(w^*)).
    \end{align}
    We denote $\Ek{\cdot}$ as $\E{\cdot \mid w_k}$ -- the expectation conditioned on $w_k.$ Thus, we have from \eqref{eq:ionqowjbdnqwj} that 
    \begin{align}
       \Ek{{\rm dist}(w_{k+1},\cS)^2 }
            &\le {\rm dist}(w_k,\cS)^2 - \frac{\theta^2}{20H_0}\Ek{(1-\1_{E(w_k)})(f_{S_k}(w_k)-f_{S_k}(w^*))}\notag\\
            &\hspace{7cm} -\; \Ek{\1_{E(w_k)}}\frac{\theta^2}{40 H_1}\notag\\
            &= {\rm dist}(w_k,\cS)^2 - \frac{\theta^2}{20H_0}\Ek{(1-\1_{E(w_k)})(f_{S_k}(w_k)-f_{S_k}(w^*))}\notag\\
            &\hspace{8cm} -\; p_k\frac{\theta^2}{40 H_1},
    \end{align}
    where $p_k \eqdef \Ek{\1_{E(w_k)}} = \Prob(E(w_k)) = \Prob(f_{S_k}(w_k) - f_{S_k}(w^*) \ge \frac{H_0}{2H_1}).$ We emphasize that $p_k$ is a random variable. If $p_k > 0$, then there is at least one $i\in[n]$, so that $f_{i}(w_k)-f_{i}(w^*) \ge \frac{H_0}{2H_1}$ for given $w_k.$ Thus, we have $p_k \ge \frac{1}{n}.$ In the opposite case, we have $p_k=0,$ and $1-\1_{E(w_k)} = 1$ for given $w_k.$ Putting all together, we continue as follows 
    \begin{align*}
        \Ek{{\rm dist}(w_{k+1},\cS)^2}
            &\le {\rm dist}(w_k,\cS)^2 - \frac{\theta^2}{20H_0}\1_{\{p_k=0\}}(f(w_k)-f(w^*)) - \1_{\{p_k > 0\}}p_k\frac{\theta^2}{40 H_1}\notag\\
            &\le {\rm dist}(w_k,\cS)^2 - \frac{\theta^2}{20H_0}\1_{\{p_k=0\}}(f(w_k)-f(w^*)) - \1_{\{p_k > 0\}}\frac{\theta^2}{40n H_1}\notag\\
            &\le {\rm dist}(w_k,\cS)^2 - \min\left\{\frac{\theta^2}{20H_0}(f(w_k)-f(w^*)), \frac{\theta^2}{40n H_1}\right\}.
    \end{align*}
    Taking full expectation and rearranging terms, we obtain
    \begin{align*}
         \sum_{k=0}^{K}\E{\min\left\{\frac{\theta^2}{20H_0}(f(w_k)-f(w^*)), \frac{\theta^2}{40n H_1}\right\}} 
         &\le \sum_{k=0}^{K+1} \E{{\rm dist}(w_k,\cS)^2} - \E{{\rm dist}(w_{k+1},\cS)^2}\\
         &\le {\rm dist}(w_0,\cS)^2.
    \end{align*}
    Dividing both sides by $\frac{\theta^2}{20H_0(K+1)}$, we obtain 
    \begin{align*}
         \frac{1}{K+1}\sum_{k=0}^{K}\E{\min\left\{f(w_k)-f(w^*), \frac{H_0}{2n H_1}\right\}} 
         &\le \frac{20H_0{\rm dist}(w_0,\cS)^2}{\theta^2(K+1)}.
    \end{align*}
    The rate above implies 
    \begin{align*}
         \min_{k<K+1}\E{\min\left\{f(w_k)-f(w^*), \frac{H_0}{2n H_1}\right\}} 
         &\le \frac{20H_0{\rm dist}(w_0,\cS)^2}{\theta^2(K+1)}.
    \end{align*}
\end{proof}

\section{Missing Proofs for GD in the Convex Setting}

In this case, we demonstrate the convergence to the minimizer $w^*$ of the convex function $f.$

\begin{proof}
   The proof mainly follows the proof of \Cref{th:aiming_deterministic_adaptive_step-size} by setting $\theta=1$ and $\cS=\{w^*\}.$
\end{proof}

\section{Lower Bounds}

\label{sec:lower_bounds_proofs}

\lowerbounds*

\begin{proof}

Consider constants $H_1, M > 1$ and the function
\begin{align*}
f(w)=
    \begin{cases}
  \frac{e^{-\sqrt{H_1} w}}{e}, & \text{if} \hspace{1mm} w<-\frac{1}{\sqrt{H_1}} \\
  \frac{H_1 w^2}{2}+\frac{1}{2}, & \text{if} \hspace{1mm} w \in \left[-\frac{1}{\sqrt{H_1}},\frac{1}{\sqrt{H_1}} \right] \\
  \frac{e^{\sqrt{H_1} w}}{e}, & \text{if} \hspace{1mm} w>\frac{1}{\sqrt{H_1}}.
\end{cases}
\end{align*}

This function is $(H_0, H_1)$-smooth with $H_0 \geq H_1$ and convex, thus it also belongs to the objective function class.

We consider \algname{GD} for the function $f$ starting from the point
\begin{equation*}
    w_0 = \frac{\log M+1}{\sqrt{H_1}}>1.
\end{equation*}

Notice that $f(w_0) = M$ and $\|\nabla f(w_0)\| = M \sqrt{H_1}$. 

If we choose the step-size $\eta$ of \algname{GD} larger than $2w_0/M\sqrt{H_1}$, it holds
\begin{equation*}
    w_1=w_0-\eta \nabla f(w_0) < w_0 - (2 w_0 /M\sqrt{H_1}) M \sqrt{H_1} = -w_0.
\end{equation*}

Thus, $w_1$ is negative and further from the optimum (which is $0$) compared to $w_0$.

By the structure of the function, we can show that $x_2$ will be even further. Since the function is totally symmetric, the effect of one step of \algname{GD} starting from $w_1$ is the same as if it would start from $-w_1$. Thus, it suffices to show that $\tilde w_1 = -w_1 - \eta \nabla f(-w_1)$ is further from $0$ compared to $-w_1$. Since $|w_1|>|w_0|$, it holds $-w_1>w_0$. We consider the function
\begin{equation*}
    g(y) = |y-\eta \nabla f(y)|-|y|
\end{equation*}
for $y>\frac{1}{\sqrt{H_1}}$. Then, we have
\begin{equation*}
    g(y) = \left |y-\eta \sqrt{H_1} \frac{e^{\sqrt{H_1}y}}{e} \right|-|y|.
\end{equation*}

It is simple to see that in the part where this function is positive and $y>\frac{1}{\sqrt{H_1}}$, it is also increasing. Since $g(w_0)>0$, $w_0>\frac{1}{\sqrt{H_1}}$ and $-w_1>w_0$, we have that $g(-w_1)>0$. This means that $|\tilde w_1|>|w_1|$. Using an induction argument, we can show that the iterates of \algname{GD} under such step-size diverge.

We conclude, that the step-size $\eta$ for our function class must satisfy
\begin{equation}
\label{eq:eta_upper_bound}
    \eta \leq \frac{2w_0}{M \sqrt{H_1}} =  \frac{2 \log f(w_0)+2}{f(w_0) H_1}.
\end{equation}

This step-size bound will be used to derive the lower complexity bounds in all cases.

To establish lower bounds for the general and convex cases, we construct a function that contains a long, flat ``runway'' region where the gradient is small but non-zero. This forces any first-order method to take many small steps to traverse it.

For a parameter $\delta>0$ (to be chosen later) and $H_0, H_1>0$, we define the following function $f_\delta(w)$;

The function is symmetric, $f_\delta(w) = f_\delta(-w)$, and defined for $x \geq 0$ as:
\begin{equation}
\label{func:non- convex_convex_lower_bound}
f_\delta(w) = 
\begin{cases} 
\frac{H_0}{2}w^2 & \text{if } 0 \leq w \leq X_1 \\
m(w - X_1) + \delta & \text{if } X_1 < w \leq X_2 \\
A e^{\sqrt{H_1}(w-X_2)} + B & \text{if } w > X_2. 
\end{cases}
\end{equation}

To make this function twice differentiable, we choose
\begin{align*}
&m = \sqrt{2 H_0 \delta}\\
&X_1 = \sqrt{2\delta / H_0}\\
&X_2 = X_1 + (1 - \delta)/m\\
&A = m / \sqrt{H_1}\\
&B = 1 - A.
\end{align*}

$f$ is $(H_0,H_1)$-smooth and its minimum is $f^* = f(0) = 0$.

\paragraph{Lower bound in the general non-convex case:}

We look for a point $w_K$ such that $\|\nabla f(w_K)\| \leq \epsilon$.
To establish the lower bound, we set the gradient on the runway to be slightly larger than our target $\epsilon$, for instance, $\|\nabla f(w)\| = m = 2\epsilon$.

This choice requires us to set the construction parameter $\delta$ as follows:

$$
 \sqrt{2H_0\delta} = 2\epsilon \implies \delta = \frac{2\epsilon^2}{H_0}. 
$$

An algorithm must traverse the linear runway to enter the quadratic bowl, which is the only region where $\|\nabla f(w)\| \leq \epsilon$ is achievable.

\algname{GD} update on the runway is $w_{k+1} = w_k - \eta \nabla f(w_k) = w_k - \eta m$, which implies that
\begin{equation*}
    w_K = w_0 - \eta K m.
\end{equation*}

Thus, if $w_0 = X_2$ (we start at the beginning of the runway) and $K < \frac{X_2-X_1}{\eta m}$, then $w_K > X_1$ and we get $\|\nabla f(w_K)\|=2\epsilon>\epsilon$. Thus, in order to get $\|\nabla f(w_K)\| \leq \epsilon$, we need to have 
\begin{equation*}
    K \geq \frac{X_2-X_1}{\eta m} = \frac{1-\delta}{\eta m^2} = \frac{1-\frac{2 \epsilon^2}{H_0}}{4 \eta \epsilon^2}.
\end{equation*}

Choosing $H_0=1$ (we can choose any positive constant) and plugging in the upper bound \eqref{eq:eta_upper_bound} for the step-size $\eta$, we get that $K$ must satisfy
\begin{equation*}
   K \geq \frac{f(w_0)H_1}{8(\log f(w_0)+1)} \frac{1-2 \epsilon^2}{\epsilon^2}. 
\end{equation*}

Noticing that $f(w_0)=1$ and $f^*=0$, it holds $f(w_0)-f^*=1$ and we get the desired lower bound:
\begin{equation*}
    K \geq \frac{H_1 (f(w_0)-f^*)}{\log (f(w_0)-f^*)+1} \frac{f(w_0)-f^*-2\epsilon^2}{8 \epsilon^2}.
\end{equation*}

\paragraph{Lower bound in the convex case:}
For this scenario, the target accuracy $\epsilon$ directly maps to our construction parameter. We set $\delta = \epsilon$ \eqref{func:non- convex_convex_lower_bound}.
The function $f_\epsilon(w)$ is convex and is constructed such that the linear runway begins at the point $(X_1, \epsilon)$. An algorithm starting at some point $w_0=X_2$ where $f(w_0) =1$ must traverse the runway from $X_2$ down to $X_1$ to achieve the desired accuracy.

On this runway, the gradient has a constant magnitude $m = \sqrt{2H_0\epsilon}$.
Similarly as before, we have that if 
$K < \frac{X_2-X_1}{\eta m}$, then $w_K > X_1$ and we get $f(w_K)-f^*> \epsilon$. Thus, we need to have
\begin{equation*}
    K \geq \frac{X_2-X_1}{\eta m} = \frac{1-\epsilon}{\eta m^2} = \frac{f(w_0)-f^*-\epsilon}{2 \eta H_0 \epsilon}
\end{equation*}
to achieve $\epsilon$ accuracy for the function value.

Substituting, the upper bound \eqref{eq:eta_upper_bound} for $\eta$ and $H_0=1$, we get the desired result.

\paragraph{Lower bound in the PL case: }
The linear runway construction is not $\mu$-PL. For the third case, we need to construct a different function. We construct a fixed function, independent of $\epsilon$.

Let $C_0>0$ and $0< \mu \leq 1$. We define a fixed connection point $w_c = \sqrt{2C_0/\mu}$. The function is symmetric and defined for $w \geq 0$ as:
\begin{equation}
f(w) = 
\begin{cases} 
\frac{\mu}{2}w^2 & \text{if } 0 \leq w \leq w_c \\
A e^{\sqrt{H_1}(x-w_c)} + B & \text{if } w > w_c 
\end{cases}
\end{equation}
where $A = \sqrt{2C_0\mu/H_1}$ and $B = C_0 - A$ are chosen to ensure the function is $C^1$ at $w_c$. This function is $\mu$-strongly convex (thus also $\mu$-PL) and belongs to the class of $(H_0,H_1)$ functions.

Our goal is to find again a point $w_K$ such that $f(w_K) - f^* \leq \epsilon$.

We analyze the performance of \algname{GD} on the quadratic part of this function, $f(w) = \frac{\mu}{2}w^2$. An algorithm starting at $w_0 = w_c$ will have an initial function value of $f(w_0) = C_0$. The update rule with a fixed step-size $\eta$ is:

$$
 w_{k+1} = w_k - \eta \nabla f(w_k) = w_k - \eta (\mu w_k) = (1 - \eta \mu) w_k. 
$$
After $K$ iterations, we have $w_K = (1-\eta\mu)^K w_0$. We want to find the number of iterations $K$ needed to ensure $f(w_K) \leq \epsilon$.

$$
 f(w_K) = \frac{\mu}{2}w_K^2 = \frac{\mu}{2} (1-\eta\mu)^{2K} w_0^2 = f(w_0) (1-\eta\mu)^{2K} \leq \epsilon .
$$
For this to hold, we need

$$
 f(w_0) (1-\eta\mu)^{2K} \leq \epsilon \implies (1-\eta\mu)^{2K} \leq \frac{\epsilon}{f(w_0)}. 
$$
Taking the logarithm of both sides and using the inequality $\log(1-z) \leq -z$:

$$
 2K \log(1 - \eta\mu) \leq \log\left(\frac{\epsilon}{f(w_0)}\right), \hspace{2mm} \text{if} \hspace{2mm} -2K(\eta\mu) \leq -\log\left(\frac{f(w_0)}{\epsilon}\right). 
$$
Solving for $K$, we get:

$$
 K \geq \frac{1}{2\eta\mu} \log\left(\frac{f(w_0)}{\epsilon}\right). 
$$
Substituting the upper bound \eqref{eq:eta_upper_bound} for the step-size $\eta$ and $f^*=0$, we get the desired lower complexity bound.

\end{proof}

\section{Experimental Details and Additional Ablations}\label{sec:additional_exp}

\subsection{Experimental Setup}

\paragraph{Language Modeling.} Our training of language models is based on the Plain LM GitHub repository \citep{ajroldi2024plainlm} with small changes. The implementation is based on NanoGPT \citep{karpathy2022nanogpt}, and it includes recent improvements such as RMSNorm \citep{zhang2019root}, Rotational Positional Embeddings \citep{su2024roformer}, and SwiGLU activations \citep{shazeer2020glu}. All details are reported in \Cref{tab:summary_experiments}.

\begin{table*}[t]
    \centering
    \caption{Detailed training details of language models and model configurations for the results in \Cref{fig:language_models_comparison}. The implementation is based on \citet{ajroldi2024plainlm}.}
    \label{tab:summary_experiments}
    \resizebox{\textwidth}{!}{
        \begin{tabular}{cccccccc}
            \toprule
            {\bf Model} &
            {\bf Configuration}  &
            \makecellnew{{\bf MLP} \\ {\bf Type}} &
            \makecellnew{{\bf Backbone}} &
            {\bf Normalization} & 
            \makecellnew{{\bf Position} \\ {\bf Embeddings}} &
            {\bf Precision} &
            {\bf Dropout}
            
            \\ \toprule

            70M & 
            \makecellnew{\# Layers: 6 \\ 
            \# heads: 8 \\
            hidden size: 512 \\ 
            seq. length: 1024 \\
            batch size: 256\\ 
            weight decay: 0\\
            cooldown steps: 20 \% \\
            grad clip: 1.0 \\ 
            tokens: 1.2B} &
            \makecellnew{SwiGLU \\ \citep{shazeer2020glu} }&
            \makecellnew{PreLN transformer \\ \citep{xiong2020layer} \\ with skip connections}& 
            \makecellnew{RMSnorm \\ \citep{zhang2019root}} &
            \makecellnew{MLP and Attention \\
            layers with variance: \\
            $0.02/\sqrt{\# \text{ layers}}$\\
            Other layers: \\ 
            $0.02$ std. dev. \\
            Biases are always \\ 
            initialized at zero} &
            \makecellnew{Mixed \\ precision \\ FP16 } & 
            \makecellnew{Disabled for both\\  hidden and \\
            attention layers}
            \\ \midrule

            160M & 
            \makecellnew{\# Layers: 12 \\ 
            \# heads: 12 \\
            hidden size: 1024\\
            seq. length: 2048 \\
            batch size: 256\\ 
            weight decay: 0.1\\
            cooldown steps: 20 \% \\
            grad clip: 1.0 \\ 
            tokens: 1.2B} &
            \makecellnew{SwiGLU \\ \citep{shazeer2020glu} }&
            \makecellnew{PreLN transformer \\ \citep{xiong2020layer} \\ with skip connections}& 
            \makecellnew{RMSnorm \\ \citep{zhang2019root}} &
            \makecellnew{MLP and Attention \\
            layers with variance: \\
            $0.02/\sqrt{\# \text{ layers}}$\\
            Other layers: \\ 
            $0.02$ std. dev. \\
            Biases are always \\ 
            initialized at zero} &
            \makecellnew{Mixed \\ precision \\ FP16 } & 
            \makecellnew{Disabled for both\\  hidden and \\
            attention layers}
            \\ \midrule

            410M & 
            \makecellnew{\# Layers: 6 \\ 
            \# heads: 8 \\
            hidden size: 512 \\
            seq. length: 2048 \\
            batch size: 256\\ 
            weight decay: 0.1 \\
            cooldown steps: 20 \% \\
            grad clip: 1.0 \\
            tokens: 3.2B} &
            \makecellnew{SwiGLU \\ \citep{shazeer2020glu} }&
            \makecellnew{PreLN transformer \\ \citep{xiong2020layer} \\ with skip connections}& 
            \makecellnew{RMSnorm \\ \citep{zhang2019root}} &
            \makecellnew{MLP and Attention \\
            layers with variance: \\
            $0.02/\sqrt{\# \text{ layers}}$\\
            Other layers: \\ 
            $0.02$ std. dev. \\
            Biases are always \\ 
            initialized at zero} &
            \makecellnew{Mixed \\ precision \\ FP16 } & 
            \makecellnew{Disabled for both\\  hidden and \\
            attention layers}
            \\

            \bottomrule

        \end{tabular}
        } 
\end{table*}

\paragraph{Image Classification.} The implementation of vision tasks is based on the GitHub repository \citep{ajroldi2025vision} with minor changes. Similarly, we report the training details of ViT training in \Cref{tab:summary_experiments_vit}. It includes LayerNorm \citep{ba2016layer}, GELU activations \citep{hendrycks2016gaussian}, and drop path.

\begin{table*}[t]
    \centering
    \caption{Detailed training details of image classification and model configurations for the results in \Cref{fig:vit_comparison}. The implementation is based on \citet{ajroldi2025vision}.}
    \label{tab:summary_experiments_vit}
    \resizebox{\textwidth}{!}{
        \begin{tabular}{ccccccc}
            \toprule
            {\bf Model} &
            {\bf Configuration}  &
            \makecellnew{{\bf MLP} \\ {\bf Type}} &
            \makecellnew{{\bf Backbone}} &
            {\bf Normalization} & 
            \makecellnew{{\bf Position} \\ {\bf Embeddings}} &
            \makecellnew{{\bf Stochastic Depth} \\ {\bf via DropPath}}
            
            \\ \toprule

            ViT-Tiny & 
            \makecellnew{\# Patch size: 4 \\ 
            \# heads: 8 \\
            Embedding size: 192 \\ 
            \# layers: 12 \\
            \# heads: 3\\ 
            MLP ratio: 3\\
            Class token: True\\
            Drop path rate: 0.1\\
            grad clip: Null} &
            \makecellnew{GELU \\ \citep{hendrycks2016gaussian} }&
            \makecellnew{PreLN transformer \\ \citep{xiong2020layer} \\ with skip connections}& 
            \makecellnew{LayerNorm \\ \citep{ba2016layer}} &
            \makecellnew{LayerNorm: 1\\
            Biases: 0\\
            Other layers: \\ 
            $0.02$ std. dev.} & 
            \makecellnew{Residual branches are \\ randomly dropped with \\ a linearly increasing drop\\ rate across depth}
            \\ \midrule

            \bottomrule

        \end{tabular}
        } 
\end{table*}

\begin{remark}
    The results in Figures~\ref{fig:llm_verification} and \ref{fig:resnet50_vit_verification} are done with gradient clipping 1.0 and a small LR $10^{-4}$ to make small steps in the loss landscape from the initialization. Such an approach allows for tracking better the smoothness-loss dependency around the initialization. 
\end{remark}

\subsection{Effect of Initialization}\label{sec:initialization}
In this section, we examine how the choice of initialization influences the empirical verification of our condition on the 70M language model trained with clipping 1.0 and a fixed learning rate. We evaluate two initialization strategies. The first follows the approach used in modern GPT-style architectures, where the variance of the weights in the MLP or GLU blocks and in the attention output layers is scaled as $0.02/\textrm{n_layer}$, with 
$\textrm{n_layer}$ denoting the layer index. We refer to this strategy as ``scaling with depth.'' The second strategy uses a fixed variance of $0.02$, independent of depth, which is the default choice in many implementations. We refer to this strategy as ``no scaling with depth.'' We estimate the smoothness throughout training using the same methodology as in \Cref{sec:experiments_main}. We highlight that in this set of experiments, we use gradient clipping to $1$, which is a standard training technique used in practice, which allows to restrict the update magnitude and make constant steps in the landscape.

We present the results in \Cref{fig:init_study} varying the fixed learning rate hyperparameter of \algname{SGD}. We observe the following results
\begin{itemize}
    \item After an initial sharpness reduction (flattening) phase, \algname{SGD} enters a gradual sharpening phase where the sharpness grows. In this regime, our condition does not describe the smoothness well anymore. After the gradual sharpening phase, \algname{SGD} enters \algname{EoS} stage where the sharpness oscillates around $2/\textrm{LR}$ stability threshold (only observed for LR 1e-2). 
    \item Importantly, the transition from sharpness reduction to the gradual sharpening phase happens at the same loss value regardless of LR choice, which indicates a strong connection between the smoothness and the loss value.
    \item Our condition describes well the sharpness reduction phase, also observed in \citep{kalra2024warmup}. In contrast to that work, we describe the reduction phase analytically. Using the proposed learning-rate warm-up strategy can avoid instabilities due to high values of the sharpness at the beginning and lead to better final performance.
    \item ``Scaling with depth'' strategy initializes the model closer to the origin. This results in a larger initial sharpness in comparison with ``no scaling with depth'' scheme. This aligns with our theoretical calculations in Proposition \ref{prop:1x1x1 network}. We hypothesize that GPT-style initialization requires a longer learning-rate warm-up phase due to such high values of the sharpness at the beginning.
\end{itemize}

\begin{figure}
    \centering
    \begin{tabular}{ccc}
    \hspace{-5mm}\makecellnew{\includegraphics[width=0.33\linewidth]{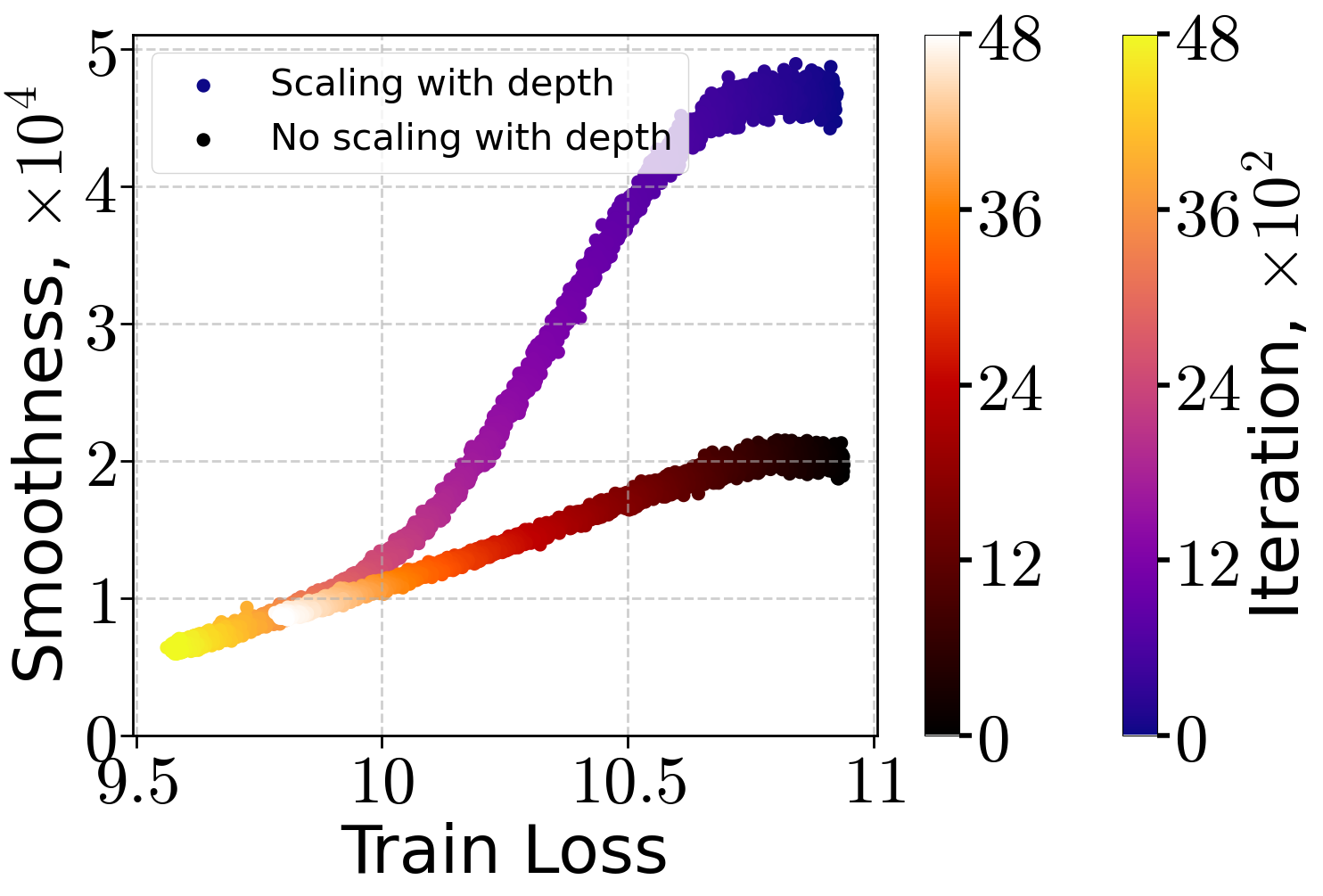} \\ LR 1e-4
    } 
    &
    \hspace{-8mm}
    \makecellnew{\includegraphics[width=0.33\linewidth]{figures/rebuttals_70M_sgd_1e-3_init.png}\\ LR 1e-3} 
    &
    \hspace{-8mm}
    \makecellnew{\includegraphics[width=0.33\linewidth]{figures/rebuttals_70M_sgd_1e-2_init.png}\\ LR 1e-2}
    \end{tabular}
    \caption{Training of 70M model on FineWeb dataset with \algname{SGD} varying fixed learning rate and initialization scheme. Colored points correspond to depth-scaled initialization, while black points correspond to fixed-variance initialization. Color indicates training progress.}
    \label{fig:init_study}
\end{figure}

\subsection{Additional Results on Verification of the Proposed Condition}

\subsubsection{Results Varying Random Seed}

In this section, we demonstrate that the obtained results in \Cref{fig:llm_verification,fig:resnet50_vit_verification} are consistent when changing the random seed. Random seed changes the initialization of the models, thus leading to exploration of various parts of the landscape. We report the results in \Cref{fig:llm_verification_appendix}. Across these runs, we observe an approximately linear decrease in the smoothness proxy with training loss early in training.

\begin{figure}
    \centering
    \begin{tabular}{ccc}
    \includegraphics[width=0.3\linewidth]{figures/70M_sgd_1e-4_100.png} & 
    \includegraphics[width=0.3\linewidth]{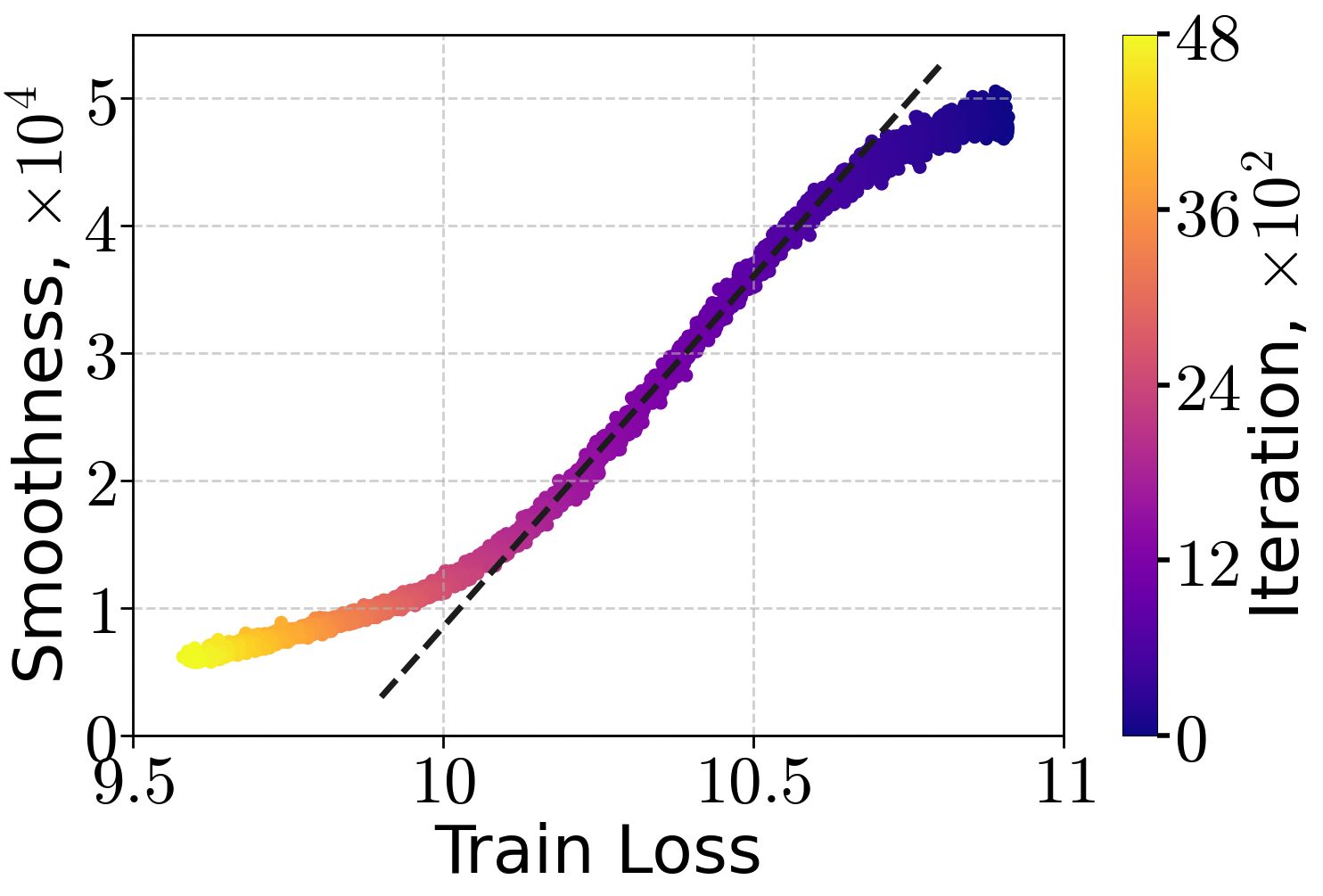} & 
    \includegraphics[width=0.3\linewidth]{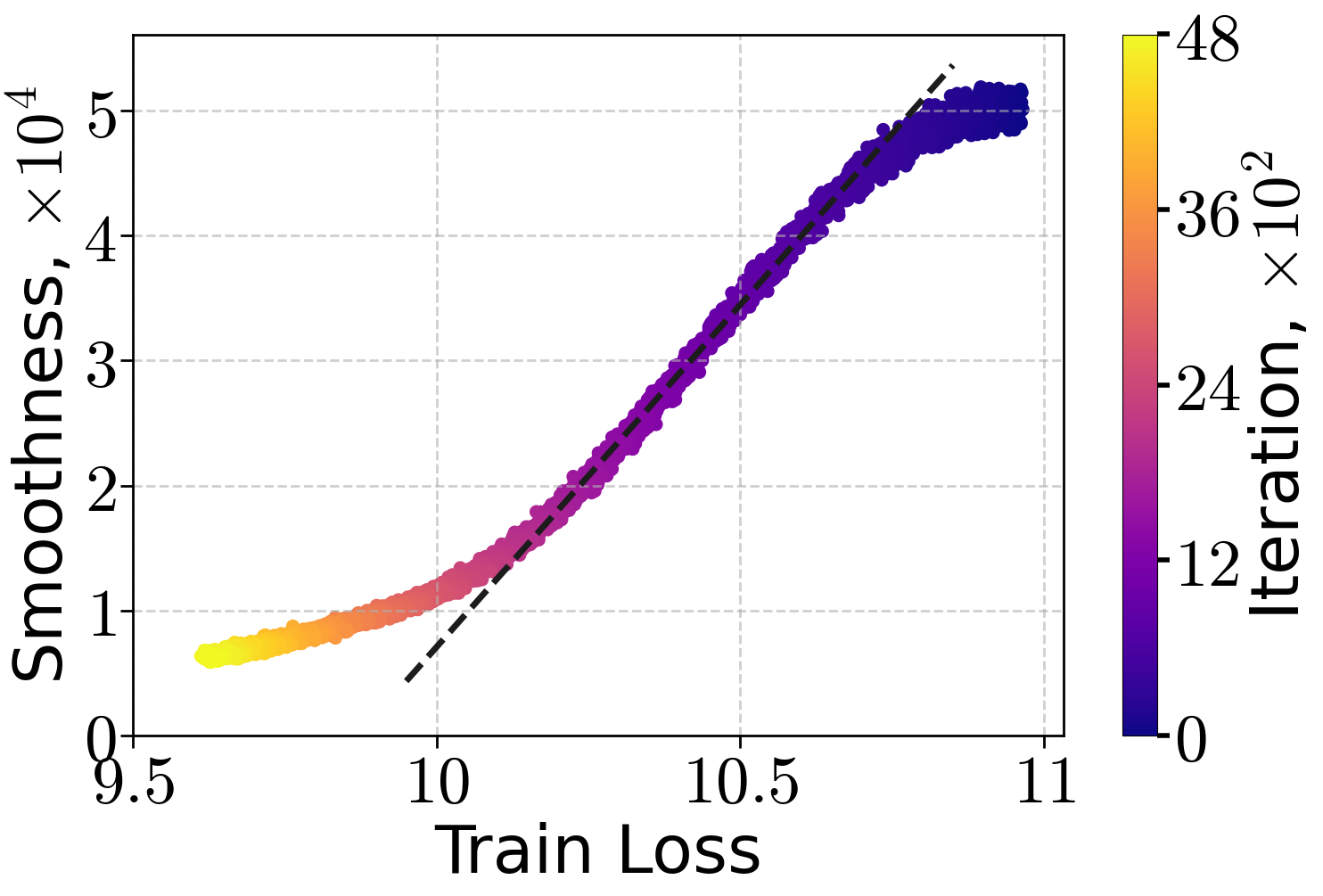} 
    \\
    &
    70M on FineWeb
    &
    \\
    \includegraphics[width=0.3\linewidth]{figures/160M_sgd_1e-4_100.png} &
    \includegraphics[width=0.3\linewidth]{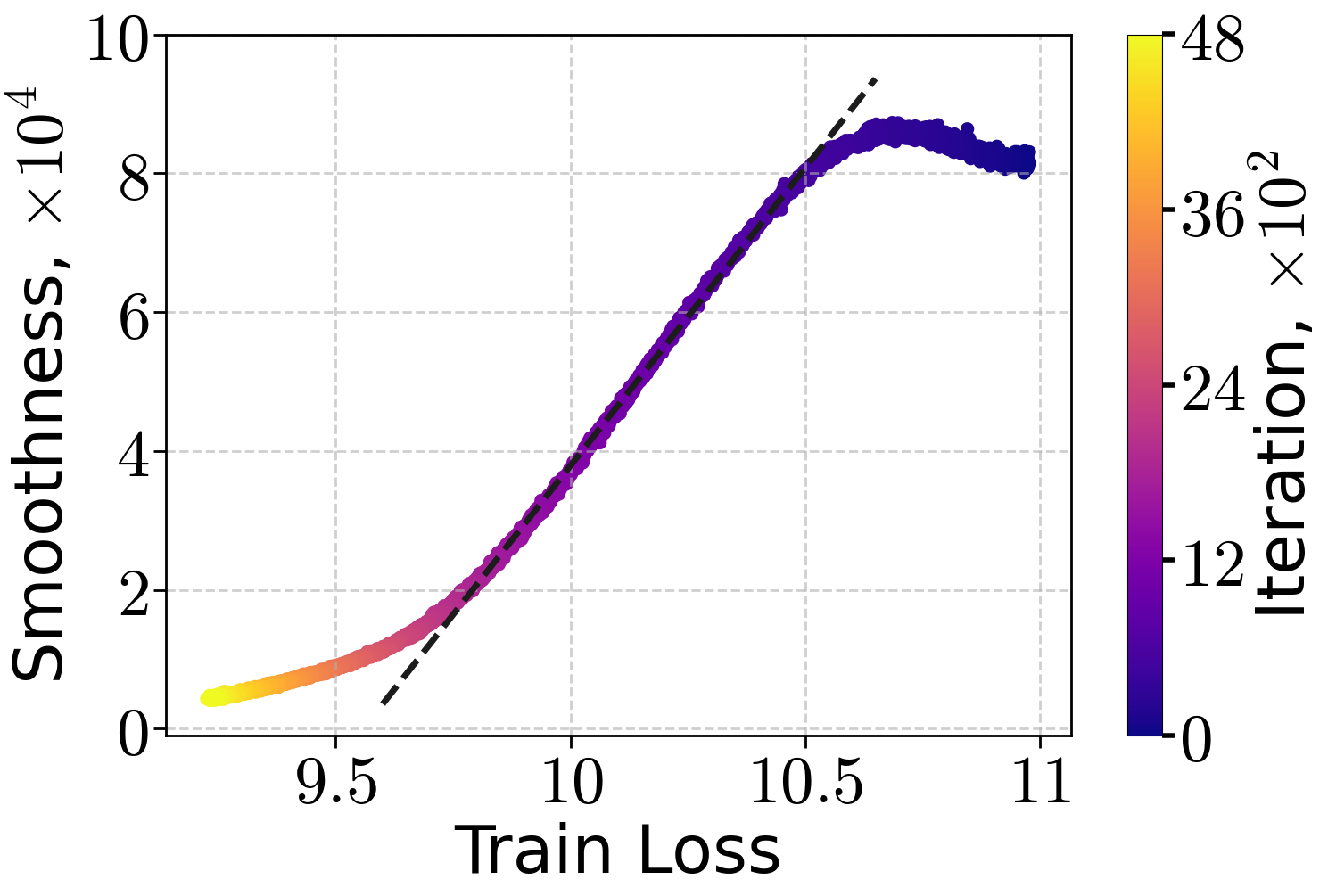} &
    \includegraphics[width=0.3\linewidth]{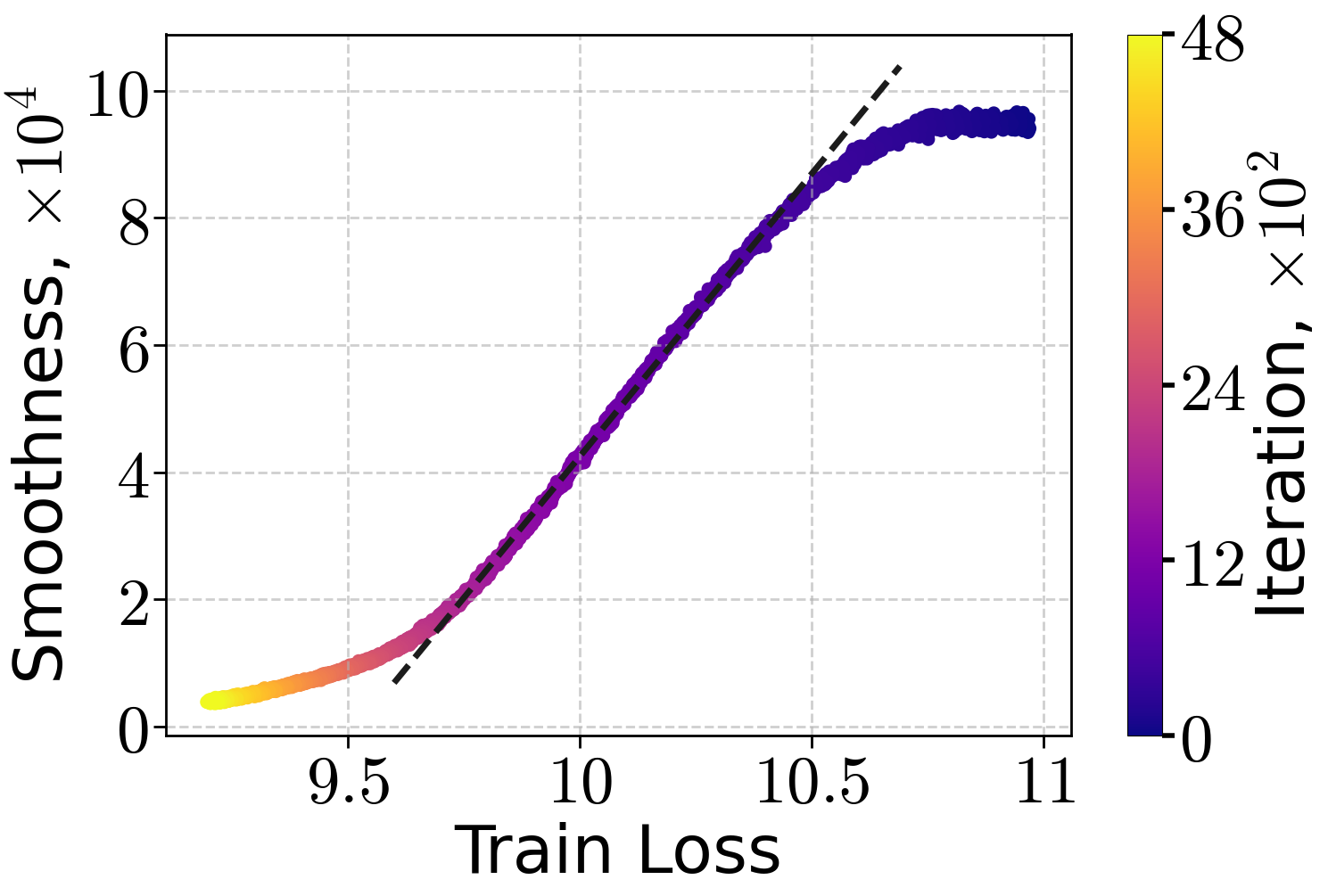} 
    \\
    &
    160M on FineWeb
    &
    \\
    \includegraphics[width=0.3\linewidth]{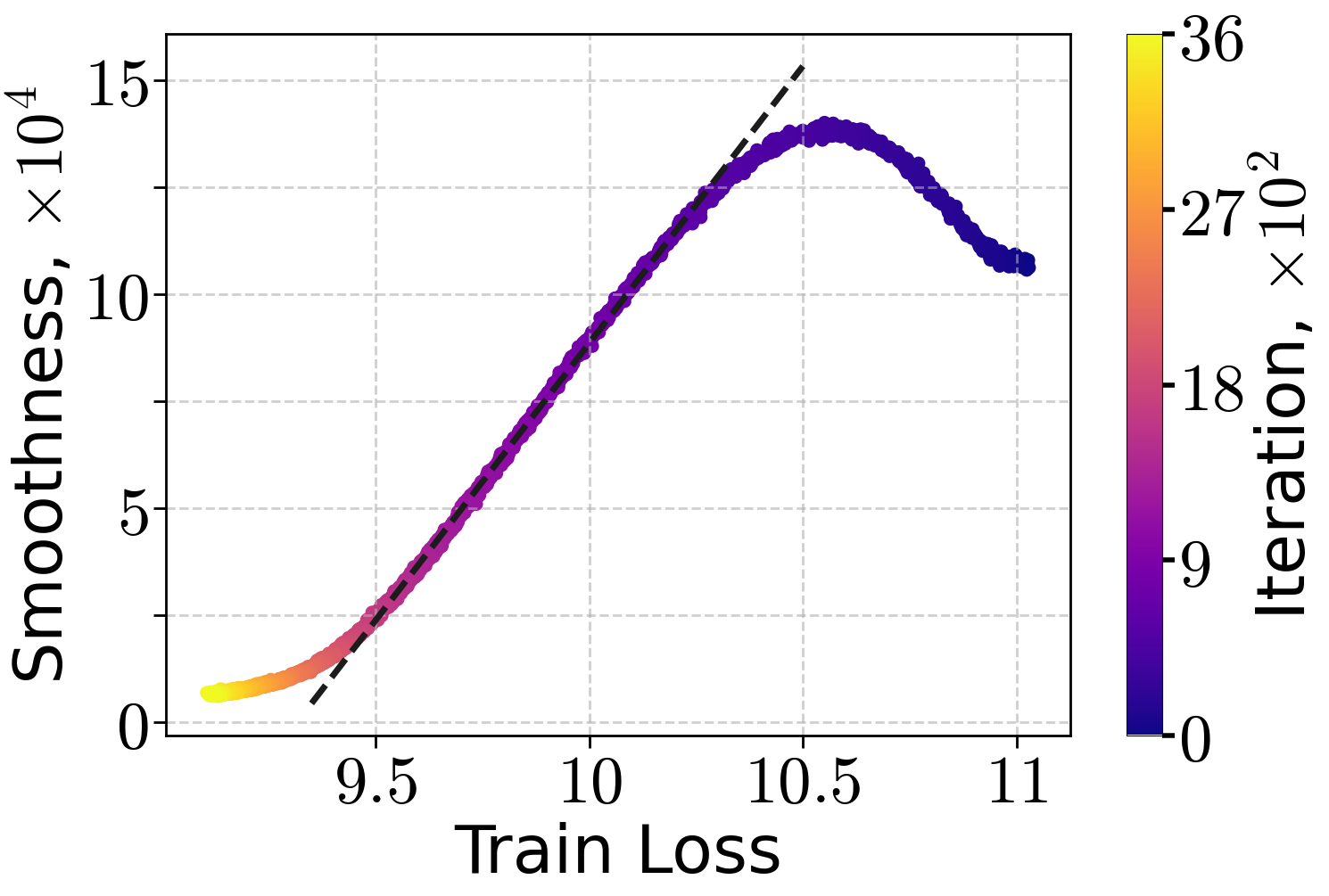} & 
    \includegraphics[width=0.3\linewidth]{figures/410M_sgd_1e-4_101.png} & 
    \includegraphics[width=0.3\linewidth]{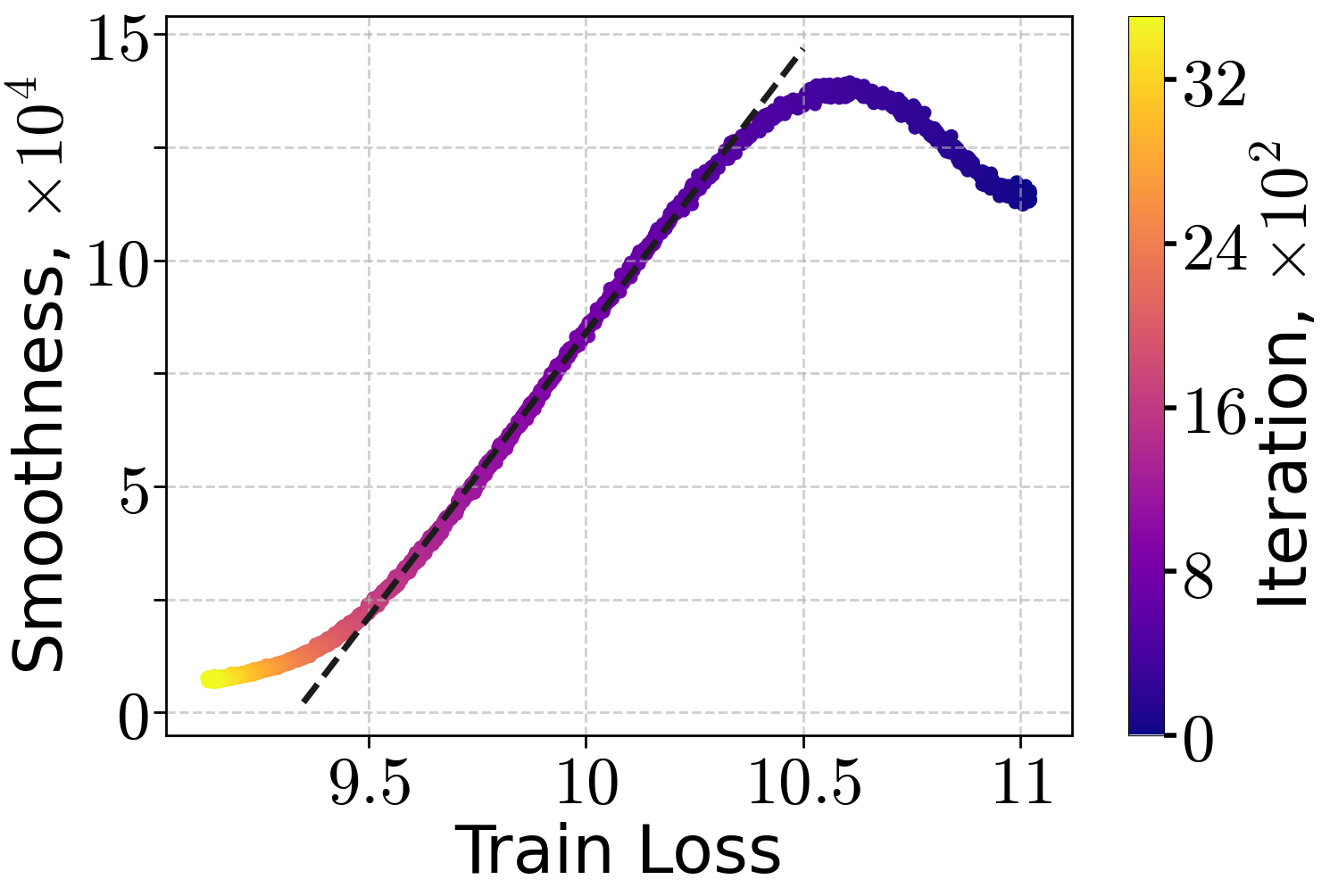} 
    \\
    &
    410M on FineWeb
    &
    \end{tabular}
    \caption{Local smoothness approximation versus training loss for language models of varying sizes and random seed on the FineWeb dataset. Models are trained with \algname{SGD} at a constant learning rate of $10^{-4}$. Each dot represents the estimated local smoothness and stochastic training loss at a given iteration, with color indicating training progress, while the black dashed line shows the best linear fit. For much of early training, the relation is well-approximated by a line, aside from the very initial phase where smoothness behaves differently. This deviation likely arises because the linear fit reflects only an upper bound, suggesting that a more complex functional dependence may be necessary.}
    \label{fig:llm_verification_appendix}
\end{figure}

\subsubsection{Verification with Adam}

Next, we switch to \algname{Adam} optimizer to verify the proposed $(H_0, H_1)$-smoothness condition. We test the results on language models of size 70M, 160M, and 410M. The results are reported in \Cref{fig:adam_llm_verification_appendix}. Similar to the setting in the main body, we use a small constant learning rate $10^{-7}$, which allows moving slowly in the landscape. We observe that \algname{Adam} also demonstrates a linear dependency between local smoothness approximation and train loss. However, we observe that \algname{Adam} stays in this linear decaying part of the landscape for fewer iterations, especially for larger models, than \algname{SGD} does. This might suggest that for \algname{Adam} the warm-up phase should be shorter.

\begin{figure}
    \centering
    \begin{tabular}{ccc}
    \includegraphics[width=0.3\linewidth]{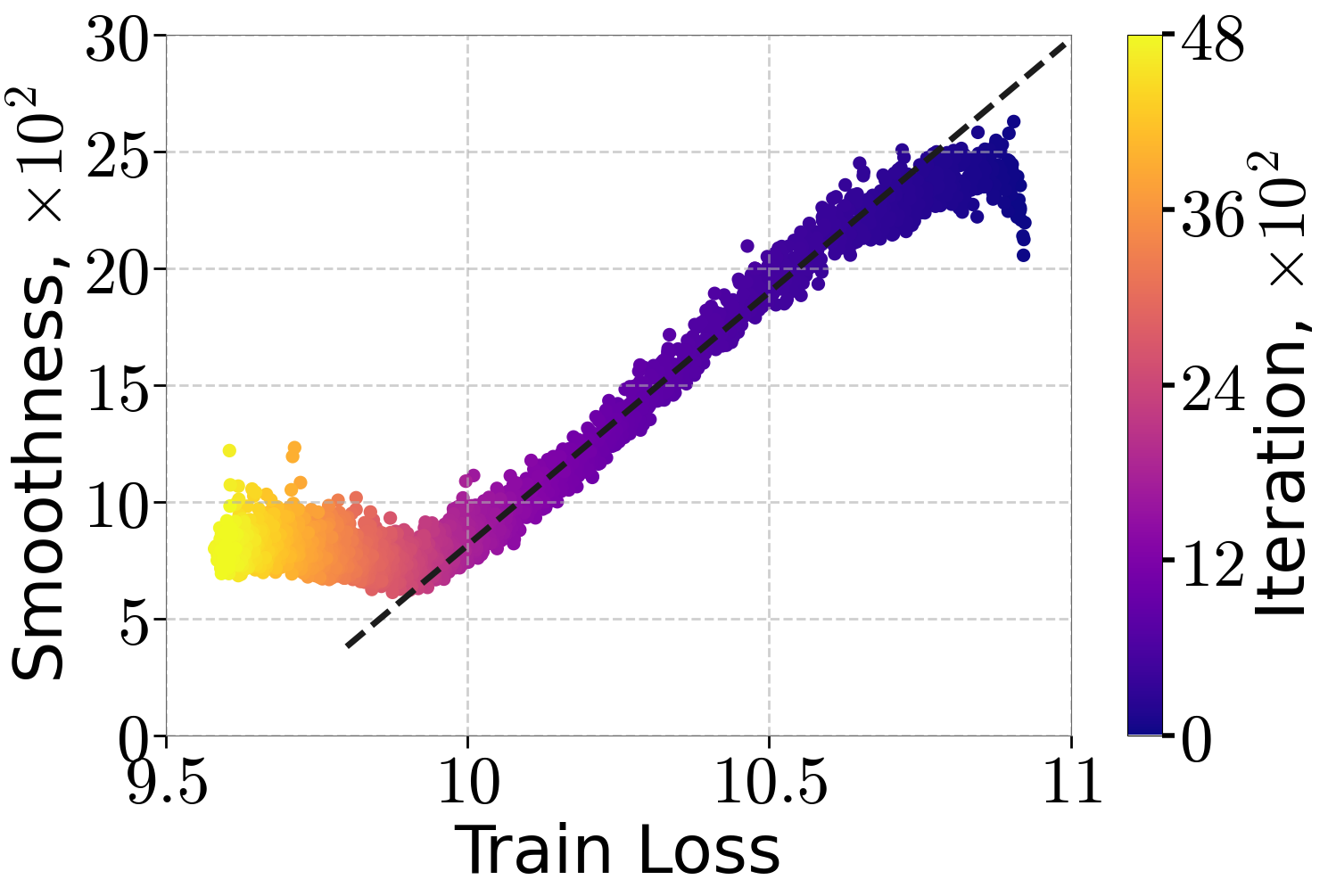} & 
    \includegraphics[width=0.3\linewidth]{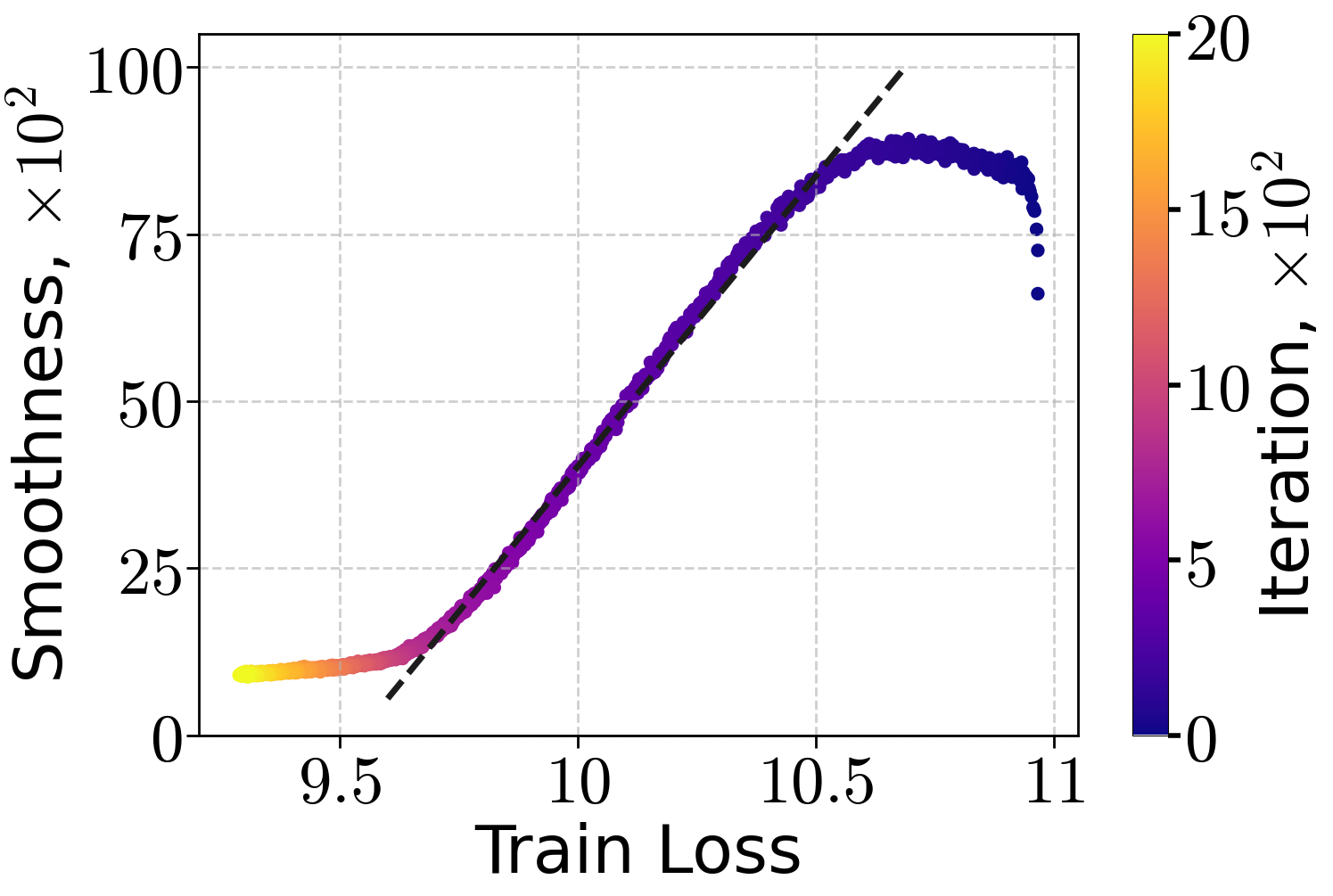} & 
    \includegraphics[width=0.3\linewidth]{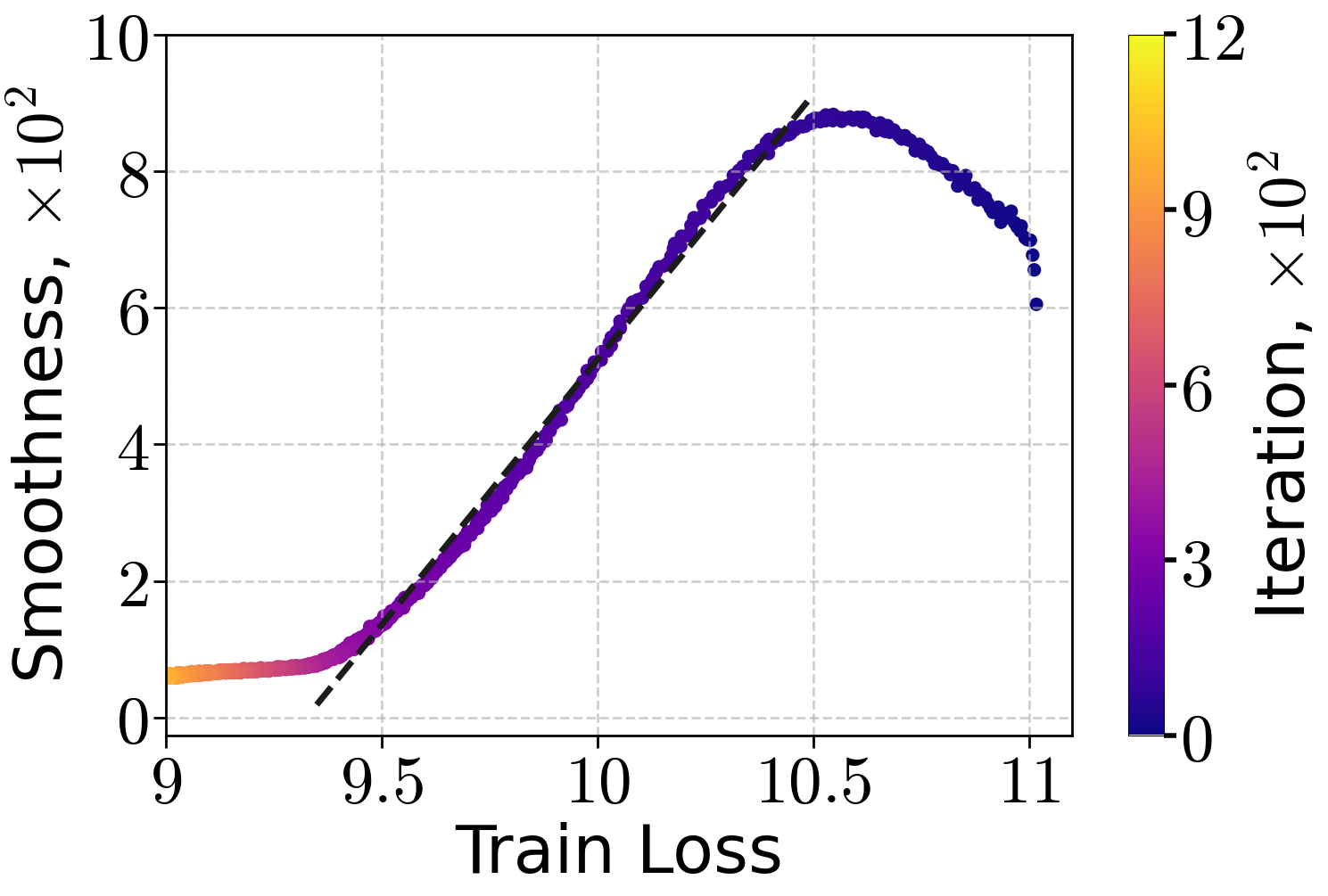} 
    \\
    70M on FineWeb &
    160M on FineWeb &
    410M on FineWeb
    \end{tabular}
    \caption{Local smoothness approximation versus training loss for language models of varying sizes and random seed on the FineWeb dataset. Models are trained with \algname{Adam} at a constant learning rate of $10^{-7}$. Each dot represents the estimated local smoothness and stochastic training loss at a given iteration, with color indicating training progress, while the black dashed line shows the best linear fit. For much of early training, the relation is well-approximated by a line, aside from the very initial phase where smoothness behaves differently. This deviation likely arises because the linear fit reflects only an upper bound, suggesting that a more complex functional dependence may be necessary.}
    \label{fig:adam_llm_verification_appendix}
\end{figure}

\subsection{Ablation Studies}

\subsubsection{Performance Varying Warm-up Length}

\paragraph{Language Modeling.} In this section, we investigate how warm-up length influences training. As shown in Figures \ref{fig:70M_llm_warmup_ablation_warmup}-\ref{fig:410M_llm_warmup_ablation_warmup}, using a 10–20\% linear warm-up yields the best validation perplexity, demonstrating that warm-up improves the final performance of the models. We also find that warm-up enables convergence even with relatively large peak learning rates $10^{-2}$ for the 70M model and $3 \cdot 10^{-3}$ for the 160M model, whereas training without warm-up performs significantly worse at these values. Similar trends have been reported by \citet{wortsman2023small}. Finally, we observe that the $(H_0, H_1)$ warm-up is less robust to the choice of peak learning rate for the 70M model, resulting in higher validation perplexity. However, once the peak learning rate is properly tuned (within $10^{-3}$–$3 \cdot 10^{-3}$), it becomes less sensitive to the choice of the constant $C$.

\begin{figure}
    \centering
    \begin{tabular}{cccc}
    \hspace{-3mm}\includegraphics[width=0.23\linewidth]{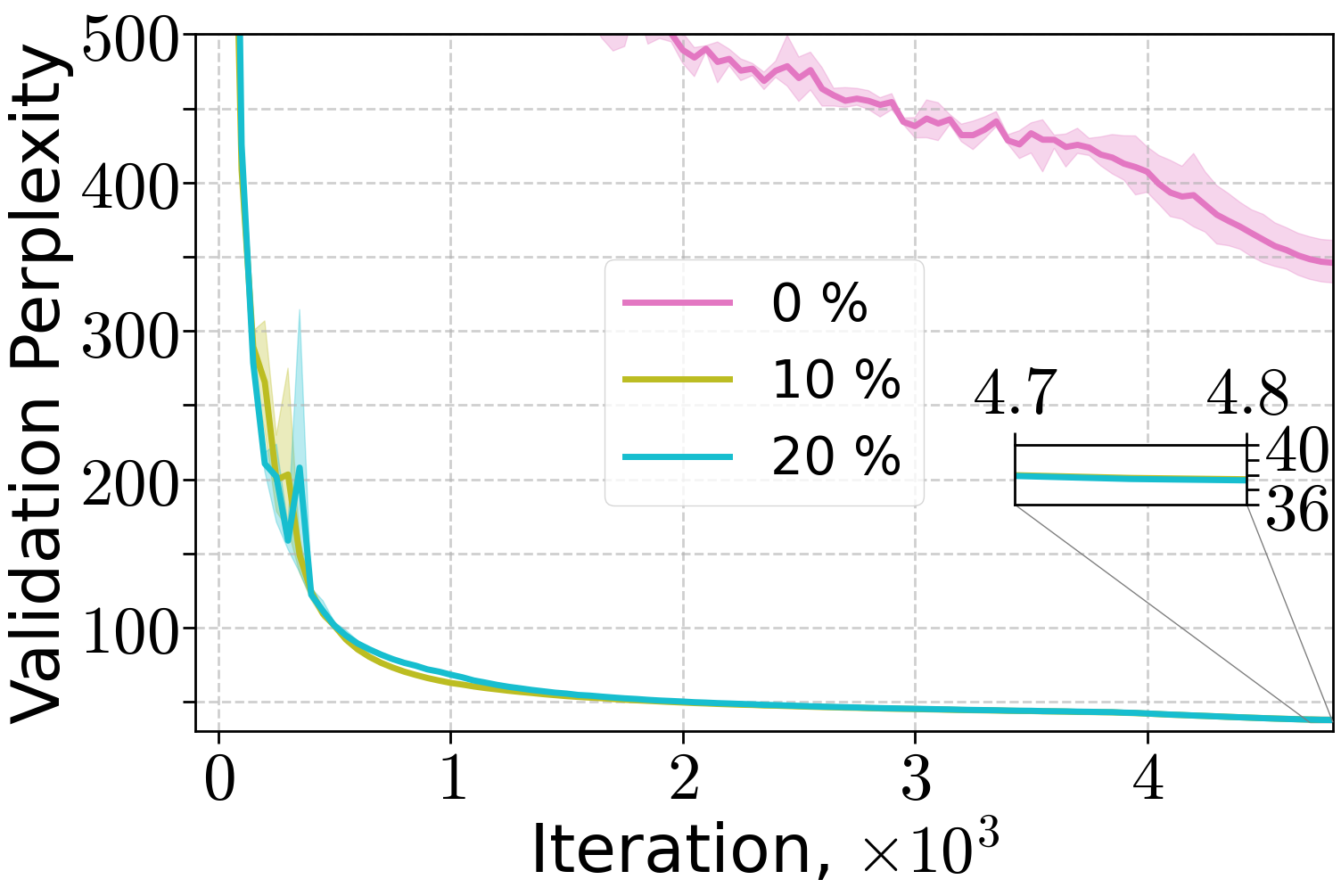} &
    \includegraphics[width=0.23\linewidth]{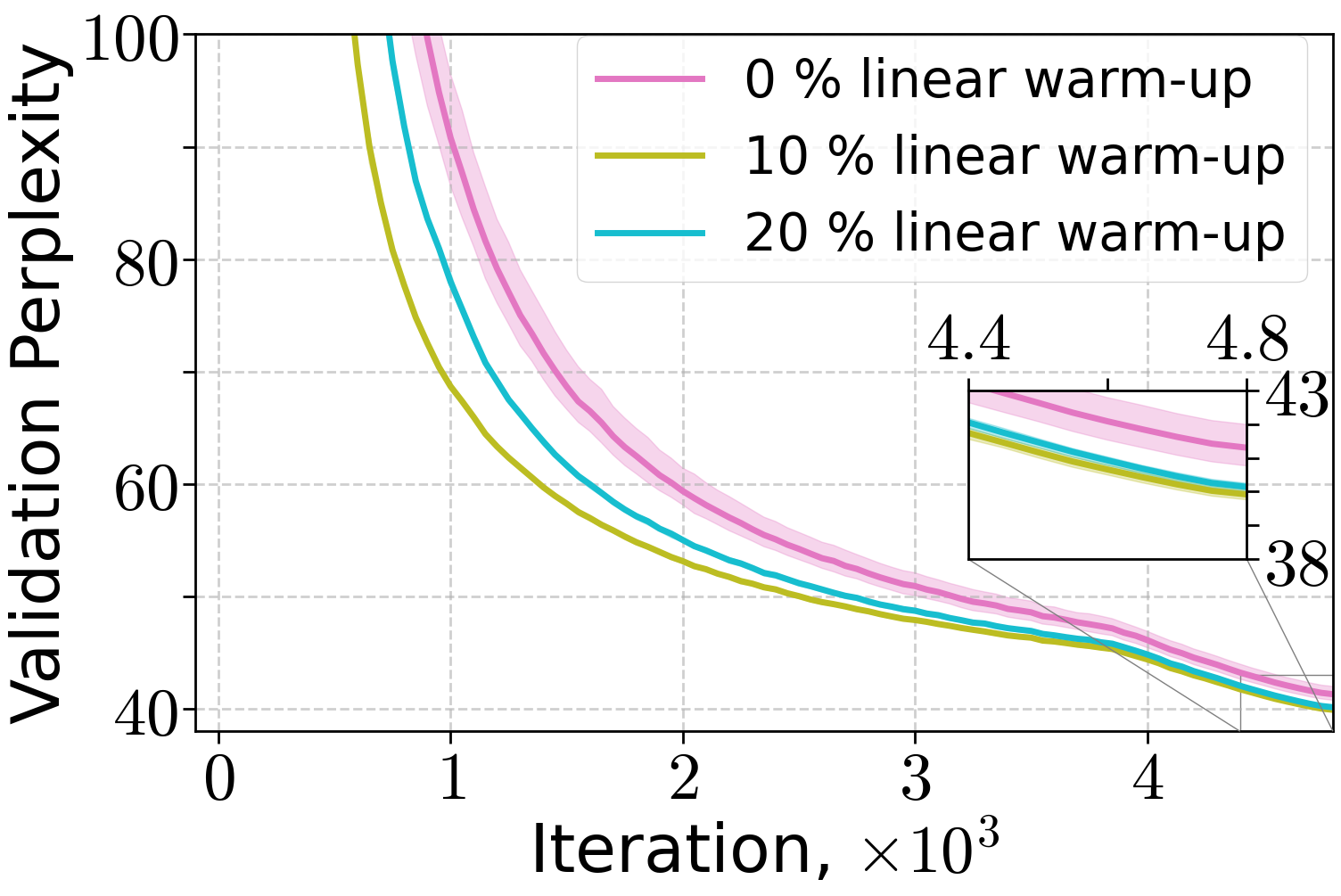} & 
     \includegraphics[width=0.23\linewidth]{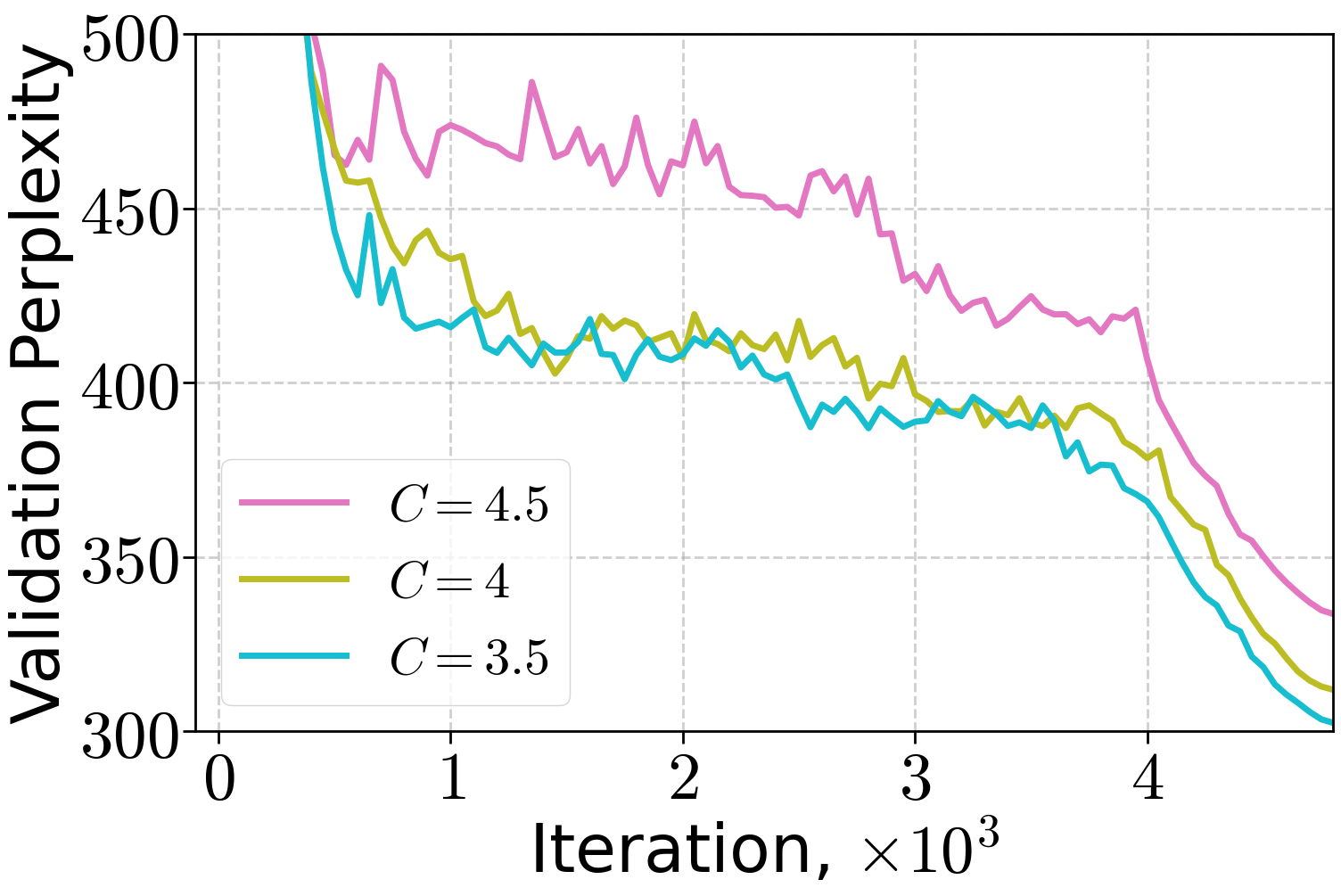} &
     \includegraphics[width=0.23\linewidth]{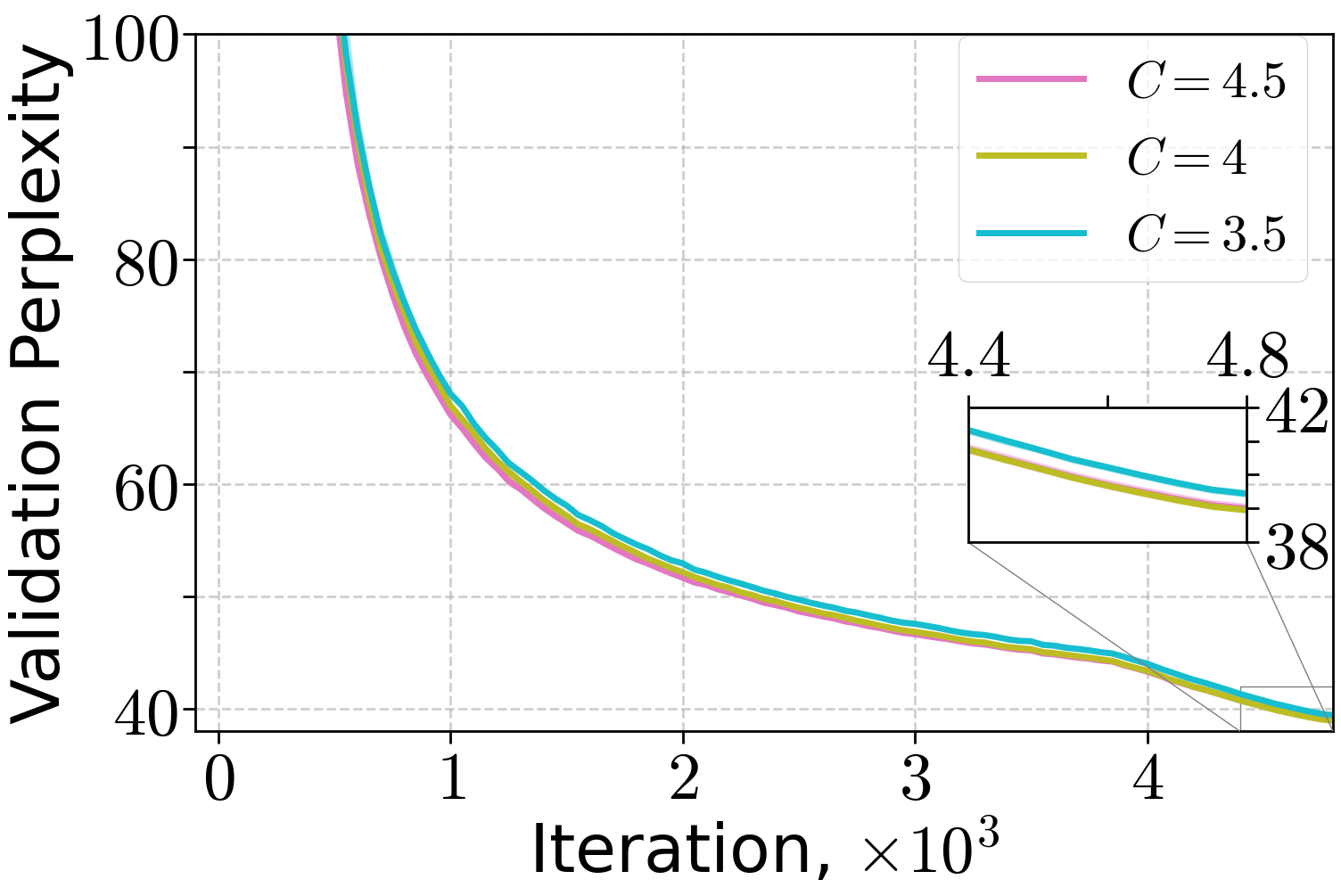} 
    \\
    \makecellnew{Linear warm-up\\ LR $10^{-2}$ }&
    \makecellnew{Linear warm-up \\ LR $10^{-3}$} &
    \makecellnew{$(H_0, H_1)$ warm-up\\ LR $10^{-2}$} &
     \makecellnew{$(H_0, H_1)$ warm-up\\ LR $10^{-3}$}
    \end{tabular}
    
    \caption{Training of 70M language model on FineWeb dataset varying the length of linear warm-up (two left figures) and threshold $C$ of $(H_0, H_1)$ warm-up (two right figures) for the peak learning rate $10^{-2}$ and $10^{-3}.$}
    \label{fig:70M_llm_warmup_ablation_warmup}
\end{figure}

\begin{figure}
    \centering
    \begin{tabular}{cccc}
    \hspace{-3mm}\includegraphics[width=0.23\linewidth]{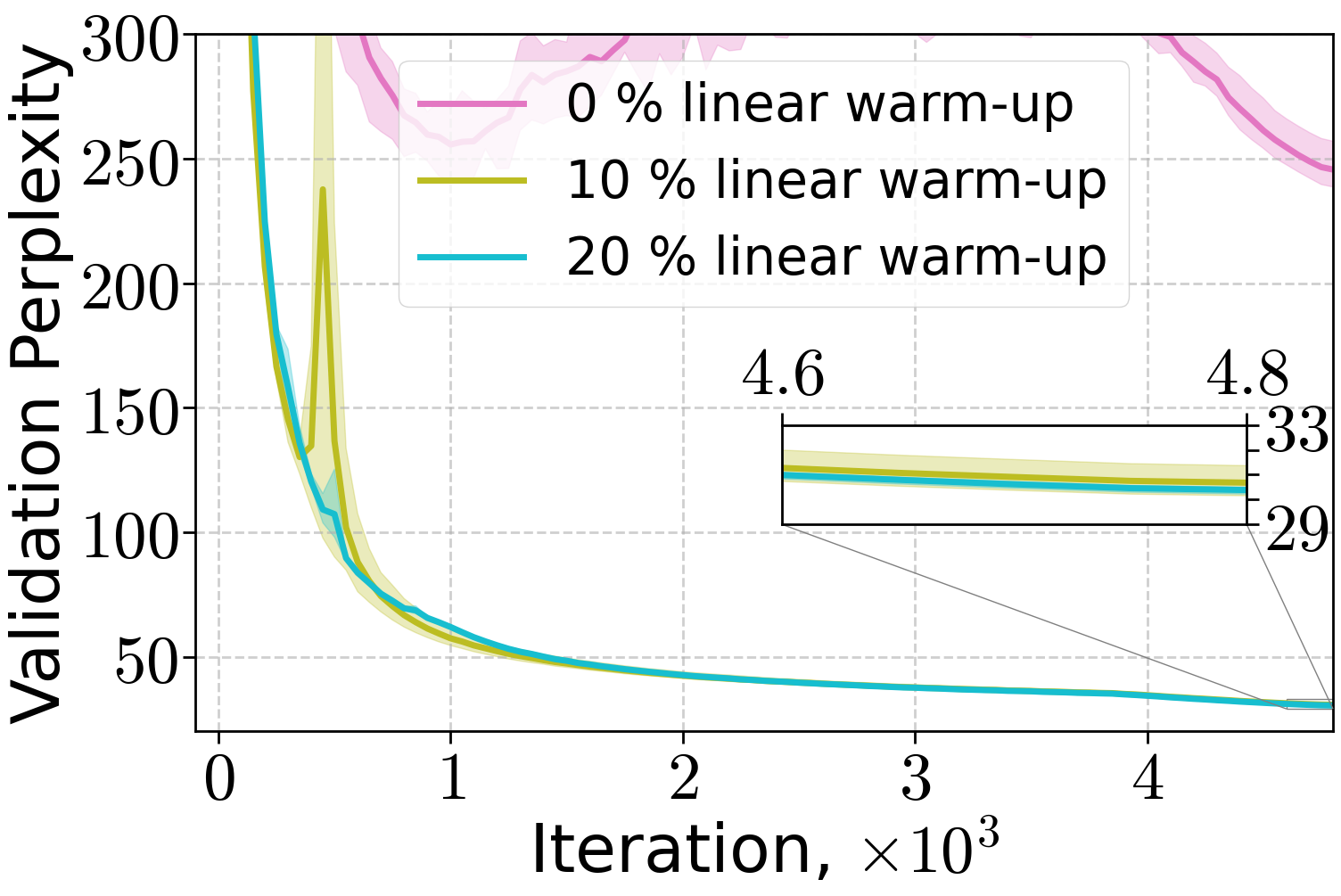} &
    \includegraphics[width=0.23\linewidth]{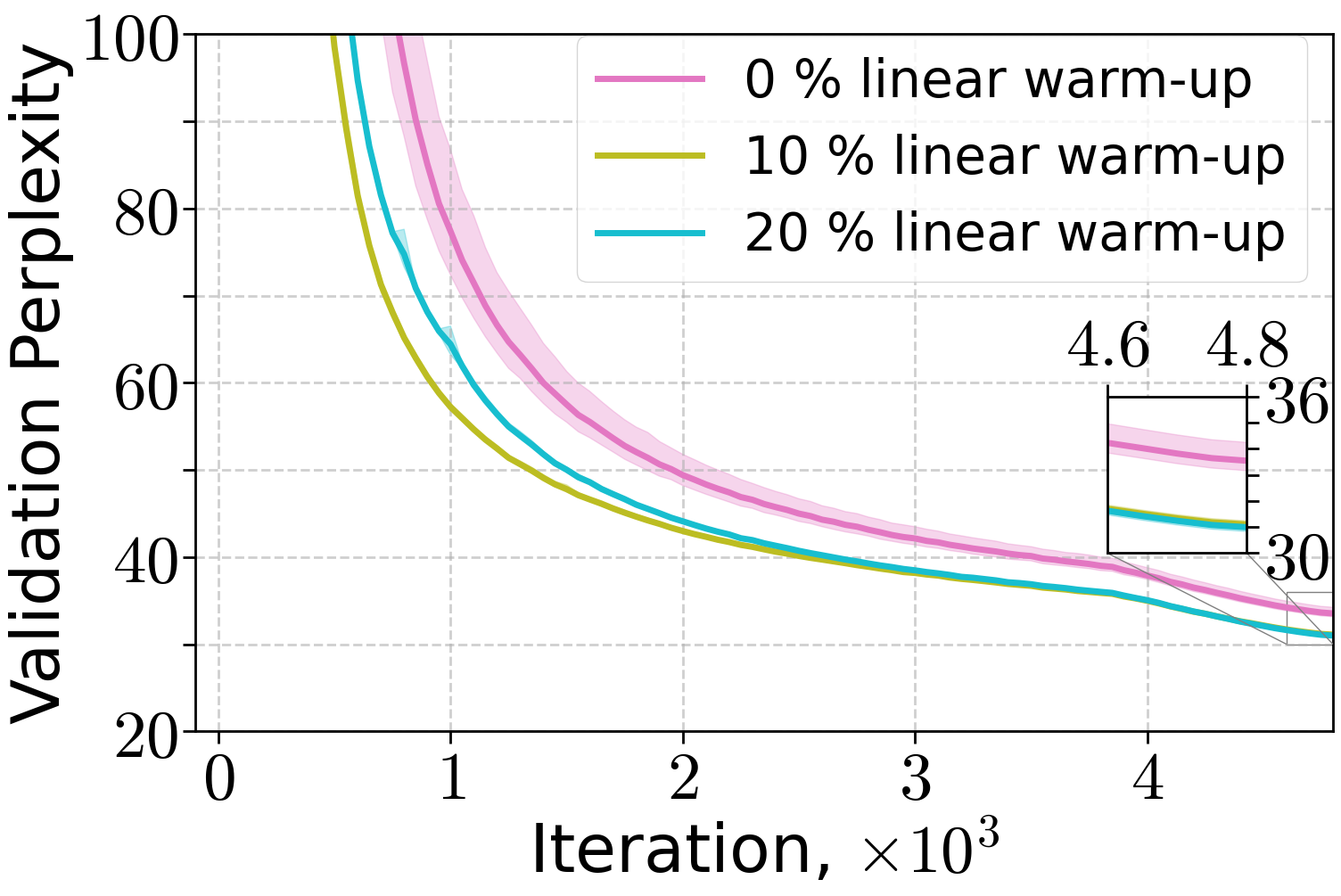} & 
     \includegraphics[width=0.23\linewidth]{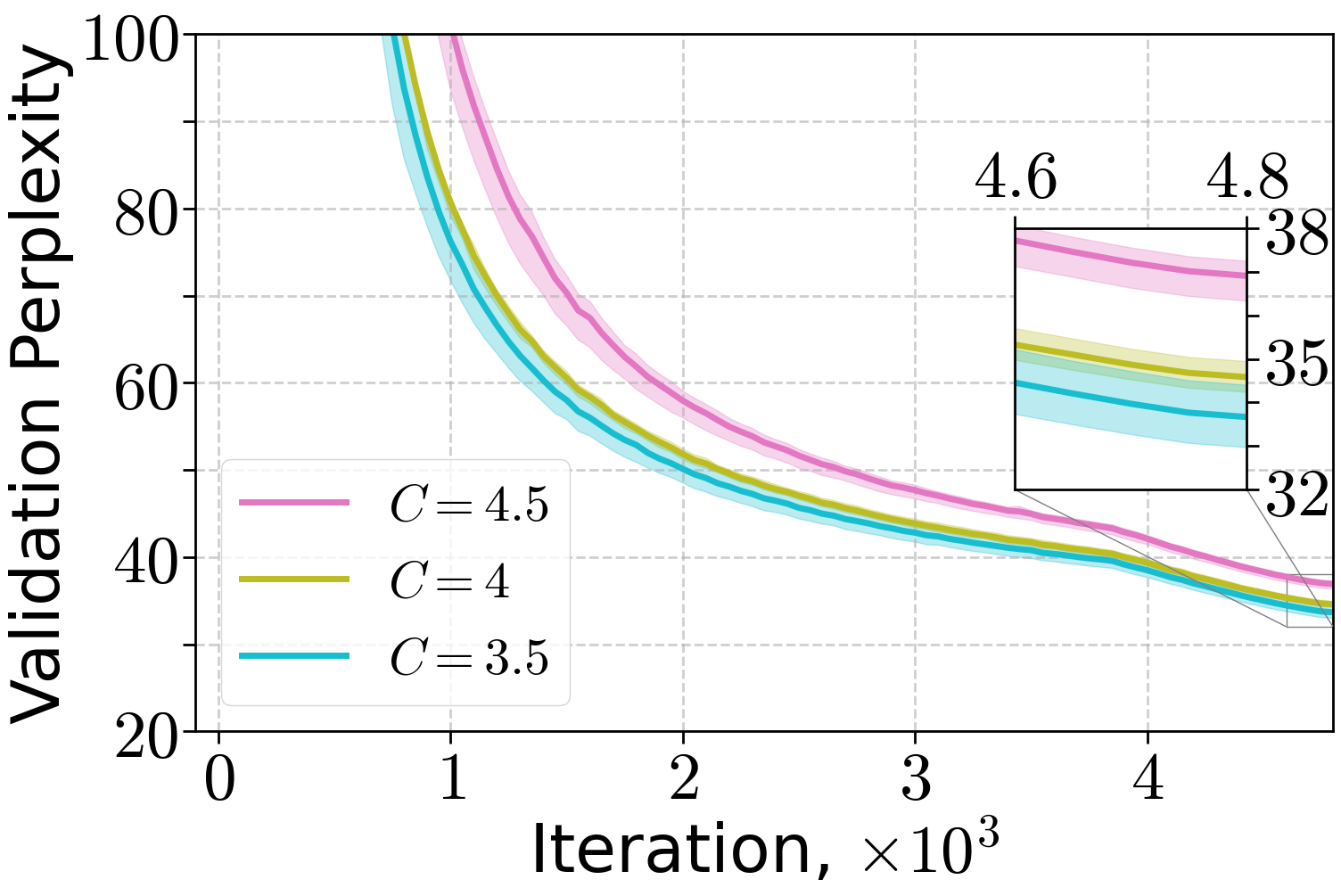} &
     \includegraphics[width=0.23\linewidth]{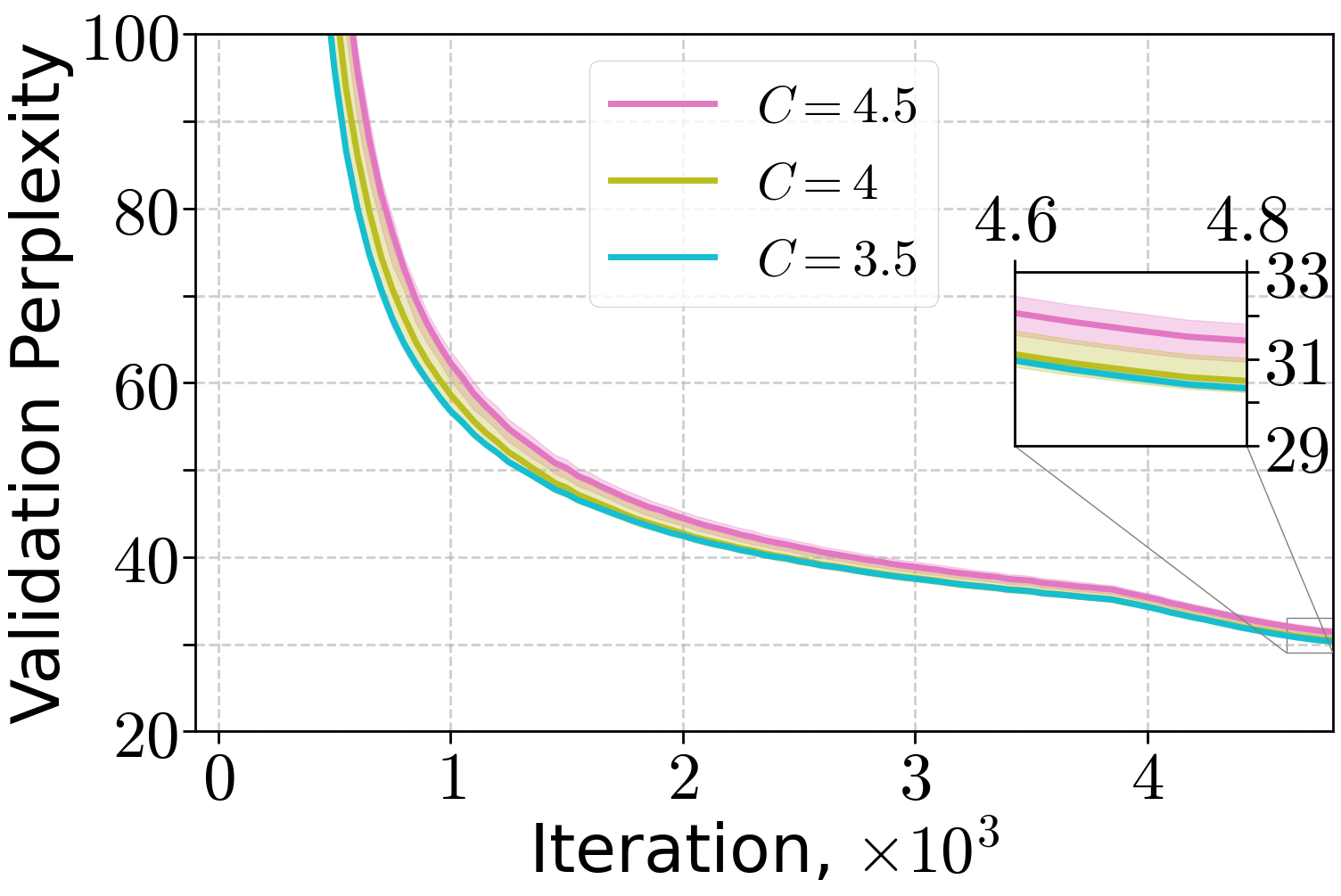} 
    \\
    \makecellnew{Linear warm-up\\ LR $3\cdot 10^{-3}$ }&
    \makecellnew{Linear warm-up \\ LR $10^{-3}$} &
    \makecellnew{$(H_0, H_1)$ warm-up\\ LR $3\cdot 10^{-3}$} &
     \makecellnew{$(H_0, H_1)$ warm-up\\ LR $10^{-3}$}
    \end{tabular}
    
    \caption{Training of 160M language model on FineWeb dataset varying the length of linear warm-up (two left figures) and threshold $C$ of $(H_0, H_1)$ warm-up (two right figures) for the peak learning rate $3\cdot 10^{-3}$ and $10^{-3}.$}
    \label{fig:160M_llm_warmup_ablation_warmup}
\end{figure}

\begin{figure}
    \centering
    \begin{tabular}{cccc}
    \hspace{-3mm}\includegraphics[width=0.23\linewidth]{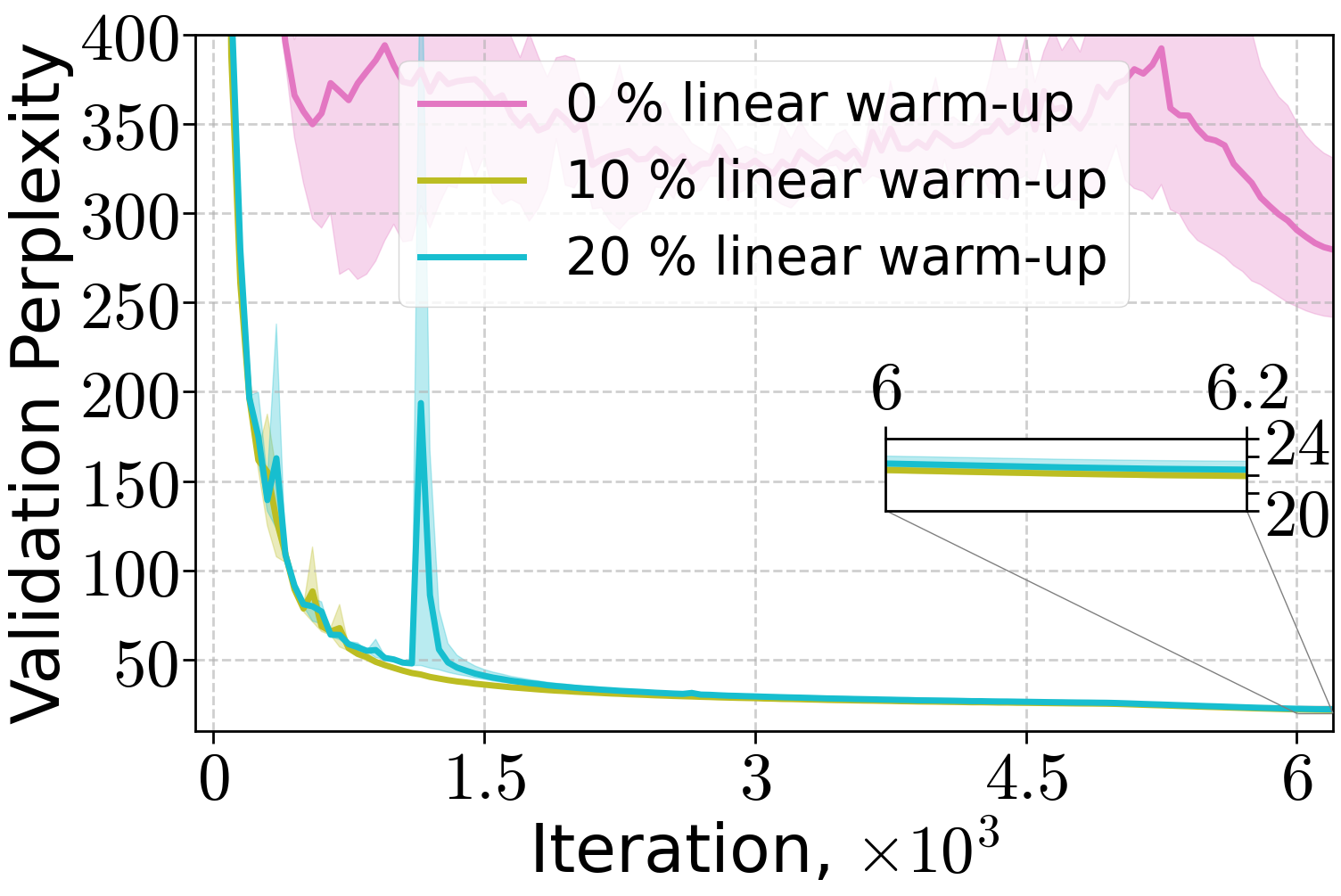} &
    \includegraphics[width=0.23\linewidth]{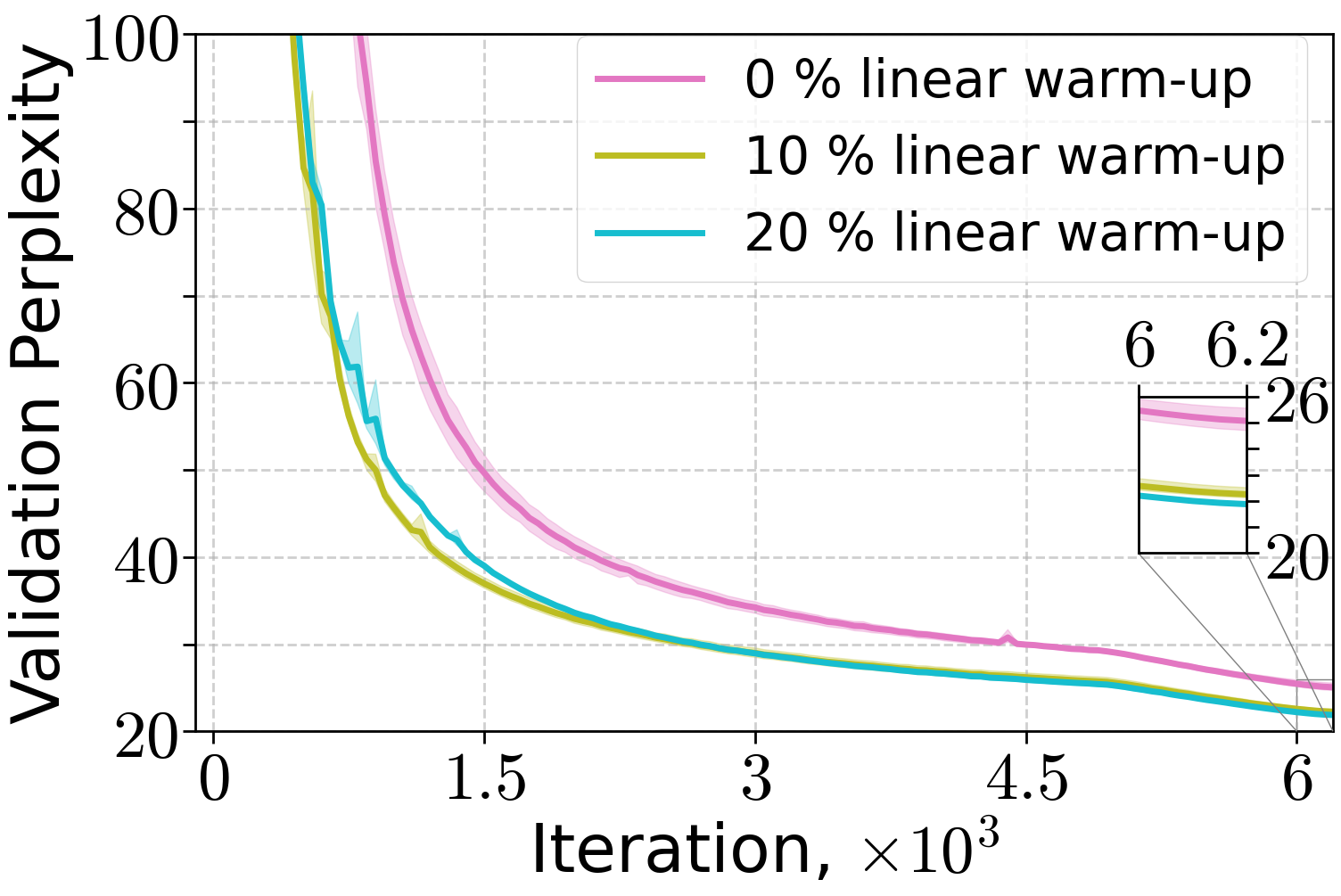} & 
     \includegraphics[width=0.23\linewidth]{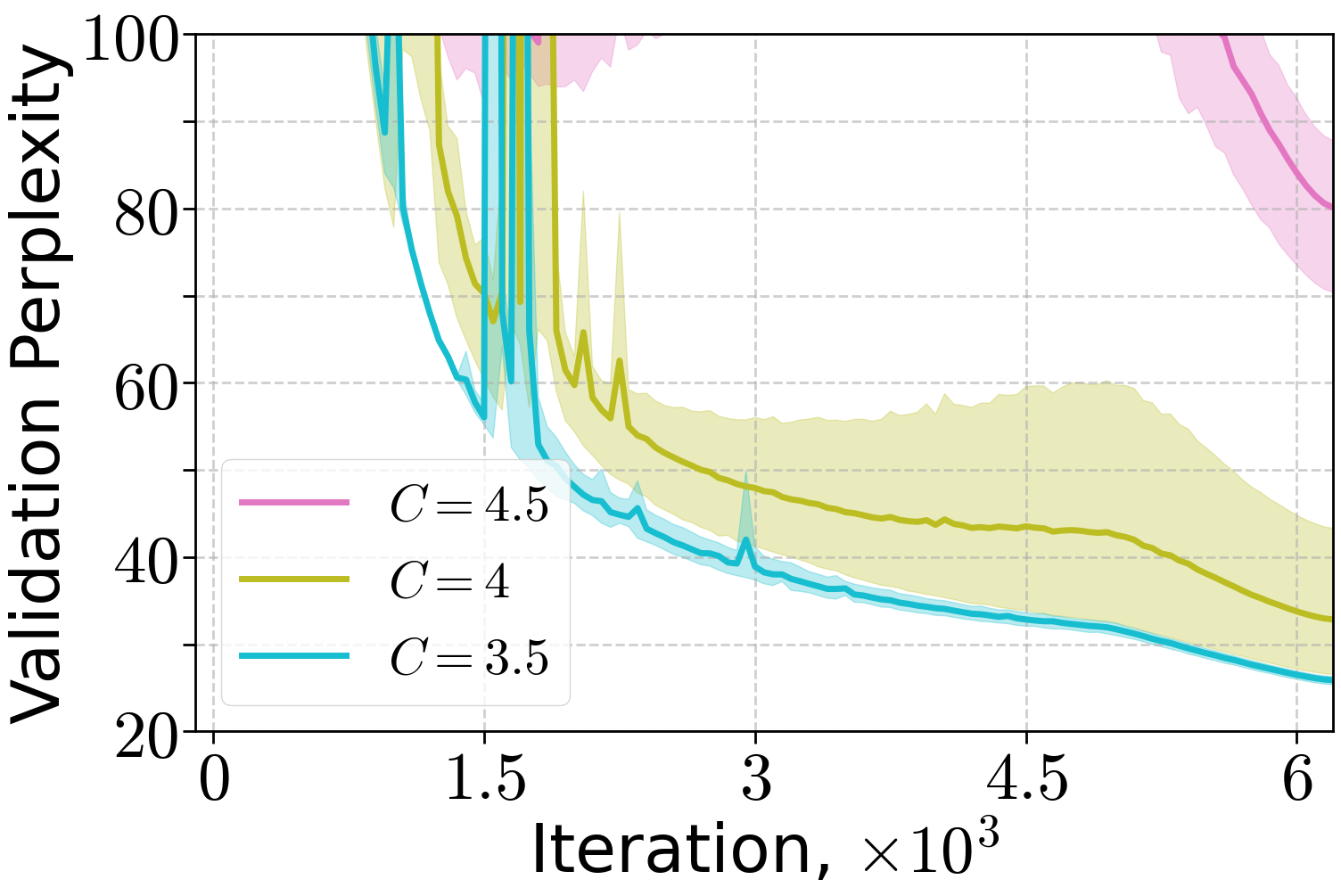} &
     \includegraphics[width=0.23\linewidth]{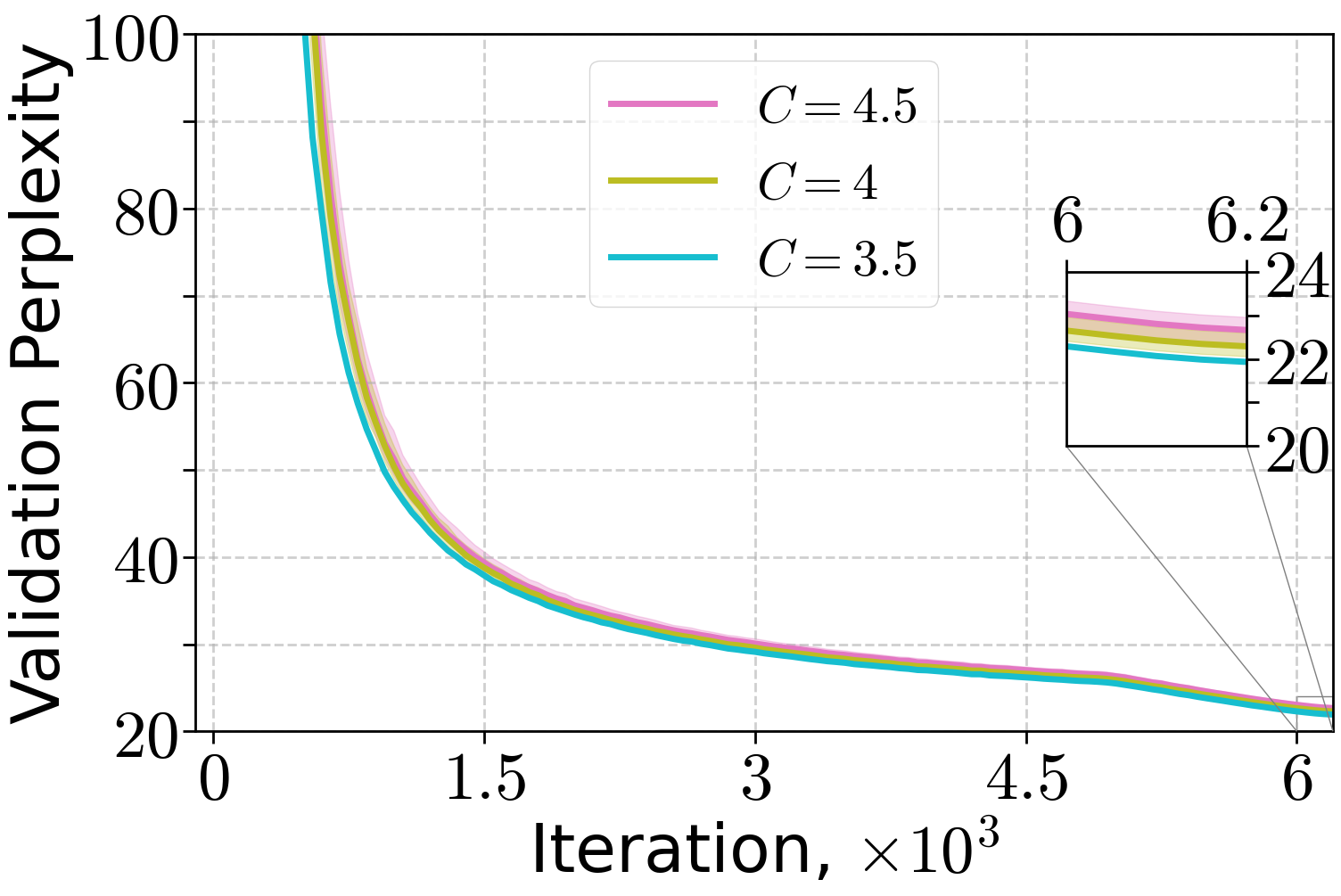} 
    \\
    \makecellnew{Linear warm-up\\ LR $3\cdot 10^{-3}$ }&
    \makecellnew{Linear warm-up \\ LR $10^{-3}$} &
    \makecellnew{$(H_0, H_1)$ warm-up\\ LR $3\cdot 10^{-3}$} &
     \makecellnew{$(H_0, H_1)$ warm-up\\ LR $10^{-3}$}
    \end{tabular}
    
    \caption{Training of 410M language model on FineWeb dataset varying the length of linear warm-up (two left figures) and threshold $C$ of $(H_0, H_1)$ warm-up (two right figures) for the peak learning rate $3\cdot 10^{-3}$ and $10^{-3}.$}
    \label{fig:410M_llm_warmup_ablation_warmup}
\end{figure}

\paragraph{Image Classification with ViT.} Now we turn to the same test, but when training the ViT model on the ImageNet-32 dataset. In contrast to language modeling results, ViT with linear and $(H_0, H_1)$ warm-up strategies demonstrates similar performance. We report the results in \Cref{fig:vit_llm_warmup_ablation_warmup}.

\begin{figure}
    \centering
    \begin{tabular}{cccc}
    \hspace{-3mm}\includegraphics[width=0.23\linewidth]{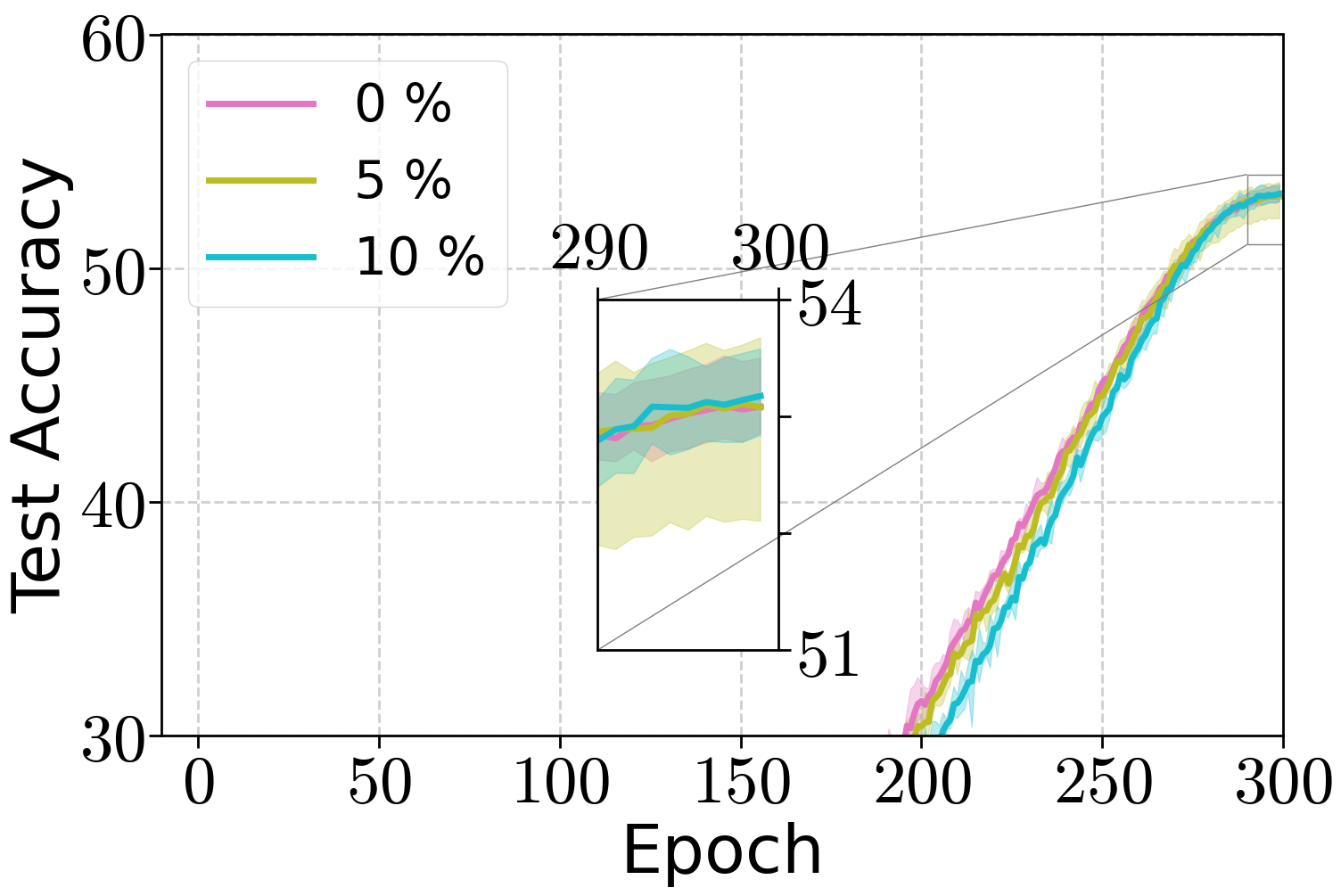} &
    \includegraphics[width=0.23\linewidth]{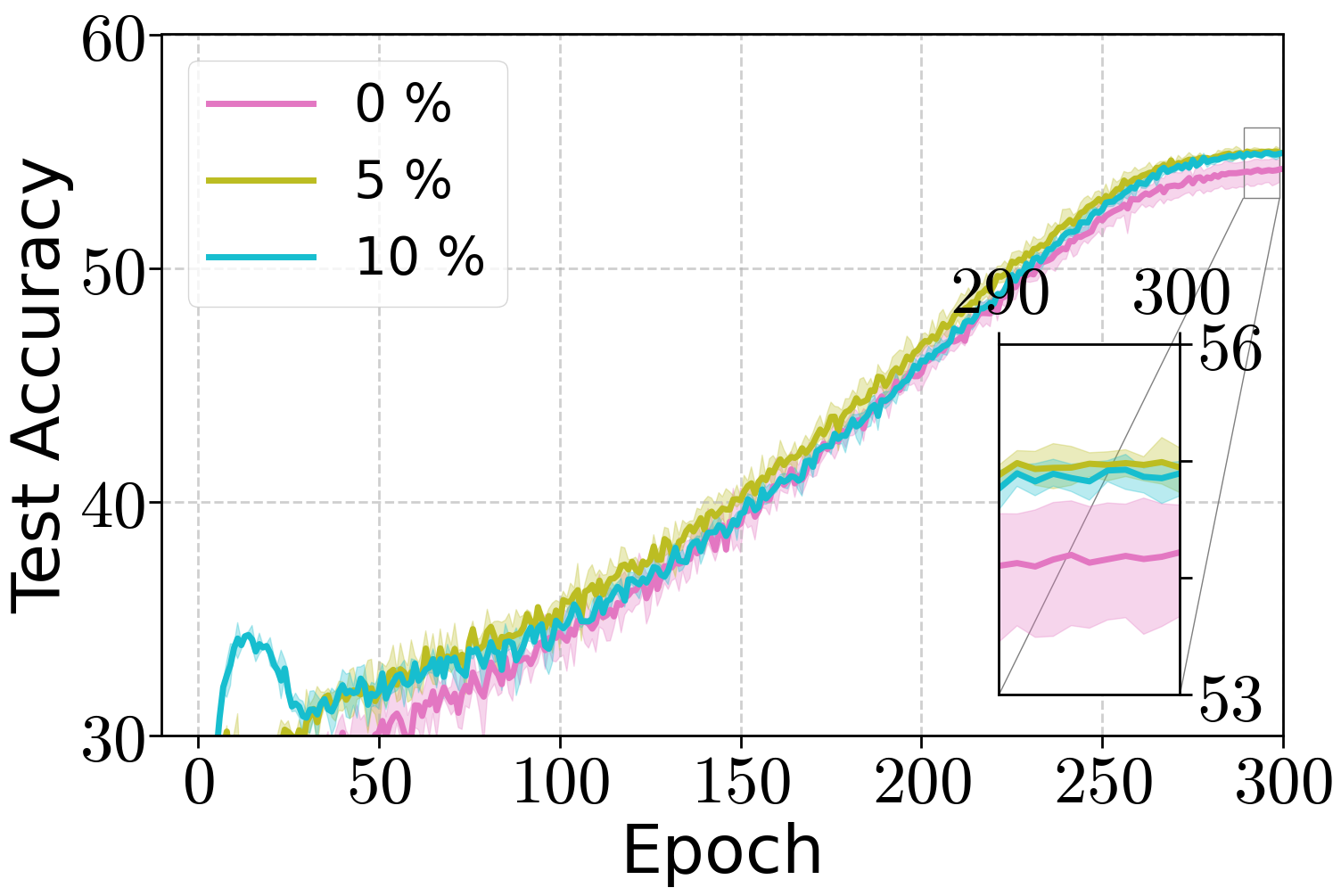} & 
     \includegraphics[width=0.23\linewidth]{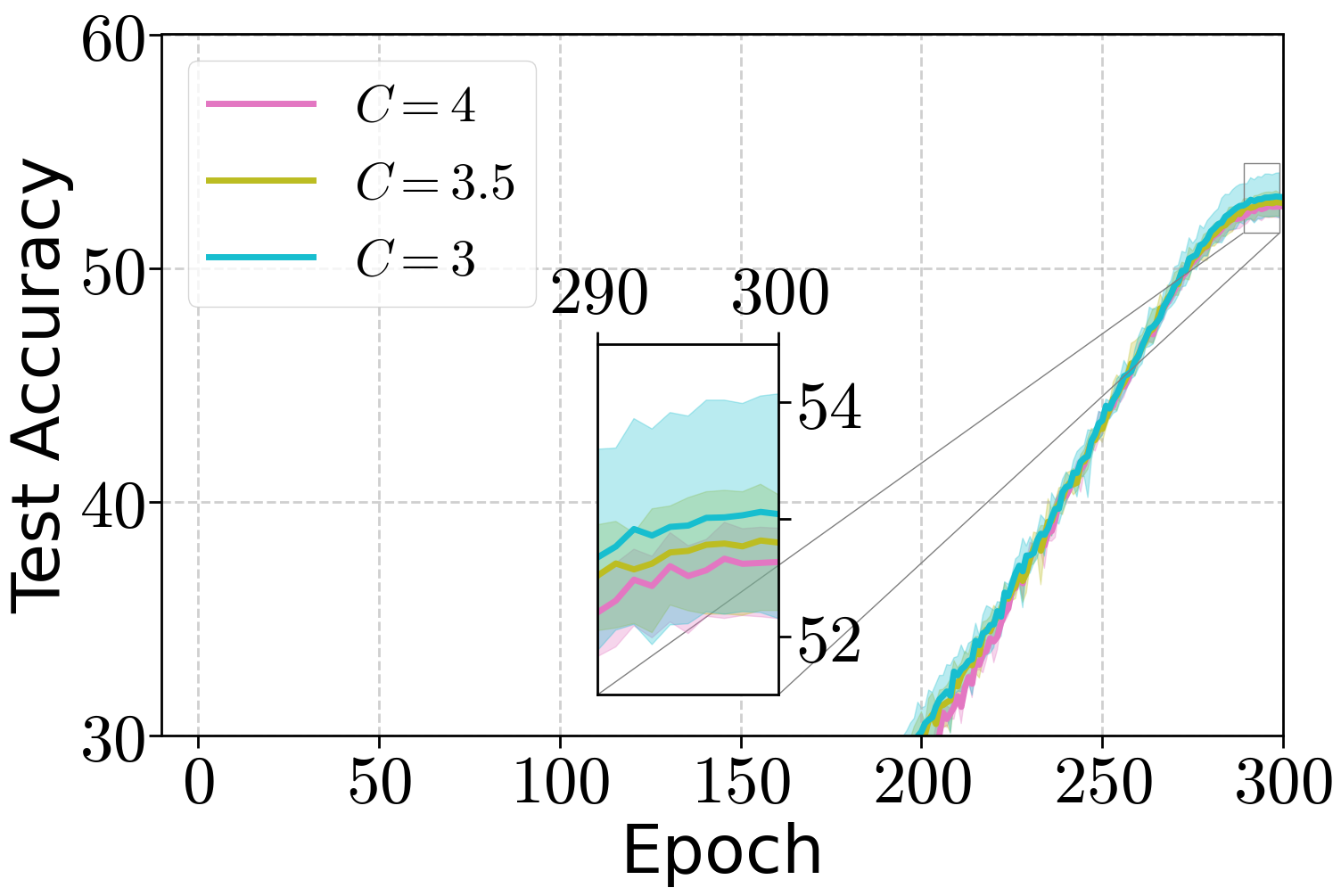} &
     \includegraphics[width=0.23\linewidth]{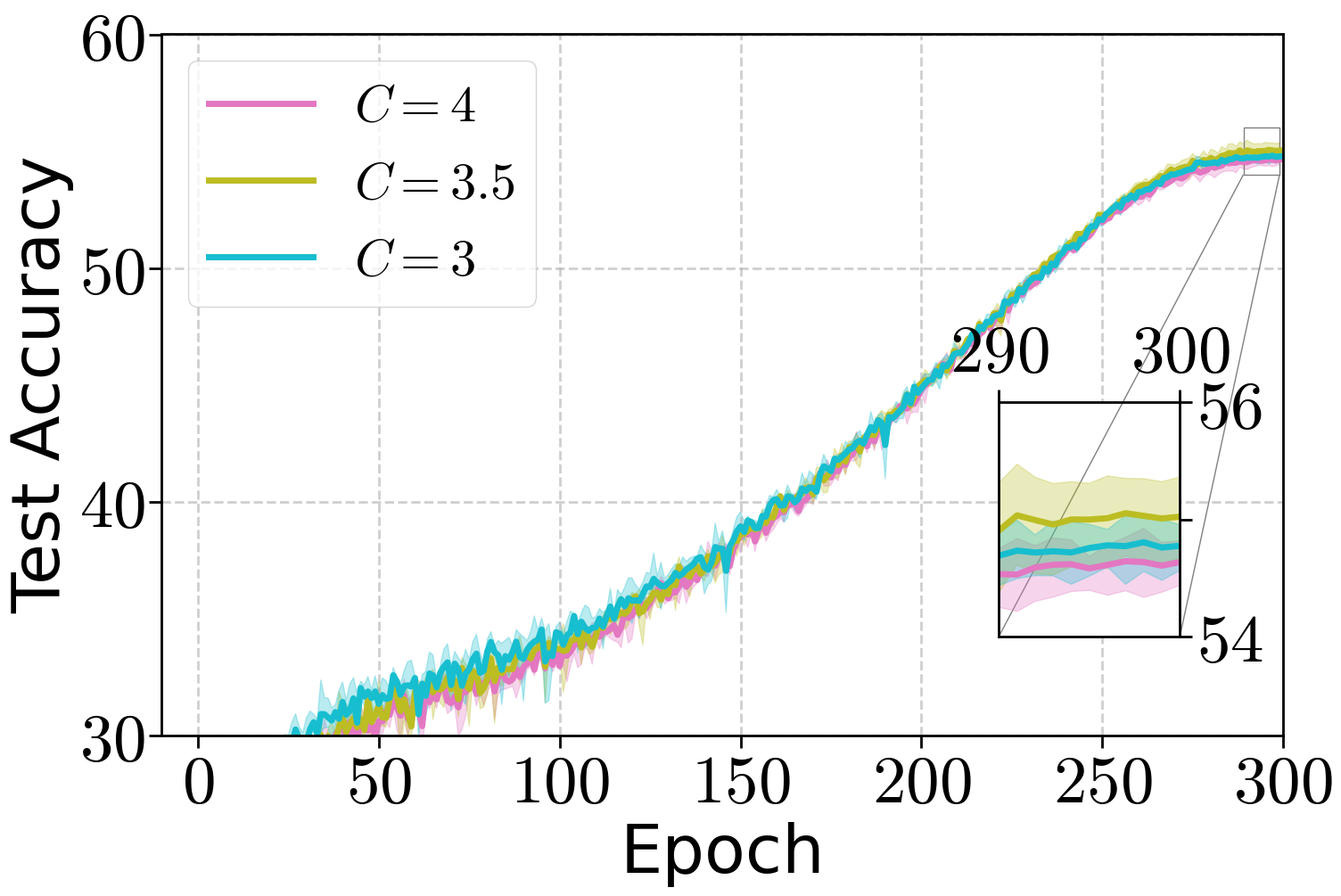} 
    \\
    \makecellnew{Linear warm-up\\ LR $3\cdot 10^{-2}$ }&
    \makecellnew{Linear warm-up \\ LR $10^{-2}$} &
    \makecellnew{$(H_0, H_1)$ warm-up\\ LR $3\cdot 10^{-2}$} &
     \makecellnew{$(H_0, H_1)$ warm-up\\ LR $10^{-2}$}
    \end{tabular}
    
    \caption{Training of ViT model on ImageNet-32 dataset varying the length of linear warm-up (two left figures) and threshold $C$ of $(H_0, H_1)$ warm-up (two right figures) for the peak learning rate $3\cdot 10^{-2}$ and $10^{-2}.$}
    \label{fig:vit_llm_warmup_ablation_warmup}
\end{figure}

\subsubsection{Performance Varying Peak Learning Rate}

\paragraph{Language Modeling.}
We now present performance curves under different peak learning rates for all warm-up strategies: 10\% linear warm-up and $(H_0, H_1)$ warm-up with $C=4$. As shown in \Cref{fig:70M_160M_llm_warmup_ablation_lr}, smaller models are less sensitive to high peak learning rates when using $(H_0, H_1)$ warm-up. However, for the largest 410M model, even slightly exceeding the optimal peak learning rate produces large spikes with $(H_0, H_1)$ warm-up, though \algname{AdamW} eventually recovers. In contrast, linear warm-up proves more robust to peak learning rate selection.

\begin{figure}
    \centering
    \begin{tabular}{cccc}
    \includegraphics[width=0.23\linewidth]{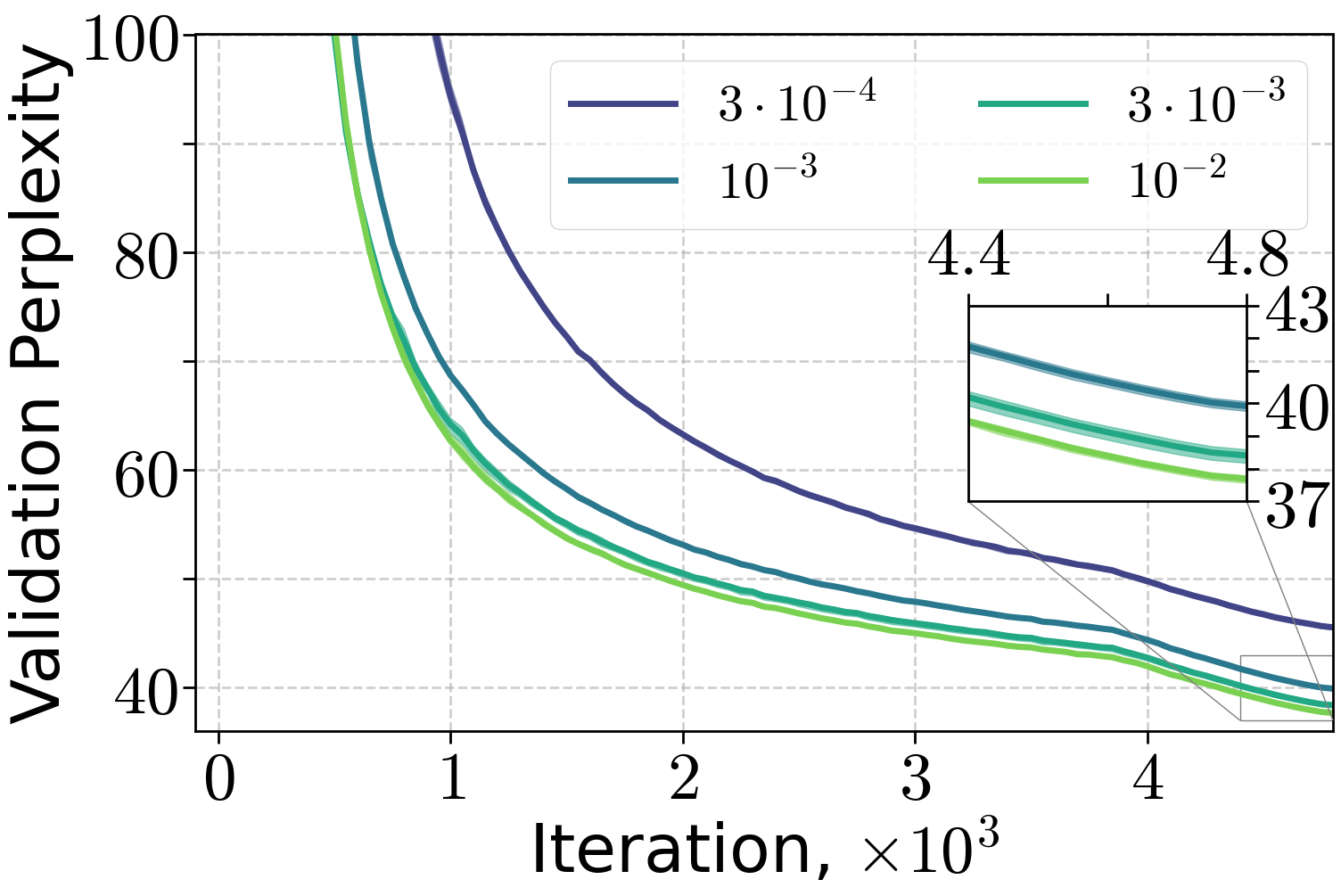} &
    \includegraphics[width=0.23\linewidth]{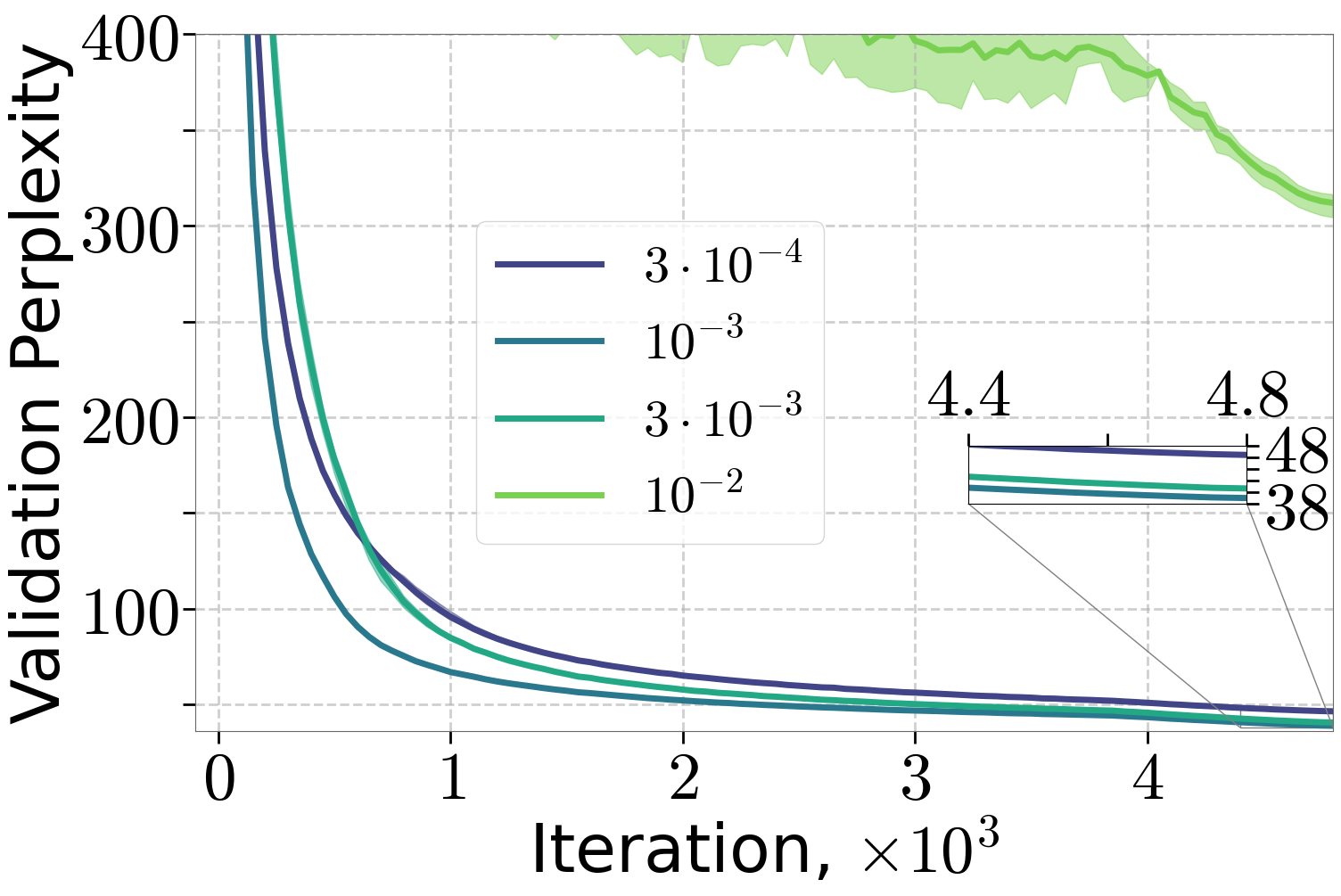} & 
    \includegraphics[width=0.23\linewidth]{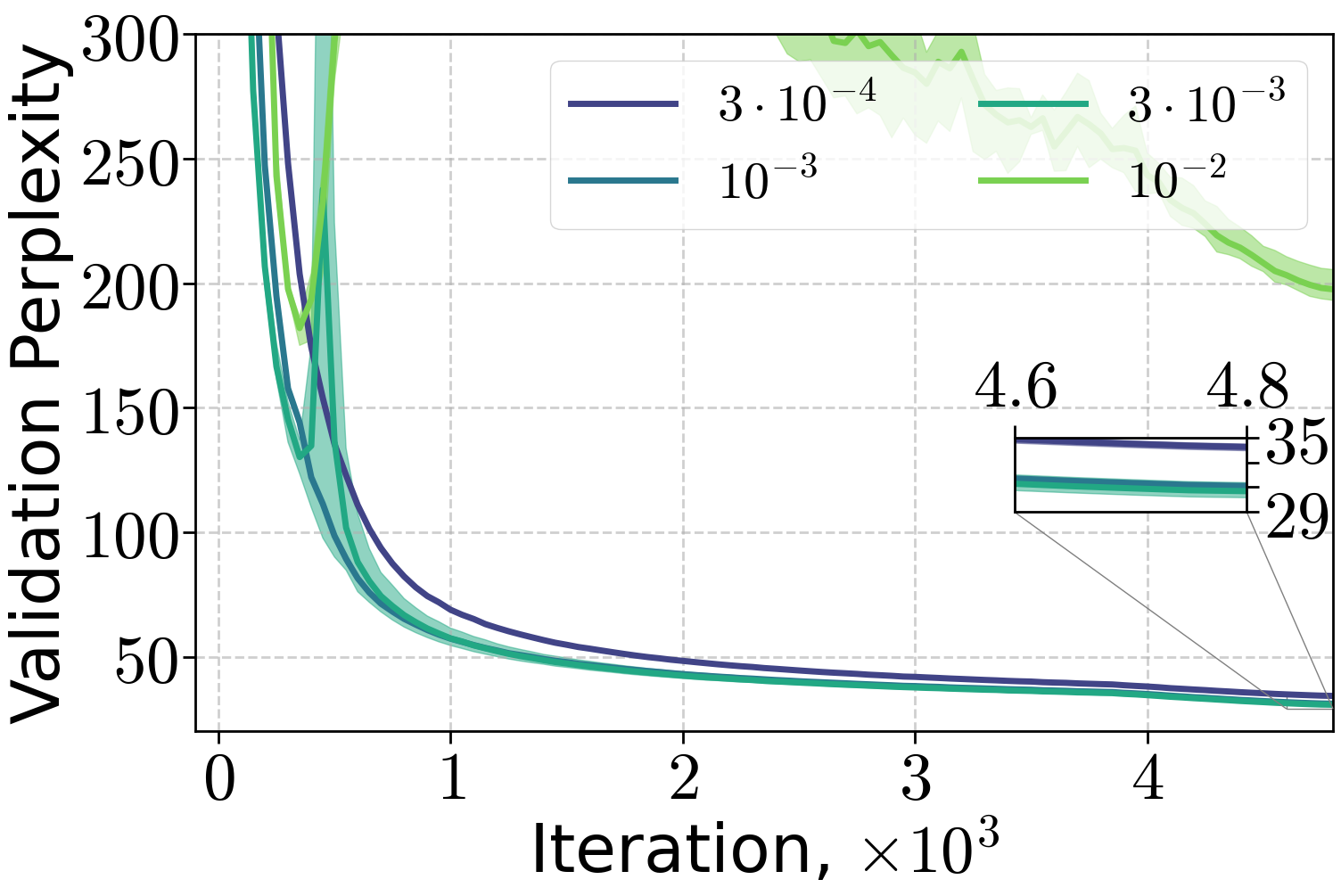} &
    \includegraphics[width=0.23\linewidth]{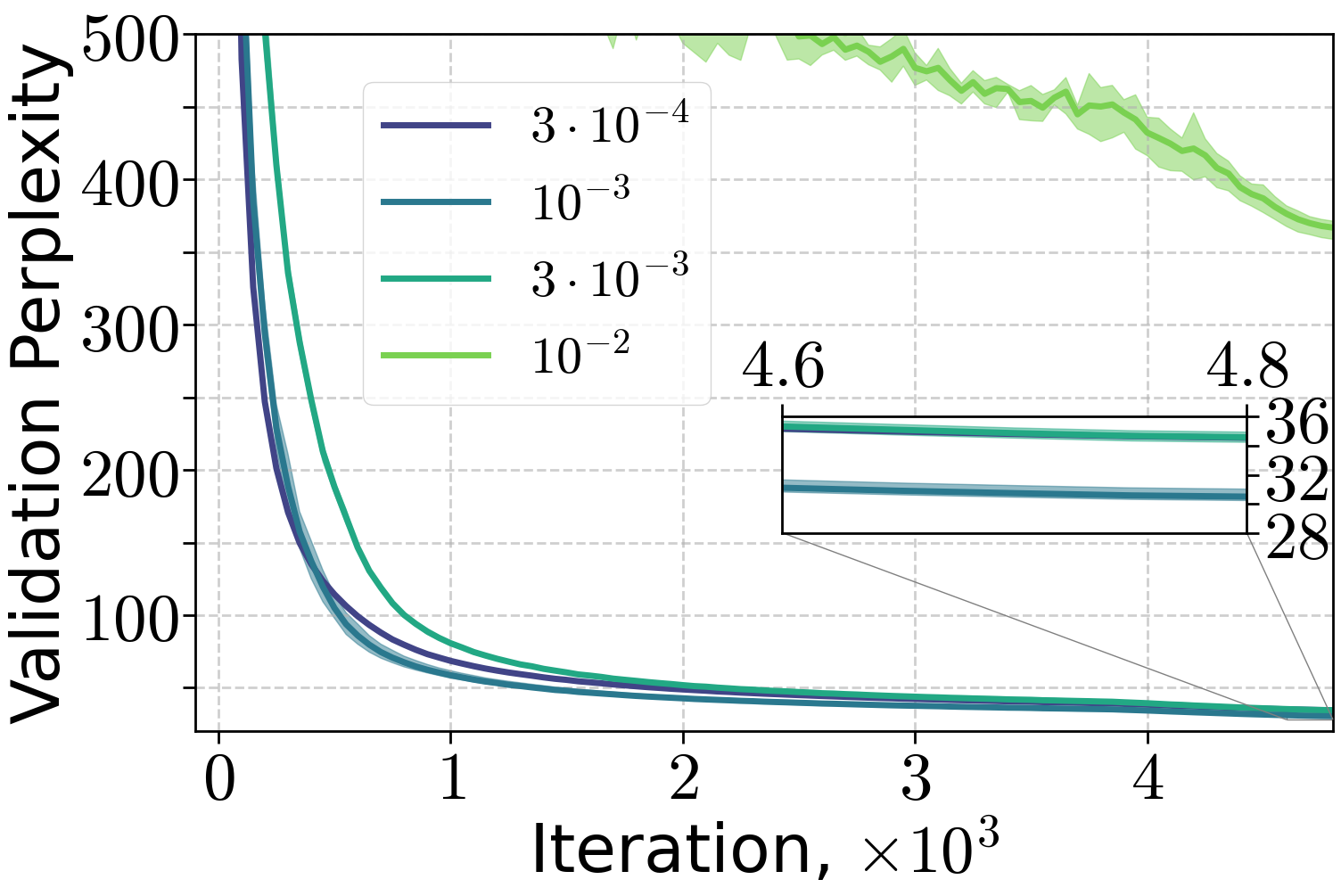} 
    \\
    \makecellnew{70M, Linear\\ warm-up 10 \% }&
     \makecellnew{70M, $(H_0, H_1)$\\ warm-up $C=4$} &
     \makecellnew{160M, Linear \\ warm-up 10 \% }&
     \makecellnew{160M, $(H_0, H_1)$\\  warm-up $C=4$} \\
     &
     \includegraphics[width=0.23\linewidth]{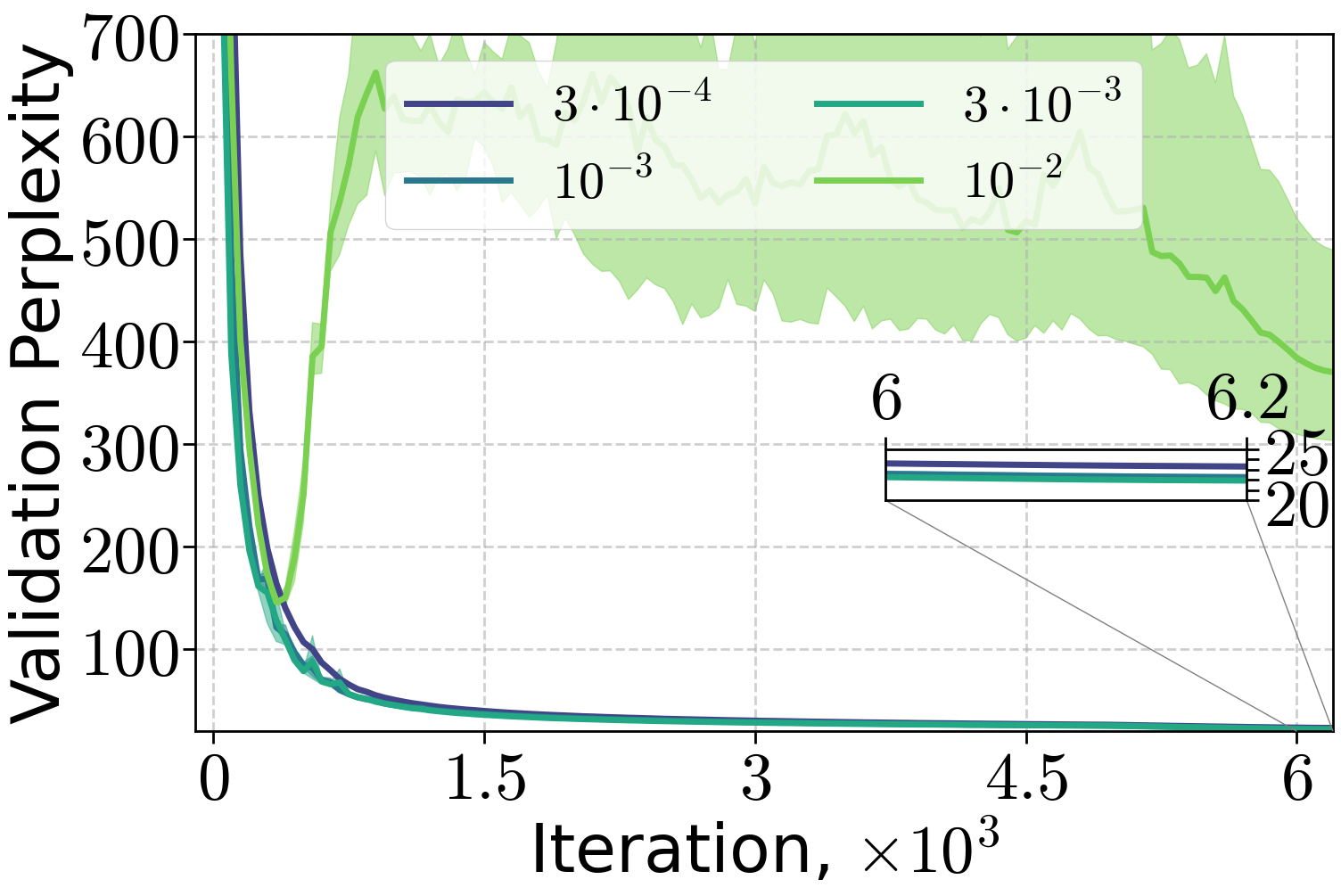} &
    \includegraphics[width=0.23\linewidth]{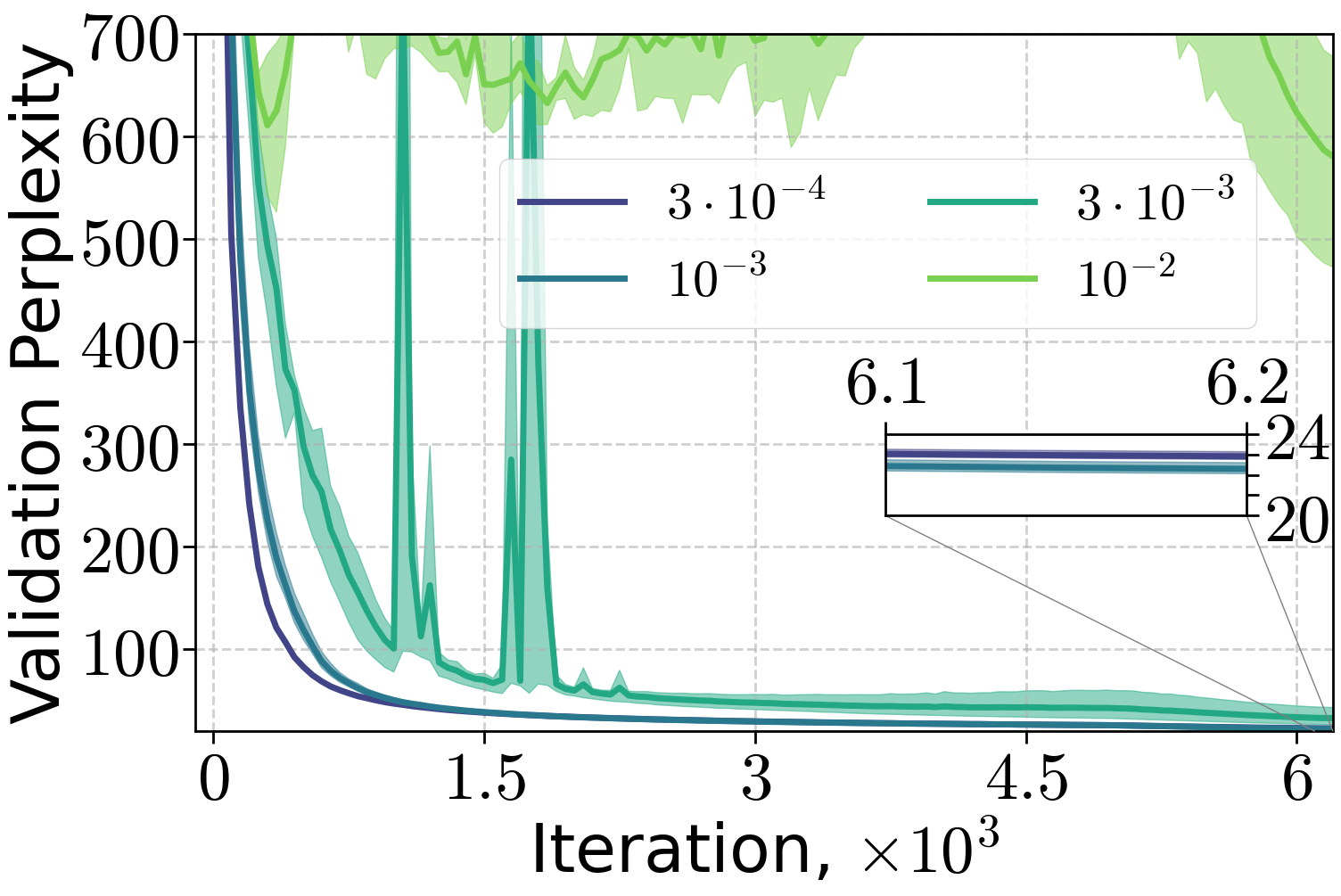} \\
     & 
     \makecellnew{410M, Linear\\ warm-up 10 \% } &
     \makecellnew{410M, $(H_0, H_1)$\\ warm-up $C=4$} &
     
    \end{tabular}
    
    \caption{Training of 70M and 160M language models on FineWeb dataset, varying the peak learning rate with 10 \% linear warm-up and $(H_0, H_1)$ warm-up with $C=4$.}
    \label{fig:70M_160M_llm_warmup_ablation_lr}
\end{figure}

\paragraph{Image Classification with ViT.}

Now we conduct similar tests as in the previous section. We report the results for three warm-up strategies: 5\% linear warm-up and $(H_0, H_1)$ warm-up with $C=3.$ In this case, we observe that both warm-up schedules achieve similar performance; see \Cref{fig:vit_warmup_ablation_lr}.

\begin{figure}
    \centering
    \begin{tabular}{cc}
    \includegraphics[width=0.23\linewidth]{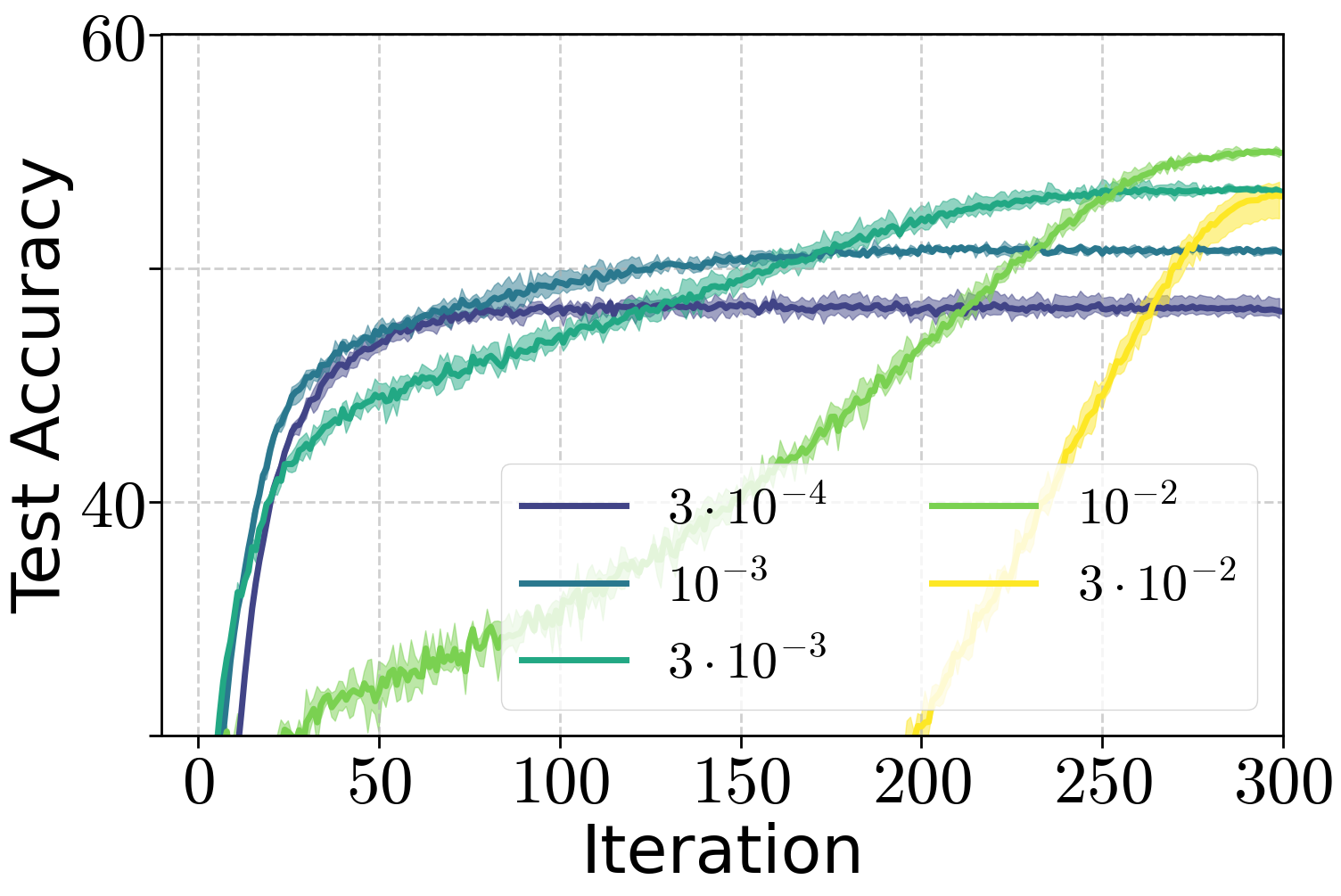} &
    \includegraphics[width=0.23\linewidth]{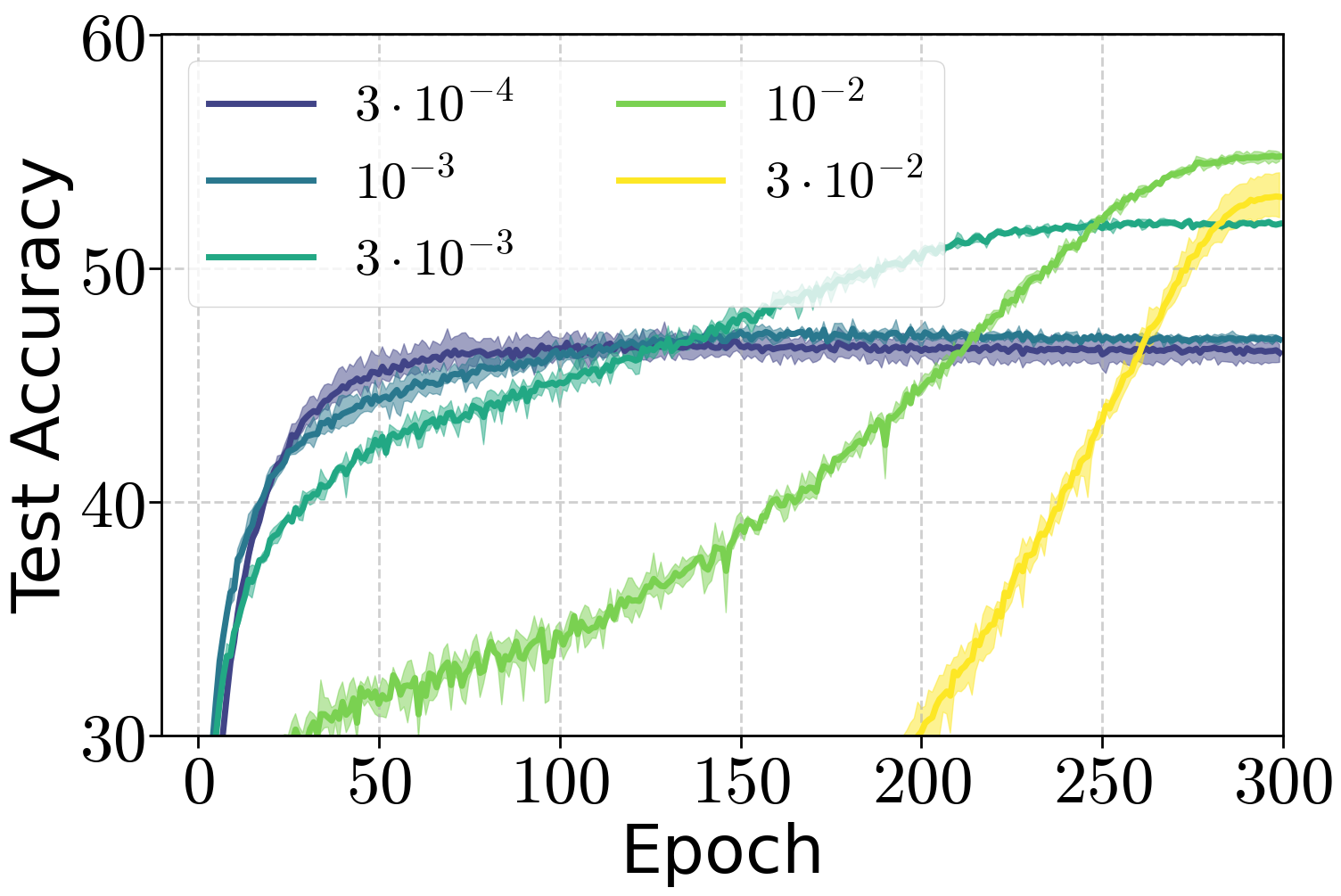} \\
    \makecellnew{ViT, Linear\\ warm-up 5 \% } &
     \makecellnew{ViT, $(H_0, H_1)$\\ warm-up $C=3$} 
     
    \end{tabular}
    
    \caption{Training of ViT model on ImageNet-32 dataset, varying the peak learning rate with 5 \% linear warm-up and $(H_0, H_1)$ warm-up with $C=3$.}
    \label{fig:vit_warmup_ablation_lr}
\end{figure}

\end{document}